\documentclass[twoside,11pt]{article}

\usepackage[nohyperref, preprint]{jmlr2e}

\usepackage{microtype}
\usepackage[tight,footnotesize]{subfigure}

\usepackage{url}
\usepackage[CJKbookmarks=true,
            bookmarksnumbered=true,
            bookmarksopen=true,
            colorlinks,
            linktocpage,
            linkcolor={blue!70!gray}, citecolor={blue!70!gray}, urlcolor={blue!70!gray}]{hyperref}

\usepackage{amsmath,amsfonts,bm}
\usepackage{pifont}
\newcommand{\cmark}{\ding{51}}
\newcommand{\xmark}{\ding{55}}
\usepackage{algorithm}
\usepackage{algorithmic}
\usepackage{booktabs}
\usepackage{multirow,bigstrut,threeparttable}
\usepackage{graphicx}
\usepackage{tcolorbox}
\usepackage{tikz}
\usepackage{tikz-3dplot}
\usepackage{blindtext}

\usepackage{adjustbox}

\newcommand{\diag}{\mathsf{diag}}

\newcommand{\dft}{\mathrm{DFT}}
\newcommand{\idft}{\mathrm{IDFT}}

\renewcommand{\vec}{\mathsf{vec}}
\renewcommand{\mod}{~\mathsf{mod}~}
\newcommand{\trans}{\mathsf{trans}}
\newcommand{\circm}{\mathsf{circ}}

\newcommand{\Appendix}[1]{the full version for}

\newtheorem{fact}{Fact}

\newcommand{\p}{\bm{p}}

\renewcommand{\u}{\bm{u}}
\newcommand{\bv}{\bm{v}}
\renewcommand{\v}{\bm{v}}

\newcommand{\x}{\bm{x}}
\newcommand{\y}{\bm{y}}
\newcommand{\z}{\bm{z}}
\newcommand{\q}{\bm{q}}

\newcommand{\C}{\bm{C}}
\newcommand{\D}{\bm{D}}
\newcommand{\E}{\bm{E}}
\newcommand{\F}{\bm{F}}

\newcommand{\I}{\bm{I}}

\renewcommand{\P}{\mathbf{P}}
\newcommand{\Q}{\mathbf{Q}}
\newcommand{\R}{\mathbb{R}}
\renewcommand{\Re}{\mathbb{R}}
\newcommand{\Co}{\mathbb{C}}
\renewcommand{\S}{\bm{S}}

\newcommand{\U}{\bm{U}}
\newcommand{\V}{\bm{V}}
\newcommand{\W}{\bm{W}}
\newcommand{\X}{\bm{X}}
\newcommand{\Y}{\bm{Y}}
\newcommand{\Z}{\bm{Z}}
\newcommand{\rank}{\textup{\textsf{rank}}}

\newcommand{\blambda}{\boldsymbol{\lambda}}

\newcommand{\0}{\mathbf{0}}

\newcommand{\tr}{\textsf{tr}}

\newcommand{\cC}{\mathcal{C}}

\newcommand{\cE}{\mathcal{E}}

\newcommand{\cL}{\mathcal{L}}

\newcommand{\cT}{\mathcal{T}}

\newcommand{\bbS}{\mathbb{S}}

\DeclareMathOperator*{\argmax}{argmax}
\DeclareMathOperator*{\argmin}{argmin}

\newcommand\blfootnote[1]{%
	\begingroup
	\renewcommand\thefootnote{}\footnote{#1}%
	\addtocounter{footnote}{-1}%
	\endgroup
}

\newcommand{\bfSigma}{\boldsymbol{\Sigma}}
\newcommand{\bfsigma}{\boldsymbol{\sigma}}

\jmlrheading{22}{2021}{1-96}{6/21}{12/21}{meila00a}{Kwan Ho Ryan Chan, Yaodong Yu, Chong You, Haozhi Qi, John Wright, and Yi Ma}

\ShortHeadings{ReduNet: A White-box Deep Network from Rate Reduction}{Chan, Yu, You, Qi, Wright, and Ma}
\firstpageno{1}

\begin{document}

\title{ReduNet: A White-box Deep Network from the Principle of Maximizing Rate Reduction$^*$}

\author{\name Kwan Ho Ryan Chan$^{\,\dagger,\, \diamond}$ \email ryanchankh@berkeley.edu 
\AND
\name Yaodong Yu$^{\,\dagger,\, \diamond}$ \email yyu@eecs.berkeley.edu
\AND
\name Chong You$^{\,\dagger,\, \diamond}$ \email cyou@berkeley.edu
\AND
\name Haozhi Qi$^{\,\dagger}$ \email  hqi@berkeley.edu 
\AND
\name John Wright$^{\,\ddagger}$ \email  johnwright@ee.columbia.edu 
\AND
\name Yi Ma$^{\,\dagger}$ \email  yima@eecs.berkeley.edu\,\, 
\\ \\
\addr $^\dagger$
Department of Electrical Engineering and Computer Sciences\\
\hspace*{2.5mm}University of California, Berkeley, CA 94720-1776, USA\\
\addr $^\ddagger$ 
Department of Electrical Engineering\\
\hspace*{2.5mm}Department of Applied Physics and Applied Mathematics\\
\hspace*{2.5mm}Columbia University, New York, NY, 10027, USA}

\editor{}

\maketitle
\blfootnote{ \hspace{-3.5mm}$^*$\,\,\,This paper integrates  previous two manuscripts:  \url{https://arxiv.org/abs/2006.08558} and \url{https://arxiv.org/abs/2010.14765}, with significantly improved organization, presentation, and new results. The first manuscript appeared as a conference paper in NeurIPS 2020~\citep{yu2020learning}. The second manuscript is the new addition.}
\blfootnote{\hspace{-3.5mm}$^ \diamond$\,\,\,The first three authors contributed equally to this work.}

\vspace{-10mm}
\begin{abstract}%
This work attempts to provide a plausible theoretical framework that aims to interpret modern deep (convolutional) networks from the principles of data compression and discriminative representation. We argue that for high-dimensional multi-class data, the optimal linear discriminative representation maximizes the coding rate difference between the whole dataset and the average of all the subsets. We show that the basic iterative gradient ascent scheme for optimizing the rate reduction objective naturally leads to a multi-layer deep network, named ReduNet, which shares common characteristics of modern deep networks. The deep layered architectures, linear and nonlinear operators, and even parameters of the network are all explicitly constructed layer-by-layer via forward propagation, although they are amenable to fine-tuning via back propagation. All components of so-obtained ``white-box'' network have precise optimization, statistical, and geometric interpretation. Moreover, all linear operators of the so-derived network naturally become multi-channel convolutions when we enforce classification to be rigorously shift-invariant. The derivation in the invariant setting suggests a trade-off between sparsity and invariance, and also indicates that such a deep convolution network is significantly more efficient to construct and learn in the spectral domain. Our preliminary simulations and experiments clearly verify the effectiveness of both the rate reduction objective and the associated ReduNet. All code and data are available at \url{https://github.com/Ma-Lab-Berkeley}.
\end{abstract}

\vspace{1mm}

\vspace{0.5mm}
\begin{quotation}
$~$ \hfill ``{\em What I cannot create, I do not understand.}''\\
$~$ \hspace{\fill} --- Richard Feynman
\end{quotation}

\newpage
\tableofcontents
\newpage

\section{Introduction and Overview}

\subsection{Background and Motivation}
In the past decade or so, the practice of deep networks has captured people’s imagination by its empirical successes in learning useful representations from large-scale real-world data such as images, speech, and natural languages \citep{lecun2015deep}. To a large extent, the dramatic revival of deep networks is attributed to remarkable technological advancements in scalable computation platforms (hardware and software) and equally importantly, ample data, real or simulated, for training and evaluation of the networks \citep{krizhevsky2012imagenet}. 

\paragraph{The practice of deep networks as a ``black box.''}
Nevertheless, the practice of deep networks has been shrouded with mystery from the very beginning. This is largely caused by the fact that deep network architectures and their components are often designed based on years of {\em trial and error}, then trained from random initialization via back propagation (BP) \citep{Back-Prop}, and then deployed as a ``black box''. Many of the popular techniques and recipes for designing and training deep networks (to be surveyed in the related work below) were developed through heuristic and empirical means, as opposed to rigorous mathematical principles, modeling and analysis. Due to lack of clear mathematical principles that can guide the design and optimization, numerous techniques, tricks, and even hacks need to be added upon the above designing and training process to make deep learning ``work at all'' or ``work better'' on real world data and tasks. Practitioners constantly face a series of challenges for any new data and tasks: What architecture or particular components they should use for the network? How wide or deep the network should be?  Which parts of the networks need to be trained and which can be determined in advance? Last but not the least, after the network has been trained, learned and tested: how to interpret functions of the operators; what are the roles and relationships among the multiple (convolution) channels; how can they be further improved or adapted to new tasks?  Besides empirical evaluation, it is challenging to provide rigorous guarantees for certain properties of so obtained networks, such as invariance and robustness. Recent studies have shown that popular networks in fact are not invariant to common deformations \citep{azulay2018deep,engstrom2017rotation}; they are prone to overfit noisy or even arbitrary labels \citep{ZhangBeHaReVi17}; and they remain vulnerable to adversarial attacks \citep{szegedy2013intriguing}; or they suffer catastrophic forgetting when  incrementally trained to learn new classes/tasks  \citep{catastrophic,delange2021continual,Wu-CVPR2021}. The lack of understanding about the roles of learned operators in the networks often provoke debates among practitioners which architecture is better than others, for example MLPs versus CNNs versus Transformers \citep{tay2021pretrained} etc. This situation seems in desperate need of improvements if one wishes deep learning to be a science rather than an ``alchemy.'' Therefore, it naturally raises a fundamental question that we aim to address in this work: \emph{how to develop a principled mathematical framework for better  understanding and design of deep networks?}

\paragraph{Existing attempts to interpret or understand deep networks.} To uncover the mystery of deep networks, a plethora of recent studies on the mathematics of deep learning has, to a large extent, improved our understanding on key aspects of deep models, including over-fitting and over-parameterization \citep{arora2019implicit,BHMM}, optimization \citep{JinGeNeKaJo17,JinNeJo17,gidel2019implicit,gunasekar2018implicit}; expressive power \citep{Rolnick-2017,hanin2017universal}, generalization bounds~\citep{BartlettFoTe17,Golowich-2017} etc.
Nevertheless, these efforts are typically based on simplified models (e.g. networks with only a handful layers \citep{zhong2017recovery,soltanolkotabi2018theoretical,zhang2019learning,Mei-2019}); or results are conservative (e.g.\ generalization bounds \citep{BartlettFoTe17,Golowich-2017}); or rely on simplified assumptions (e.g. ultra-wide deep networks \citep{Jacot2018-dv,du2018gradient,allen2019convergence,buchanan2020deep}); or only explain or justify one of the components or characteristics of the network (e.g. dropout \citep{cavazza2018dropout,pmlr-v80-mianjy18b,pmlr-v119-wei20d}). 
On the other hand, the predominant methodology in practice still remains \textit{trial and error}.
In fact, this is often pushed to the extreme by searching for effective network structures and training strategies through extensive random search techniques, such as Neural Architecture Search \citep{NAS-1,Baker2017DesigningNN}, AutoML \citep{automl}, and Learning to Learn \citep{andrychowicz2016learning}. 

\paragraph{A new theoretical framework based on data compression and representation.} Our approach deviates from the above efforts in a very significant way. Most of the existing theoretical work view the deep networks themselves as the object for study. They try to understand why deep networks work by examining their capabilities in fitting certain input-output relationships (for given class labels or function values). We, in this work, however, advocate to shift the attention of study back to the data and try to understand what deep networks should do. We start our investigation with a fundamental question: What exactly do we try to {\em learn} from and about the data? With the objective clarified, maybe deep networks, with {\em all} their characteristics, are simply necessary means to achieve such an objective. More specifically, in this paper, we develop a new theoretical framework for understanding deep networks around the following two questions:
\begin{enumerate}
    \item {\em Objective of Representation Learning:} What intrinsic structures of the data should we learn, and how should we represent such structures? What is a principled objective function for learning a good representation of such structures, instead of choosing heuristically or arbitrarily?\vspace{-2mm}
    \item {\em Architecture of Deep Networks: } Can we justify the structures of modern deep networks from such a principle? In particular, can the networks' layered architecture and operators (linear or nonlinear) all be derived from this objective, rather than designed heuristically and evaluated empirically?
\end{enumerate}
This paper will provide largely positive and constructive answers to the above questions.

We will argue that, at least in the classification setting, a principled objective for a deep network is {\em to learn a low-dimensional linear discriminative representation of the data} (Section \ref{sec:principled-objective-via-compression}). The optimality of such a representation can be evaluated by a principled measure from (lossy) data compression, known as {\em rate reduction} (Section \ref{sec:principled-objective}). Appropriately structured deep networks can be naturally interpreted as {\em optimization schemes for maximizing this measure} (Section \ref{sec:vector} and \ref{sec:shift-invariant}). Not only does this framework offer new perspectives to understand and interpret modern deep networks, they also provide new insights that can potentially change and improve the practice of deep networks. For instance, the resulting networks will be entirely a ``white box'' and back propagation from random initialization is no longer the only choice for training the networks (as we will verify through extensive experiments in Section \ref{sec:experiments}).

For the rest of this introduction, we will first provide a brief survey on existing work that are related to the above two questions. Then we will give an overview of our approach and contributions before we delve into the technical details in following sections.

\subsection{Related Work}
Given a random vector $\x \in \Re^D$ which is drawn from a mixture of $k$ distributions $\mathcal{D} = \{\mathcal{D}^j\}_{j=1}^k$, one of the most fundamental problems in machine learning is how to effectively and efficiently {\em learn the distribution} from a finite set of i.i.d samples, say $\X = [\x^1, \x^2, \ldots, \x^m] \in \Re^{D\times m}$. To this end, we {\em seek a good representation}  through a continuous mapping, $f(\x, \bm \theta): \Re^D \rightarrow \Re^n$, that captures intrinsic structures of $\x$ and best facilitates subsequent tasks such as classification or clustering.\footnote{ Classification is where deep learning demonstrated the initial success that has catalyzed the  explosive interest in deep networks \citep{krizhevsky2012imagenet}. Although our study in this paper focuses on classification, we believe the ideas and principles can be naturally generalized to other settings such as regression.}

\subsubsection{Objectives for Deep Learning}\label{sec:intro-objective}
\paragraph{Supervised learning via cross entropy.} To ease the task of learning $\mathcal{D}$, in the popular supervised setting, a true class label, represented as a one-hot vector $\y^i \in \Re^k$, is given for each sample $\x^i$. Extensive studies have shown that for many practical datasets (images, audio, and natural language, etc.), the mapping from the data $\bm x$ to its class label $\bm y$ can be effectively modeled by training a deep network~\citep{goodfellow2016deep}, here denoted as $f(\x, \bm \theta):\x \mapsto \y$ with network parameters $\bm \theta \in \Theta$. This is typically done by minimizing the {\em cross-entropy loss} over a training set $\{(\x^i, \y^i)\}_{i=1}^m$, 
through backpropagation over the network parameters $\bm \theta$: 
\begin{equation}
   \min_{\bm \theta \in \Theta} \; \mbox{CE}(\bm \theta, \x, \y) \doteq - \mathbb{E}[\langle \y, \log[f(\x, \bm \theta)] \rangle] \, \approx - \frac{1}{m}\sum_{i=1}^m \langle \y^i, \log[f(\x^i, \bm \theta)] \rangle.
   \label{eqn:cross-entropy}
\end{equation}
Despite its effectiveness and enormous popularity, there are two serious limitations with this approach:  1) It aims only to predict the labels $\y$ even if they might be mislabeled. Empirical studies show that deep networks, used as a ``black box,'' can even fit random labels \citep{zhang2017understanding}. 2) With such an end-to-end data fitting, 
despite plenty of empirical efforts in trying to interpret the so-learned features~\citep{Zeiler-ECCV2014}, 
it is not clear to what extent the intermediate features learned by the network capture the intrinsic structures of the data that make meaningful classification possible in the first place. Recent work of  \cite{papyan2020prevalence,fang2021layer,zhu2021geometric} shows that the representations learned via the  cross-entropy loss \eqref{eqn:cross-entropy} exhibit a \emph{neural collapsing} phenomenon,\footnote{Essentially, once an over-parameterized network fits the training data, regularization (such as weight decay) would collapse weight components or features that are not the most relevant for fitting the class labels. Besides the most salient feature, informative and discriminative features that also help define a class can be suppressed.} where within-class variability and structural information are getting suppressed and ignored, as we will also see in the experiments. The precise geometric and statistical properties of the learned features are also often obscured, which leads to the lack of interpretability and subsequent performance guarantees (e.g., generalizability, transferability, and robustness, etc.) in deep learning. Therefore, {\em one of the goals of this work is to address such limitations by reformulating the objective towards learning explicitly meaningful and useful representations for the data $\x$, in terms of feature linearity, discriminativeness, and richness.}

\paragraph{Minimal versus low-dimensional representations from deep learning.} Based on strong empirical evidence that intrinsic structures of high-dimensional (imagery) data are rather low-dimensional\footnote{For example, the digits in MNIST approximately live on a manifold with intrinsic dimension no larger than $15$ \citep{hein2005intrinsic}, the images in CIFAR-$10$ live closely on a $35$-dimensional manifold \citep{spigler2019asymptotic}, and the images in ImageNet have intrinsic dimension of $\sim40$ \citep{pope2021the}.}, it has been long believed that the role of deep networks is to learn certain (nonlinear) low-dimensional representations of the data, which has an advantage over classical linear dimension reduction methods such as PCA \citep{Hinton504}. Following this line of thought, one possible approach to interpret the role of deep networks is to view outputs of intermediate layers of the network as selecting certain low-dimensional latent features $\z = f(\x, \bm \theta) \in \Re^n$ of the data that are discriminative among multiple classes. Learned representations $\z$ then  facilitate the subsequent classification task for predicting the class label $\y$ by optimizing a classifier $g(\z)$: 
\vspace{-1mm}
\begin{equation}
    \x   \xrightarrow{\hspace{2mm} f(\x, \bm \theta)\hspace{2mm}} \z(\bm \theta)  \xrightarrow{\hspace{2mm} g(\z) \hspace{2mm}} \y.
\end{equation}
The {\em information bottleneck} (IB) formulation \citep{Tishby-ITW2015} further hypothesizes that the role of the network is to learn $\z$ as the minimal sufficient statistics for predicting $\y$. Formally, it seeks to maximize the mutual information $I(\z, \y)$~\citep{Thomas-Cover}
between $\z$ and $\y$ while minimizing $I(\x, \z)$ between $\x$ and $\z$:
\begin{equation}
    \max_{\bm \theta\in \Theta}\; \mbox{IB}(\x, \y, \z(\bm \theta)) \doteq I(\z(\bm \theta), \y) - \beta I(\x, \z(\bm \theta)), \quad \beta >0. 
\label{eqn:information-bottleneck}
\end{equation}
Given one can overcome some caveats associated with this framework \citep{kolchinsky2018caveats-ICLR2018}, such as how to accurately evaluate mutual information with finitely samples of degenerate distributions, this framework can be helpful in explaining certain behaviors of deep networks. For example, recent work \citep{papyan2020prevalence} indeed shows that the representations learned via the cross-entropy loss expose a \emph{neural collapse} phenomenon. That is, features of each class are mapped to a one-dimensional vector whereas all other information of the class is suppressed. More discussion on neural collapse will be given in Section \ref{sec:rate-reduction-properties}.
But by being task-dependent (depending on the label $\y$) and seeking a {\em minimal} set of most informative features for the task at hand (for predicting the label $\y$ only), the so learned network may sacrifice robustness in case the labels can be corrupted or transferrability when we want to use the features for different tasks. To address this, {\em our framework uses the label $\y$ as only side information to assist learning discriminative yet diverse (not minimal) representations; these representations optimize a different intrinsic objective based on the principle of rate reduction.}\footnote{As we will see in the experiments in Section \ref{sec:experiments}, indeed this makes learned features much more robust to mislabeled data.}

\paragraph{Reconciling contractive and contrastive learning.} 
{Complementary to the above supervised discriminative approach, {\em  auto-encoding} \citep{Baldi89,Kramer1991NonlinearPC,Hinton504} is another popular {\em unsupervised} (label-free) framework used to learn good latent representations, which can be viewed as a nonlinear extension to the classical PCA \citep{Jolliffe2002}.} The idea is to learn a compact latent representation $\z \in \Re^n$ that adequately regenerates  the original data $\x$ to certain extent, through optimizing decoder or generator $g(\z, \bm \eta)$: 
\begin{equation}
     \x \xrightarrow{\hspace{2mm} f(\x, \bm \theta)\hspace{2mm}} \z(\bm \theta)  \xrightarrow{\hspace{2mm} g(\z,\bm{\eta}) \hspace{2mm}} \widehat{\x}(\bm \theta, \bm\eta).
     \label{eqn:generative}
\end{equation}
Typically, such representations are learned in an end-to-end fashion by imposing certain heuristics on  geometric or statistical ``compactness'' of $\z$, such as its dimension, energy, or volume. For example, the {\em contractive} autoencoder  \citep{contractive-ICML11} penalizes local volume expansion of learned features approximated by the Jacobian 
\begin{equation}
\min_{\bm \theta} \left\|\frac{\partial \z}{\partial \bm \theta}\right\|.
\end{equation}
When the data contain complicated {\em multi-modal low-dimensional} structures, naive heuristics or inaccurate metrics may fail to capture all internal subclass structures or to explicitly discriminate among them for classification or clustering purposes. For example, one consequence of this is the phenomenon of {\em mode collapsing} in learning generative models for data that have mixed multi-modal structures~\citep{li2020multimodal-IJCV}. To address this, {\em we propose a principled rate reduction measure (on $\z$) that promotes both the within-class compactness and between-class discrimination of the features for data with mixed structures.}

If the above contractive learning seeks to reduce the dimension of the learned representation, {\em contrastive learning} \citep{hadsell2006dimensionality,oord2018representation,he2019momentum} seems to do just the opposite. For data that belong to $k$ different classes, a randomly chosen pair $(\x^i, \x^j)$ is of high probability belonging to difference classes if $k$ is large.\footnote{For example, when $k \ge 100$, a random pair is of probability 99\% belonging to different classes.} Hence it is desirable that the representation $\z^i = f(\x^i, \bm \theta)$ of a sample $\x^i$ should be highly incoherent to those $\z^j$ of other samples $\x^j$ whereas coherent to feature of its transformed version $\tau(\x^i)$, denoted as $\z(\tau(\x^i))$ for $\tau$ in certain augmentation set $\mathcal T$ in consideration. Hence it was proposed heuristically that, to promote  discriminativeness of the learned representation, one may seek to minimize the so-called {\em contrastive loss}: \begin{equation}
\min_{\bm \theta} - \log \frac{\exp(\langle \z^i, \z(\tau(\x^i))\rangle)}{\sum_{j\not= i} \exp(\langle \z^i, \z^j\rangle)},
\label{eqn:contrastive-loss}
\end{equation}
which is small whenever the inner product $\langle \z^i, \z(\tau(\x^i))\rangle$ is large and $\langle \z^i, \z^j\rangle$ is small for $i \not= j$. 

As we may see from the practice of both contractive learning and contrastive learning, for a good representation of the given data, people have striven to achieve certain trade-off between the compactness and discriminativeness of the representation. Contractive learning aims to compress the features of the entire ensemble, whereas contrastive learning aims to expand features of any pair of samples. Hence it is not entirely clear why either of these two seemingly opposite heuristics seems to help learn good features. Could it be the case that both mechanisms are needed but each acts on different part of the data? {\em As we will see, the rate reduction principle precisely reconciles the tension between these two seemingly contradictory objectives by explicitly specifying to compress (or contract) similar features in each class whereas to  expand (or contrast) the set of all features in multiple classes.}

\subsubsection{Architectures for Deep Networks}\label{sec:intro-architecture}
The ultimate goal of any good theory for deep learning is to facilitate a better understanding of deep networks and to design better network architectures and algorithms with performance guarantees. So far we have surveyed many popular objective functions that promote certain desired properties of the learned representation $\z  = f(\x, \bm \theta)$ of the data $\x$. The remaining question is how the mapping $f(\x, \bm \theta)$ should be modeled and learned. 

\paragraph{Empirical designs of deep (convolution) neural networks.}
The current popular practice is to model the mapping with an empirically designed artificial deep neural network and learn the  parameters $\bm \theta$ from random initialization via backpropagation (BP) \citep{Back-Prop}. Starting with the AlexNet \citep{krizhevsky2012imagenet}, the architectures of modern deep networks continue to be empirically revised and improved. Network architectures such as VGG \citep{simonyan2014very}, ResNet \citep{he2016deep}, DenseNet \citep{dense-net},  Recurrent CNN or LSTM \citep{LSTM}, and mixture of experts (MoE) \citep{MoE} etc. have continued to push the performance envelope.

As part of the effort to improve deep networks' performance, almost every component of the networks has been scrutinized empirically and various revisions and improvements have been proposed. They are not limited to the nonlinear activation functions \citep{maas2013rectifier,klambauer2017self,xu2015empirical,nwankpa2018activation,hendrycks2016gaussian}, skip connections \citep{ronneberger2015u,he2016deep}, normalizations \citep{ioffe2015batch,ba2016layer,ulyanov2016instance,wu2018group,miyato2018spectral}, up/down sampling or pooling \citep{scherer2010evaluation}, convolutions~\citep{lecun1998gradient,krizhevsky2012imagenet}, etc.
Nevertheless, almost all such modifications are developed through years of empirical {trial and error} or ablation study. Some recent practices even take to the extreme by searching for effective network structures and training strategies through extensive random search techniques, such as Neural Architecture Search \citep{NAS-1,Baker2017DesigningNN}, AutoML \citep{automl}, and Learning to Learn \citep{andrychowicz2016learning}.

However, there has been apparent lack of direct justification of the resulting network architectures from the desired learning objectives, e.g. cross entropy or contrastive learning. As a result, it is challenging if not impossible to rigorously justify why the resulting network is the best suited for the objective, let alone to interpret the learned operators and parameters inside. In this work, {\em we will attempt to derive network architectures and components as entirely a ``white box'' from the desired objective (say, rate reduction)}.

\paragraph{Constructive approaches to deep (convolution) networks.} 
For long, people have noticed structural similarities between deep networks and iterative optimization algorithms, especially those for solving sparse coding. Even before the revival of deep networks, \citet{gregor2010learning} has argued that algorithms for sparse coding, such as the FISTA algorithm \citep{BeckA2009}, can be viewed as a deep network and be trained using BP for better coding performance, known as LISTA (learned ISTA). Later \citet{papyan2017convolutional,Giryes-2018,monga2019algorithm,deep-sparse} have proposed similar interpretations of deep networks as unrolling algorithms for sparse coding in  convolutional or recurrent settings. However, it remains unclear about the role of the convolutions (dictionary) in each layer and exactly why such low-level sparse coding is needed for the high-level classification task.  To a large extent, {\em this work will provide a new perspective to elucidate the role of the sparsifying convolutions in a deep network: not only will we reveal why sparsity is needed for ensuring invariant classification but also the (multi-channel) convolution operators can be explicitly derived and constructed.}

Almost all of the above networks inherit architectures and initial parameters from sparse coding algorithms, but still rely on back propagation \citep{Back-Prop} to tune these parameters.  There have been efforts that try to construct the network in a {\em purely forward fashion}, without any back propagation. For example, to ensure translational invariance for a wide range of signals, \cite{scattering-net,Wiatowski-2018} have proposed to use  wavelets to construct convolution networks, known as ScatteringNets.
As a ScatteringNet is oblivious to the given data and feature selection for classification, the required number of convolution kernels grow exponentially with the depth. \cite{Zarka2020Deep, zarka2021separation} have also later proposed hybrid deep networks based on scattering transform and dictionary learning to alleviate scalability. Alternatively, \cite{chan2015pcanet} has proposed to construct much more compact (arguably the simplest) networks using principal components of the input data as the convolution kernels, known as PCANets. However, in both cases of ScatteringNets and PCANets the forward-constructed networks seek a representation of the data that is not directly related to a specific (classification) task. To resolve limitations of both the ScatteringNet and the PCANet, {\em this work shows how to construct a data-dependent deep convolution network in a forward fashion that leads to a discriminative representation directly beneficial to the classification task.} More discussion about the relationships between our construction and these networks will be given in Section \ref{sec:architecture-comparison} and Appendix \ref{app:ReduNet-Scattering}.

\subsection{A Principled Objective for Discrimiative Representation via Compression}\label{sec:principled-objective-via-compression} 
Whether the given data $\X$ of a mixed distribution $\mathcal{D}  = \{\mathcal{D}^j\}_{j=1}^k$ can be effectively classified depends on how separable (or discriminative) the component distributions $\mathcal{D}^j$ are (or can be made).
One popular working assumption is that the distribution of each class has relatively {\em low-dimensional} intrinsic structures. There are several reasons why this assumption is plausible: 1). High dimensional data are highly redundant; 2). Data that belong to the same class should be similar and correlated to each other; 3). Typically we only care about equivalent structures of $\x$ that are invariant to certain classes of deformation and augmentations. 
Hence we may assume the distribution $\mathcal{D}^j$ of each class has a support on a low-dimensional submanifold, say $\mathcal{M}^j$ with dimension $d_j \ll D$, and the distribution $\mathcal D$ of $\x$ is supported on the mixture of those submanifolds, $\mathcal M = \cup_{j=1}^k \mathcal{M}^j$,  in the high-dimensional ambient space $\Re^D$, as illustrated in Figure~\ref{fig:low-dim} left. 

\begin{figure*}[t]
\begin{center}
    \subfigure{
    \tdplotsetmaincoords{60}{110}
\begin{tikzpicture}[scale=1.5]
  \coordinate (o1) at (0,0);
  \coordinate (o2) at (4.8,0.8);
  \coordinate (o3) at (6.1,1.2);
  
  \draw[very thick,->] (2.9,0.9) .. controls (3.5,0.7) .. (3.9,0.7);
  \draw (3.3,0.75) node[anchor=north]{$f(\x,{\bm \theta})$};
  \tdplotsetrotatedcoords{100}{0}{0}
  \tdplotsetrotatedcoordsorigin{(o1)}
  \begin{scope}[tdplot_rotated_coords]
  \draw[thick,->] (0,0,0) -- (2,0,0);
  \draw[thick,->] (0,0,0) -- (0,2.0,0);
  \draw[thick,->] (0,0,0) -- (0,0,2);
  \draw (0.3,0.1,1.7) node{$\Re^D$};
  \draw (3.9,0.8,1.7) node{$\Re^n$};

  \draw[thick] (0.5,0.5,0.5) .. controls (1.2,1,0.55) .. (2.4,0.5,0.6);
  \draw[thick] (2.4,0.5,0.6) .. controls (2.4,1.5,0.55) .. (2.9,2.1,0.5);
  \draw[thick] (2.9,2.1,0.5) .. controls (2.0,2.5,0.65) .. (0.8,2.4,0.8);
  \draw[thick] (0.8,2.4,0.8) .. controls (0.55,1.5,0.65) .. (0.5,0.5,0.5);
  \draw (0.6,0.6,0.55) node[anchor=south west]{\small $\mathcal{M}$};
  
  \draw[red] (0.85,0.85,0.75) .. controls (1.6,1.5,0.8) .. (2.6,1.5,0.8);
  \draw[black!30!green] (0.9,1.8,0.75) .. controls (1.6,1.5,0.8) .. (2.4,0.8,0.8);
  \draw[black!10!blue] (1.6,0.6,0.75) .. controls (1.7,1.5,0.8) .. (1.7,2.1,0.8);
  \draw (2.2,2.2,0.55) node[red]{\small{$ \mathcal{M}^1$}};
  \draw (2.15,1.1,0.55) node[black!30!green]{\small $\mathcal{M}^2$};
  \draw (1.35,2.4,0.55) node[black!10!blue]{\small $\mathcal{M}^j$}; 
  \draw (1.4,1.4,0.55) node{$\x^i$}; 
  
  \def\points{(0.95,0.93,0.75), (1.05,1.01,0.75), (1.4,1.35,0.75), (1.75,1.5,0.75), (2.0,1.55,0.75), (2.3,1.56,0.75)}
  \foreach \p in \points {
    \draw plot [mark=*, mark size=0.8, mark options={draw=red, fill=red}] coordinates{\p}; 
  }
  \def\points{(0.95,1.75,0.75), (1.15,1.7,0.75), (1.4,1.6,0.75), (1.75,1.4,0.75), (2.0,1.2,0.75), (2.3,0.95,0.75)}
  \foreach \p in \points {
    \draw plot [mark=*, mark size=0.8, mark options={draw=black!30!green, fill=black!30!green}] coordinates{\p};
  }
  \def\points{(1.61,0.7,0.75), (1.62,0.8,0.75), (1.64,1.1,0.75), (1.65,1.4,0.75), (1.66,1.8,0.75), (1.67,2.1,0.75)}
  \foreach \p in \points {
    \draw plot [mark=*, mark size=0.8, mark options={draw=black!10!blue, fill=black!10!blue}] coordinates{\p};
  }
  \end{scope}
  
  \tdplotsetrotatedcoords{-30}{0}{10}
  \tdplotsetrotatedcoordsorigin{(o2)}
  \begin{scope}[tdplot_rotated_coords]
    \draw [red, ->] (-1.2,0,0) -- (1.2,0,0) node[anchor=north east]{$\mathcal{S}^1$};
    \draw [black!30!green, ->] (0,-1,0) -- (0,1,0) node[anchor=north west]{$\mathcal{S}^2$};
    \draw [black!10!blue, ->] (0,0,-1) -- (0,0,1) node[anchor=east]{$\mathcal{S}^j$};
    \draw (0,0,0.25) -- (-0.25,0,0.25) -- (-0.25,0,0);
    \def\points{(-0.8,0,0), (-0.6,0,0), (-0.3,0,0), (-0.1,0,0), (0.3,0,0), (0.6,0,0)}
    \foreach \p in \points {
      \draw plot [mark=*, mark size=0.8, mark options={draw=red, fill=red}] coordinates{\p}; 
    }
    \def\points{(0,-0.85,0), (0,-0.6,0), (0,-0.3,0), (0,0.1,0), (0,0.4,0), (0,0.6,0)}
    \foreach \p in \points {
      \draw plot [mark=*, mark size=0.8, mark options={draw=black!30!green, fill=black!30!green}] coordinates{\p};
    }
    \def\points{(0,0,-0.9), (0,0,-0.7), (0,0,-0.2), (0,0,-0.1), (0,0,0.4), (0,0,0.8)}
    \foreach \p in \points {
    \draw plot [mark=*, mark size=0.8, mark options={draw=black!10!blue, fill=black!10!blue}] coordinates{\p};
  }
  \draw (0.3,0.0,-0.25) node{$\z^i$}; 
  \end{scope}
\end{tikzpicture}}
    \hspace{5mm} \subfigure{\includegraphics[width=0.27\textwidth]{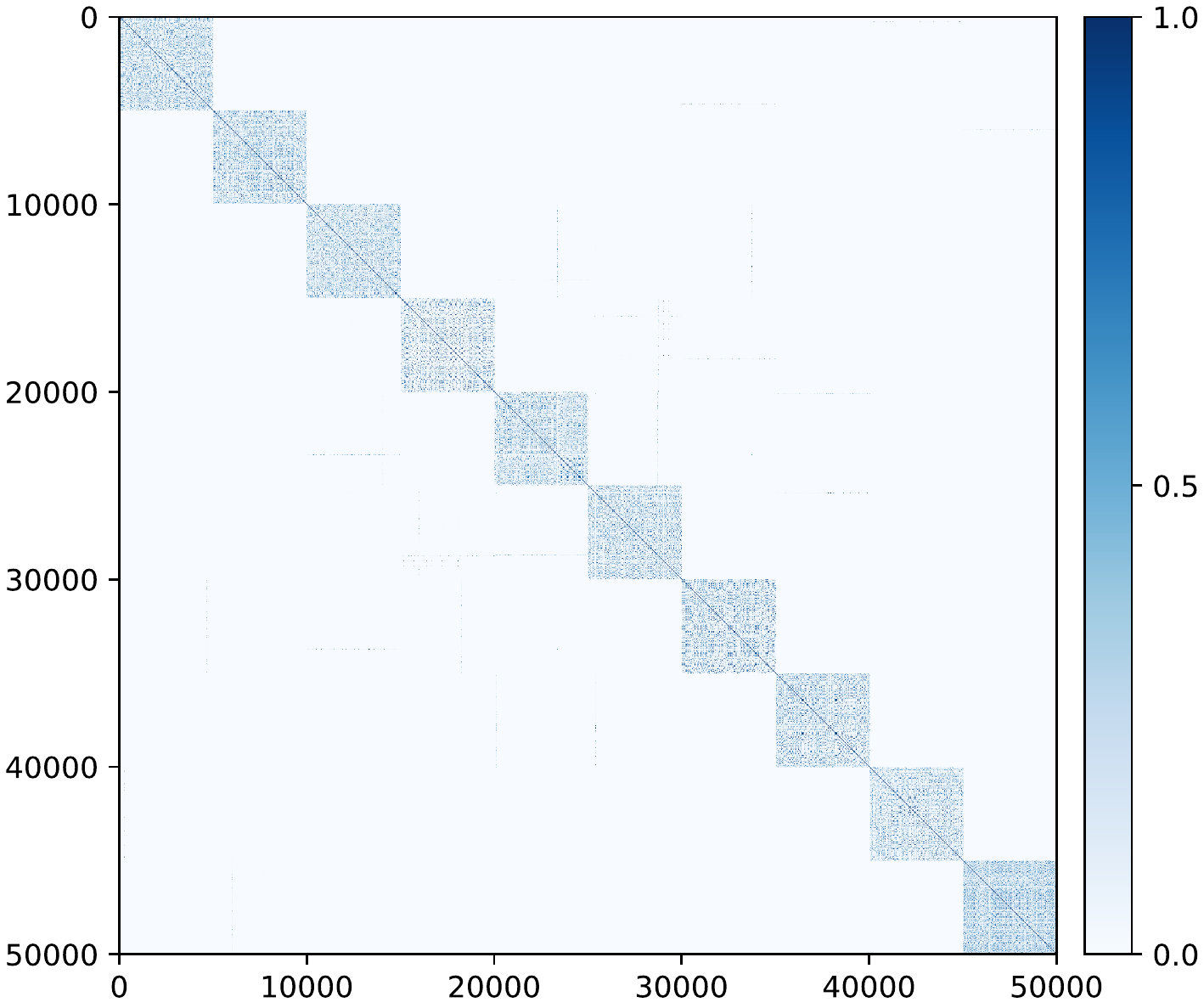}}
\caption{\small \textbf{Left and Middle:} The distribution $\mathcal D$ of high-dim data $\x\in \Re^D$ is supported on  a manifold $\mathcal{M}$ and its classes on low-dim submanifolds $\mathcal{M}^j$, we learn a map $f(\x, \bm  \theta)$ such that $\z^i = f(\x^i, \bm \theta)$ are on a union of maximally uncorrelated subspaces $\{\mathcal{S}^j\}$. \textbf{Right:} Cosine similarity between learned features by our method for the CIFAR10 training dataset. Each class has 5,000 samples and their features span a subspace of over 10 dimensions (see Figure~\ref{fig:train-test-loss-pca-3}).}
\label{fig:low-dim}
\end{center}
\vskip -0.2in
\end{figure*}
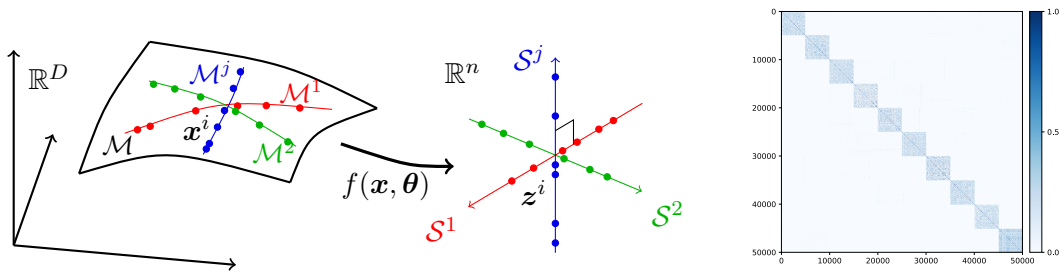

With the manifold assumption in mind, we want to learn a mapping {$\z = f(\x, \bm \theta)$} that maps each of the submanifolds $\mathcal{M}^j \subset \Re^D$  to a {\em linear} subspace $\mathcal{S}^j \subset \Re^n$ (see Figure~\ref{fig:low-dim} middle). To do so, we require our learned representation to have the following properties, called a {\em linear discriminative representation} (LDR): 
\begin{enumerate}
    \item {\em Within-Class Compressible:} Features of samples from the same class/cluster should be relatively {\em correlated} in a sense that they belong to a low-dimensional linear subspace.\footnote{Linearity is a desirable property for many reasons. For example, in engineering, it makes interpolating/extrapolating data easy via superposition which is very useful for generative purposes. There is also scientific evidence that the brain represents objects such as faces as a linear subspace \citep{Chang-Cell-2017}.}
    \item {\em Between-Class Discriminative:} Features of samples from different classes/clusters should be highly {\em uncorrelated} and belong to different low-dimensional linear subspaces.
    \item {\em Diverse Representation:} Dimension (or variance) of features for each class/cluster should be {\em as large as possible} as long as they stay uncorrelated from the other classes.
\end{enumerate}
The third item is desired because we want the learned features to reveal all possible causes why this class is different from all other classes\footnote{For instance, to tell ``Apple'' from ``Orange'', not only do we care about the color, but also the shape and the leaves. Interestingly, some images of computers are labeled as ``Apple'' too.} (see Section \ref{sec:principled-objective} for more detailed justification). 

Notice that the first two items align well with the spirit of the classic {\em  linear discriminant analysis} (LDA)  \citep{HastieTiFr09}. Here, however, although the intrinsic structures of each class/cluster may be low-dimensional, they are by no means simply linear (or Gaussian) in their original representation $\x$ and they need be to made linear through a nonlinear transform $\z = f(\x)$. Unlike LDA (or similarly SVM\footnote{For instance, \cite{RSVM} has proposed to use SVMs in a recursive fashion to build (deep) nonlinear classifiers for complex data.}), here we do not directly seek a discriminant (linear) classifier. Instead, we use the nonlinear transform to seek a {\em linear discriminative representation}\footnote{In this work, to avoid confusion, we always use the word ``discriminative'' to describe a representation, and ``discriminant'' a classifier. To some extent, one may view classification and representation are {\em dual} to each other. Discriminant methods (LDA or SVM) are typically more  natural for two-class settings (despite many extensions to multiple classes), whereas discriminative representations are natural for multi-class data and help reveal the data intrinsic structures more directly.} (LDR) for the data such that the subspaces that represent all the classes are  maximally incoherent.  To some extent, the resulting multiple subspaces $\{\mathcal{S}^j\}$ can be viewed as discriminative {\em generalized principal components} \citep{GPCA} or, if orthogonal, {\em independent components} \citep{HYVARINEN2000411} of the resulting features $\z$ for the original data $\bm x$. As we will see in Section \ref{sec:vector}, deep networks precisely play the role of modeling this nonlinear transform from the data to an LDR!

For many clustering or classification tasks (such as object detection in images), we consider two samples as {\em equivalent} if they differ by certain classes of domain deformations or augmentations $\cT = \{\tau \}$. Hence, we are only interested in low-dimensional structures that are {\em invariant} to such deformations~(i.e., $\x \in \mathcal{M}$ iff $\tau(\x) \in \mathcal{M}$ for all $\tau \in \cT$\,), 
which are known to have sophisticated geometric and topological structures \citep{Wakin-2005} and can be difficult to learn precisely in practice even with rigorously designed CNNs \citep{Cohen-ICML-2016, cohen2019general}. In our framework, this is formulated in a very natural way: all equivariant instances are to be embedded into the same subspace, so that the subspace itself is invariant to the  transformations under consideration (see Section \ref{sec:shift-invariant}).

There are previous attempts to directly enforce subspace structures on features learned by a deep network for supervised \citep{lezama2018ole} or unsupervised learning \citep{Ji-NIPS2017,zhang2018scalable,peng2017deep,zhou2018deep,zhang2019neural,zhang2019self,lezama2018ole}. However, the {\em self-expressive} property of subspaces exploited by these work does not enforce all the desired properties listed above as shown by \cite{haeffele2020critique}. Recently \cite{lezama2018ole} has explored a nuclear norm based geometric loss to enforce orthogonality between classes, but does not promote diversity in the learned representations, as we will soon see. Figure~\ref{fig:low-dim} right illustrates a representation learned by our method on the CIFAR10 dataset. More details can be found in the experimental Section \ref{sec:experiments}. 

In this work, to learn a discriminative linear  representation for intrinsic low-dimensional structures from high-dimensional data, we propose an information-theoretic measure that maximizes the coding rate difference between the whole dataset and the sum of each individual class, known as {\em rate reduction}. This new objective provides a more unifying view of above objectives such as cross-entropy, information bottleneck, contractive and contrastive learning. {\em We can rigorously show that when the intrinsic dimensions the submanifolds are known and this objective is optimized, the resulting representation indeed has the desired properties listed above.}

\subsection{A Constructive Approach to Deep Networks via Optimization}
Despite tremendous advances made by numerous empirically designed deep networks, there is still a lack of rigorous theoretical justification of the need or reason for ``deep layered'' architectures\footnote{After all, at least in theory, \cite{Barron1991ApproximationAE} already proved back in early 90's that a single-layer neural network can efficiently approximate a very general class of functions or mappings.} and a lack of fundamental understanding of the roles of the associated operators, e.g. linear (multi-channel convolution) and nonlinear activation in each layer. Although many works mentioned in Section \ref{sec:intro-architecture} suggest the layered architectures might be interpreted as unrolled optimization algorithms (say for sparse coding), there is lack of explicit and direct connection between such algorithms and the objective (say, minimizing the cross entropy for classification). As a result, there is lack of principles for network design: How wide or deep should the network be? What has improved  about features learned between adjacent layers? Why are multi-channel convolutions necessary for image classification instead of separable convolution kernels\footnote{Convolution kernels/dictionaries such as wavelets have been widely practiced to model and process 2D signals such as images. Convolution kernels used in modern CNNs are nevertheless multi-channel, often involving hundreds of channels altogether at each layer. Most theoretical modeling and analysis for images have been for the 2D kernel case. The precise reason and role for multi-channel convolutions have been elusive and equivocal. This work aims to provide a rigorous explanation.}? Can the network parameters be better initialized than purely randomly?

In this paper, we attempt to provide some answers to the above questions and offer a plausible interpretation of deep neural networks by deriving a class of deep (convolution) networks from the perspective of learning an LDR for the data. We contend that all key features and structures of modern deep (convolution) neural networks can be naturally derived from optimizing the {\em rate reduction} objective, which seeks an optimal (invariant) linear discriminative  representation of the data. More specifically, the basic iterative {\em projected gradient ascent} scheme for optimizing this objective naturally takes the form of a deep neural network, one layer per iteration. In this framework, the width of the network assumes a precise role as the {\em statistical resource} needed to preserve the low-dimensional (separable) structures of the data whereas the network depth as the {\em computational resource} needed to map the (possibly nonlinear) structures to a linear discriminative representation. 

This principled approach brings a couple of nice surprises: First, architectures, operators, and parameters of the network can be constructed explicitly layer-by-layer in a {\em forward propagation} fashion, and all inherit precise optimization, statistical and geometric interpretation. As a result, the so constructed ``white-box'' deep network already gives a rather  discriminative representation for the given data even {\em without any back propagation training} (see Section \ref{sec:vector}). Nevertheless, the so-obtained network is actually amenable to be further fine-tuned by back propagation for better performance, as our experiments will show. Second, in the case of seeking a representation {\em rigorously} {invariant to shift or  translation}, the network naturally lends itself to a multi-channel convolutional network (see Section \ref{sec:shift-invariant}). Moreover, the derivation indicates such a convolutional network is computationally more efficient to construct in the {\em spectral (Fourier) domain}, analogous to how neurons in the visual cortex encode and transit information with their spikes \citep{spiking-neuron-book,Belitski5696}.

\section{The Principle of Maximal Coding Rate Reduction}\label{sec:principled-objective}

\subsection{Measure of Compactness for Linear  Representations}\label{sec:lossy-coding}
Although the three properties listed in Section \ref{sec:principled-objective-via-compression} for linear discriminative representations (LDRs) are all highly desirable for the latent representation $\z$, they are by no means easy to obtain: Are these properties compatible so that we can expect to achieve them all at once? If so, is there  a {\em simple but principled} objective that can measure the goodness of the resulting representations in terms of all these properties? The key to  these questions {is to find} a principled ``measure of compactness'' for the distribution of a random variable $\z$ or from its finite samples $\Z$. Such a measure should directly and accurately characterize intrinsic geometric or statistical properties of the distribution, in terms of its intrinsic dimension or {volume}. Unlike cross-entropy \eqref{eqn:cross-entropy} or information bottleneck \eqref{eqn:information-bottleneck}, such a measure should not depend exclusively on class labels so that it can work in all supervised, self-supervised, semi-supervised, and unsupervised settings.

\paragraph{Measure of compactness from information theory. } 
In information theory \citep{Thomas-Cover}, the notion of entropy $H(\z)$ is designed to be such a measure. 
However, entropy is not  well-defined for continuous random variables with degenerate distributions. 
This is unfortunately the case for data with relatively low intrinsic dimension. 
The same difficulty resides with evaluating mutual information $I(\x, \z)$ for degenerate distributions. 
To alleviate this difficulty, another related concept in information theory, more specifically in lossy data compression, that measures the ``compactness'' of a random distribution is the so-called {\em rate distortion} \citep{Thomas-Cover}: Given a random variable $\z$ and a prescribed precision $\epsilon >0$, the rate distortion $R(\z, \epsilon)$ is the minimal number of binary bits needed to encode $\z$ such that the expected decoding error is less than $\epsilon$, i.e., the decoded $\widehat \z$ satisfies $\mathbb E[\|\z - \widehat \z \|_2] \le \epsilon$. This quantity has been shown to be useful in explaining feature selection in deep networks \citep{rate-distortion}. However, the rate distortion of an arbitrary high-dimensional distribution is intractable, if not impossible, to compute, except for simple distributions such as discrete and Gaussian.\footnote{The same difficulties lie with the information bottleneck framework \citep{Tishby-ITW2015} where one needs to evaluate (difference of) mutual information for degenerate distributions in a high-dimensional space \eqref{eqn:information-bottleneck}.} Nevertheless, as we have discussed in Section \ref{sec:principled-objective-via-compression}, our goal here is to learn a final representation of the data as linear subspaces. Hence we only need a measure of compactness/goodness for this class of distributions. Fortunately, as we will explain below, the rate distortions for this class of distributions can be accurately and easily computed, actually in closed form!

\paragraph{Rate distortion for finite samples on a subspace.} 
Another practical difficulty in evaluating the rate distortion is that we normally do not know the distribution of $\z$.  Instead, we have a finite number of samples as learned representations $\{\z^{i} = f(\x^i, \bm \theta) \in \R^{n}, i = 1,\ldots, m\}$, for the given data samples $\X = [\x^1, \ldots, \x^m]$. Fortunately, \citet{ma2007segmentation} provides a precise estimate on the number of binary bits needed to encode finite samples from a subspace-like distribution. In order to encode the learned representation $\Z = [\z^1, \dots, \z^m]$ up to a precision, say $\epsilon$, the total number of bits needed is given by the following expression:
$\cL(\Z, \epsilon) \doteq \left(\frac{m + n}{2}\right)\log \det\left(\I + \frac{n}{m\epsilon^{2}}\Z\Z^{*}\right)$.\footnote{We use superscript $^*$ to indicate (conjugate) transpose of a vector or a matrix} This formula can be derived either by packing $\epsilon$-balls into the space spanned by $\Z$ as a Gaussian source or by computing the number of bits needed to quantize the SVD of $\Z$ subject to the precision, see \citet{ma2007segmentation} for proofs. Therefore, the compactness of learned features {\em as a whole} can be measured in terms of the average coding length per sample (as the sample size $m$ is large),  a.k.a. the {\em coding rate} subject to the distortion $\epsilon$:
\begin{equation}
R(\Z,\epsilon) \doteq \frac{1}{2}\log\det\left(\I + \frac{n}{m\epsilon^{2}}\Z\Z^{*}\right).
\label{eqn:coding-length-eval}
\end{equation} 
See Figure \ref{fig:lossy-diagram} for an illustration.

\begin{figure}[t]
  \begin{center}
    \includegraphics[width=0.45\textwidth]{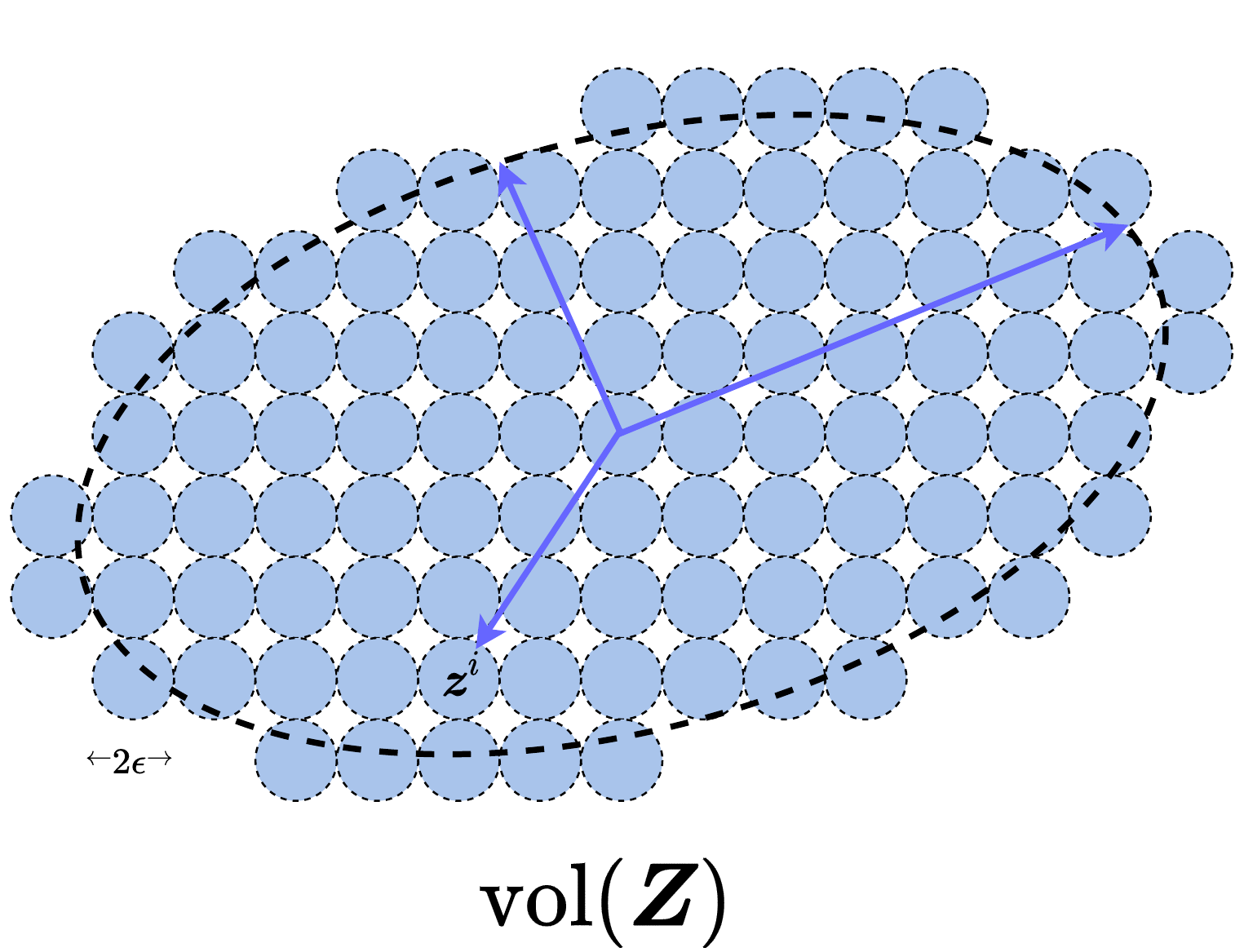}
  \end{center}
  \vspace{-0.2in}
\caption{\small Lossy coding scheme: Given a precision $\epsilon$, we pack the space/volume spanned by the data $\bm Z$ with small balls of diameter $2\epsilon$. The number of balls needed to pack the space gives the number of bits needed to record the location of each data point $\bm z^i$, up to the given precision $\mathbb E[\|\z - \widehat \z \|_2] \le \epsilon$.}\label{fig:lossy-diagram}
\end{figure}

\paragraph{Rate distortion of samples on a mixture of subspaces.} {In general,}  the features $\Z$ of multi-class data may belong to multiple low-dimensional subspaces. To evaluate the rate distortion of such mixed data {\em more accurately}, we may partition the data $\Z$ into multiple subsets: $\Z = \Z^1 \cup \Z^2\cup \ldots \cup \Z^k$, {with} each $\Z^j$ containing samples in one low-dimensional subspace.\footnote{By a little abuse of notation, we here use $\bm Z$ to denote the set of all samples from all the $k$ classes. For convenience, we often represent $\bm Z$ as a matrix whose columns are the samples.} So the above coding rate \eqref{eqn:coding-length-eval} is accurate for each subset. For convenience, let $\bm{\Pi} = \{\bm{\Pi}^j \in \Re^{m \times m}\}_{j=1}^{k}$ be a set of diagonal matrices whose diagonal entries encode the membership of the $m$ samples in the $k$ classes. More specifically, the diagonal entry $\bm \Pi^j(i,i)$ of $\bm \Pi^j$ indicates the probability of sample $i$ belonging to subset $j$. Therefore $\bm{\Pi}$ lies in a simplex: ${\Omega} \doteq \{\bm{\Pi} \mid \bm{\Pi}^j \ge \mathbf{0}, \; \bm{\Pi}^1 + \cdots + \bm{\Pi}^k = \I\}$.
Then, according to \citet{ma2007segmentation}, with respect to this partition, the average number of bits per sample (the coding rate) is
\begin{equation}
R_c(\Z,  \epsilon \mid \bm{\Pi}) \doteq \sum_{j=1}^{k}
\frac{\tr(\bm{\Pi}^j)}{2m}\log\det\left(\I + \frac{n}{\tr(\bm{\Pi}^j)\epsilon^{2}}\Z\bm{\Pi}^j\Z^{*}\right).
\label{eqn:compress-loss-eval}
\end{equation}
When $\Z$ is given, $R_c(\Z, \epsilon \mid \bm{\Pi})$ is a concave function of $\bm{\Pi}$.
The function $\log\det(\cdot)$ in the above expressions has been long known as an effective heuristic for rank minimization problems, with guaranteed convergence to local minimum \citep{fazel2003log-det}. {As it nicely characterizes the rate distortion of Gaussian or subspace-like distributions, $\log\det(\cdot)$ can be very effective in clustering or classification of mixed data \citep{ma2007segmentation,wright2008classification,kang2015logdet}.}

\subsection{Principle of Maximal Coding Rate Reduction}\label{sec:principle-mcr2}
On one hand, for learned features to be discriminative, features of different classes/clusters are preferred to be {\em maximally incoherent} to each other. Hence they together should span a space of the largest possible volume (or dimension) and the coding rate of the whole set $\Z$ should be as large as possible. On the other hand, learned features of the same class/cluster should be highly correlated and coherent. Hence each class/cluster should only span a space (or subspace) of a very small volume and the coding rate should be as small as possible. Shortly put, learned features should follow the basic rule that {\em similarity contracts and dissimilarity contrasts.}

To be more precise, a good (linear) discriminative representation $\Z$ of $\X$ is one such that, given a partition $\bm{\Pi}$ of $\Z$, achieves a large difference between the coding rate for the whole and that for all the subsets:
\begin{equation}
\Delta R(\Z, \bm{\Pi}, \epsilon) \doteq R(\Z, \epsilon) - R_c(\Z, \epsilon \mid  \bm{\Pi}).
\label{eqn:coding-length-reduction}
\end{equation}
If we choose our feature mapping to be $\z = f(\x,\bm \theta)$ (say modeled by a deep neural network), the overall process of the feature representation and the resulting rate reduction w.r.t. certain partition $\bm{\Pi}$ can be illustrated by the following diagram:
\begin{equation}
    \X 
    \xrightarrow{\hspace{2mm} f(\x, \,\bm \theta)\hspace{2mm}} \Z(\bm \theta) \xrightarrow{\hspace{2mm} \bm{\Pi},\epsilon \hspace{2mm}} \Delta R(\Z(\bm \theta), \bm{\Pi}, \epsilon).
    \label{eqn:flow}
\end{equation}

\paragraph{The role of normalization.} Note that $\Delta R$ is {\em monotonic} in the scale of the features $\Z$. So to make the amount of reduction comparable between different representations,
we need to {\em normalize the scale} of the learned features, either by imposing the Frobenius norm of each class $\Z^j$ to scale with the number of features in $\Z^j \in \mathbb R^{n \times m_j}$: $\|\Z^j\|_F^2 = m_j$ or by normalizing each feature to be on the unit sphere: $\z^i \in \mathbb{S}^{n-1}$. This can be compared to the use of ``batch normalization'' in the practice of training deep neural networks \citep{ioffe2015batch}.\footnote{Notice that normalizing the scale of the learned representations helps ensure that the mapping of each layer of the network is approximately {\em isometric}. As it has been shown in the work of \cite{ISOnet}, ensuring the isometric property alone is adequate to ensure good performance of deep networks, even without the batch normalization.} Besides normalizing the scale, normalization could also act as a precondition mechanism that helps accelerate gradient descent \citep{liu2021convolutional}.\footnote{Similar to the role that preconditioning plays in the classic conjugate gradient descent method \citep{Shewchuk-CG,NocedalJ2006}.} This interpretation of normalization becomes even more pertinent when we realize deep networks as an iterative scheme to optimize the rate reduction objective in the next two sections. In this work, to simplify the analysis and derivation, we adopt the simplest possible normalization schemes, by simply enforcing each sample on a sphere or the Frobenius norm of each subset being a constant.\footnote{In practice, to strive for better performance on specific data and tasks, many other normalization schemes can be considered such as layer normalization \citep{ba2016layer}, instance normalization \citep{ulyanov2016instance}, group normalization \citep{wu2018group}, spectral normalization \citep{miyato2018spectral}. }  

Once the representations can be compared fairly, our goal becomes to learn a set of features $\Z(\bm \theta) = f(\X, \bm \theta)$ and their partition $\bm \Pi$ (if not given in advance) such that they maximize the reduction between the coding rate of all features and that of the sum of features w.r.t. their classes:
\vspace*{0.05in} 
\begin{equation}
 \max_{\bm \theta, \,\bm{\Pi}} \;  \Delta R\big(\Z(\bm \theta), \bm{\Pi}, \epsilon\big) = R(\Z(\bm \theta), \epsilon) - R_c(\Z(\bm \theta),  \epsilon \mid \bm{\Pi}), \quad \mbox{s.t.} \ \ \,  \|\Z^j(\bm \theta)\|_F^2 = m_j, \, \bm{\Pi} \in {\Omega}.
\label{eqn:maximal-rate-reduction}
\end{equation}
We refer to this as the principle of {\em maximal coding rate reduction} (MCR$^2$), 
{an embodiment of Aristotle's famous quote:  ``{\em the whole is greater than the sum of the  parts.}''}  Note that for the clustering purpose alone, one may only care about the sign of $\Delta R$ for deciding whether to partition the data or not, which leads to the greedy algorithm in \citep{ma2007segmentation}. More specifically, in the context of clustering {\em finite} samples, one needs to use the more precise measure of the coding length mentioned earlier, see \citep{ma2007segmentation} for more details. Here to seek or learn the most discriminative representation, we further desire that {\em  the whole is maximally greater than the sum of the parts}. This principle is illustrated with a simple example in Figure \ref{fig:sphere-packing}.  
\begin{figure}[t]
  \begin{center}
    \includegraphics[width=0.65\textwidth]{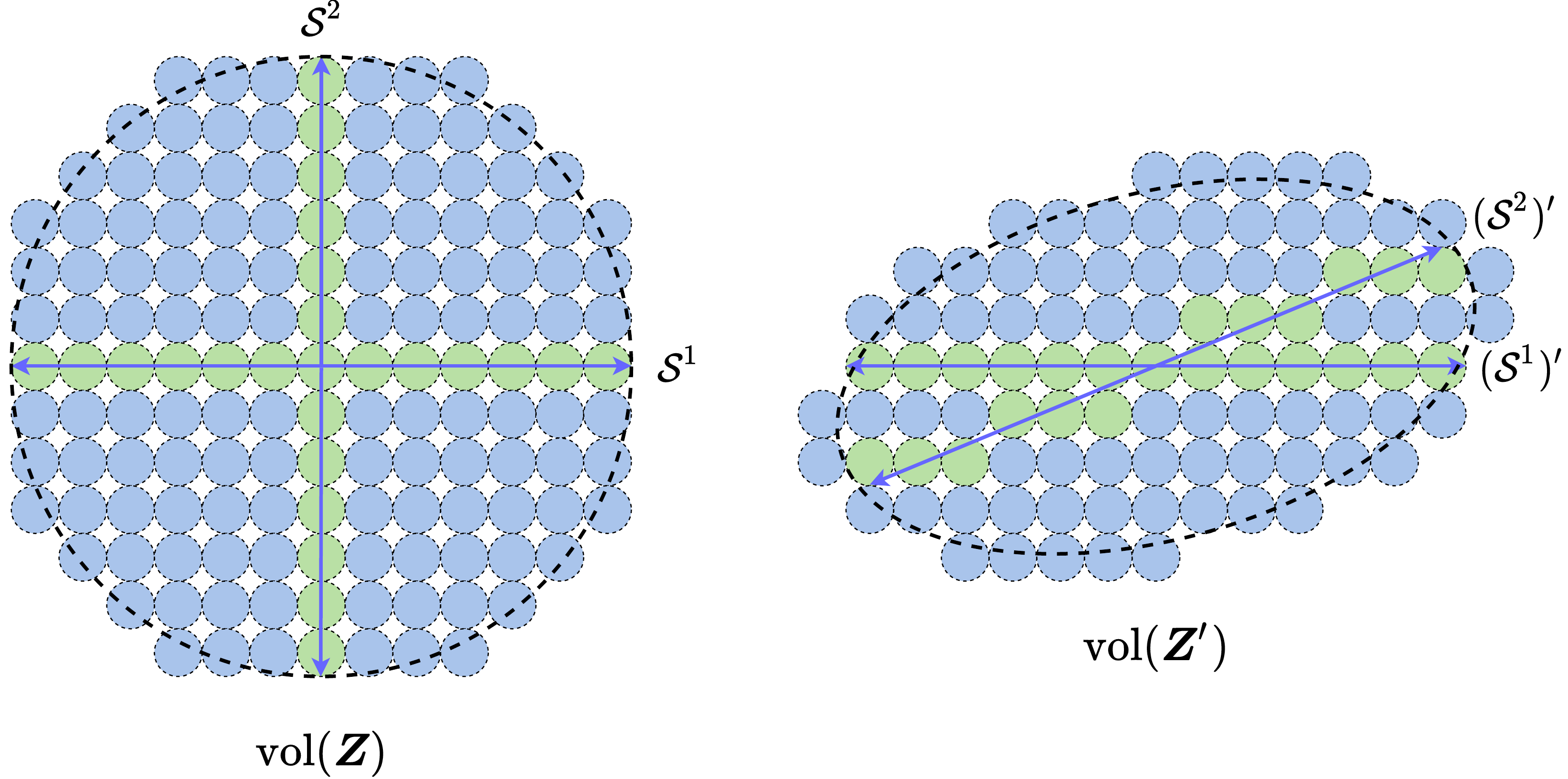}
  \end{center}
    \vspace{-0.2in}
\caption{\small Comparison of two learned representations $\Z$ and $\Z'$ via reduced rates: $R$ is the number of $\epsilon$-balls packed in the joint distribution and $R_c$ is the sum of the numbers for all the subspaces (the green balls). $\Delta R$ is their difference (the number of blue balls). The MCR$^2$ principle prefers $\Z$ (the left one).}\label{fig:sphere-packing}
\end{figure}

\paragraph{Relationship to information gain.} 
The maximal coding rate reduction can be viewed as a generalization to {\em information gain} (IG), which aims to maximize the reduction of entropy of a random variable, say $\bm z$, with respect to an observed attribute, say $\bm \pi$:
$
\max_{\bm \pi} \; \mbox{IG}(\bm z, \bm \pi) \doteq H(\bm z) - H(\bm z \mid \bm \pi),
$
i.e., the {\em mutual information} between $\z$ and $\bm \pi$ \citep{Thomas-Cover}. Maximal information gain has been widely used in areas such as decision trees \citep{decision-trees}. However, MCR$^2$ is used differently in several ways: 1) One typical setting of MCR$^2$ is when the data class labels are given, i.e. $\bm \Pi$ is known, MCR$^2$ focuses on learning representations $\bm z(\bm \theta)$ rather than fitting labels. 2) In traditional settings of IG, the number of attributes in $\bm z$ cannot be so large and their values are discrete (typically binary). Here the ``attributes'' $\bm \Pi$  represent the probability of a multi-class partition for all samples and their values can even be continuous. 3) As mentioned before, entropy $H(\bm z)$ or mutual information $I(\bm z, \bm \pi)$ \citep{hjelm2018learning} is not well-defined for degenerate continuous distributions or hard to compute for high-dimensional distributions,  whereas here the rate distortion $R(\bm z, \epsilon)$ is and can be accurately and efficiently computed for (mixed)  subspaces, at least.

\subsection{Properties of the Rate Reduction Function}
\label{sec:rate-reduction-properties}
In theory, the MCR$^2$ principle \eqref{eqn:maximal-rate-reduction} is very general and can be applied to representations $\Z$ of {\em any} distributions with {\em any} attributes $\bm \Pi$ as long as the rates $R$ and $R_c$ for the distributions can be accurately and efficiently evaluated. The optimal representation $\Z_\star$ and partition $\bm \Pi_\star$ should have some interesting geometric and statistical properties. We here reveal nice properties of the optimal representation with the special case of linear subspaces, which have many important use  cases in machine learning. When the desired representation for $\Z$ is multiple subspaces (or Gaussians), the rates $R$ and $R_c$ in \eqref{eqn:maximal-rate-reduction} are given by \eqref{eqn:coding-length-eval}
and \eqref{eqn:compress-loss-eval}, respectively. At any optimal representation, denoted as $\Z_\star = \Z^1_\star\cup \cdots \cup \Z^k_\star \subset \Re^n$, it should achieve the maximal rate reduction. One can show that $\Z_\star$ has the following desired properties (see Appendix \ref{ap:rate-reduction} for a formal statement and detailed proofs).

\begin{theorem}[Informal Statement]
Suppose $\Z_\star = \Z^1_\star\cup \cdots \cup \Z^k_\star$ is the optimal solution that maximizes the rate reduction~\eqref{eqn:maximal-rate-reduction} with  the rates $R$ and $R_c$ given by \eqref{eqn:coding-length-eval}
and \eqref{eqn:compress-loss-eval}. Assume that the optimal solution satisfies $\rank{(\Z^j_\star)}\le d_j$.\footnote{Notice that here we assume we know a good upper bound for the dimension $d_j$ of each class. This requires us to know the intrinsic dimension of the submanifold $\mathcal{M}^j$. In general, this can be a very challenging problem itself even when the submanifold is linear but noisy, which is still an active research topic \citep{hong2020selecting}. Nevertheless, in practice, we can decide the dimension empirically through ablation experiments, see for example Table \ref{table:ablation-supervise} for experiments on the CIFAR10 dataset.} We have:
\begin{itemize}
\item {\em Between-class Discriminative}: As long as the ambient space is adequately large ($n \ge \sum_{j=1}^{k} d_j$), the subspaces are all orthogonal to each other, {\em i.e.} $(\Z^i_\star)^{*} \Z^{j}_\star = \bm 0$ for $i \not= j$.
\item {\em Maximally Diverse Representation}: 
As long as the coding precision is adequately high, i.e., $\epsilon ^4 < \min_{j}\Big\{ \frac{m_j}{m}\frac{n^2}{d_j^2}\Big\}$, each subspace achieves its maximal dimension, i.e. $\rank{(\Z^j_\star)}= d_j$. In addition, the largest $d_j - 1$ singular values of $\Z^{j}_\star$ are equal. \label{thm:MCR2-properties}
\end{itemize}
\end{theorem}

In other words, the MCR$^2$ principle promotes embedding of data into multiple independent subspaces,\footnote{In this case, the subspaces can be viewed as the {\em independent components} \citep{HYVARINEN2000411} of the so learned features. However, when the condition $n \ge \sum_{j=1}^{k} d_j$ is violated, the optimal subspaces may not be orthogonal. But experiments show that they tend to be maximally incoherent, see Appendix \ref{sec:appendix-gaussian}.} with features distributed {\em isotropically}  in each subspace (except for possibly one dimension). In addition, among all such discriminative representations, it prefers the one with the highest dimensions in the ambient space (see Section \ref{sec:experiment-objective} and Appendix \ref{ap:additional-exp} for experimental verification). This is substantially different from objectives such as the cross entropy \eqref{eqn:cross-entropy} and information bottleneck~\eqref{eqn:information-bottleneck}. The optimal representation associated with MCR$^2$ is indeed an LDR according to definition given in Section \ref{sec:principled-objective-via-compression}.

\paragraph{Relation to neural collapse. } A line of recent work \cite{papyan2020prevalence,han2021neural,mixon2020neural,zhu2021geometric} discovered, both empirically and theoretically, that deep networks trained via cross-entropy or mean squared losses produce \emph{neural collapse} features. That is, features from each class become identical, and different classes are maximally separated from each other. 
In terms of the coding rate function, the neural collapse solution is preferred by the objective of minimizing the coding rate of the ``parts'', namely $R_c(\Z(\bm \theta),  \epsilon \mid \bm{\Pi})$ (see Table~\ref{table:simulations} of the Appendix \ref{ap:additional-exp}).  However, the neural collapse solution leads to a small coding rate for the ``whole'', namely $R(\Z(\bm \theta))$, hence is not an optimal solution for maximizing the rate reduction. Therefore, the benefit of MCR$^2$ in preventing the collapsing of the features from each class and producing maximally diverse representations can be attributed to introducing and maximizing the term $R(\Z(\bm \theta))$.

\paragraph{Comparison to the geometric OLE loss.} To encourage the learned features to be uncorrelated between classes, the work of \citet{lezama2018ole} has proposed to maximize the difference between the nuclear norm of the whole $\Z$ and its subsets $\Z^j$, called the {\em orthogonal low-rank embedding} (OLE) loss:
$
    \max_{\bm \theta}\,
    \mbox{OLE}(\Z(\bm \theta), \bm \Pi) \doteq  \|\Z(\bm \theta)\|_* - \sum_{j=1}^k \|\Z^j(\bm \theta)\|_*,
$
added as a regularizer to the cross-entropy loss \eqref{eqn:cross-entropy}. The nuclear norm $\|\cdot \|_*$ is a {\em nonsmooth convex}
surrogate for low-rankness and the nonsmoothness potentially poses additional difficulties in using this loss to learn features via gradient descent, whereas $\log\det(\cdot)$ is {\em smooth concave} instead. Unlike the rate reduction $\Delta R$, OLE is always {\em negative} and achieves the maximal value $0$ when the subspaces are orthogonal, regardless of their dimensions. So in contrast to $\Delta R$, this loss serves as a geometric heuristic and does not promote diverse representations.  In fact, OLE typically promotes learning one-dimensional representations per class, whereas MCR$^2$ encourages learning subspaces with maximal dimensions (Figure~7 of \cite{lezama2018ole} versus our Figure~\ref{fig:pca-plot}). More importantly, as we will see in the next section, the precise form of the rate distortion plays a crucial role in deriving the deep network operators, with precise statistical and geometrical meaning. 

\paragraph{Relation to contractive or contrastive learning.}
If samples are {\em evenly} drawn from $k$ classes, a randomly chosen pair $(\x^i, \x^j)$ is of high probability belonging to different classes if $k$ is large. For example, when $k \ge 100$, a random pair is of probability 99\% belonging to different classes. 
We may view the learned features of two samples together with their augmentations $\Z^i$ and $\Z^j$ as two classes. Then the rate reduction $\Delta R^{ij} = R(\Z^i\cup \Z^j, \epsilon) - \frac{1}{2}(R(\Z^i, \epsilon) + R(\Z^j, \epsilon))$ gives a ``distance'' measure for how far the two sample sets are. We may try to further ``expand'' pairs that likely belong to different classes. From Theorem \ref{thm:MCR2-properties},  the (averaged) rate reduction $\Delta R^{ij}$ is maximized  when features from different samples are uncorrelated $(\Z^i)^* \Z^j = \bm 0$  (see Figure~\ref{fig:sphere-packing}) and features $\Z^i$ from augmentations of the same sample are compressed into the same subspace. Hence, when applied to sample pairs, MCR$^2$ naturally conducts the so-called {\em contrastive learning} \citep{hadsell2006dimensionality,oord2018representation,he2019momentum} and {\em contractive learning} \citep{contractive-ICML11} together that we have discussed in the introduction Section \ref{sec:intro-objective}. But MCR$^2$ is {\em not} limited to expand or compress pairs of samples and can uniformly conduct ``contrastive/contractive learning'' for a subset with {\em any number} of samples as long as we know they likely belong to different (or the same) classes, say by randomly sampling subsets from a large number of classes or with membership derived from a good clustering method.

\section{Deep Networks from Maximizing Rate Reduction}\label{sec:vector}
In the above section, we have presented rate  reduction \eqref{eqn:maximal-rate-reduction} as a principled objective for learning a  linear discriminative representation (LDR) for the data. We have, however, not specified the architecture of the feature mapping $\z(\bm \theta) = f(\x, \bm \theta)$ for extracting such a representation from input data $\x$. 
A straightforward choice is to use a conventional deep network, such as ResNet, for implementing $f(\x, \bm \theta)$. 
As we show in the experiments (see Section~\ref{sec:experiment-objective}), we can effectively optimize the MCR$^2$ objective with a ResNet architecture and obtain discriminative and diverse representations for real image data sets.

There remain several unanswered problems with using a ResNet. Although the learned feature representation is now more interpretable, the network itself is still {\em not}. It is unclear why any chosen ``black-box'' network is able to optimize the desired MCR$^2$ objective at all. The good empirical results (say with a ResNet) do not necessarily justify the particular choice in architectures and operators of the network:  Why is a deep layered model necessary;\footnote{Especially it is already long known that even a single layer neural network is already a universal functional approximator with tractable model complexity \citep{Barron1991ApproximationAE}.} what do additional layers try to improve or simplify; how wide and deep is adequate; or is there any rigorous justification for the convolutions (in the popular multi-channel form) and nonlinear operators (e.g. ReLu or softmax) used? In this section, we show that using gradient ascent to maximize the rate reduction $\Delta R(\Z )$ naturally leads to a ``white-box'' deep network that represents such a mapping. All network layered architecture, linear/nonlinear operators, and parameters are {\em explicitly constructed in a purely forward propagation fashion}.

\subsection{Gradient Ascent for Rate Reduction on the Training} 
From the previous section, we see that mathematically, we are essentially seeking a continuous mapping $f(\cdot): \x \mapsto \z$ from the data $\X = [\x^1, \ldots, \x^m] \in \Re^{D \times m}$ (or initial features extracted from the data\footnote{As we will see the necessity of such a feature extraction in the next section.}) to an optimal representation $\Z = [\z^1, \ldots, \z^m] \subset \Re^{n \times m}$ that maximizes the following coding rate reduction objective:
\begin{equation}\label{eq:mcr2-formulation}
\begin{split}
\Delta R(\Z, \bm{\Pi}, \epsilon) &= R(\Z, \epsilon) - R_c(\Z, \epsilon \mid  \bm{\Pi})\\
&\doteq \underbrace{\frac{1}{2}\log\det \Big(\I + {\alpha} \Z \Z^{*} \Big)}_{R(\Z, \epsilon)} \;-\; \underbrace{\sum_{j=1}^{k}\frac{\gamma_j}{2} \log\det\Big(\I + {\alpha_j} \Z \bm{\Pi}^{j} \Z^{*} \Big)}_{R_c(\Z, \epsilon \mid\bm \Pi)},
\end{split}
\end{equation}
where for simplicity we denote $\alpha=\frac{n}{m\epsilon^2}$, $\alpha_j=\frac{n}{\textsf{tr}(\bm{\Pi}^{j})\epsilon^2}$, $\gamma_j=\frac{\textsf{tr}(\bm{\Pi}^{j})}{m}$ for $j = 1,\ldots, k$.

The question really boils down to whether there is a {\em constructive} way of finding such a continuous mapping $f(\cdot)$ from $\bm x$ to $\bm z$? To this end, let us consider incrementally maximizing the objective $\Delta R(\Z )$ as a function of $\Z \subset \mathbb{S}^{n-1}$. Although there might be many optimization schemes to choose from, for simplicity we first consider the arguably simplest projected {\em gradient ascent} (PGA)  scheme:\footnote{Notice that we use superscript $j$ on $\Z^j$ to indicate features in the $j$th class and subscript $\ell$ on $\Z_\ell$ to indicate all features at $\ell$-th iteration or layer.} 
\begin{equation}
\bm Z_{\ell+1}   \; \propto \; \bm Z_{\ell} + \eta \cdot \frac{\partial \Delta R}{\partial \bm Z}\bigg|_{\Z_\ell}
\quad \mbox{s.t.} \quad \Z_{\ell+1} \subset \mathbb{S}^{n-1}, \; \ell = 1, 2, \ldots,
\label{eqn:gradient-descent}
\end{equation}
for some step size $\eta >0$ and the iterate starts with the given data $\bm Z_1 = \bm X$\footnote{Again, for simplicity, we here first assume the initial features $\bm Z_1$ are the data themselves. Hence the data and the features have the same dimension $n$. This needs not to be the case though. As we will see in the next section, the initial features can be some (lifted) features of the data to begin with and could in principle have a different (much higher) dimension. All subsequent iterates have the same dimension.}. 
This scheme can be interpreted as how one should incrementally adjust locations of the current features $\Z_\ell$, initialized as the input data $\bm X$, in order for the resulting $\Z_{\ell +1}$ to improve the rate reduction $\Delta R(\Z)$, as illustrated in Figure \ref{fig:gradient-flow}. 
\begin{figure}
\centering
    \includegraphics[width=0.85\linewidth]{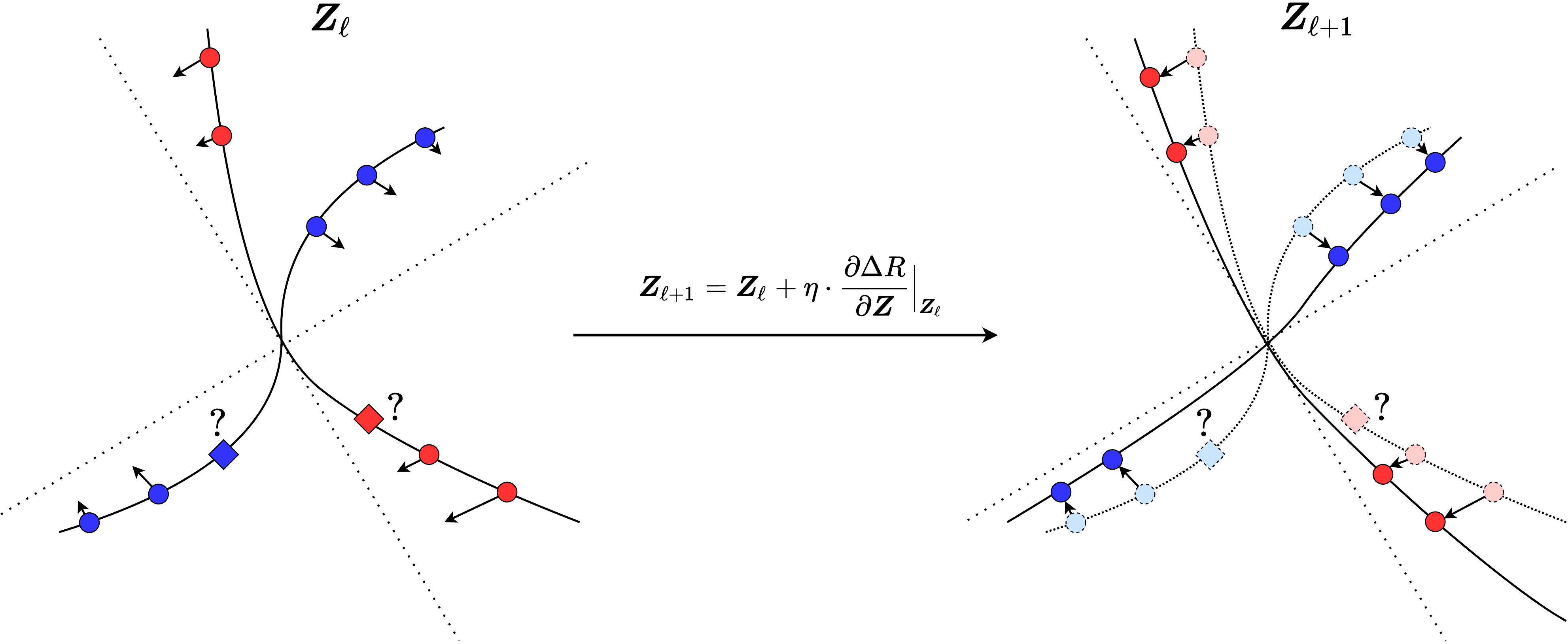} 
    \caption{Incremental deformation via gradient flow to both flatten data of each class into a subspace and push different classes apart. Notice that for points whose memberships are unknown, those marked as ``$\diamond$'', their gradient cannot be directly calculated.}
    \label{fig:gradient-flow}
\end{figure} 

Simple calculation shows that the gradient $\frac{\partial \Delta R}{\partial \bm Z}$ entails evaluating the following derivatives of the two terms in $\Delta R(\Z )$:
\begin{equation}\label{eqn:expand-directions}
    \frac{1}{2}\frac{\partial \log \det (\I \!+\! \alpha \Z \Z^{*} )}{\partial \bm Z}\bigg|_{\Z_\ell} = \,\, \underbrace{\alpha(\I \!+\! \alpha\Z_\ell \Z_\ell^{*})^{-1}}_{\E_{\ell} \; \in \Re^{n\times n}}\Z_\ell,
\end{equation}

\begin{equation}\label{eqn:compress-directions}
\frac{1}{2}\frac{\partial \left( \gamma_j  \log \det (\I + \alpha_j \Z \bm \Pi^j \Z^{*} )  \right)}{\partial \bm Z}\bigg|_{\Z_\ell} \\= \gamma_j  \underbrace{ \alpha_j  (\I +  \alpha_j \Z_\ell \bm \Pi^j \Z_\ell^{*})^{-1}}_{\bm C_{\ell}^j \; \in \Re^{n\times n}} \Z_{\ell} \bm \Pi^j.
\end{equation}
Notice that in the above, the matrix $\bm E_{\ell}$  only depends on $\Z_{\ell}$ and it aims to {\em expand} all the features to increase the overall coding rate; the matrix $\bm C_{\ell}^{j}$ depends on features from each class and aims to {\em compress} them to reduce the coding rate of each class. 
Then the complete gradient $\frac{\partial \Delta R}{\partial \bm Z}\big|_{\Z_\ell} \in \Re^{n\times m}$ is of the  form:
\begin{equation}
\frac{\partial \Delta R}{\partial \bm Z}\bigg|_{\Z_\ell}  = \underbrace{\bm E_{\ell}}_{\text{Expansion}} \Z_{\ell} \;-\; \sum_{j=1}^k \gamma_j \underbrace{\bm C^{j}_{\ell}}_{\text{Compression}}  \Z_{\ell} \bm{\Pi}^j.
\label{eqn:DR-gradient}
\end{equation}

\begin{remark}[Interpretation of $\bm E_\ell$ and $\bm C^j_\ell$ as Linear Operators]\label{rem:regression-interpretation} 
For any $\z_\ell \in \mathbb{R}^n$,
\begin{gather}
    \bm E_\ell \z_\ell = \alpha(\z_\ell - \Z_\ell [\q_\ell]_\star) \quad
    \mbox{where}\quad [\q_\ell]_\star \doteq \argmin_{\q_\ell} \alpha \|\z_\ell - \Z_\ell \q_\ell\|_2^2 + \|\q_\ell\|_2^2.
\end{gather}
Notice that $[\q_\ell]_\star$ is exactly the solution to the ridge regression by all the data points $\Z_\ell$ concerned. Therefore, $\E_\ell$ (similarly for $\bm C_\ell^j$) is approximately (i.e. when $m$ is large enough) the projection onto the orthogonal complement of the subspace spanned by columns of $\Z_\ell$.
Another way to interpret the matrix $\E_\ell$ is through eigenvalue decomposition of the covariance matrix $\Z_\ell \Z_\ell^*$. Assuming that $\Z_\ell \Z_\ell^* \doteq \U_\ell \bm \Lambda_\ell \U_\ell^*$ where $\bm \Lambda_\ell \doteq \diag\{\lambda_\ell^1, \ldots, \lambda_\ell^n \}$, we have 
\begin{equation}\E_\ell = \alpha\, \U_\ell\, \diag\left\{\frac{1}{1+\alpha\lambda_\ell^1}, \ldots, \frac{1}{1+\alpha\lambda_\ell^n}\right\} \U_\ell^*.
\end{equation}
Therefore, the matrix $\E_\ell$ operates on a vector $\z_\ell$ by stretching in a way that directions of large variance are shrunk while directions of vanishing variance are kept. These are exactly the directions \eqref{eqn:expand-directions} in which we move the features so that the overall volume expands and the coding rate will increase, hence the positive sign. To the opposite effect, the directions associated with \eqref{eqn:compress-directions} are ``residuals'' of features of each class deviate from the subspace to which they are supposed to belong. These are exactly the directions in which the features need to be compressed back onto their respective subspace, hence the negative sign (see Figure~\ref{fig:regression-interpretation}). 

\begin{figure}[t]
    \centering
    \includegraphics[width=0.65\linewidth]{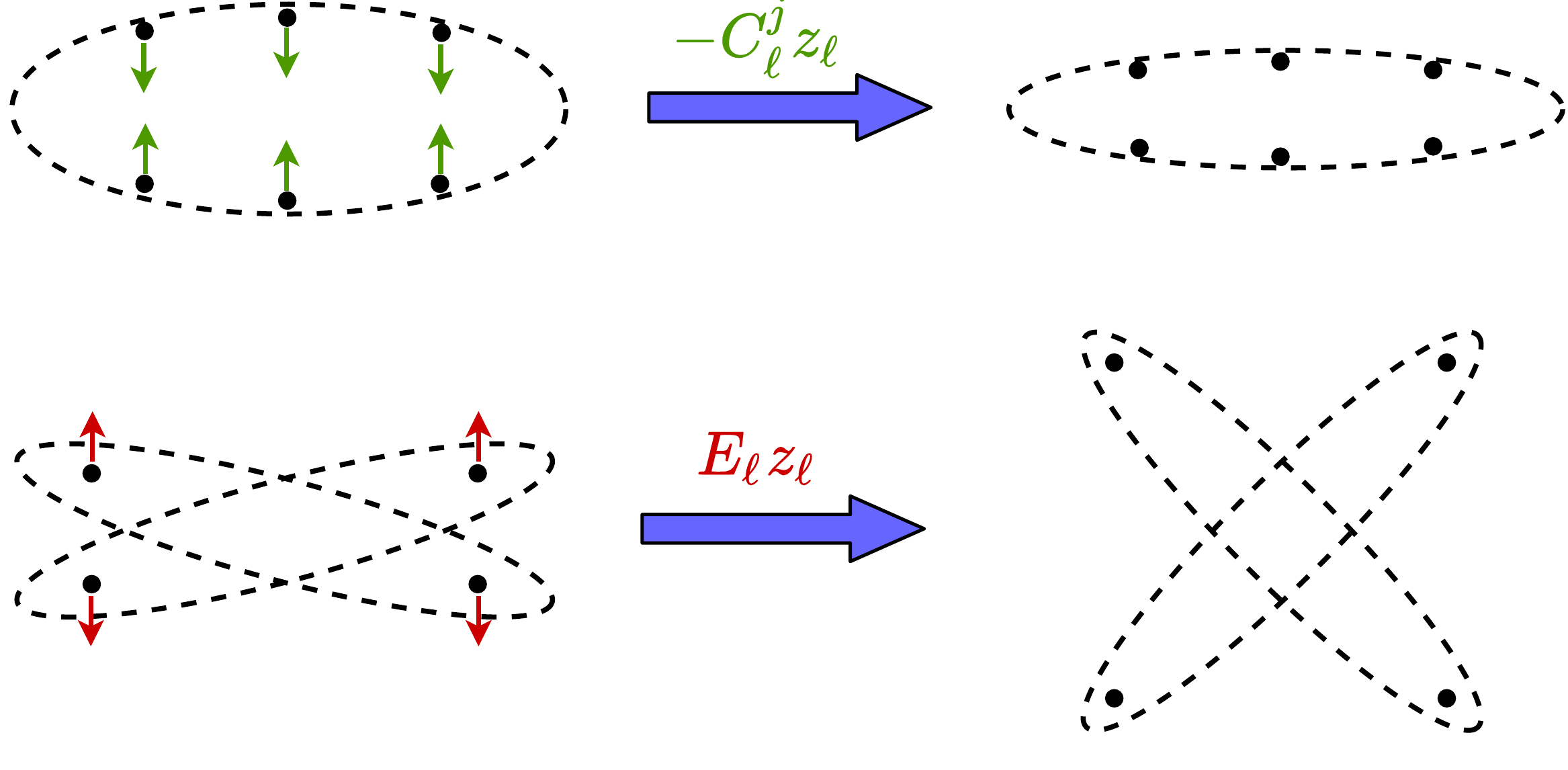}
    \caption{\small Interpretation of $\bm C_\ell^j$ and $\bm E_\ell$: $\bm C_\ell^j$ compresses each class by contracting the features to a low-dimensional subspace; $\bm E_\ell$ expands all features by contrasting and repelling features across different classes.}
    \label{fig:regression-interpretation}
    \vspace{-0.1in}
\end{figure}

Essentially, the linear operations $\bm E_\ell$ and $\bm C^j_\ell$ in gradient ascend for rate reduction are determined by training data conducting ``auto-regressions". The recent renewed understanding about ridge regression in an over-parameterized setting \citep{yang2020rethinking,Wu2020OnTO} indicates that using seemingly redundantly sampled data (from each subspaces) as regressors do not lead to overfitting.
\end{remark}

\subsection{Gradient-Guided Feature Map Increment}
Notice that in the above, the gradient ascent considers all the features $\Z_{\ell} = [\z_{\ell}^{1}, \dots, \z_{\ell}^{m}]$ as free variables. The increment $\Z_{\ell+1} - \Z_{\ell} = \eta \frac{\partial \Delta R}{\partial \bm Z}\big|_{\Z_\ell}$ does not yet give a transform on the entire feature domain $\z_\ell \in \Re^n$. According to equation \eqref{eqn:DR-gradient}, the gradient cannot be evaluated at a point whose membership is not known, as illustrated in Figure \ref{fig:gradient-flow}. Hence, in order to find the optimal $f(\x,\bm  \theta)$ explicitly, we may consider constructing a small increment transform $g(\cdot, \bm{\theta}_{\ell})$ on the $\ell$-th layer feature $\z_\ell$ to emulate the above (projected) gradient scheme:
\begin{equation}
\z_{\ell + 1}   \; \propto \; \z_{\ell} + \eta\cdot  g(\z_{\ell}, \bm{\theta}_{\ell}) \quad \mbox{subject to} \quad \z_{\ell +1} \in \mathbb{S}^{n-1}
\label{eqn:gradient-descent-transform}
\end{equation}
such that: $\big[g(\z^1_{\ell}, \bm \theta_{\ell}), \ldots, g(\z^m_{\ell}, \bm \theta_{\ell}) \big] \approx \frac{\partial \Delta R}{\partial \bm Z}\big|_{\Z_\ell}.$ That is, we need to approximate the gradient flow $\frac{\partial \Delta R}{\partial \bm Z}$ that locally deforms all (training) features $\{\z^i_\ell\}_{i=1}^m$ with a continuous mapping $g(\z)$ defined on the entire feature space $\z_\ell \in \Re^n$.  
Notice that one may interpret the increment \eqref{eqn:gradient-descent-transform} as a discretized version of a continuous  differential equation:
\begin{equation}
\dot{\z} = g(\z, \theta).
\end{equation}
Hence the (deep) network so constructed can be interpreted as certain neural ODE \citep{chen2018neural}. Nevertheless, unlike neural ODE where the flow $g$ is chosen to be some generic structures, here our $g(\z, \theta)$ is to emulate the gradient flow of the rate reduction on the feature set (as shown in Figure \ref{fig:gradient-flow}): \begin{equation}
    \dot{\Z} = \frac{\partial \Delta R}{\partial \bm Z},
\end{equation} 
and its structure is entirely derived and fully determined from this objective, without any other priors or heuristics.

By inspecting the structure of the gradient \eqref{eqn:DR-gradient}, it suggests that a natural candidate for the increment transform $g(\z_\ell, \bm \theta_\ell)$ is of the form:
\begin{equation}
    g(\z_\ell, \bm \theta_\ell) \; \doteq \; \E_\ell \z_\ell - \sum_{j=1}^k \gamma_j \bm C^{j}_{\ell}  \z_{\ell} \bm \pi^j(\z_\ell) \quad \in \Re^n,
    \label{eqn:DR-gradient-transform}
\end{equation}
where  $\bm \pi^j(\z_\ell) \in [0,1]$ {indicates the probability of $\z_{\ell}$ belonging to the $j$-th class.}
The increment depends on: First, a set of linear maps represented by $\bm E_{\ell}$ and $\{ \bm C_{\ell}^{j}\}_{j=1}^{k}$ that depend only on statistics of features of the training $\Z_\ell$; Second, membership $\{ \bm \pi^j(\z_\ell)\}_{j=1}^k$ of any feature $\z_\ell$. 
Notice that on the training samples $\Z_\ell$, for which the memberships $\bm \Pi^j$ are known,  the so defined $g(\z_\ell, \bm \theta)$ gives exactly the values for the gradient $\frac{\partial \Delta R}{\partial \bm Z}\big|_{\Z_\ell}$.

Since we only have the membership $\bm \pi^j$ for the training samples, the function $g(\cdot)$ defined in \eqref{eqn:DR-gradient-transform} can only be evaluated on the training. To extrapolate $g(\cdot)$ to the entire feature space, we need to estimate $\bm \pi^j(\z_\ell)$ in its second term. In the conventional deep learning, this map is typically modeled as a deep network and learned from the training data, say via {\em back propagation}. Nevertheless, our goal here is not to learn a precise classifier $\bm \pi^{j}(\z_\ell)$ already. Instead, we only need a good enough estimate of the class information in order for $g(\cdot)$ to approximate the gradient $\frac{\partial \Delta R}{\partial \bm Z}$ well.

From the geometric interpretation of the linear maps $\bm E_\ell$ and $\bm C^j_\ell$ given by Remark \ref{rem:regression-interpretation}, the term $\p^{j}_{\ell} \doteq \bm C_{\ell}^j \z_{\ell}$ can be viewed as (approximately) the projection of $\z_{\ell}$ onto the orthogonal complement of each class $j$. 
Therefore, $\|\p^{j}_{\ell}\|_2$ is small if $\z_\ell$ is in class $j$ and large otherwise. This motivates us to estimate its membership based on the following softmax function:
\begin{equation}
\widehat{\bm \pi}^j(\z_\ell) \doteq \frac{\exp{( -\lambda \|\bm C_{\ell}^j  \z_{\ell} \|)} }{ \sum_{j=1}^{k} \exp{(-\lambda \|\bm C_{\ell}^j  \z_{\ell}  \|)}} \in [0,1].    
\end{equation}
Hence the second term of \eqref{eqn:DR-gradient-transform} can be approximated by this estimated membership:
\begin{align}
\sum_{j=1}^k \gamma_j \bm C^{j}_{\ell}  \z_{\ell} \bm \pi^j(\z_\ell)
\; \approx \;  \sum_{j=1}^k \gamma_j  \bm{C}_{\ell}^j  \z_{\ell} \cdot \widehat{\bm \pi}^j(\z_\ell) 
\; \doteq \; \bm \sigma\Big([\bm{C}_{\ell}^{1} \z_{\ell}, \dots, \bm{C}_{\ell}^{k} \z_{\ell}]\Big),
\label{eqn:soft-residual}
\end{align}
which is denoted as a nonlinear operator $\bm \sigma(\cdot)$ on outputs of the feature $\z_\ell$ through $k$ groups of filters: $[\bm{C}_{\ell}^{1}, \dots, \bm{C}_{\ell}^{k}]$. Notice that the nonlinearality arises due to a ``soft'' assignment of class membership based on the feature responses from those filters.

\begin{remark}[Approximate Membership with a ReLU Network]\label{rem:ReLU}
The choice of the softmax is mostly for its simplicity as it is widely used in other (forward components of) deep networks for purposes such as clustering, gating \citep{MoE} and routing \citep{capsule-net}. 
In practice, there are many other simpler nonlinear activation functions that one can use to approximate the membership $\widehat{\bm \pi}(\cdot)$ and subsequently the nonlinear operation $\bm \sigma$ in \eqref{eqn:soft-residual}. Notice that the geometric meaning of $\bm \sigma$ in \eqref{eqn:soft-residual} is to compute the ``residual'' of each feature against the subspace to which it belongs. There are many different ways one may approximate this quantity. For example, when we restrict all our features to be in the first (positive) quadrant of the feature space,\footnote{Most current neural networks seem to adopt this regime.}
one may approximate this residual using the rectified linear units operation, ReLUs, on $\p_{j} = \bm C_{\ell}^j \z_{\ell}$ or its orthogonal complement:
\begin{equation}
\bm \sigma(\z_\ell) \; \propto \; \z_\ell - \sum_{j=1}^k   \mbox{ReLU}\big(\bm P_{\ell}^j \z_{\ell}\big), 
\label{eq:approx-relu}
\end{equation}
where $\bm P_{\ell}^j = (\bm C_{\ell}^j)^\perp$ is (approximately) the  projection onto the $j$-th class 
and $\text{ReLU}(x) = \max(0, x)$. The above approximation is good under the more restrictive assumption that projection of $\z_{\ell}$ on the correct class via  $\bm P_{\ell}^j$ is mostly large and positive and yet small or negative for other classes. 
\end{remark}
Overall, combining \eqref{eqn:gradient-descent-transform},  \eqref{eqn:DR-gradient-transform}, and \eqref{eqn:soft-residual}, 
the increment feature transform from $\z_{\ell}$ to $\z_{\ell+1}$ now becomes:
\begin{equation}\label{eqn:layer-approximate}
\begin{aligned}
\z_{\ell+1}  &\propto \; \z_\ell +  \eta \cdot  \bm E_{\ell} \z_{\ell} - \eta\cdot  \bm \sigma\big([\bm{C}_{\ell}^{1} \z_{\ell}, \dots, \bm{C}_{\ell}^{k} \z_{\ell}]\big)  \\
&= \; \z_\ell +  \eta \cdot g(\z_\ell, \bm \theta_\ell) \quad \mbox{s.t.} \quad \z_{\ell +1} \in \mathbb{S}^{n-1},
\end{aligned}
\end{equation}
with the nonlinear function $\bm \sigma(\cdot)$ defined above and $\bm \theta_\ell$ collecting all the layer-wise parameters. That is  $\bm \theta_\ell =\left\{\E_\ell, \bm{C}_{\ell}^{1}, \dots, \bm{C}_{\ell}^{k}, \gamma_j, \lambda\right\}$. Note features at each layer are always ``normalized'' by projecting onto the unit sphere $\mathbb S^{n-1}$, denoted as $\mathcal P_{\mathbb S^{n-1}}$. The form of increment in \eqref{eqn:layer-approximate} can be illustrated by a diagram in  Figure~\ref{fig:arch} left.
\begin{figure}[t]
\subcapcentertrue
\begin{center}
    \subfigure[\textbf{ReduNet}]{\includegraphics[width=0.32\textwidth]{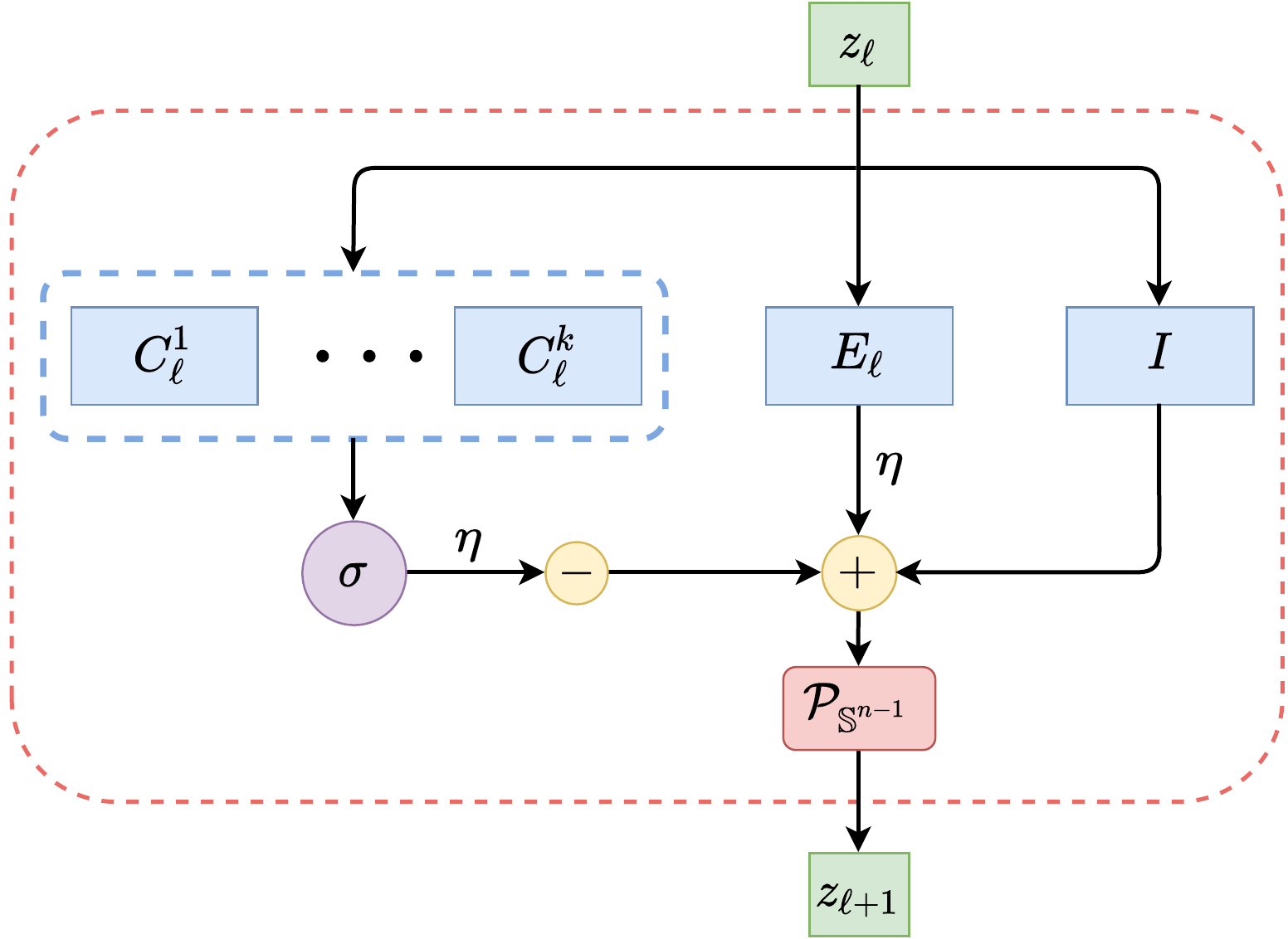}} 
    \hspace{4mm}
    \subfigure[\textbf{ResNet} and \textbf{ResNeXt}.]{\includegraphics[width=0.635\textwidth]{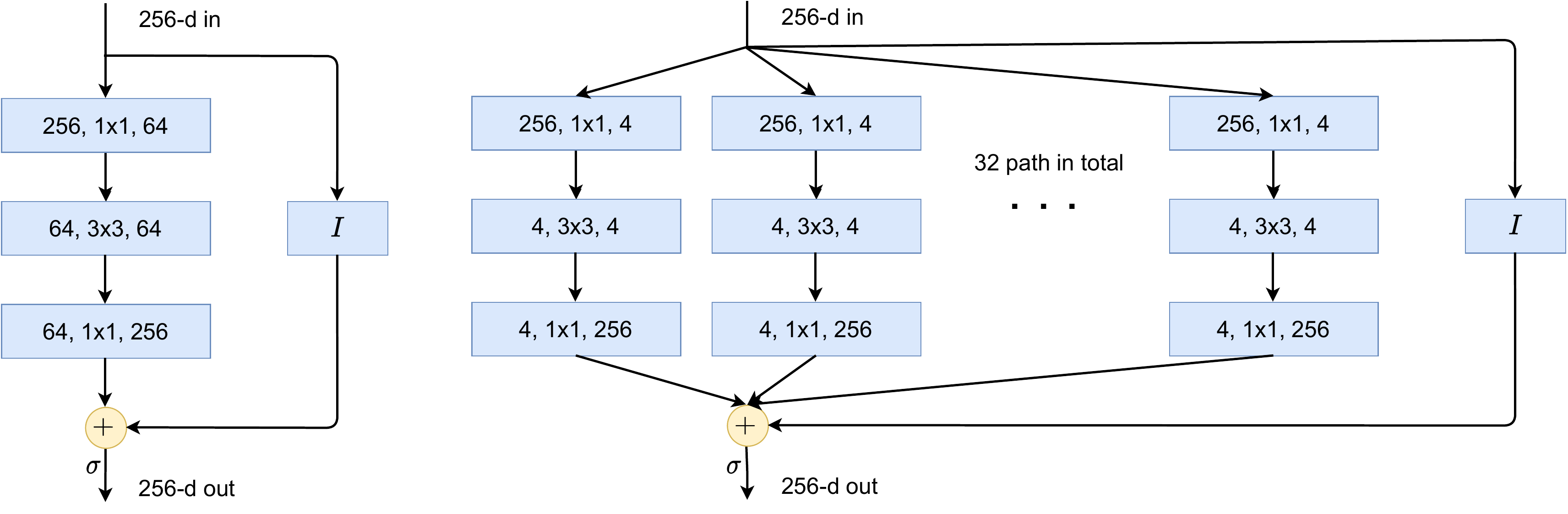}}
    \caption{\small Network Architectures of the ReduNet and comparison with others. \textbf{(a)}: Layer structure of the \textbf{ReduNet} derived from one iteration of gradient ascent for optimizing rate reduction. \textbf{(b)} (left): A layer of ResNet~\citep{he2016deep}; and \textbf{(b)} (right): A layer of ResNeXt~\citep{ResNEXT}. As we will see in Section~\ref{sec:shift-invariant}, the linear operators $\bm E_\ell$ and $\bm{C}^j_\ell$ of the ReduNet naturally become (multi-channel) convolutions when shift-invariance is imposed.}
    \label{fig:arch}
\end{center}
\end{figure}

\subsection{Deep Network for Optimizing Rate Reduction} 
Notice that the increment is constructed to emulate the gradient ascent for the rate reduction $\Delta R$. Hence by transforming the features iteratively via the above process, we expect the rate reduction to increase, as we will see in the experimental section. This iterative process, once converged say after $L$ iterations, gives the desired feature map $f(\x, \bm \theta)$ on the input $\z_1 = \x$, precisely in the form of a {\em deep network}, in which each layer has the structure shown in Figure~\ref{fig:arch} left:
\begin{equation}\label{eqn:ReduNet}
\begin{aligned}
f(\x, \bm \theta)\; =&  \;\;\phi_L \circ \phi_{L-1} \circ  \cdots \circ \phi_2 \circ \phi_1(\z_1),  \\ 
\phi_\ell(\z_\ell, \bm \theta_\ell) \; \doteq & \;\; \z_{\ell+1} = \mathcal{P}_{\mathbb{S}^{n-1}}[\z_{\ell} + \eta\cdot g(\z_{\ell}, \bm \theta_{\ell})], \\
g(\z_{\ell}, \bm \theta_{\ell}) \; =&\;\; \bm E_{\ell} \z_{\ell} -  \bm \sigma\big([\bm{C}_{\ell}^{1} \z_{\ell}, \dots, \bm{C}_{\ell}^{k} \z_{\ell}]\big).
\end{aligned}
\end{equation}
As this deep network is derived from maximizing the rate \textbf{redu}ced, we call it the \textbf{ReduNet}. 
We summarize the training and evaluation of ReduNet in Algorithm~\ref{alg:training} and Algorithm~\ref{alg:evaluating}, respectively. 
Notice that all parameters of the network are explicitly constructed layer by layer in a {\em forward propagation} fashion. The construction does not need any back propagation! The so learned features can be directly used for classification, say via a nearest subspace classifier. 

\begin{algorithm}[t]
	\caption{\textbf{Training Algorithm} for ReduNet}
	\label{alg:training}
	\begin{algorithmic}[1]
		\REQUIRE $\X = [\x^1, \ldots, \x^m]\in \Re^{D \times m}$, $\bm{\Pi}$, $\epsilon > 0$, feature dimension $n$, $\lambda$, and a learning rate $\eta$.\!\!\!
		\STATE Set $\alpha = {n}/{(m \epsilon^2)}$, $\{\alpha_j = {n}/{(\textsf{tr}\left(\bm{\Pi}^{j}\right)\epsilon^{2})}\}_{j=1}^k$, 
		$\{\gamma_j = {\textsf{tr}\left(\bm{\Pi}^{j}\right)}/{m}\}_{j=1}^k$.
		\STATE Set $\Z_1 \doteq [\z_1^1, \ldots, \z_1^m] = \X \in \Re^{n \times m}$ (assuming $n = D$ for simplicity).
		\FOR{$\ell = 1, 2, \dots, L$} 
		    \STATE {\texttt{\# Step 1:Compute network parameters} $\E_\ell$ \texttt{and} $\{\C_\ell^j\}_{j=1}^k$.}
    		\STATE $\E_\ell \doteq \alpha(\I \!+\! \alpha\Z_\ell \Z_\ell^{*})^{-1} \in \Re^{n\times n}$, \; $\{ \C^j_\ell \doteq \alpha_j  (\I +  \alpha_j \Z_\ell \bm \Pi^j \Z_\ell^{*})^{-1} \in \Re^{n\times n} \}_{j=1}^k$;
    		\STATE \texttt{\# Step 2:Update feature  $\Z_\ell$.}
    		\FOR{$i=1, \ldots, m$}
    		    \STATE \texttt{\# Compute soft  assignment $\{\widehat{\bm \pi}^j(\z_\ell^i)\}_{j=1}^k$.}
            	\STATE $\left\{\widehat{\bm \pi}^j(\z_\ell^i) \doteq \frac{\exp{( -\lambda \|\bm C_{\ell}^j  \z_{\ell}^i \|)} }{ \sum_{j=1}^{k} \exp{(-\lambda \|\bm C_{\ell}^j  \z_{\ell}^i  \|)}} \in [0,1] \right\}_{j=1}^k $;
    		    \STATE \texttt{\# Update feature $\z^i_\ell$.} 
            	\STATE $\z^i_{\ell+1} = \mathcal P_{\mathbb S^{n-1}} \left( \z^i_\ell + \eta\left(\E_\ell \z^i_\ell - \sum_{j=1}^k \gamma_j  \bm{C}_{\ell}^j  \z^i_{\ell} \cdot \widehat{\bm \pi}^j(\z_\ell^i)\right) \right) \quad \in \Re^n$;
        	\ENDFOR
		\ENDFOR
		\ENSURE features $\Z_{L+1}$, the learned parameters $\{\E_\ell\}_{\ell=1}^L$ and  $\{\C^j_\ell\}_{j=1, \ell=1}^{k, L}$, $\{\gamma_j\}_{j=1}^k$.
	\end{algorithmic}
\end{algorithm}

\begin{algorithm}[t]
	\caption{\textbf{Evaluation Algorithm} for ReduNet}
	\label{alg:evaluating}
	\begin{algorithmic}[1]
		\REQUIRE $\x \in \Re^{D}$,  network parameters $\{\E_\ell\}_{\ell=1}^L$ and $\{\C^j_\ell\}_{j=1, \ell=1}^{k, L}$, $\{\gamma_j\}_{j=1}^k$, feature dimension $n$, $\lambda$, and a learning rate $\eta$.
		\STATE Set $\z_1 = \x \in \Re^{n}$ (assuming $n = D$ for simplicity).
		\FOR{$\ell = 1, 2, \dots, L$} 
    		    \STATE \texttt{\# Compute soft  assignment $\{\widehat{\bm \pi}^j(\z_\ell)\}_{j=1}^k$.}
            	\STATE $\left\{\widehat{\bm \pi}^j(\z_\ell) \doteq \frac{\exp{( -\lambda \|\bm C_{\ell}^j  \z_{\ell} \|)} }{ \sum_{j=1}^{k} \exp{(-\lambda \|\bm C_{\ell}^j  \z_{\ell}  \|)}} \in [0,1] \right\}_{j=1}^k $;
    		    \STATE \texttt{\# Update feature $\z_\ell$.} 
            	\STATE $\z_{\ell+1} = \mathcal P_{\mathbb S^{n-1}} \left( \z_\ell + \eta\left(\E_\ell \z_\ell - \sum_{j=1}^k \gamma_j  \bm{C}_{\ell}^j  \z_{\ell} \cdot \widehat{\bm \pi}^j(\z_\ell)\right) \right) \quad \in \Re^n$;
		\ENDFOR
		\ENSURE feature $\z_{L+1}$
	\end{algorithmic}
\end{algorithm}

\subsection{Comparison with Other Approaches and Architectures}\label{sec:comparison-vector}
Like all networks that are inspired by unfolding certain iterative optimization schemes, the structure of the ReduNet naturally contains a skip connection between adjacent layers as in the ResNet \citep{he2016deep} (see Figure \ref{fig:arch} middle). 
Empirically, people have found that additional skip connections across multiple layers may improve the network performance, e.g. the Highway networks \citep{srivastava2015highway} and DenseNet \citep{dense-net}. In our framework, the role of each layer is precisely interpreted as one iterative gradient ascent step for the objective function $\Delta R$. In our experiments (see Section \ref{sec:experiments}), we have observed that the basic gradient scheme sometimes converges slowly, resulting in deep networks with hundreds of layers (iterations)!  To improve the efficiency of the basic ReduNet, one may consider in the future accelerated gradient methods such as the Nesterov acceleration \citep{nesterov1983method} or perturbed accelerated gradient descent \citep{Jin-2018}. Say to minimize or maximize a function $h(\z)$, such accelerated methods usually take the form:
\begin{equation}
\left\{
\begin{array}{ccl}
    \bm q_{\ell +1} &=& \z_\ell + \beta_\ell\cdot (\z_\ell - \z_{\ell-1}),\\
    \z_{\ell+1} &=& \bm q_{\ell+1} + \eta \cdot \nabla h(\bm q_{\ell+1}). 
\end{array}  
\right.
\label{eqn:acceleration}
\end{equation}
They require introducing additional skip connections among three layers $\ell-1$, $\ell$ and $\ell+1$. For typical convex or nonconvex programs, the above accelerated schemes can often reduce the number of iterations by a magnitude \citep{Wright-Ma-2021}.

Notice that, structure wise, the $k+1$ parallel groups of channels $\bm E, \bm C^j$ of the ReduNet correspond to the ``residual'' channel of the ResNet (Figure \ref{fig:arch} middle). Remarkably, here in ReduNet, we know they precisely correspond to {\em the residual of data auto-regression} (see Remark \ref{rem:regression-interpretation}). Moreover, the multiple parallel groups actually draw resemblance to the parallel structures that people later empirically found to further improve the ResNet, e.g. ResNeXt \citep{ResNEXT} (shown in Figure \ref{fig:arch} right) or the mixture of experts (MoE) module adopted in \citet{MoE}.\footnote{The latest large language model, the switched transformer \citep{Switch-Transformers}, adopts the MoE architecture, in which the number of parallel banks (or experts) $k$ are in the thousands and the total number of parameters of the network is about 1.7 trillion.} Now in ReduNet, each of those groups $\bm C^j$ can be precisely interpreted as {\em an expert classifier for each class of objects}. But a major difference here is that all above networks need to be initialized randomly and trained via back propagation whereas all components (layers, operators, and parameters) of the ReduNet are by explicit construction in a forward propagation. They all have precise optimization, statistical and geometric interpretation. 

Of course, like any other deep networks, the so-constructed ReduNet is amenable to fine-tuning via back-propagation if needed. A recent study from \cite{giryes2018tradeoffs} has shown that such fine-tuning may achieve a better trade off between accuracy and efficiency of the unrolled network (say when only a limited number of layers, or iterations are allowed in practice). Nevertheless, for the ReduNet, one can start with the nominal values obtained from the forward construction, instead of random initialization. Benefits of fine-tuning and initialization will be verified in the experimental section (see Table \ref{table:backprop_acc}).

\section{Deep Convolution Networks from Invariant Rate Reduction}\label{sec:shift-invariant}
So far, we have considered the data $\bm x$ and their features $\bm z$ as vectors. In many applications, such as serial data or imagery data, the semantic meaning (labels) of the data are  {\em invariant} to certain transformations $\mathfrak{g} \in \mathbb{G}$ (for some group $\mathbb{G}$) \citep{CohenW16,deep-sets-NIPS2017}. For example, the meaning of an audio signal is invariant to shift in time; and the identity of an object in an image is invariant to translation in the image plane. Hence, we prefer the feature mapping $f(\x,\bm \theta)$ is rigorously invariant to such transformations:
\begin{equation}
\mbox{\em Group Invariance:}\;   f(\x\circ \mathfrak{g}, \bm \theta) \sim f(\x,\bm \theta), \quad \forall \mathfrak{g} \in \mathbb{G},
\end{equation}
where ``$\sim$'' indicates two features belonging to the same equivalent class. Although to ensure invariance or equivarience, convolutional operators has been common practice in deep networks \citep{CohenW16}, it remains challenging in practice to train an (empirically designed) convolution network from scratch that can {\em guarantee} invariance even to simple transformations such as translation and rotation \citep{azulay2018deep,engstrom2017rotation}. An alternative approach is to carefully design convolution filters of each layer so as to ensure translational invariance for a wide range of signals, say using wavelets as in ScatteringNet \citep{scattering-net} and followup works \citep{Wiatowski-2018}. However, in order to ensure invariance to generic signals, the number of convolutions needed usually grows exponentially with network depth. That is the reason why this type of network cannot be constructed so deep, usually only several layers.

In this section, we show that the MCR$^2$ principle is compatible with invariance in a very natural and precise way: we only need to assign all transformed versions $\{\x\circ \mathfrak{g} \mid \mathfrak{g} \in \mathbb G\}$  into the same class as the data $\x$ and map their features $\z$ all to the same subspace $\mathcal S$. Hence, all group equivariant information is encoded only inside the subspace, and any classifier defined on the resulting set of subspaces will be automatically invariant to such group transformations. See Figure \ref{fig:ortho-invariance-diagram} for an illustration of the examples of 1D rotation and 2D translation.  We will rigorously show in the next two sections (as well as Appendix~\ref{app:1D} and Appendix~\ref{ap:2D-translation}) that, when the group $\mathbb G$ is circular 1D shifting or 2D translation, the resulting deep network naturally becomes a {\em multi-channel convolution network}! Because the so-constructed network only needs to ensure invariance for the given data $\X$ or their features $\Z$, the number of convolutions needed actually remain constant through a very deep network, as oppose to the ScatteringNet. 

\begin{figure}[t]
  \begin{center}
    \subfigure{\includegraphics[width=0.4\textwidth]{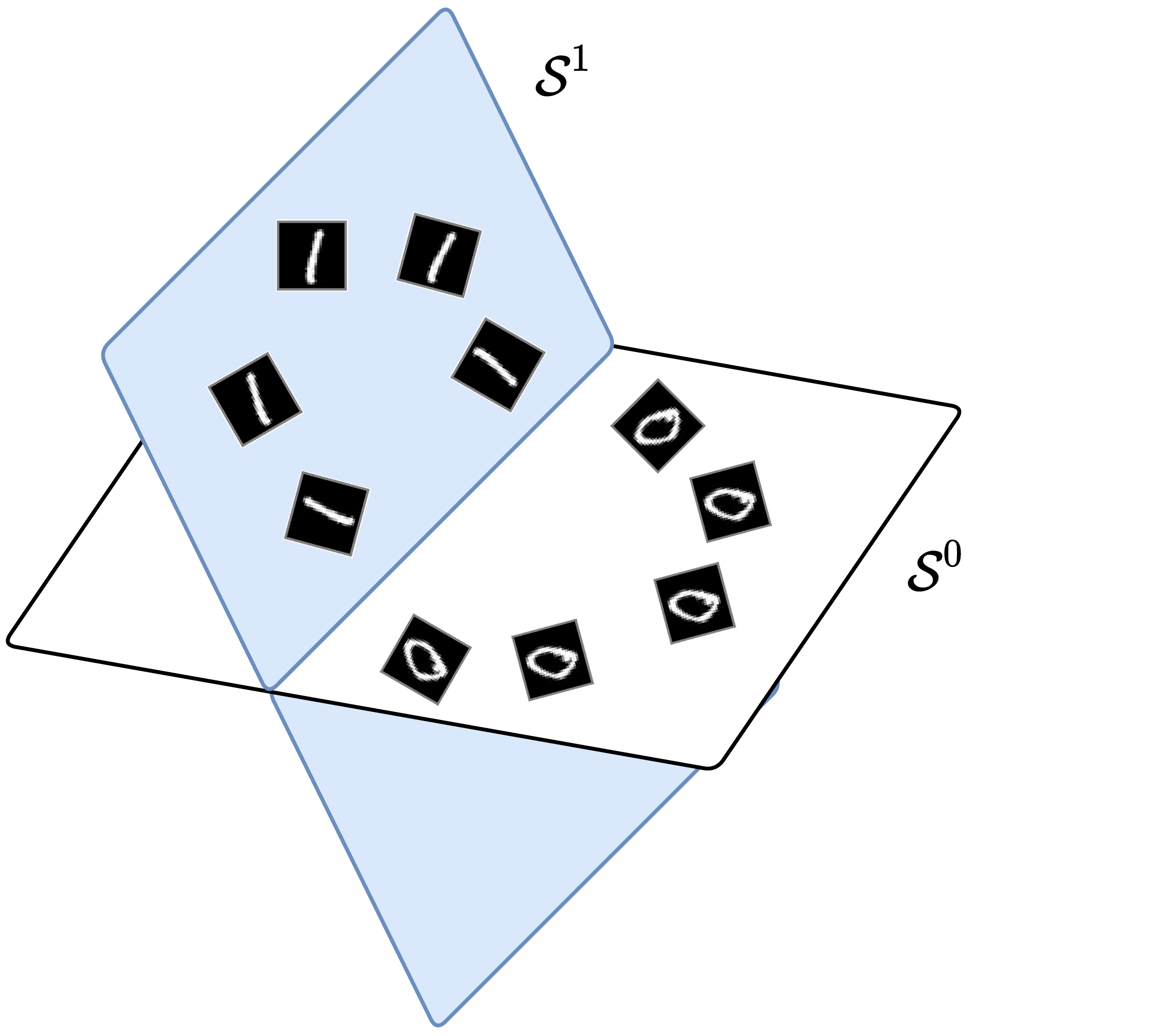}} 
    \hspace{5mm}
    \subfigure{\includegraphics[width=0.4\textwidth]{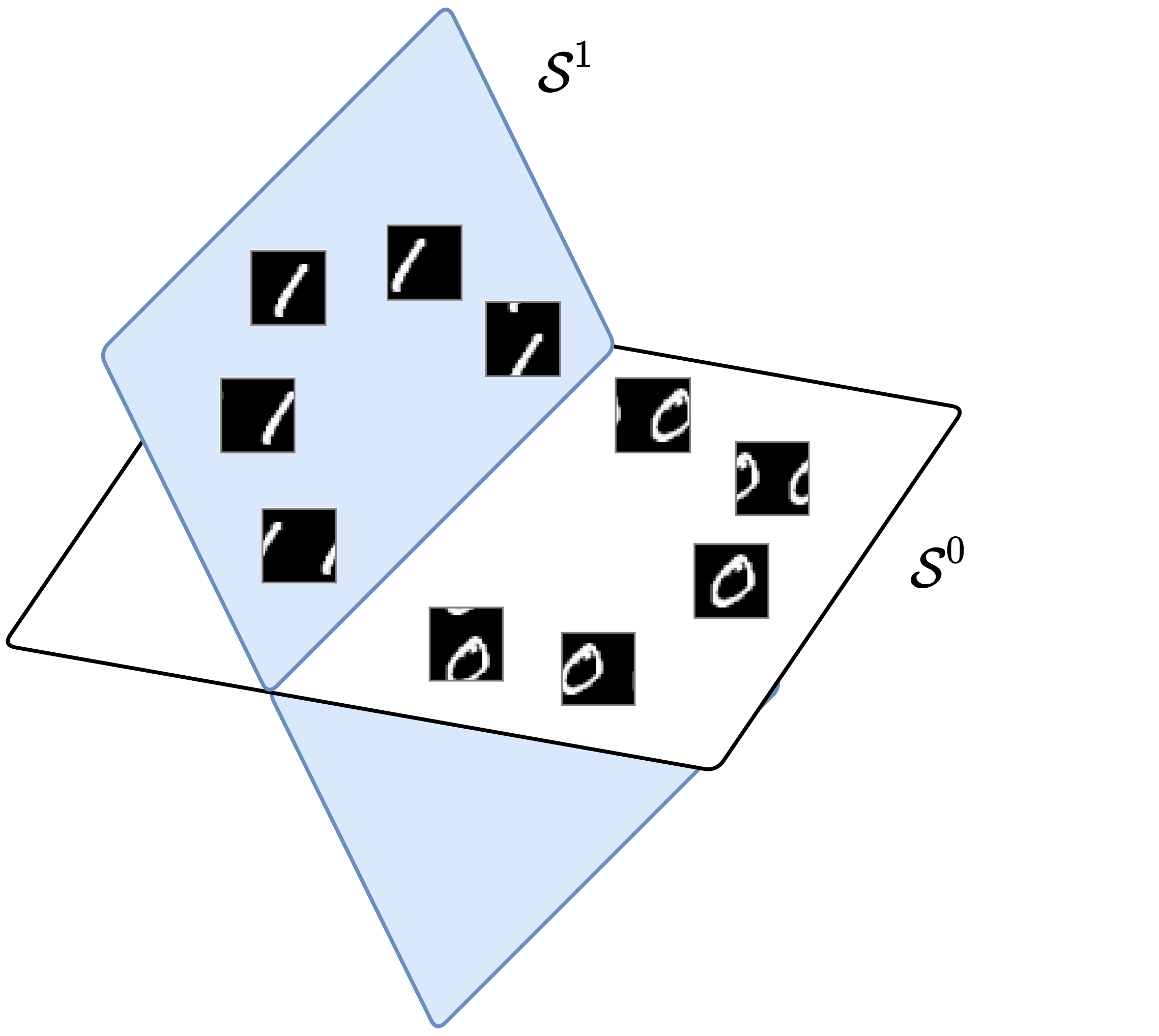}}
  \end{center}
\caption{\small Illustration of the sought representation that is equivariant/invariant to image rotation (left) or translation (right): all transformed images of each class are mapped into the same subspace that are incoherent to other subspaces. The features embedded in each subspace are equivariant to transformation group whereas each subspace is invariant to such transformations.}\label{fig:ortho-invariance-diagram}
\end{figure}

\subsection{1D Serial Data and  Shift Invariance}\label{sec:invariance}
To classify one-dimensional data $\x = [x(0), x(1), \ldots, x(D-1)] \in \Re^D$ invariant under shifting, we take $\mathbb{G}$ to be the group of all circular shifts. Each observation $\x^i$ generates a family $\{ \x^i \circ \mathfrak{g} \, | \, \mathfrak{g} \in \mathbb G \}$ of shifted copies, which are the columns of the circulant matrix $\circm(\x^i) \in \Re^{D \times D}$ given by
\begin{equation}
\circm(\x) \,\doteq\, \left[ \begin{array}{ccccc} x(0) & x(D-1) & \dots & x(2) & x(1) \\ x(1) & x(0) & x(D-1) & \cdots & x(2) \\ \vdots & x(1) & x(0) &\ddots & \vdots \\ x(D-2) &  \vdots & \ddots & \ddots & x(D-1) \\ x(D-1) & x(D-2) & \dots & x(1) & x(0)   \end{array} \right] \quad \in \Re^{D \times D}.
\end{equation}
We refer the reader to Appendix \ref{ap:circulant} or \cite{Kra2012OnCM} for properties of circulant matrices. For simplicity, let $\bm Z_1 \doteq [ \z_1^1, \dots, \z_1^m ] = \X \in \Re^{n \times m}$\footnote{Again, to simplify discussion, we assume for now that the initial features $\Z_1$ are $\X$ themselves hence have the same dimension $n$. But that does not need to be the case as we will soon see that we need to lift $\X$ to a higher dimension.}. Then what happens if we construct the ReduNet from their circulant families $\circm(\bm Z_1) = \left[ \circm(\z_1^1), \dots, \circm(\z_1^m) \right] \in \Re^{n \times nm}$? That is, we want to compress and map all these into the same subspace by the ReduNet. 

Notice that now the data covariance matrix: 
\begin{eqnarray}
\circm(\bm Z_1) \circm(\bm Z_1)^* 
&=& \left[ \circm(\z^1), \dots, \circm(\z^m) \right] \left[ \circm(\z^1), \dots, \circm (\z^m) \right]^* \\
&=& \sum_{i =1}^m \circm(\z_1^i) \circm(\z_1^i)^* \;\in \Re^{n\times n}
\end{eqnarray}
associated with this family of samples is {\em automatically} a (symmetric) circulant matrix. Moreover, because the circulant property is preserved under sums, inverses, and products, the matrices $\bm E_1$ and $\bm C_1^j$ are also automatically circulant matrices, whose application to a feature vector $\bm z \in \Re^n$ can be implemented using circular convolution ``$\circledast$''.
Specifically, we have the following proposition. 

\begin{proposition}[Convolution structures of $\bm E_1$ and $\bm C_1^j$]
The matrix 
\begin{equation}
    \E_1 = \alpha\big(\bm I + \alpha \circm(\bm Z_1) \circm(\bm Z_1)^* \big)^{-1}
\end{equation}
is a circulant matrix and represents a circular convolution: 
$$\E_1 \z = \bm e_1 \circledast \z,$$ 
where $\bm e_1 \in \Re^n$ is the first column vector of $\E_1$ and ``$\circledast$'' is circular convolution defined as
\begin{equation*}
    (\bm e_1 \circledast \bm z)_{i} \doteq \sum_{j=0}^{n-1} e_1(j) x(i+ n-j \,\, \textsf{mod} \,\,n).
\end{equation*}
Similarly, the matrices $\bm C_1^j$ associated with any subsets of $\bm Z_1$ are also circular convolutions. 
\label{prop:circular-conv-1}
\end{proposition}

Not only the first-layer parameters $\E_1$ and $\C_1^j$ of the ReduNet become circulant convolutions but also the next layer features remain circulant matrices. 
That is, the incremental feature transform in \eqref{eqn:layer-approximate} applied to all shifted versions of a $\z_1 \in \Re^n$, given by
\begin{equation}
    \circm(\z_1) + \eta \cdot \E_1 \circm(\z_1) - \eta \cdot \bm \sigma \Big([\C_1^1 \circm(\z_1), \ldots, \C_1^k \circm(\z_1)] \Big),
\end{equation}
is a circulant matrix. 
This implies that there is no need to construct circulant families from the second layer features as we did for the first layer. 
By denoting
\begin{equation}
\z_{2} \propto \z_{1} +\eta\cdot g(\z_{1}, \bm \theta_{1}) =  \z_1 + \eta \cdot \bm e_{1} \circledast \z_{1} -  \eta \cdot \bm \sigma\Big([\bm{c}_{1}^{1} \circledast \z_{1}, \dots, \bm{c}_{1}^{k} \circledast \z_{1}]\Big),
\label{eqn:approximate-convolution}
\end{equation}
the features at the next level can be written as
$$\circm(\bm Z_2) = \big[ \circm( \bm z_2^1), \dots, \circm( \bm z_2^m ) \big] = \big[ \circm( \bm z_1^1 + \eta g( \bm z_1^1, \bm \theta_1)), \dots, \circm( \bm z_1^m + \eta g(\bm z_1^m, \bm \theta_1)) \big].$$
Continuing inductively, we see that all matrices $\bm E_\ell$ and $\bm C_\ell^j$ based on such $\circm(\bm Z_\ell)$ are circulant, and so are all features. 
By virtue of the properties of the data, ReduNet has taken the form of a convolutional network, {\em with no need to explicitly choose this structure!}

\subsection{A Fundamental Trade-off between Invariance and Sparsity}
There is one problem though: In general, the set of all circular permutations of a vector $\z$ gives a full-rank matrix. That is, the $n$ ``augmented'' features associated with each sample (hence each class) typically already span the entire space $\Re^n$. For instance, all shifted versions of a delta function $\delta(n)$ can generate any other signal as their (dense) weighted superposition. The MCR$^2$ objective \eqref{eqn:maximal-rate-reduction} will not be able to distinguish classes as different subspaces.

One natural remedy is to improve the separability of the data by ``lifting'' the original signal to a higher dimensional space, e.g., by taking their responses to multiple, filters $\bm k_1, \ldots, \bm k_C \in \Re^n$: 
\begin{equation}
\bm z[c] = \bm k_c \circledast \bm x  =  \circm(\bm k_c) \bm x \quad \in \Re^n, \quad c = 1, \ldots, C.
\label{eqn:lift-1d}
\end{equation} 
The filters can be pre-designed invariance-promoting filters,\footnote{For 1D signals like audio, one may consider the conventional short time Fourier transform (STFT); for 2D images, one may consider 2D wavelets as in the ScatteringNet \citep{scattering-net}.} or adaptively learned from the data,\footnote{For learned filters, one can learn filters as the principal components of samples as in the PCANet \citep{chan2015pcanet} or from convolution dictionary learning \citep{li2019multichannel,qu2019nonconvex}.} or randomly selected as we do in our experiments. This operation lifts each original signal $\x \in \Re^n$ to a $C$-channel feature, denoted as $\bar{\z}  \doteq [\z[1], \ldots, \z[C]]^* \in \Re^{C\times n}$. 
Then, we may construct the ReduNet on vector representations of $\bar{\z}$, denoted as
$\vec(\bar\z) \doteq [\z[1]^*, \ldots, \z[C]^*] \in \Re^{nC}$. 
The associated circulant version $ \circm(\bar{\z})$ and its data covariance matrix, denoted as $\bar{\bm \Sigma}(\bar\z)$, for all its shifted versions are given as:
\begin{equation}
 \circm(\bar{\z}) \doteq \left[\begin{smallmatrix}
    \circm(\z[1])  \\ \vdots \\ \circm(\z[C]) \end{smallmatrix} \right] \in \Re^{nC\times n},\;\;  \bar{\bm \Sigma}(\bar\z) \doteq 
    \left[\begin{smallmatrix}
    \circm(\z[1]) \\ \vdots \\ \circm(\z[C]) \end{smallmatrix} \right]
    \left[\begin{smallmatrix}\circm(\z[1])^*,\ldots, \circm(\z[C])^*\end{smallmatrix} \right] \in \Re^{nC\times nC},
    \label{eqn:W-multichannel}
\end{equation}
where $\circm(\z[c]) \in \Re^{n\times n}$ with $c \in [C]$ is the circulant version of the $c$-th channel of the feature $\bar \z$. Then the columns of $\circm(\bar\z)$  will only span at most an $n$-dimensional proper subspace in $\Re^{nC}$. 

However, this simple lifting operation (if linear) is not sufficient to render the classes separable yet---features associated with other classes will span the {\em same} $n$-dimensional subspace. This reflects a fundamental conflict between invariance and linear (subspace) modeling: {\em one cannot hope for arbitrarily shifted and superposed signals to belong to the same class.} 

\begin{figure}[t]
	\centerline{
\includegraphics[width=0.9\textwidth]{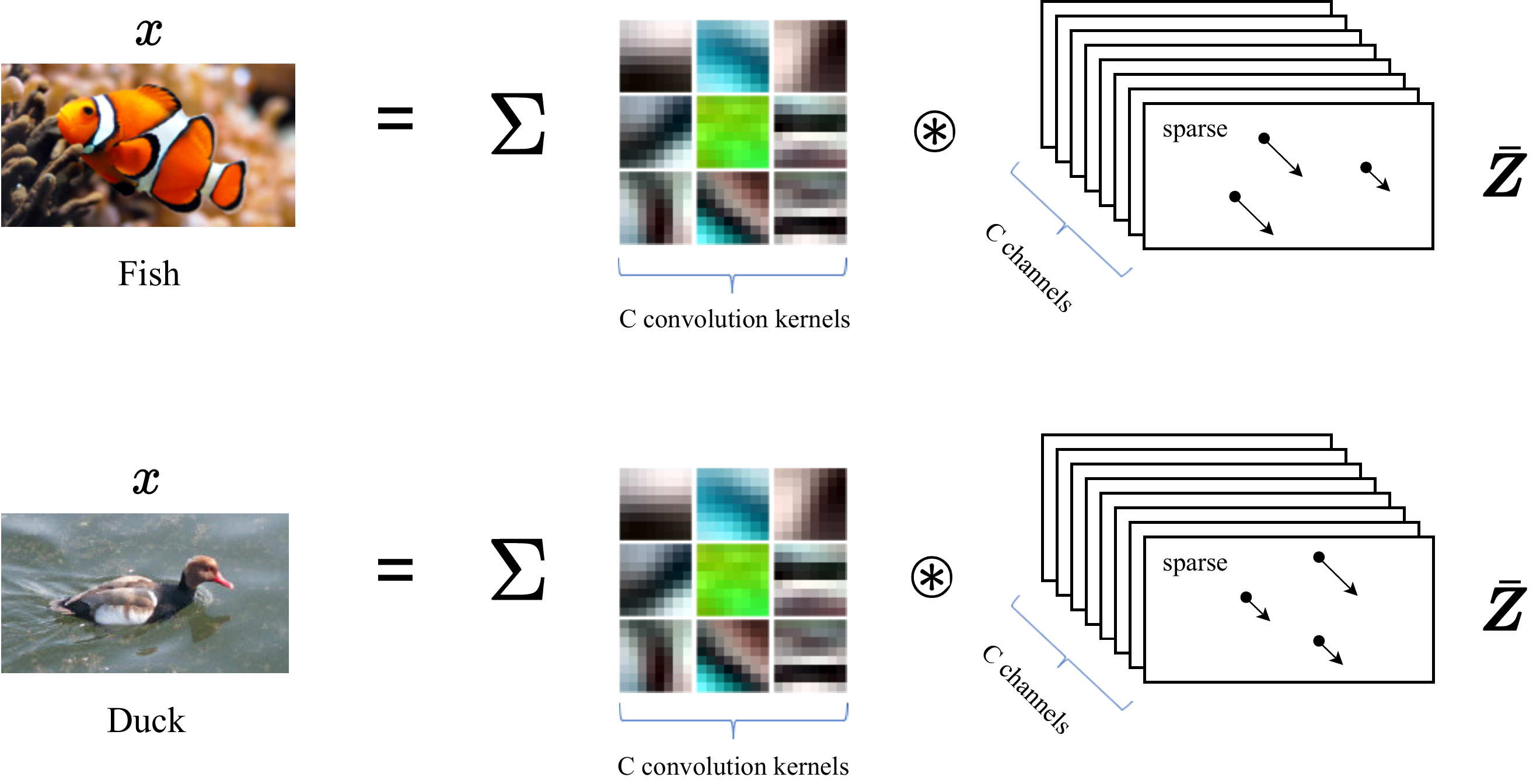}}
	\caption{Each input signal $\bm x$ (an image here) can be represented as a superposition of sparse convolutions with multiple kernels $\bm d_c$ in a dictionary $\bm D$.}
	\label{fig:multi-channel-sparse-representation}
\end{figure}
One way of resolving this conflict is to leverage additional structure within each class, in the form of {\em sparsity}: Signals within each class are not generated as arbitrary linear superposition of some base atoms (or motifs), but only {\em sparse} combinations of them and their shifted versions, as shown in Figure \ref{fig:multi-channel-sparse-representation}. More precisely, let $\bm D^j = [\bm d_1^j, \ldots, \bm d_c^j]$ denote a matrix with a collection of atoms associated for class $j$, also known as a dictionary, then each signal $\x$ in this class is sparsely generated as: 
\begin{equation}
    \x = \bm d_1^j \circledast \z_1 + \ldots + \bm d_c^j \circledast \z_c = \circm(\bm{D}^j)\z,
\end{equation}
for some sparse vector $\z$. Signals in different classes are then generated by different dictionaries whose atoms (or motifs) are incoherent from one another. Due to incoherence, signals in one class are unlikely to be sparsely represented by atoms in any other class. Hence all signals in the $k$ class can be represented as
\begin{equation}
\x = \big[\circm(\bm{D}^1), \circm(\bm{D}^2), \ldots, \circm(\bm{D}^k)\big] \bar \z,
\end{equation}
where $\bar \z$ is sparse.\footnote{Notice that similar sparse representation models have long been proposed and used for classification purposes in applications such a face recognition, demonstrating excellent effectiveness  \citep{Wright:2009,wagner2012toward}. Recently, the convolution sparse coding model has been proposed by \cite{papyan2017convolutional} as a framework for interpreting the structures of deep convolution networks.} There is a vast literature on how to learn the most compact and optimal sparsifying dictionaries from sample data, e.g.  \citep{li2019multichannel,qu2019nonconvex} and subsequently solve the inverse problem and compute the associated sparse code $\z$ or $\bar \z$. Recent studies of \cite{qu2020nonconvex,qu2020geometric} even show under broad conditions the convolution dictionary learning problem can be solved effectively and efficiently. 

Nevertheless, for tasks such as classification, we are not necessarily interested in the precise optimal dictionary nor the precise sparse code for each individual signal. We are mainly interested if collectively the set of sparse codes for each class are adequately separable from those of other classes. Under the assumption of the sparse generative model, if the convolution kernels $\{\bm k_c\}_{c=1}^C$  match well with the ``transpose'' or ``inverse'' of the above sparsifying dictionaries $\bm D = [\bm D^1, \ldots, \bm D^k]$, also known as the {\em analysis filters} \citep{Cosparse-Nam,Analysis-Filter}, signals in one class will only have high responses to a small subset of those filters and low responses to others (due to the incoherence assumption).  Nevertheless, in practice, often a sufficient number of, say $C$, random filters $\{\bm k_c\}_{c=1}^C$ suffice the purpose of ensuring so extracted $C$-channel features:
\begin{equation}
\big[\bm k_1 \circledast \x, \bm k_2 \circledast \x, \ldots, \bm k_C \circledast \x\big]^* = \big[\circm(\bm k_1) \x, \ldots, \circm(\bm k_C) \x \big]^* \in \Re^{C\times n}
\end{equation}
for different classes have different response patterns to different filters hence make different classes separable \citep{chan2015pcanet}. 

Therefore, in our framework, to a large extent the number of channels (or the width of the network) truly plays the role as the  {\em statistical resource} whereas the number of layers (the depth of the network) plays the role as the {\em computational resource}. The theory of compressive sensing precisely characterizes how many measurements are needed in order to preserve the intrinsic low-dimensional structures (including separability) of the data \citep{Wright-Ma-2021}. As optimal sparse coding is not the focus of this paper, we will use the simple random filter design in our experiments, which is adequate to verify the concept.\footnote{Although better learned or designed sparsifying dictionaries and sparse coding schemes may surely lead to better classification performance, at a higher computational cost.} 

\begin{figure}[t]
	\centerline{
\includegraphics[width=0.95\textwidth]{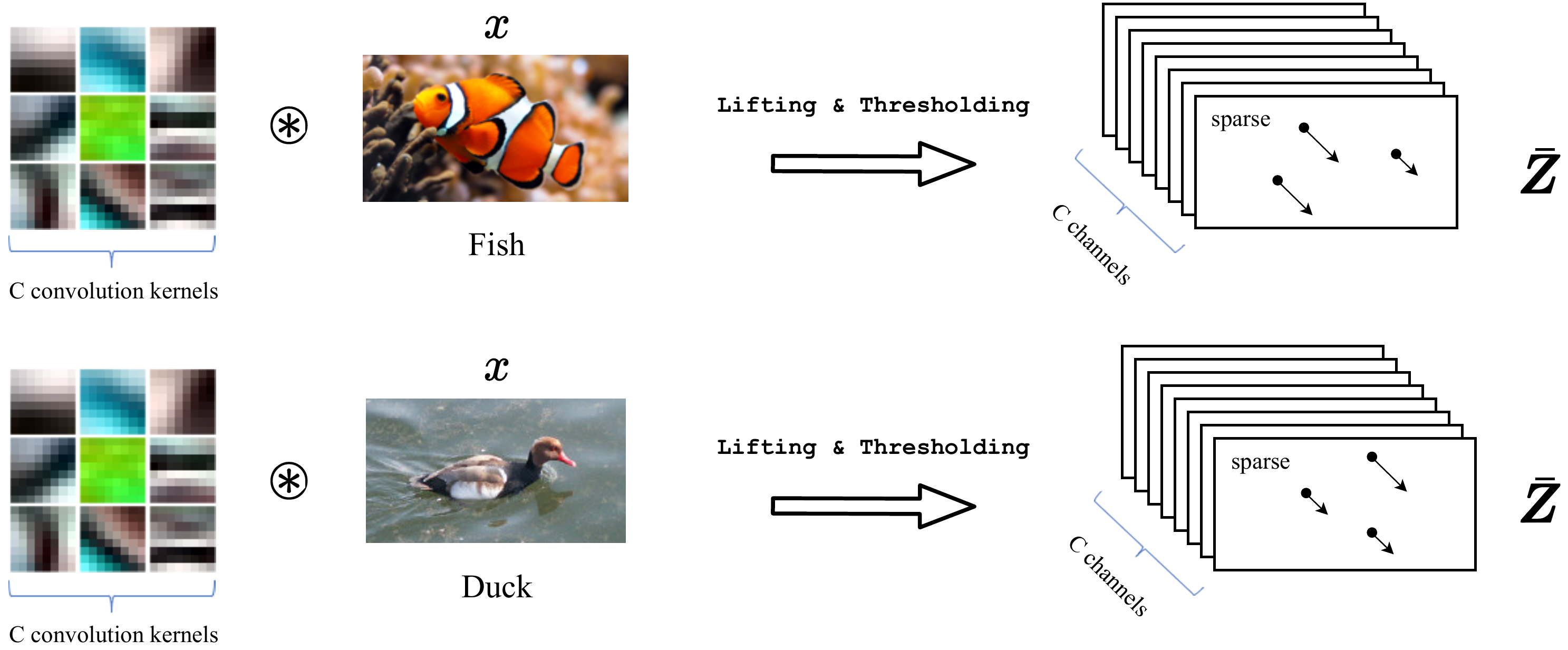}}
\caption{Estimate the sparse code $\bar{\bm z}$ of an input signal $\bm x$ (an image here) by taking convolutions with multiple kernels $\bm k_c$ and then sparsifying.}
		\label{fig:multi-channel-sparse-lifting}
\end{figure}
The multi-channel responses $\bar \z$ should be sparse. So to approximate the sparse code $\bar \z$, we may take an entry-wise {\em sparsity-promoting nonlinear thresholding}, say $\bm \tau(\cdot)$, on the above filter outputs by setting low (say absolute value below $\epsilon$) or negative  responses to be zero:
\begin{equation}
\bar \z \doteq \bm \tau \left( \big[\circm(\bm k_1) \x, \ldots, \circm(\bm k_C) \x \big]^* \right) \quad \in \Re^{C \times n}.
\label{eqn:sparse-lifting}
\end{equation}
Figure \ref{fig:multi-channel-sparse-lifting} illustrates the basic ideas. One may refer to \citet{Analysis-Filter} for a more systematical study on the design of the sparsifying thresholding operator. Nevertheless, here we are not so interested in obtaining the best sparse codes as long as the codes are sufficiently separable. Hence the nonlinear operator $\bm \tau$ can be simply chosen to be a soft thresholding or a ReLU. 
These presumably sparse features $\bar \z$ can be assumed to lie on a lower-dimensional (nonlinear) submanifold of $\mathbb{R}^{nC}$, which can be linearized and separated from the other classes by subsequent ReduNet layers, as illustrated later in Figure \ref{fig:learn-to-classify-diagram}. 

The ReduNet constructed from circulant version of these multi-channel features $\bar\Z \doteq [\bar \z^1, \ldots, \bar \z^m] \in \Re^{C \times n \times m}$, i.e., $\circm(\bar\Z) \doteq [ \circm(\bar\z^1), \dots, \circm(\bar\z^m)] \in \Re^{nC \times nm}$, retains the good invariance properties described above: the linear operators, now denoted as $\bar{\bm E}$ and $\bar{\bm C}^j$, remain block circulant, and represent {\em multi-channel 1D circular convolutions. }
Specifically, we have the following result (see Appendix \ref{ap:multichannel-circulant} for a proof).
\begin{proposition}[Multi-channel convolution structures of $\bar{\bm E}$ and $\bar{\bm C}^j$]
The matrix 
\begin{equation}
\label{eq:def-E-bar}
\bar\E \doteq \alpha\left(\bm I + \alpha \circm(\bar\Z) \circm(\bar\Z)^* \right) ^{-1}  
\end{equation}
is block circulant, i.e.,
\begin{equation*}
    \bar{\bm E} = 
    \left[\begin{smallmatrix}
        \bar{\bm E}_{1, 1} & \cdots & \bar{\bm E}_{1, C}\\
        \vdots & \ddots & \vdots \\
        \bar{\bm E}_{C, 1} & \cdots & \bar{\bm E}_{C, C}\\
    \end{smallmatrix}\right] \in \Re^{nC \times nC},
\end{equation*}
where each $\bar{\bm E}_{c, c'}\in \Re^{n \times n}$ is a circulant matrix. Moreover, $\bar{\bm E}$ represents a multi-channel circular convolution, i.e., for any multi-channel signal $\bar\z \in \Re^{C \times n}$ we have 
$$\bar\E \cdot \vec(\bar\z) = \vec( \bar{\bm e} \circledast \bar\z).$$ 
In above, $\bar{\bm e} \in \Re^{C \times C \times n}$ is a multi-channel convolutional kernel with $\bar{\bm e}[c, c'] \in \Re^{n}$ being the first column vector of $\bar{\bm E}_{c, c'}$, and $\bar{\bm e} \circledast \bar\z \in \Re^{C \times n}$ is the multi-channel circular convolution defined as
\begin{equation*}
    (\bar{\bm e} \circledast \bar\z)[c] \doteq \sum_{c'=1}^C \bar{\bm e}[c, c'] \circledast \bar{\z}[c'], \quad \forall c = 1, \ldots, C.
\end{equation*}
Similarly, the matrices $\bar{\bm C}^j$ associated with any subsets of $\bar{\bm Z}$ are also multi-channel circular convolutions. 
\label{prop:multichannel-circular-conv-1}
\end{proposition}
From Proposition \ref{prop:multichannel-circular-conv-1}, shift invariant ReduNet is a deep convolutional network for multi-channel 1D signals by construction. Notice that even if the initial lifting kernels are separated \eqref{eqn:sparse-lifting}, the matrix inverse  in \eqref{eq:def-E-bar} for computing $\bar{\bm E}$ (similarly for $\bar{\bm C^j}$) introduces ``cross talk'' among all $C$ channels. This multi-channel mingling effect will become more clear when we show below how to compute $\bar{\bm E}$ and $\bar{\bm C}^j$ efficiently in the frequency domain. Hence, unlike Xception nets~\citep{Xception}, these multi-channel convolutions in general are {\em not} depth-wise separable.\footnote{It remains open what additional structures on the data would lead to depth-wise separable convolutions.}

\subsection{Fast Computation in the Spectral Domain} The calculation of $\bar\E$ in \eqref{eq:def-E-bar} requires inverting a matrix of size $nC \times nC$, which has complexity $O(n^3C^3)$. 
By using the relationship between circulant matrix and Discrete Fourier Transform (DFT) of a 1D signal, this complexity can be significantly reduced.

Specifically, let $\F \in \Co^{n \times n}$ be the DFT matrix,\footnote{Here we scaled the matrix $\bm F$ to be unitary, hence it differs from the conventional DFT matrix by a $1/\sqrt{n}$.} and $\dft(\z) \doteq \F \z \in \Co^{n \times n}$ be the DFT of $\z \in \Re^n$, where $\Co$ denotes the set of complex numbers. 
We know all circulant matrices can be simultaneously diagonalized by the discrete Fourier transform matrix $\F$:
\begin{equation}
\label{eq:circ-dft-main}
    \circm(\z) = \F^* \diag(\dft(\z)) \F. 
\end{equation}  
We refer the reader to Fact \ref{fact:circulant} of the Appendix \ref{ap:1D-shift} for more detailed properties of circulant matrices and DFT. Hence the covariance matrix $\bar{\bm \Sigma}(\bar{\z})$ of the form \eqref{eqn:W-multichannel} can be converted to a standard ``blocks of diagonals'' form:
\begin{equation}
\bar{\bm \Sigma}(\bar{\z}) = 
\left[ 
\begin{matrix}
\F^* & \bm 0 & \bm 0  \\
\bm 0 & \footnotesize{\ddots} & \bm 0 \\
\bm 0 & \bm 0 & \F^*
\end{matrix}
\right]\left[ 
\begin{matrix}
\D_{11}(\bar{\z}) & \cdots & \D_{1C}(\bar{\z}) \\
{\footnotesize \vdots} & {\footnotesize \ddots} & {\footnotesize \vdots} \\
\D_{C1}(\bar{\z}) & \cdots & \D_{CC}(\bar{\z})
\end{matrix}
\right]\left[ 
\begin{matrix}
\F & \bm 0 & \bm 0  \\
\bm 0 & {\footnotesize \ddots} & \bm 0 \\
\bm 0 & \bm 0 & \F
\end{matrix}
\right]\in \Re^{nC \times nC},
\label{eqn:diagonal-block}
\end{equation}
where $\D_{cc'}(\bar{\z}) \doteq \diag(\dft(\z[c])) \cdot \diag(\dft(\z[c']))^* \in \Co^{n \times n}$ is a diagonal matrix. 
The middle of RHS of \eqref{eqn:diagonal-block} is a block diagonal matrix after a permutation of rows and columns.

Given a collection of multi-channel features $\{\bar\z^i \in \Re^{C \times n}\}_{i=1}^m$, we can use the relation in \eqref{eqn:diagonal-block} to compute $\bar{\E}$ (and similarly for $\bar{\C}^j$) as
\begin{equation}
    \bar\E = 
    \left[\begin{matrix}
    \F^* & \bm 0 & \bm 0  \\
    \bm 0 & \footnotesize{\ddots} & \bm 0 \\
    \bm 0 & \bm 0 & \F^*
    \end{matrix}\right]
    \cdot \alpha  \left(\I + \alpha \sum_{i=1}^m
    \left[\begin{matrix}
    \D_{11}(\bar{\z}^i) & \cdots & \D_{1C}(\bar{\z}^i) \\
    {\footnotesize \vdots} & {\footnotesize \ddots} & {\footnotesize \vdots} \\
    \D_{C1}(\bar{\z}^i) & \cdots & \D_{CC}(\bar{\z}^i)
    \end{matrix}\right]
    \right)^{-1} \cdot
    \left[\begin{matrix}
    \F & \bm 0 & \bm 0  \\
    \bm 0 & \footnotesize{\ddots} & \bm 0 \\
    \bm 0 & \bm 0 & \F
    \end{matrix}\right] \in \Re^{nC \times nC}.
\end{equation}
The matrix in the inverse operator is a block diagonal matrix with $n$ blocks of size $C \times C$ after a permutation of rows and columns. 
Hence, to compute $\bar \E$ and $\bar{\C}^j \in \Re^{nC \times nC}$, we only need to compute in the frequency domain the inverse of $C\times C$ blocks for $n$ times and the overall complexity is $O(nC^3)$.~\footnote{There is strong scientific evidence that neurons in the visual cortex encode and transmit information in the rate of spiking, hence the so-called spiking neurons \citep{spking-neuron-1993,spiking-neuron-book}. Nature might be exploiting the computational efficiency in the frequency domain for achieving shift invariance.} 

The benefit of computation with DFT motivates us to construct the ReduNet in the spectral domain. 
Let us consider the \emph{shift invariant coding rate reduction} objective for shift invariant features $\{\bar\z^i \in \Re^{C \times n}\}_{i=1}^m$:
\begin{multline}
    \Delta R_\circm(\bar\Z, \bm{\Pi}) \doteq \frac{1}{n}\Delta R(\circm(\bar\Z), \bar{\bm{\Pi}}) \\= 
    \frac{1}{2n}\log\det \Bigg(\I + \alpha  \circm(\bar\Z)  \circm(\bar\Z)^{*} \Bigg) 
    - \sum_{j=1}^{k}\frac{\gamma_j}{2n}\log\det\Bigg(\I + \alpha_j  \circm(\bar\Z) \bar{\bm{\Pi}}^{j} \circm(\bar\Z)^{*} \Bigg),
\end{multline}
where $\alpha = \frac{Cn}{mn\epsilon^{2}} = \frac{C}{m\epsilon^{2}}$, $\alpha_j = \frac{Cn}{\textsf{tr}\left(\bm{\Pi}^{j}\right)n\epsilon^{2}} = \frac{C}{\textsf{tr}\left(\bm{\Pi}^{j}\right)\epsilon^{2}}$, $\gamma_j = \frac{\textsf{tr}\left(\bm{\Pi}^{j}\right)}{m}$, and $\bar{\bm \Pi}^j$ is augmented membership matrix in an obvious way.
The normalization factor $n$ is introduce because the circulant matrix $\circm(\bar\Z)$ contains $n$ (shifted) copies of each signal.
Next, we derive the ReduNet for maximizing $\Delta R_\circm(\bar\Z, \bm{\Pi})$.

Let $\dft(\bar\Z) \in \Co^{C \times n \times m}$ be data in spectral domain obtained by taking DFT on the second dimension
and denote $\dft(\bar\Z)(p) \in \Co^{C \times m}$ the $p$-th slice of $\dft(\bar\Z)$ on the second dimension. 
Then, the gradient of $\Delta R_\circm(\bar\Z, \bm{\Pi})$ w.r.t. $\bar\Z$ can be computed from the expansion $\bar\cE \in \Co^{C \times C \times n}$ and compression $\bar\cC^j \in \Co^{C \times C \times n}$ operators in the spectral domain, defined as
\begin{equation}
\begin{aligned}
    \bar\cE(p) \doteq&\;  \alpha \cdot \left[\I + \alpha \cdot \dft(\bar\Z)(p) \cdot \dft(\bar\Z)(p)^* \right]^{-1} \quad \in \Co^{C\times C}, \\
    \bar\cC^j(p) \doteq& \; \alpha_j \cdot\left[\I + \alpha_j \cdot \dft(\bar\Z)(p) \cdot \bm{\Pi}_j \cdot \dft(\bar\Z)(p)^*\right]^{-1} \quad \in \Co^{C\times C}.
\end{aligned}
\end{equation}
In above, $\bar\cE(p)$ (resp., $\bar\cC^j(p)$) is the $p$-th slice of $\bar\cE$ (resp., $\bar\cC^j$) on the last dimension. 
Specifically, we have the following result (see Appendix \ref{ap:1D-shift} for a complete proof).
 
\begin{theorem} [Computing multi-channel convolutions $\bar{\bm E}$ and $\bar{\bm C}^j$]
\label{thm:1D-convolution}
Let $\bar\U \in \Co^{C \times n \times m}$ and $\bar\W^{j} \in \Co^{C \times n \times m}, j=1,\ldots, k$ be given by 
\begin{eqnarray}
    \bar\U(p) &\doteq& \bar\cE(p) \cdot \dft(\bar\Z)(p), \\
    \bar\W^{j}(p) &\doteq& \bar\cC^j(p) \cdot \dft(\bar\Z)(p),
\end{eqnarray}
for each  $p \in \{0, \ldots, n-1\}$. Then, we have
\begin{eqnarray}
    \frac{1}{2n}\frac{\partial \log \det (\I + \alpha \cdot \circm(\bar\Z) \circm(\bar\Z)^{*} )}{\partial \bar\Z} &=& \idft(\bar\U), 
    \\
    \frac{\gamma_j}{2n}\frac{\partial  \log\det (\I + \alpha_j \cdot \circm(\bar\Z) \bm \bar{\bm \Pi}^j \circm(\bar\Z)^{*})}{\partial \bar\Z}&=&
    \gamma_j \cdot \idft(\bar\W^{j} \bm \Pi^j).
\end{eqnarray}
In above, $\idft(\bar\U)$ is the time domain signal obtained by taking inverse DFT on each channel of each signal in $\bar\U$. 
\end{theorem}
By this result, the gradient ascent update in \eqref{eqn:gradient-descent} (when applied to $\Delta R_\circm(\bar\Z, \bm{\Pi})$) can be equivalently expressed as an update in spectral domain on $\bar\V_\ell \doteq \dft(\bar\Z_\ell)$ as
\begin{equation}
    \bar\V_{\ell+1}(p) \; \propto \; \bar\V_{\ell}(p) + \eta \; \Big(\bar\cE_\ell(p) \cdot \bar\V_\ell(p) - \sum_{j=1}^k \gamma_j \bar\cC_\ell^j(p) \cdot \bar\V_\ell(p) \bm \Pi^j \Big), \quad p = 0, \ldots, n-1,
\end{equation}
and a ReduNet can be constructed in a similar fashion as before. For implementation details, we refer the reader to Algorithm~\ref{alg:training-1D} of Appendix~\ref{ap:1D-shift}.

\begin{figure}[t]
  \begin{center}
    \includegraphics[width=0.95\textwidth]{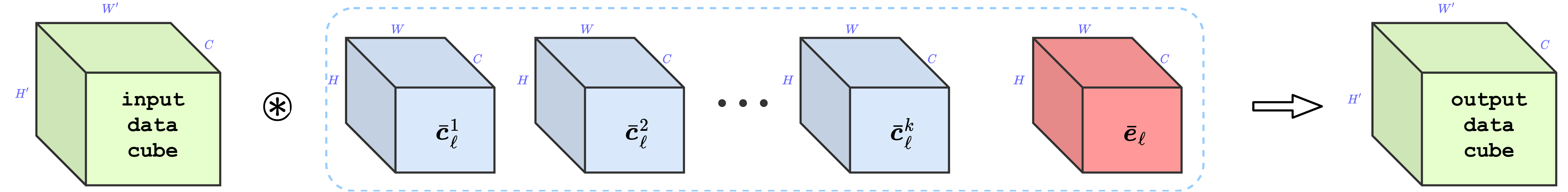}
  \end{center}
\caption{\small For invariance to 2D translation, $\bar{\bm E}$ and $\bar{\bm C}^j$ are automatically multi-channel 2D  convolutions.}\label{fig:multi-channel-convolution}
\end{figure}

\subsection{2D Translation Invariance}
In the case of classifying images invariant to arbitrary 2D translation, we may view the image (feature) $\z \in \Re^{(W\times H)\times C}$ as a function defined on a torus $\mathcal{T}^2$ (discretized as a  $W\times H$ grid) and consider $\mathbb{G}$ to be the (Abelian) group of all 2D (circular) translations on the torus. As we will show in the Appendix \ref{ap:2D-translation}, the associated linear operators $\bar \E$ and $\bar{\C}^j$'s act on the image feature $\z$ as {\em multi-channel 2D circular convolutions}, as shown in Figure \ref{fig:multi-channel-convolution}. The resulting network will be a deep convolutional network that shares the same multi-channel convolution structures as empirically designed  CNNs for 2D images \citep{LeNet,krizhevsky2012imagenet} or ones suggested for promoting sparsity \citep{papyan2017convolutional}! The difference is that, again, the convolution architectures and parameters of our network are derived from the rate reduction objective, and so are all the nonlinear activations. Like the 1D signal case, the derivation in Appendix \ref{ap:2D-translation} shows that this convolutional network can be constructed much more efficiently in the spectral domain. See Theorem \ref{thm:2D-convolution} of Appendix \ref{ap:2D-translation} for a rigorous statement and justification.

\subsection{Overall Network Architecture and Comparison}
\label{sec:architecture-comparison}
Following the above derivation, we see that in order to find a linear discriminative representation (LDR)  for multiple classes of signals/images that is invariant to translation, sparse coding, a multi-layer architecture with multi-channel convolutions, different nonlinear activation, and spectrum computing all become {\em necessary} components for achieving the objective effectively and efficiently. Figure \ref{fig:learn-to-classify-diagram} illustrates the overall process of learning such a representation via invariant rate reduction on the input sparse codes. 

\begin{figure}[t]
    \centering
    \includegraphics[width=0.98\linewidth]{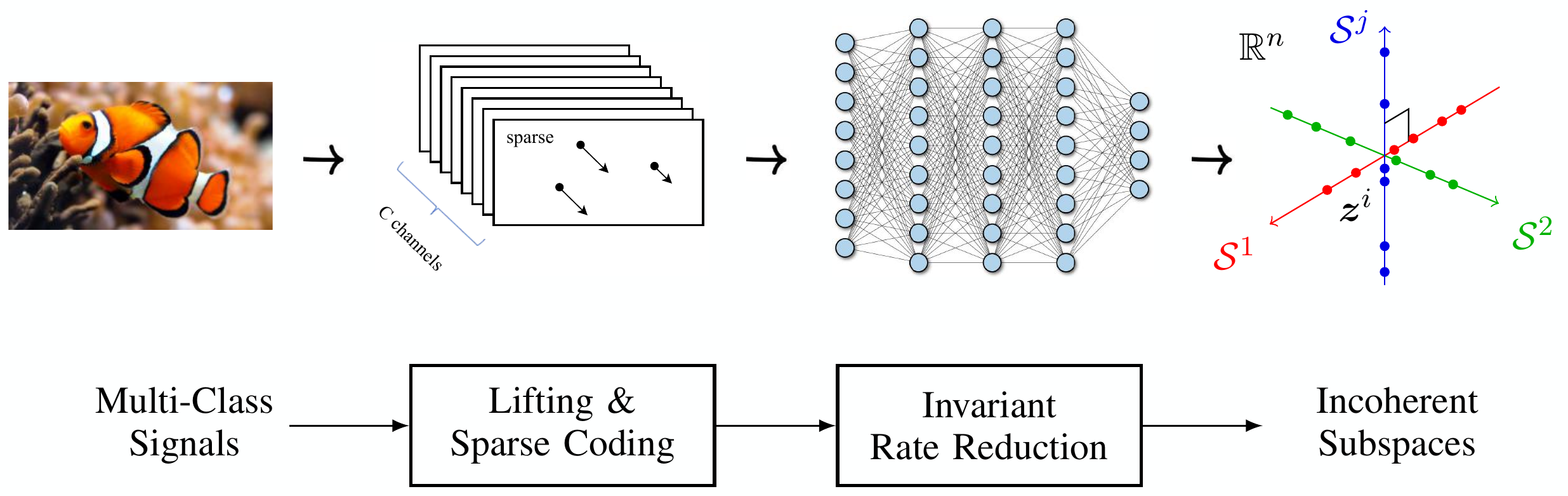}
    \caption{The overall process for classifying multi-class signals with shift invariance: Multi-channel lifting,  sparse coding, followed by a multi-channel convolution ReduNet for invariant rate reduction. These components are {\em necessary} in order to map shift-invariant multi-class signals to incoherent (linear) subspaces as an LDR. Note that the architectures of most modern deep neural networks resemble this process. The so-learned LDR facilitates subsequent tasks such as classification.}
    \label{fig:learn-to-classify-diagram}
    \vspace{-0.1in}
\end{figure}

\paragraph{Connections to convolutional and recurrent sparse coding.}
As we have discussed in the introduction, it has been long noticed that there are connections between sparse coding and deep networks, including the work of Learned ISTA  \citep{gregor2010learning} and many of its convolutional and recurrent variants  \citep{Wisdom2016InterpretableRN,papyan2017convolutional,sulam2018multilayer,monga2019algorithm}. Although both sparsity and convolution are advocated as desired characteristics for such networks, their precise roles for the classification task have never been fully revealed. For instance, \cite{papyan2017convolutional} has suggested a convolutional sparse coding framework for the CNNs. It actually aligns well with the early lifting and sparse coding stage of the overall process. Nevertheless, our new framework suggests, once such sparse codes are obtained, one needs subsequent ReduNet to transform them to the desired LDRs. Through our derivations, we see how lifting and sparse coding (via random filtering or sparse deconvolution), multi-channel convolutions ($\bar{\bm E}, \bar{\bm C}^j$), and different nonlinear operations: grouping $\widehat{\bm \pi}^j(\cdot)$, normalization $\mathcal{P}_{\mathbb{S}^{n-1}}(\cdot)$, and soft thresholding $\bm \tau(\cdot)$ are all derived as {\em necessary processes} from the objective of maximizing rate reduction of the learned features while enforcing shift invariance.

\paragraph{Comparison with ScatteringNet and PCANet.}
Using convolutional operators/with pooling to ensure equivarience/invariance have been common practice in deep networks \citep{lecun1995convolutional,Cohen-ICML-2016}, but the number of convolutions needed has never been clear and their parameters need to be learned via back propagation from randomly initialized ones. Of course, one may also choose a complete basis for the convolution filters in  each layer to ensure translational equivariance/invariance for a wide range of signals. ScatteringNet \citep{scattering-net} and many followup works \citep{Wiatowski-2018} have shown to use modulus of 2D wavelet transform to construct invariant features. However, the number of convolutions needed usually grow {\em exponentially} in the number of layers. That is the reason why ScatteringNet type networks cannot be so deep and typically limited to only 2-3 layers. In practice, it is often used in a hybrid setting \cite{Zarka2020Deep, zarka2021separation} for better performance and scalability. On the other hand, PCANet \citep{chan2015pcanet} argues that one can significantly reduce the number of convolution channels by learning features directly from the data. 
In contrast to ScatteringNet and PCANet, the proposed invariant ReduNet \textit{by construction} learns equivariant features from the data that are discriminative between classes. As experiments on MNIST in Appendix \ref{app:ReduNet-Scattering} show, the representations learned by ReduNet indeed preserve translation information. In fact, scattering transform can be used with ReduNet in a complementary fashion. One may replace the aforementioned random (lifting) filters with the scattering transform. As experiments in the Appendix \ref{app:ReduNet-Scattering} show this can yield significantly better classification performance.

Notice that, none of the previous ``convolution-by-design'' approaches explain why we need {\em multi-channel} (3D) convolutions instead of separable (2D) ones, let alone how to design them. In contrast, in the new rate reduction framework, we see that both the forms and roles of the multi-channel convolutions ($\bar{\bm E}, \bar{\bm C}^j$) are explicitly derived and justified, the number of filters (channels) remains constant through all layers, and even values of their parameters are determined by the data of interest. Of course, as mentioned before the values of the parameters can be further fine-tuned to improve performance, as we will see in the experimental section (Table \ref{table:backprop_acc}).

\paragraph{Sparse coding, spectral computing, and subspace embedding in nature.} Notice that {\em sparse coding} has long been hypothesized as the guiding organization principle for the visual cortex of primates \citep{olshausen1996emergence}. Through years of evolution, the visual cortex has learned to sparsely encode the visual input with all types of localized and oriented filters. Interstingly, there have been strong scientific evidences that neurons in the visual cortex transmit and process information in terms of {\em rates of spiking}, i.e. in the spectrum rather than the magnitude of signals, hence the so-called ``spiking neurons'' \citep{spking-neuron-1993,spiking-neuron-book,Belitski5696}. Even more interestingly, recent studies in neuroscience have started to reveal how these mechanisms might be integrated in the inferotemporal (IT) cortex, where neurons encode and process information about high-level object identity (e.g. face recognition), invariant to various transformations \citep{Majaj13402,Chang-Cell-2017}. In particular, \cite{Chang-Cell-2017} went even further to hypothesize that high-level neurons encode the face space as a {\em linear subspace} with each cell likely encoding one axis of the subspace (rather than previously thought ``an exemplar''). The framework laid out in this paper suggests that such a ``high-level'' compact (linear and discriminative) representation can be efficiently and effectively learned in the spectrum domain via an arguably much simpler and more natural ``forward propagation'' mechanism. Maybe, just maybe, nature has already learned to exploit what mathematics reveals as the most parsimonious and economic.

\section{Experimental Verification}\label{sec:experiments}
In this section, we conduct experiments to (1). \textit{Validate} the effectiveness of the proposed  maximal coding rate reduction (\textbf{MCR$^2$}) principle analyzed in Section~\ref{sec:principled-objective}; and (2). \textit{Verify} whether the constructed \textbf{ReduNet}, including the basic vector case ReduNet in Section~\ref{sec:vector} and the invariance ReduNet in Section~\ref{sec:shift-invariant}, achieves its design objectives through experiments. 
Our goal in this work is not to push the state of the art performance on any real datasets with additional engineering ideas and heuristics, although the experimental results clearly suggest this potential in the future.  All code is implemented in Python mainly using NumPy and PyTorch. All of our experiments are conducted in a computing node with 2.1 GHz Intel Xeon Silver CPU,  256GB of memory and 2 Nvidia RTX2080. Implementation details and many more experiments and  can be found in Appendix~\ref{ap:additional-exp} and \ref{sec:appendix-redunet-exp}.

\subsection{Experimental Verification of the MCR$^2$ Objective}\label{sec:experiment-objective}
In this subsection, we present experimental results on investigating the MCR$^2$ objective function for training neural networks. Our theoretical analysis in Section~\ref{sec:principled-objective} shows how the {\em maximal coding rate reduction} (MCR$^2$) is a principled measure for learning discriminative and diverse representations for mixed data. In this section, we demonstrate experimentally how this principle alone, {\em without any other heuristics,} is adequate to learning good representations. 
More specifically, we apply the widely adopted neural network architectures (such as ResNet~\citep{he2016deep}) as the feature mapping $\z = f(\x,\bm \theta)$ and optimize the neural network parameters $\bm \theta$ to achieve maximal coding rate reduction. 
Our goal here is to validate effectiveness of this principle through its most basic usage and fair comparison with existing frameworks. More implementation details and experiments are given in Appendix~\ref{ap:additional-exp}.  The code for reproducing the results on the effectiveness of MCR$^2$ objective in this section can be found in \url{https://github.com/Ma-Lab-Berkeley/MCR2}.

\vspace{-0.05in}
\paragraph{Supervised learning via rate reduction.} When class labels are provided during training, we assign the membership (diagonal) matrix $\bm{\Pi} = \{\bm{\Pi}^j\}_{j=1}^{k}$ as follows: for each sample $\bm{x}^i$ with label $j$, set $\bm{\Pi}^{j}(i,i) = 1$ and  $\bm{\Pi}^l(i,i)=0, \forall l \not= j$. Then the mapping $f(\cdot, \bm{\theta})$ can be learned by optimizing \eqref{eqn:maximal-rate-reduction}, where $\bm{\Pi}$ 
remains constant. We apply stochastic gradient descent to optimize MCR$^2$, and for each iteration we use mini-batch data $\{(\bm{x}^i, \bm{y}^{i})\}_{i=1}^{m}$ to approximate the MCR$^2$ loss.

\begin{figure*}[t]
\subcapcentertrue
  \begin{center}
    \subfigure[\label{fig:train-test-loss-pca-1} Evolution of $R, R_c, \Delta R$ during the training process.]{\includegraphics[width=0.32\textwidth]{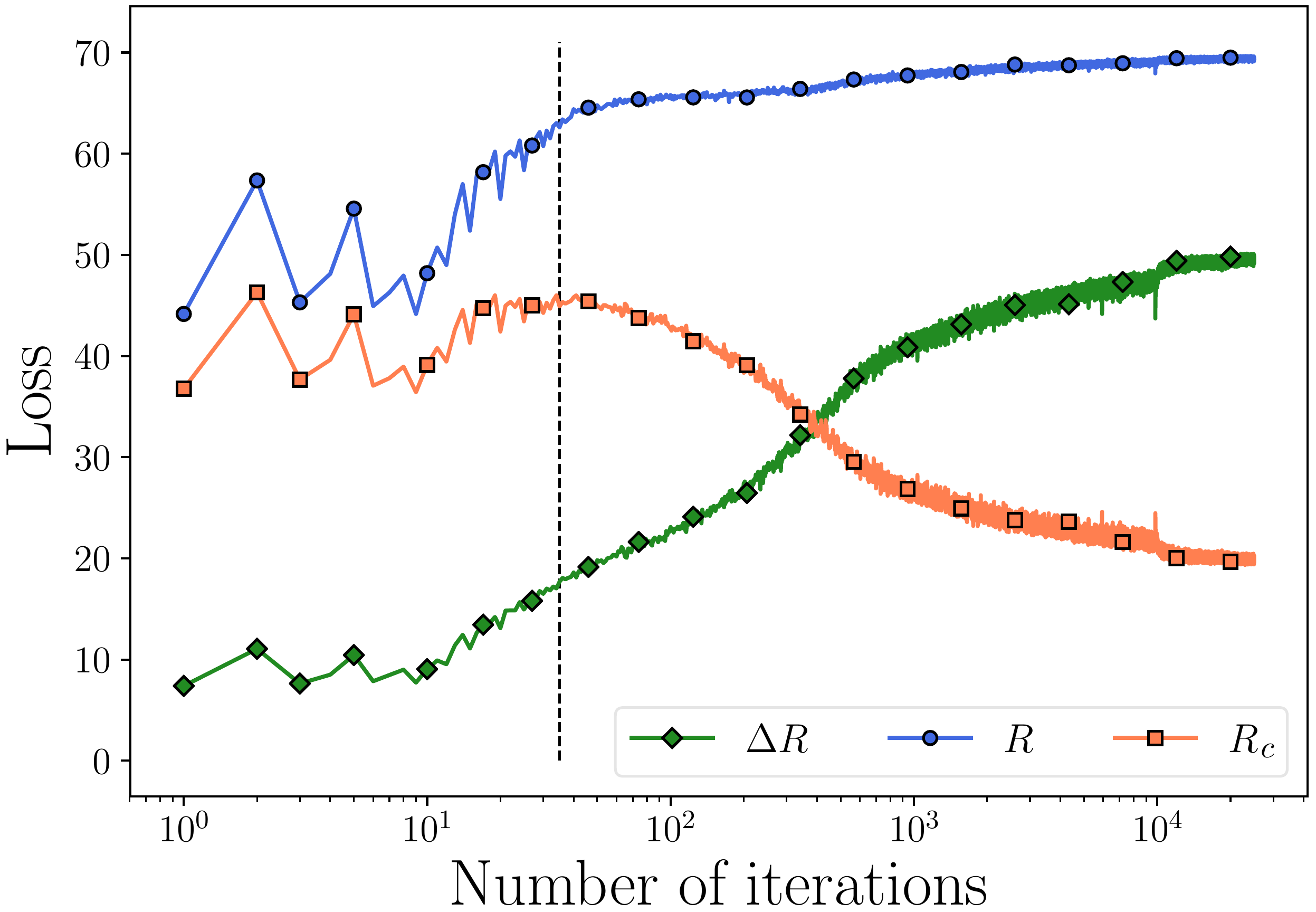}}
    \subfigure[\label{fig:train-test-loss-pca-2}Training loss versus testing loss.]{\includegraphics[width=0.32\textwidth]{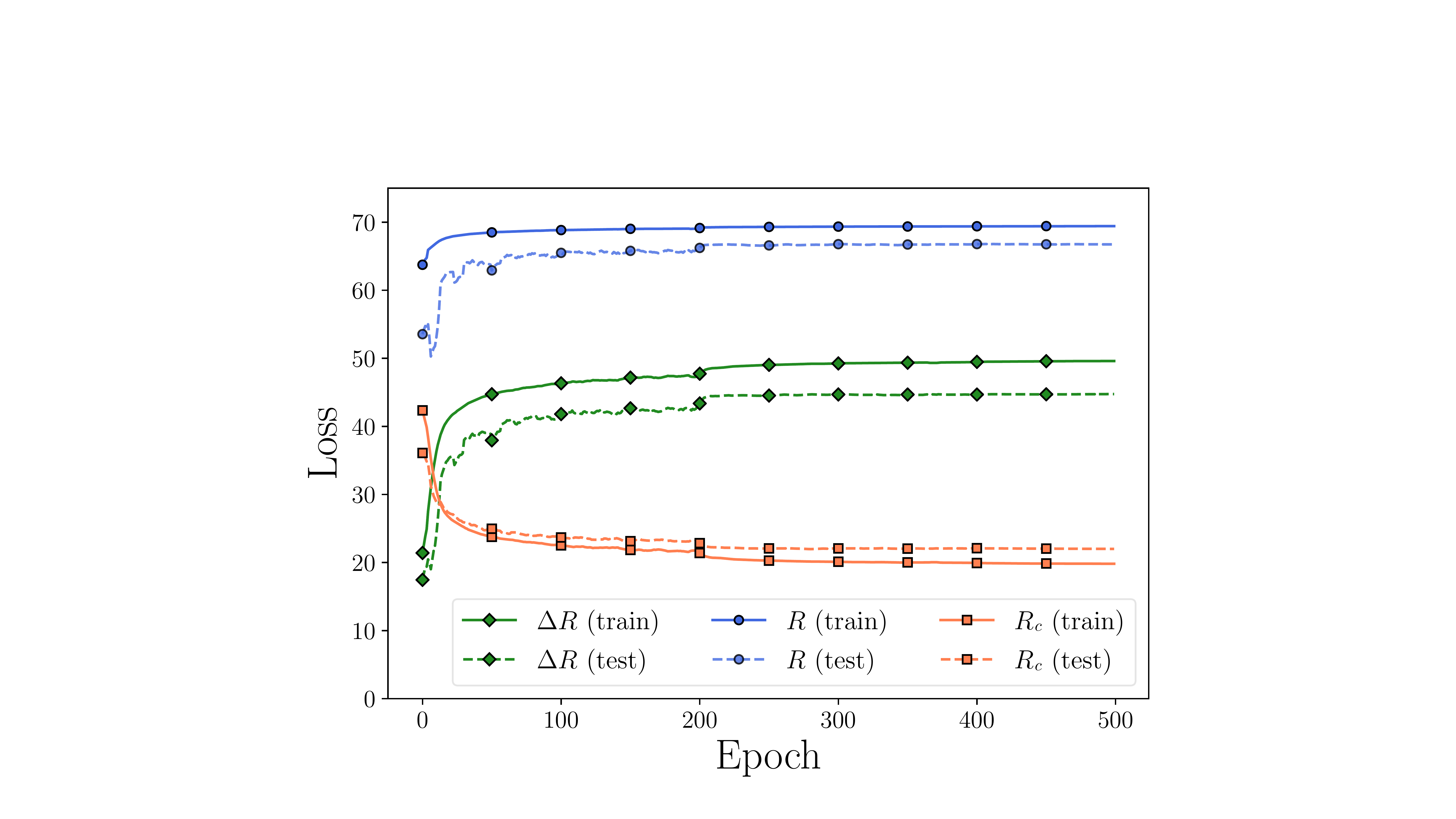}}
    \subfigure[\label{fig:train-test-loss-pca-3}PCA: {\small (\textbf{red}) overall data; (\textbf{blue}) individual classes}.]
    {\includegraphics[width=0.32\textwidth]{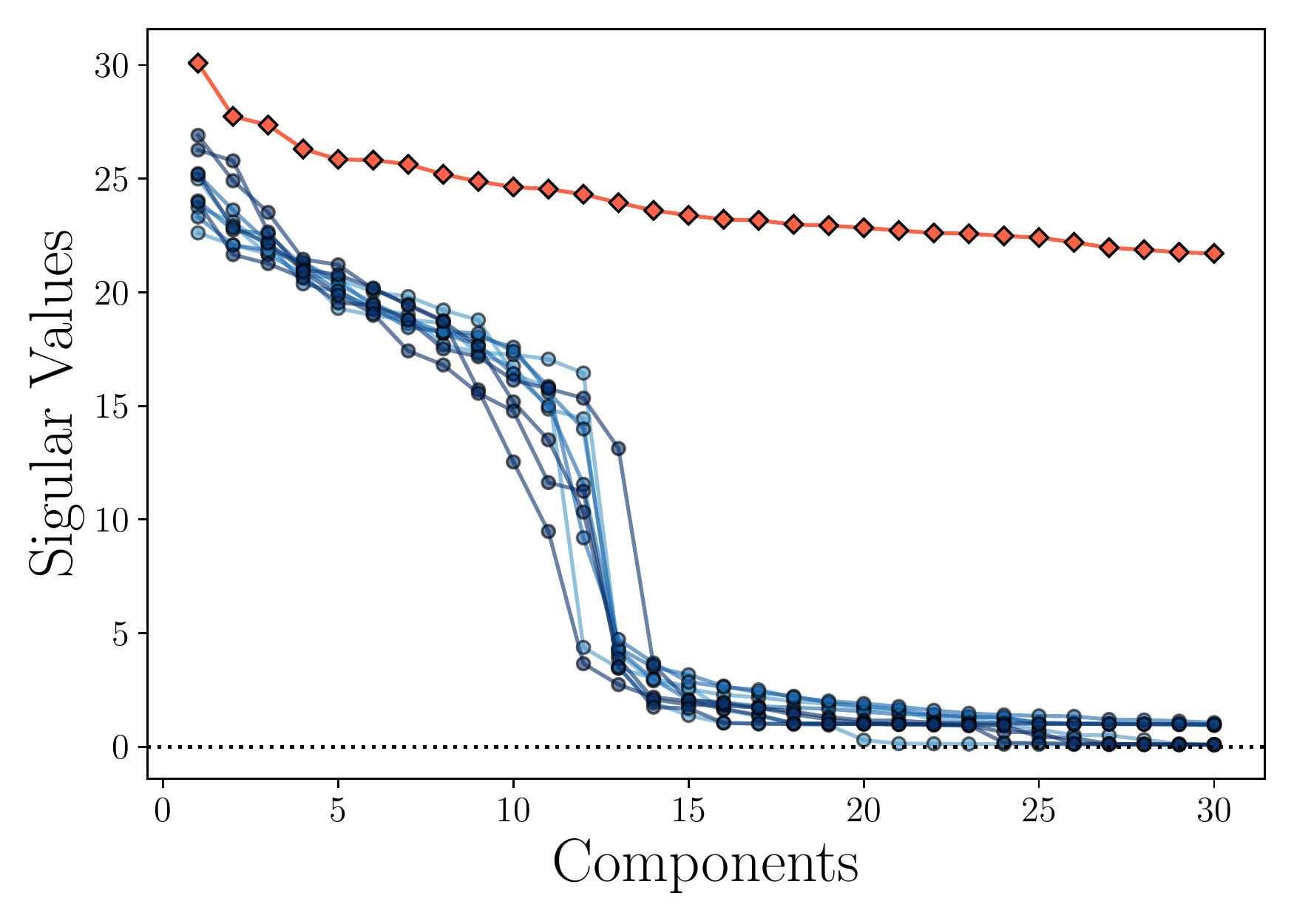}}
    \vskip -0.05in
    \caption{\small Evolution of the rates of MCR$^2$ in the training process and principal components of learned features.}
    \label{fig:train-test-loss-pca}
  \end{center}
  \vskip -0.4in
\end{figure*} 

\vspace{-0.05in}
\paragraph{Evaluation via classification.} As we will see, in the supervised setting, the learned representation has very clear subspace structures. So to evaluate the learned representations, we consider a natural nearest subspace classifier. For each class of learned features $\Z^j$, let $\bm{\mu}^j \in \Re^n$ be the mean of the representation vectors of $j$-th class and $\U^j \in \Re^{n \times r_j}$ be the first $r_j$ principal components for $\Z^j$, where $r_j$ is the estimated  dimension of class $j$. The predicted label of a test data $\x'$ is given by
$j' = \argmin_{j \in \{1, \ldots, k\}}\|(\I - \U^j (\U^j)^*) (f(\x', \bm{\theta}) - \bm{\mu}^j)\|_2^2.$

\vspace{-0.05in}
\paragraph{Experiments on real data.} We consider CIFAR10 dataset~\citep{krizhevsky2009learning} and ResNet-18 \citep{he2016deep} for $f(\cdot, \bm{\theta})$. We replace the last linear layer of ResNet-18 by a two-layer fully connected network with ReLU activation function such that the output dimension is 128. We set the mini-batch size as $m = 1,000$ and the precision parameter $\epsilon^2 = 0.5$. More results can be found in Appendix~\ref{sec:appendix-subsec-sup}.

Figure~\ref{fig:train-test-loss-pca-1} illustrates how the two rates and their difference (for both training and test data) evolves over epochs of training: After an initial phase, $R$ gradually increases while $R_c$ decreases, indicating that features $\bm Z$ are expanding as a whole while each class $\bm Z^j$ is being compressed.  
Figure~\ref{fig:train-test-loss-pca-3} shows the distribution of singular values per $\Z^j$ and Figure~\ref{fig:low-dim} (right) shows the angles of features sorted by class.  Compared to the geometric loss \citep{lezama2018ole}, our features are {\em not only orthogonal but also of much higher dimension}. We compare the singular values of representations, both overall data and individual classes, learned by using cross-entropy and MCR$^2$ in Figure~\ref{fig:pca-plot} and Figure~\ref{fig:heatmap-plot} in Appendix \ref{sec:subsec-pca}. We 
find that the representations learned by using MCR$^2$ loss are much more diverse than the ones learned by using cross-entropy loss. In addition, we find that we are able to select diverse images from the same class according to the ``principal'' components of the learned features (see Figure~\ref{fig:visual-class-2-8} and Figure~\ref{fig:visual-overall-data}).

One potential caveat of MCR$^2$ training is how to optimally select the output dimension $n$ and training batch size $m$. For a given output dimension $n$, a sufficiently large batch size $m$ is needed in order to achieve good classification performance. More detail study on varying output dimension $n$ and batch size $m$ is listed in Table~\ref{table:ablation-supervise} of Appendix~\ref{sec:appendix-subsec-sup}.

\begin{figure*}[t]
\subcapcentertrue
\begin{center}
    \subfigure[\label{fig:visual-bird}Bird]{\includegraphics[width=0.47\textwidth]{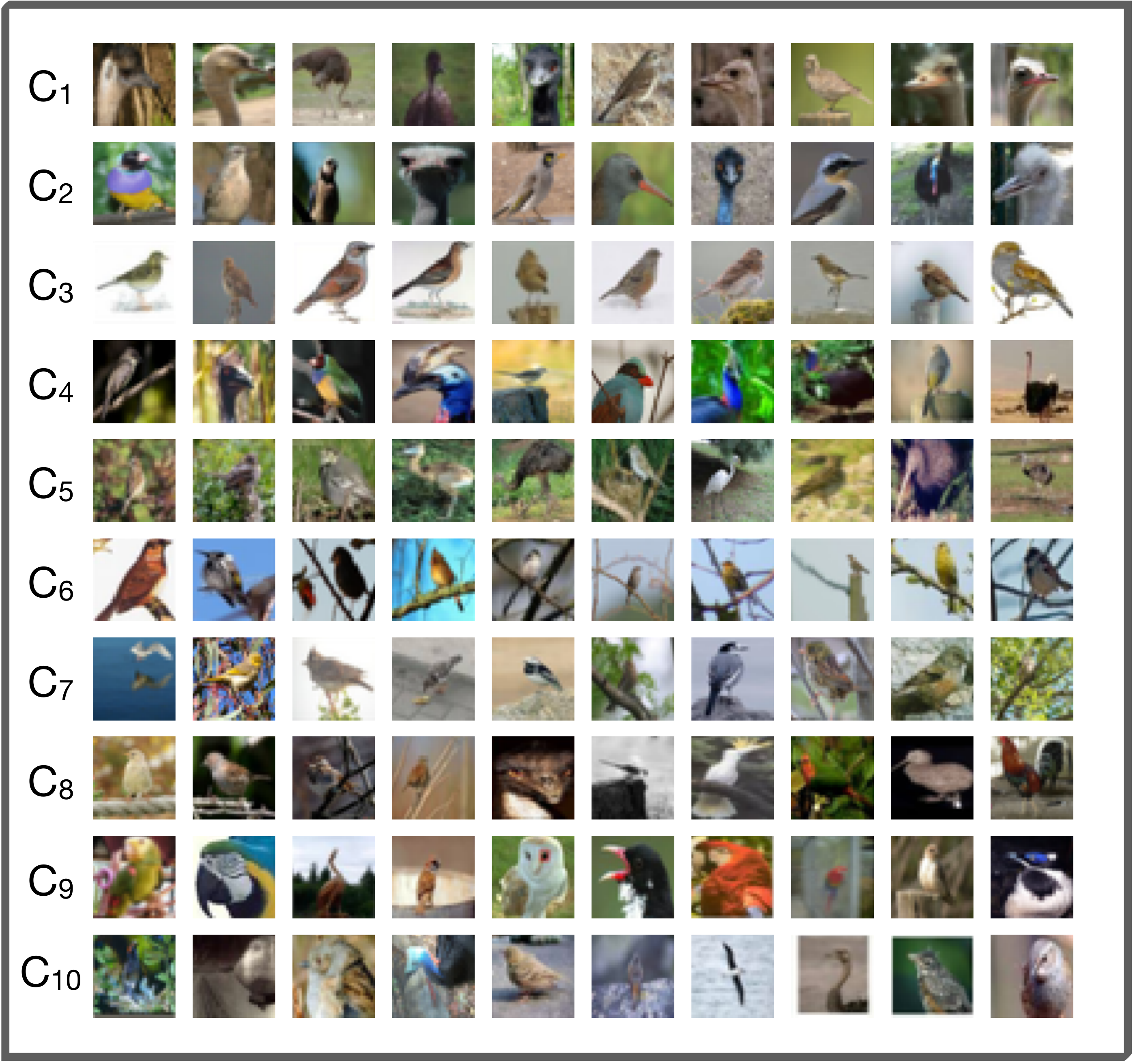}}
    \subfigure[\label{fig:visual-ship}Ship]{\includegraphics[width=0.47\textwidth]{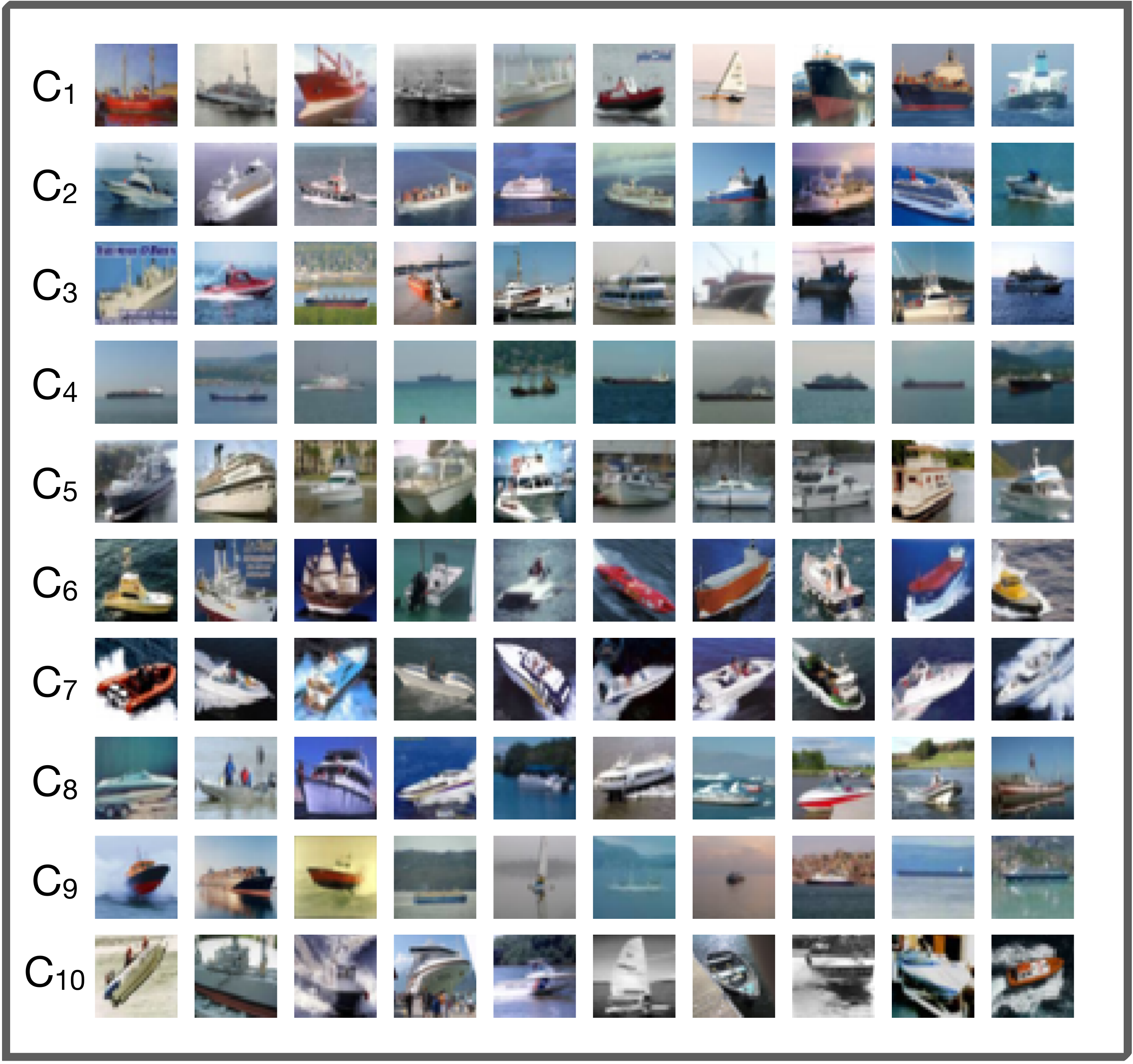}}
    \caption{\small Visualization of principal components learned for class 2-`Bird' and class 8-`Ship'. For each class $j$, we first compute the top-10 singular vectors of the SVD of the learned features $\Z^j$. Then for the $l$-th singular vector of class $j$ (denoted by $\u_{j}^{l}$), and for the feature of the $i$-th image of class $j$ (denoted by $\z_{j}^{i}$), we calculate the absolute value of inner product, $| \langle  \z_{j}^{i}, \u_{j}^{l} \rangle|$, then we select the top-10 images according to  $| \langle  \z_{j}^{i}, \u_{j}^{l} \rangle|$ for each singular vector. 
    In the above two figures, each row corresponds to one singular vector (component $C_l$). The rows are sorted based on the magnitude of the associated singular values, from large to small.}
\label{fig:visual-class-2-8}
\end{center}
\vskip -0.4in
\end{figure*}

\begin{figure*}[t]
  \begin{center}
    \subfigure[\label{fig:train-label-noise-1} $\Delta R\big(\Z(\bm{\theta}), \bm{\Pi}, \epsilon\big)$.]{\includegraphics[width=0.32\textwidth]{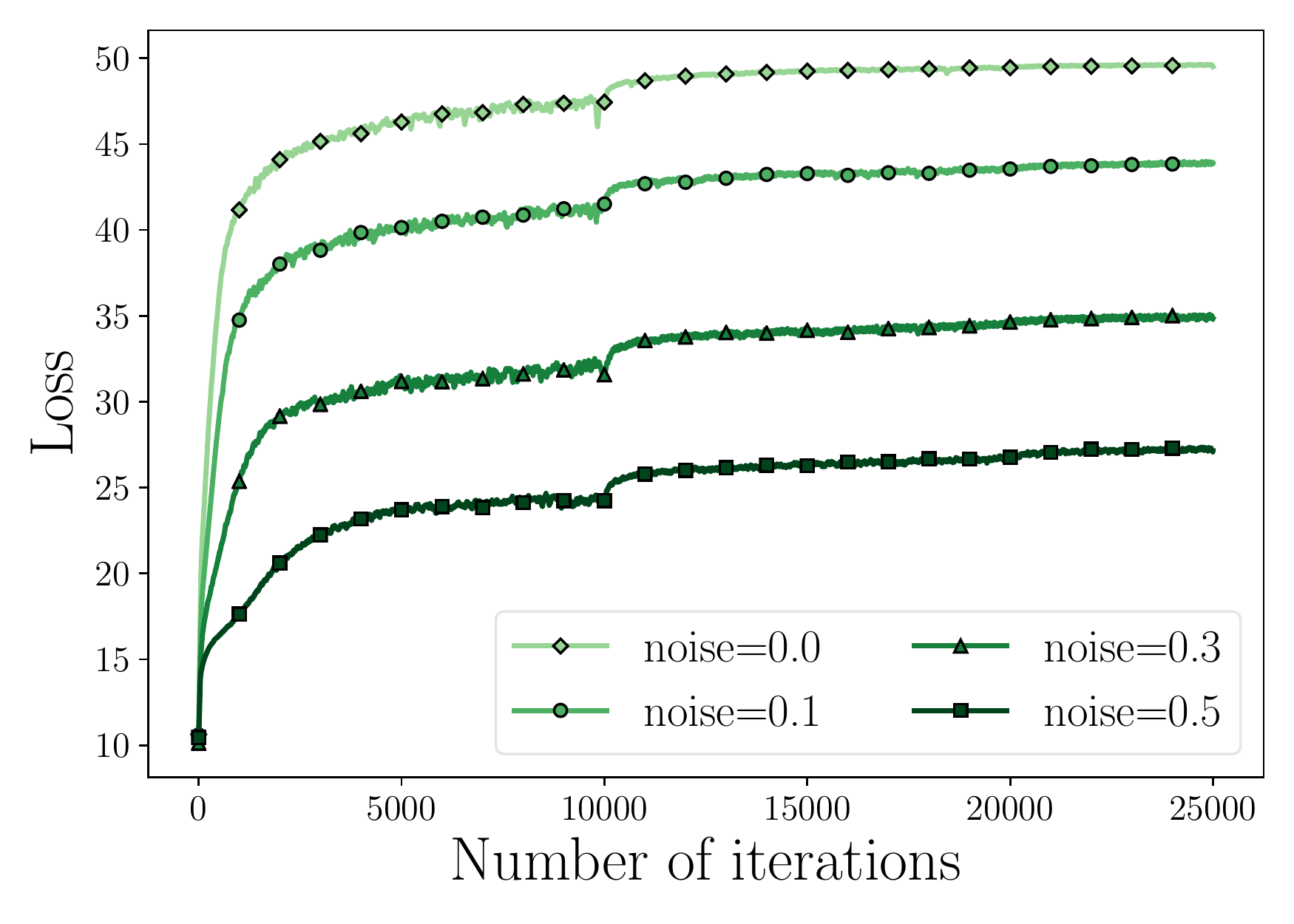}}
    \subfigure[\label{fig:train-label-noise-2} $R(\Z(\bm{\theta}), \epsilon)$.]{\includegraphics[width=0.32\textwidth]{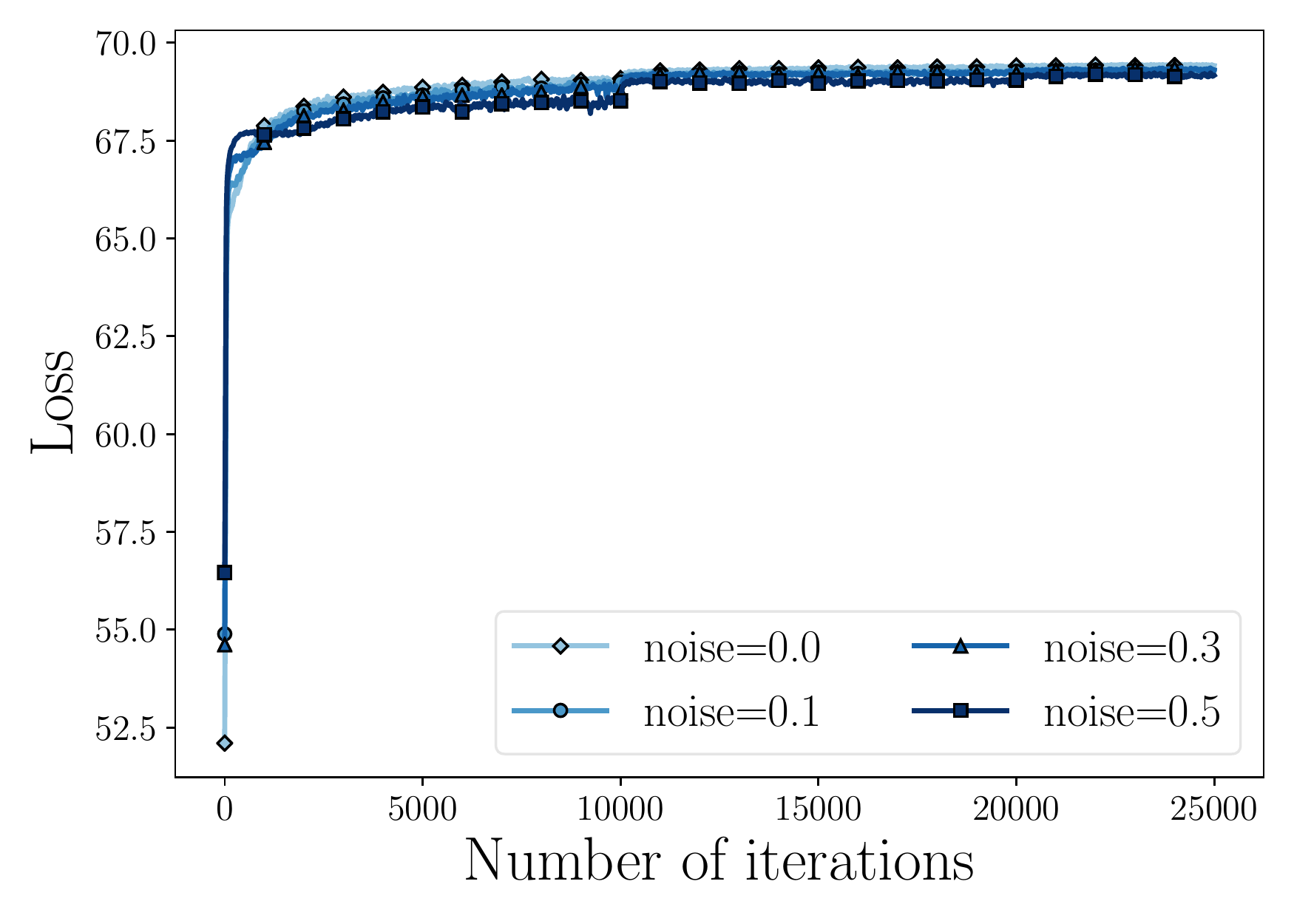}}
    \subfigure[\label{fig:train-label-noise-3}  $R_c(\Z(\bm{\theta}),  \epsilon \mid \bm{\Pi})$.]{\includegraphics[width=0.32\textwidth]{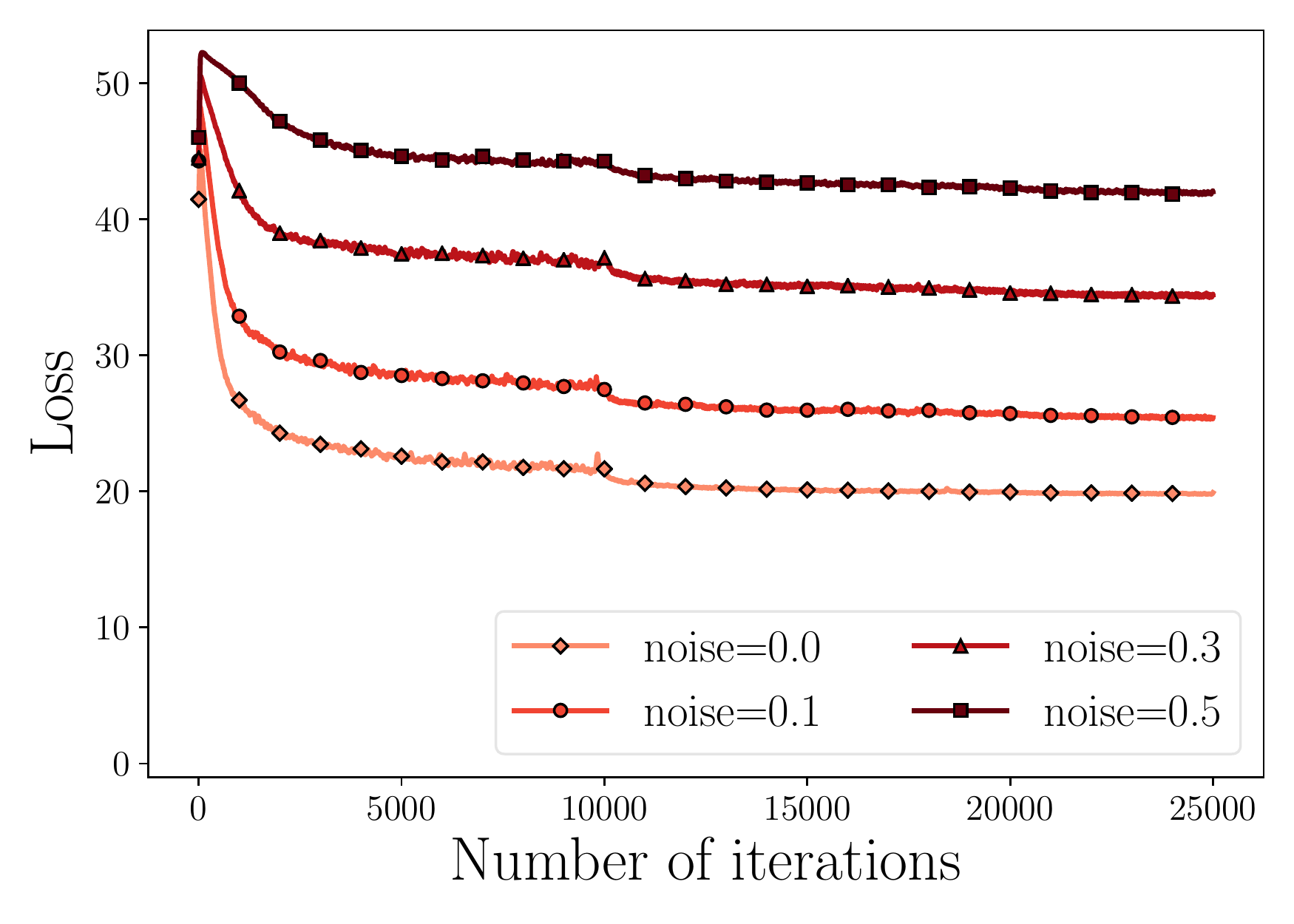}}
    \vskip -0.05in
    \caption{\small Evolution of rates $R, R_c, \Delta R$ of MCR$^2$ during training with corrupted labels.}
    \label{fig:train-label-noise}
  \end{center}
\vskip -0.3in
\end{figure*}

\begin{table}[h]
\begin{center}
\vskip 0.05in
\begin{small}
\begin{sc}
\begin{adjustbox}{max width=\textwidth}
\begin{tabular}{l |c c c c c c }
\toprule
 & Ratio=0.0 & Ratio=0.1 &  Ratio=0.2 &  Ratio=0.3 &  Ratio=0.4 &  Ratio=0.5 \\
\midrule
CE Training & 0.939 & 0.909 & 0.861 & 0.791 & 0.724  & 0.603 \\
MCR$^2$ Training  & \textbf{0.940} & \textbf{0.911} & \textbf{0.897} & \textbf{0.881} & \textbf{0.866} &  \textbf{0.843}\\
\bottomrule
\end{tabular}
\end{adjustbox}
\end{sc}
\end{small}
\caption{\small Classification results with features learned with labels corrupted at different levels.}
\label{table:label-noise}
\end{center}
\vspace{-4mm}
\end{table}

\noindent \textbf{Robustness to corrupted labels.} 
Because MCR$^2$ by design encourages richer representations that preserves intrinsic structures from the data $\X$, training relies less on class labels than traditional loss such as cross-entropy (CE). To verify this, we train the same network using both CE and MCR$^2$ with certain ratios of \textit{randomly corrupted} training labels.  Figure~\ref{fig:train-label-noise} illustrates the learning process: for different levels of corruption, while the rate for the whole set always converges to the same value, the rates for the classes are inversely proportional to the ratio of corruption, indicating our method only compresses samples with valid labels. The classification results are summarized in Table~\ref{table:label-noise}. By applying \textit{exact the same} training parameters, MCR$^2$ is significantly more robust than CE, especially with  higher ratio of corrupted labels. This can be an advantage in the settings of self-supervised learning or constrastive learning when the grouping information can be very noisy. More detailed comparison between MCR$^2$ and  OLE~\citep{lezama2018ole}, Large Margin Deep Networks~\citep{elsayed2018large}, and ITLM~\citep{shen2019learning} on learning from noisy labels can be found in Appendix~\ref{sec:appendix-label-noise-related-work} (Table~\ref{table:label-noise-related-work}).

Beside the supervised learning setting, we explore the MCR$^2$ objective in the self-supervised learning setting. We find that the MCR$^2$ objective can learn good representations without using any label and achieve better performance over other highly engineered methods on clustering tasks. More details on self-supervised learning can be found in Section~\ref{sec:appendix-selfsup}.

\subsection{Experimental Verification of the ReduNet}
In this section, we \textit{verify} whether the so constructed ReduNet (developed in  Section~\ref{sec:vector} and \ref{sec:shift-invariant}) achieves its design objectives through experiments on synthetic data and real images. The datasets and experiments are chosen to clearly demonstrate the properties and behaviors of the proposed ReduNet, in terms of learning the correct truly invariant discriminative (orthogonal) representation for the given data. 
Implementation details and more experiments and  can be found in Appendix~\ref{sec:appendix-redunet-exp}. The code for reproducing the ReduNet results can be found in \url{https://github.com/Ma-Lab-Berkeley/ReduNet}.

\vspace{-0.05in}
\paragraph{Learning Mixture of Gaussians in $\mathbb{S}^2$.}
Consider a mixture of three Gaussian distributions in $\R^{3}$ that is projected onto $\mathbb{S}^2$. We first generate data points from these two distributions,  $\X^{1}=[\x^1_{1}, \ldots, \x^{m}_{1}] \in \R^{3\times m}$, $\x^{i}_{1} \sim \mathcal{N}(\bm{\mu}_{1}, \sigma_{1}^{2} \I)$, and $\bm{\pi}(\x^{i}_{1}) = 1$; $\X^{2}=[\x^1_{2}, \ldots, \x^{m}_{2}] \in \R^{2\times m}$, $\x^{i}_{2} \sim \mathcal{N}(\bm{\mu}_{2}, \sigma_{2}^{2} \I)$, and $\bm{\pi}(\x^{i}_{2}) = 2$; $\X^{3}=[\x^1_{3}, \ldots, \x^{m}_{3}] \in \R^{3\times m}$, $\x^{i}_{3} \sim \mathcal{N}(\bm{\mu}_{3}, \sigma_{3}^{2} \I)$, and $\bm{\pi}(\x^{i}_{2}) = 3$. We set $m=500, \sigma_{1}=\sigma_{2}=\sigma_{3}=0.1$ and $\bm{\mu}^{1}, \bm{\mu}^{2}, \bm{\mu}^{3} \in \mathbb{S}^2$. Then we project all the data points onto $\mathbb{S}^{2}$, i.e., $\x^i_{j}/\|\x^i_{j}\|_{2}$. To construct the network (computing $\E_{\ell}, \C_{\ell}^{j}$ for $\ell$-the layer), we set the number of iterations/layers $L=2,000$\footnote{We do this only to demonstrate our framework leads to stable deep networks even with thousands of layers! In practice this is not necessary and one can stop whenever adding new layers gives diminishing returns. For this example, a couple of hundred is sufficient. Hence the clear optimization objective gives a natural criterion for the depth of the network needed. Remarks in Section \ref{sec:comparison-vector} provide possible ideas to further reduce the number of layers.}, step size $\eta=0.5$, and precision $\epsilon=0.1$. As shown in Figure~\ref{fig:gaussian3d-scatter-heatmap-a}-\ref{fig:gaussian3d-scatter-heatmap-b}, we can observe that after the mapping $f(\cdot, \bm{\theta})$, samples from the same class converge to a single cluster and the angle between two different clusters is approximately $\pi/4$, which is well aligned with the optimal solution $\Z_{\star}$ of the MCR$^2$ loss in $\mathbb{S}^2$. MCR$^2$ loss of features on different layers can be found in Figure~\ref{fig:gaussian3d-scatter-heatmap-c}. Empirically, we find that our constructed network is able to maximize MCR$^2$ loss and converges stably and samples from the same class converge to one cluster and different clusters are orthogonal to each other. Moreover, we sample new data points from the same distributions for both cases and find that new samples form the same class consistently converge to the same cluster center as the training samples. More simulation examples and details can be found in Appendix~\ref{sec:appendix-gaussian}.

\begin{figure*}[t]
  \begin{center}
    \subfigure[\label{fig:gaussian3d-scatter-heatmap-a}$\X_{\text{train}}$ ($3D$)]{
    \includegraphics[width=0.3\textwidth]{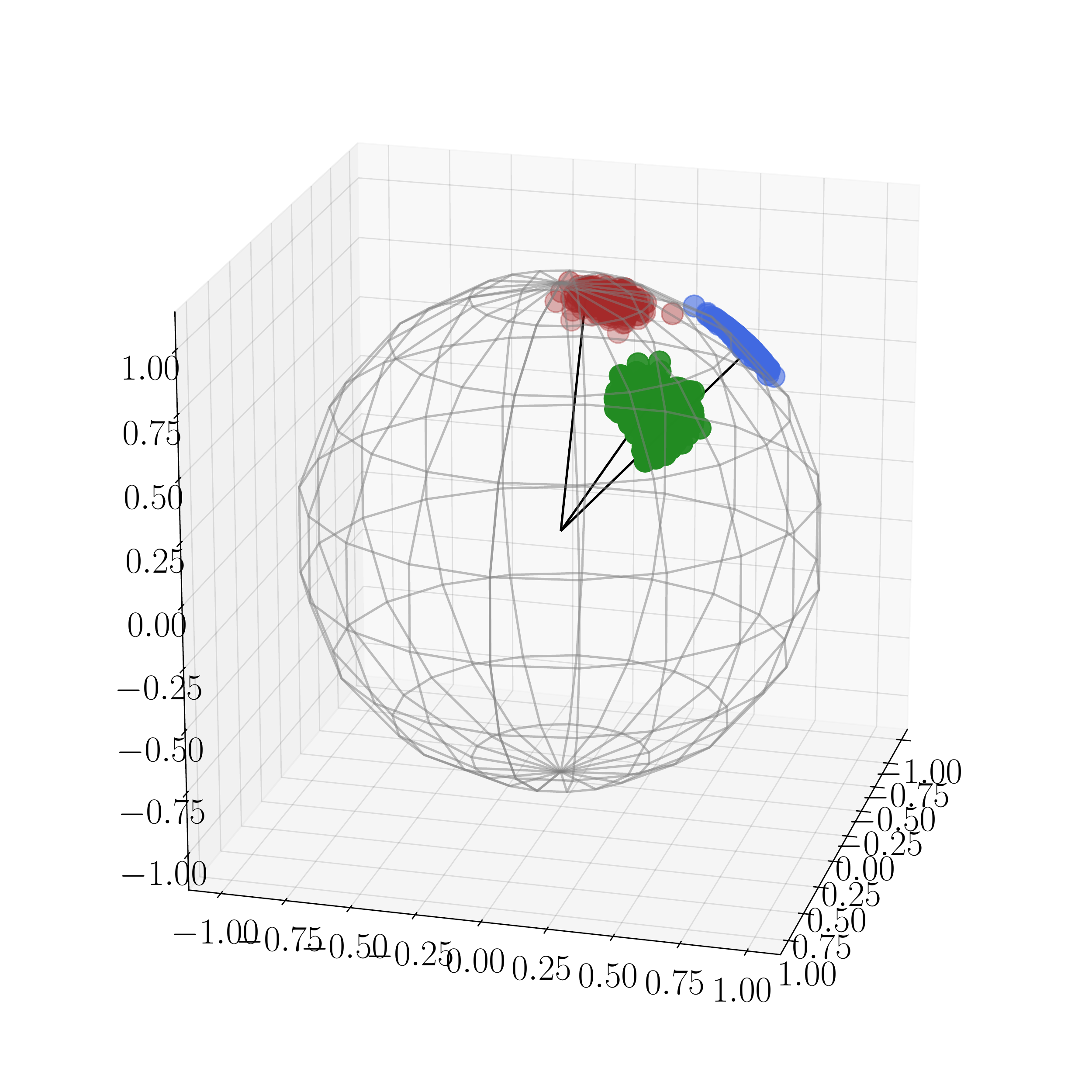}
    }
    \subfigure[\label{fig:gaussian3d-scatter-heatmap-b}$\Z_{\text{train}}$ ($3D$)]{
    \includegraphics[width=0.3\textwidth]{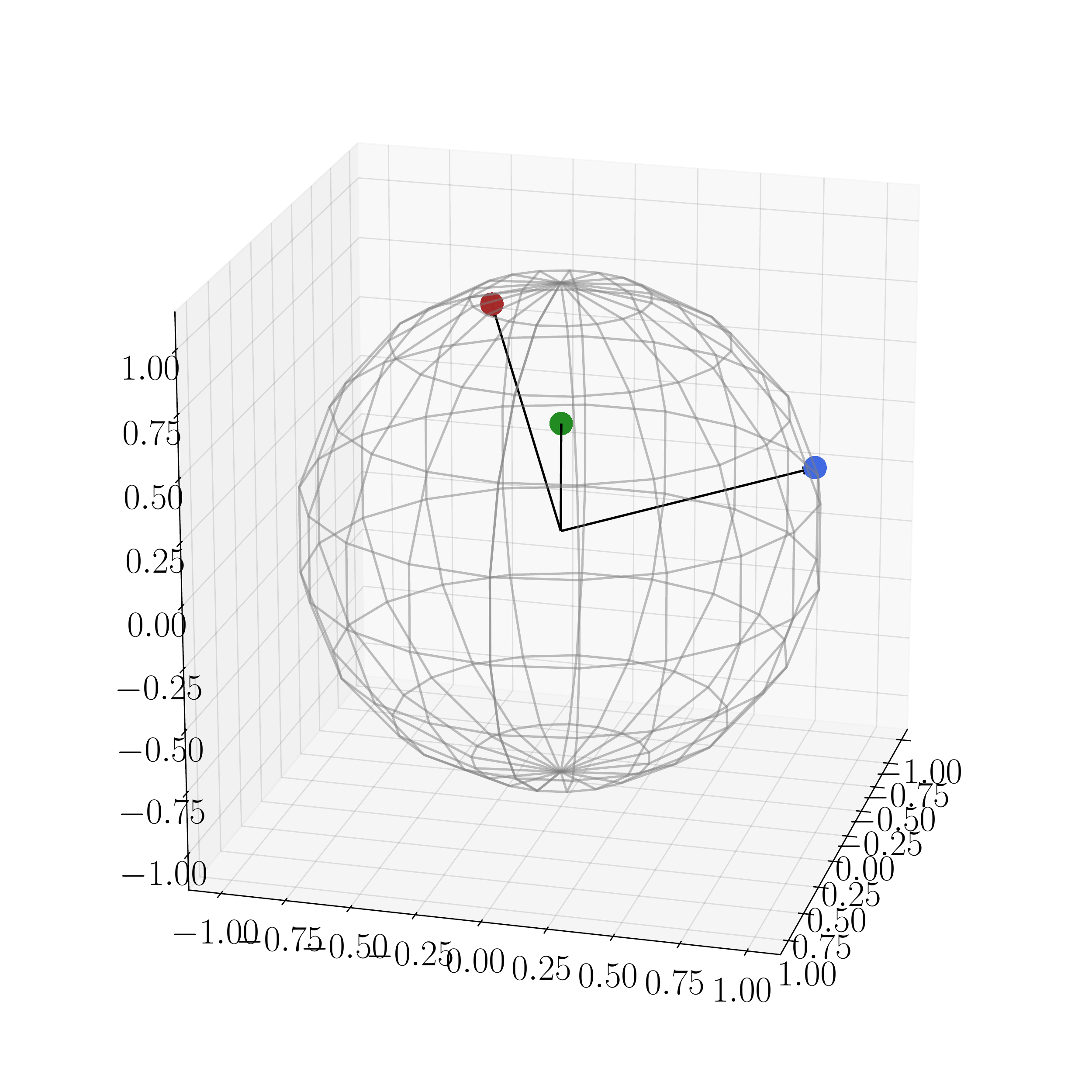}
    }
    \subfigure[\label{fig:gaussian3d-scatter-heatmap-c}Loss ($3D$)]{
    \includegraphics[width=0.3\textwidth]{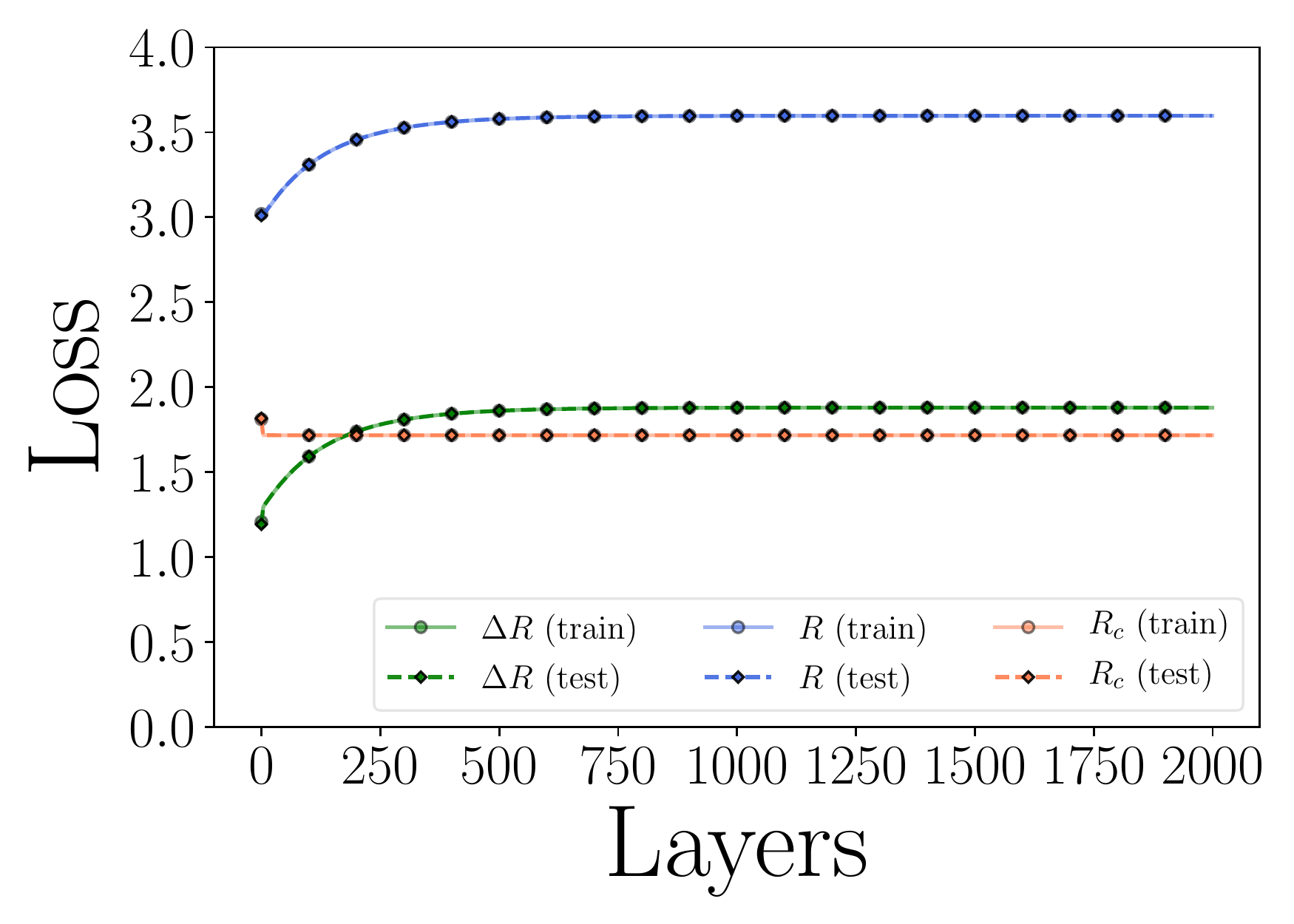}
    }
    \vskip -0.1in
    \caption{\small Original samples and learned representations for 3D Mixture of Gaussians. We visualize data points $\X$ (before mapping $f(\cdot, \bm{\theta})$) and learned features $\Z$ (after mapping $f(\cdot, \bm{\theta})$) by scatter plot. In each scatter plot, each color represents one class of samples. We also show the plots for the progression of values of the objective functions.
    }\label{fig:gaussian3d-scatter-heatmap}
  \end{center}
\vspace{-0.4in}
\end{figure*}

\vspace{-0.05in}
\paragraph{Rotational Invariance on MNIST Digits.}
We now study the ReduNet on learning \textit{rotation} invariant features on the real 10-class MNIST dataset~\citep{lecun1998mnist}. We impose a polar grid on the image $\x\in\R^{H\times W}$, with its geometric center being the center of the 2D polar grid (as illustrated in Figure \ref{fig:appendix-mnist-rotation-visualize} in the Appendix). For each radius $r_i$, $i \in [C]$, we can sample $\Gamma$ pixels with respect to each angle $\gamma_l =l\cdot({2\pi}/\Gamma)$ with $l \in [\Gamma]$. 
Then given a sample image $\x$ from the dataset, we represent the image in the (sampled) polar coordinate as a multi-channel signal $\x_p \in\R^{\Gamma\times C}$.  Our goal is to learn a rotation invariant representation, i.e., we expect to learn $f(\cdot, \bm{\theta})$ such that $\{f(\x_p \circ \mathfrak{g}, \bm{\theta})\}_{\mathfrak{g} \in\mathbb{G}}$ lie in the same subspace, where $\mathfrak{g}$ is the cyclic-shift in polar angle.  We use $m=100$ training samples ($10$ from each class) and set $\Gamma=200$, $C=15$ for polar sampling. By performing the above sampling in polar coordinate, we can obtain the data matrix $\X_p \in \mathbb{R}^{(\Gamma\cdot C) \times m}$. For the ReduNet, we set the number of layers/iterations $L=40$, precision $\epsilon=0.1$, step size $\eta=0.5$. Before the first layer, we perform lifting of the input by 1D circulant-convolution with 20 random Gaussian kernels of size 5.

\begin{figure*}[ht]
  \begin{center}
    \subfigure[\label{fig:1d-invariance-plots-a}$\X_{\text{shift}}$ (RI-MNIST)]{
    \includegraphics[width=0.22\textwidth]{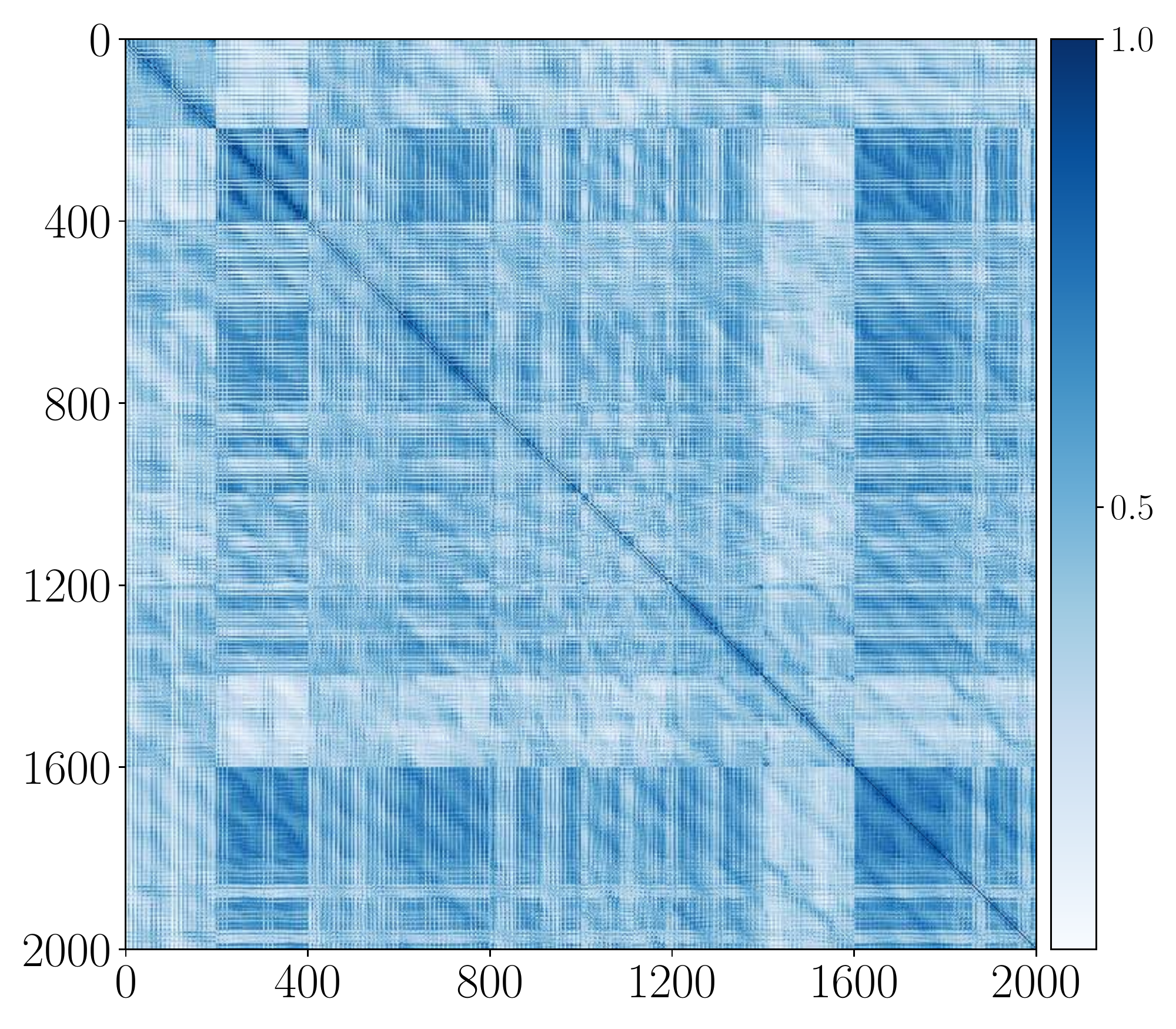}
    }
    \subfigure[\label{fig:1d-invariance-plots-b}$\bar{\Z}_{\text{shift}}$ (RI-MNIST)]{
    \includegraphics[width=0.22\textwidth]{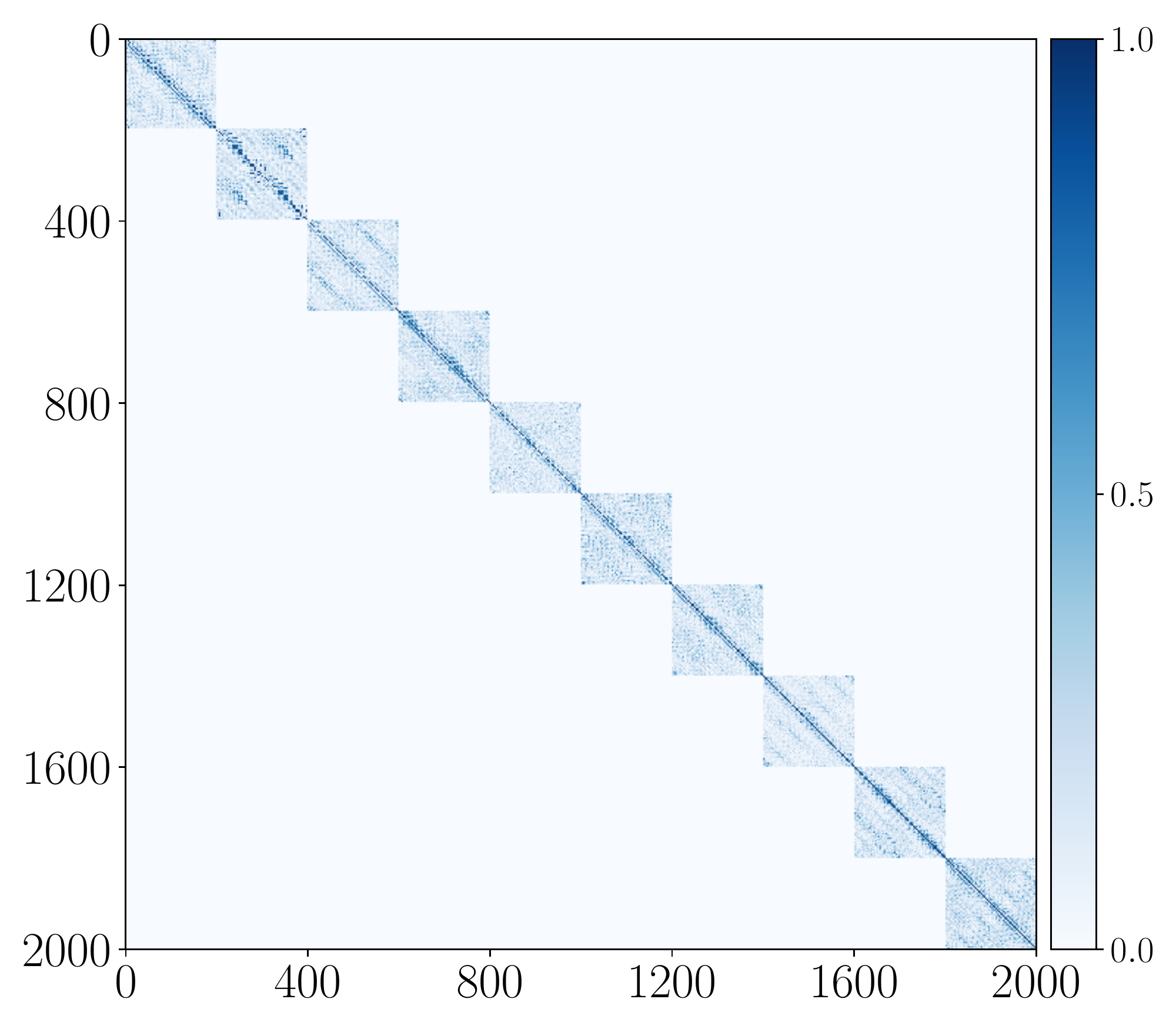}
    }
    \subfigure[\label{fig:1d-invariance-plots-c}Similarity (RI-MNIST)]{
    \includegraphics[width=0.23\textwidth]{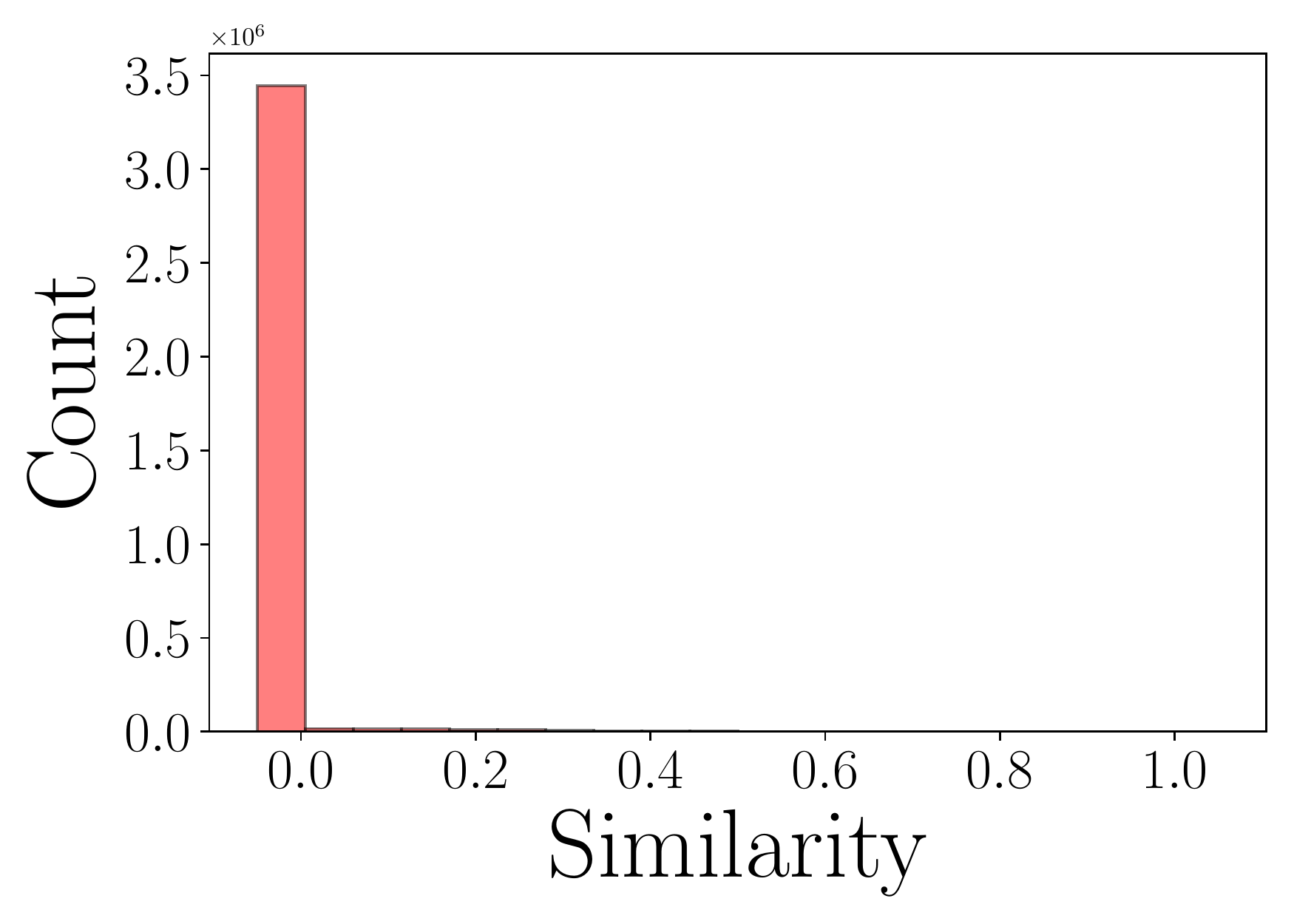}
    }
    \subfigure[\label{fig:1d-invariance-plots-d}Loss (RI-MNIST)]{
    \includegraphics[width=0.23\textwidth]{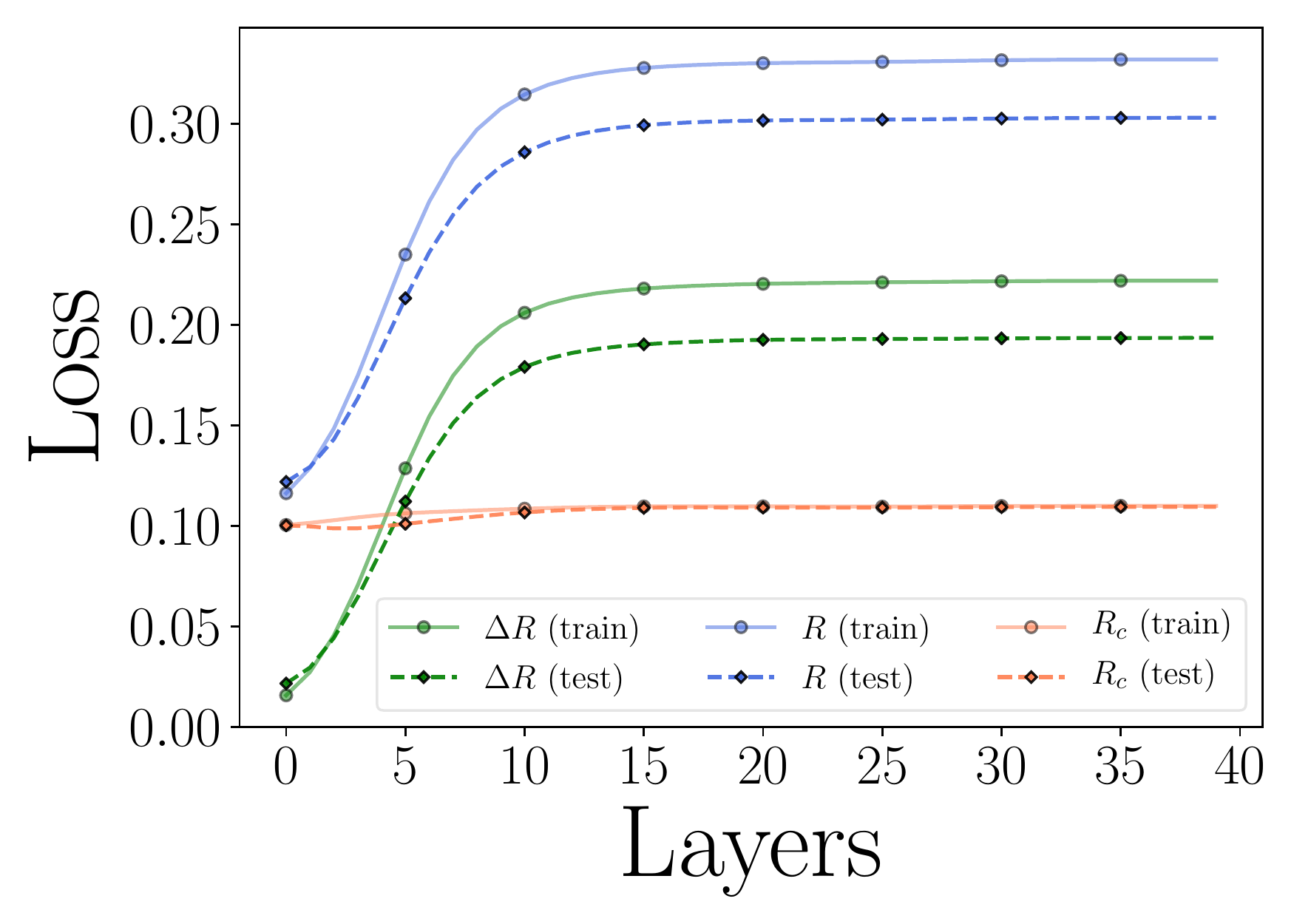}
    }
    \vskip -0.1in
    \subfigure[\label{fig:1d-invariance-plots-e}$\X_{\text{shift}}$ (TI-MNIST)]{
    \includegraphics[width=0.22\textwidth]{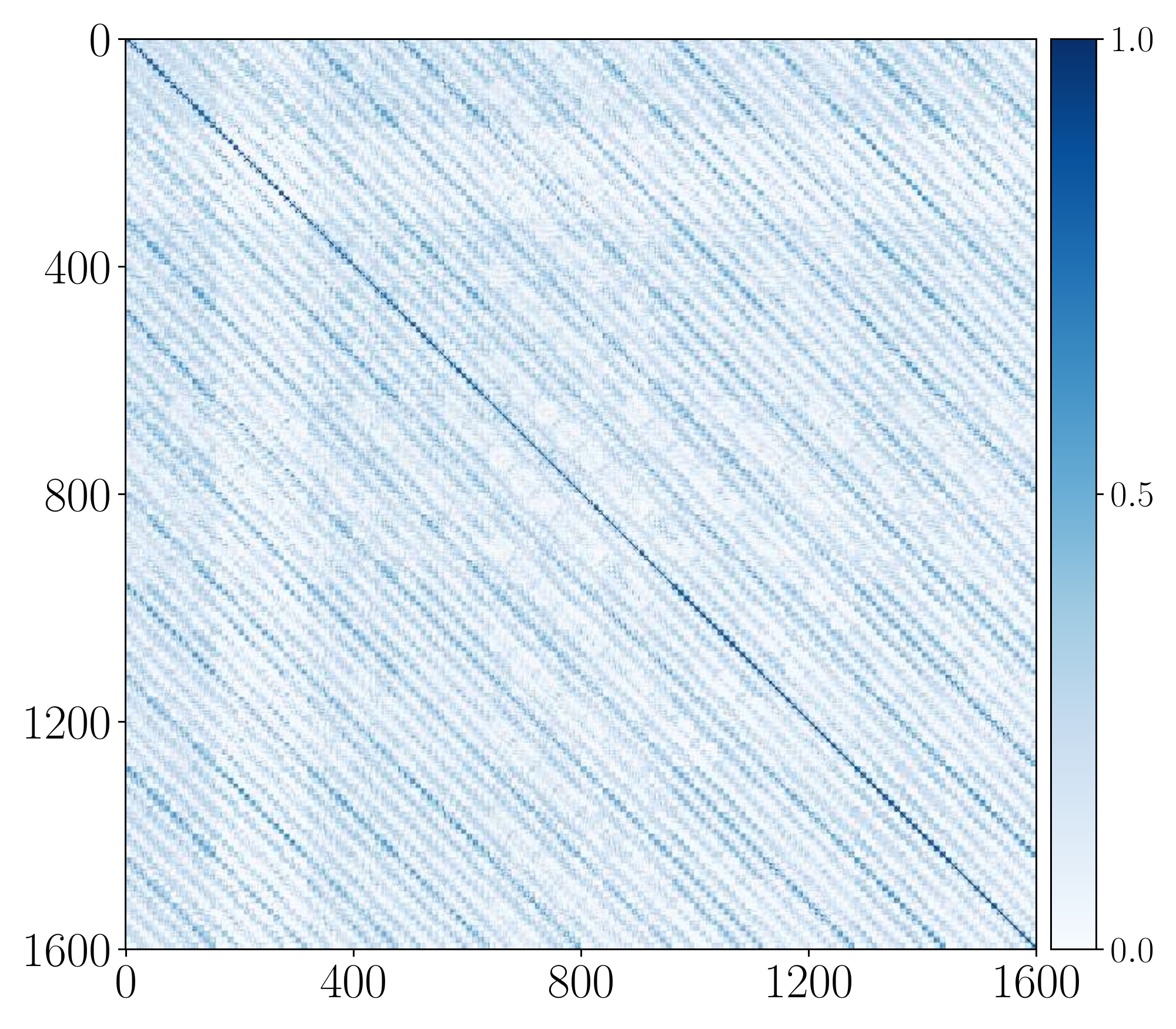}
    }
    \subfigure[\label{fig:1d-invariance-plots-f}$\bar{\Z}_{\text{shift}}$ (TI-MNIST)]{
    \includegraphics[width=0.22\textwidth]{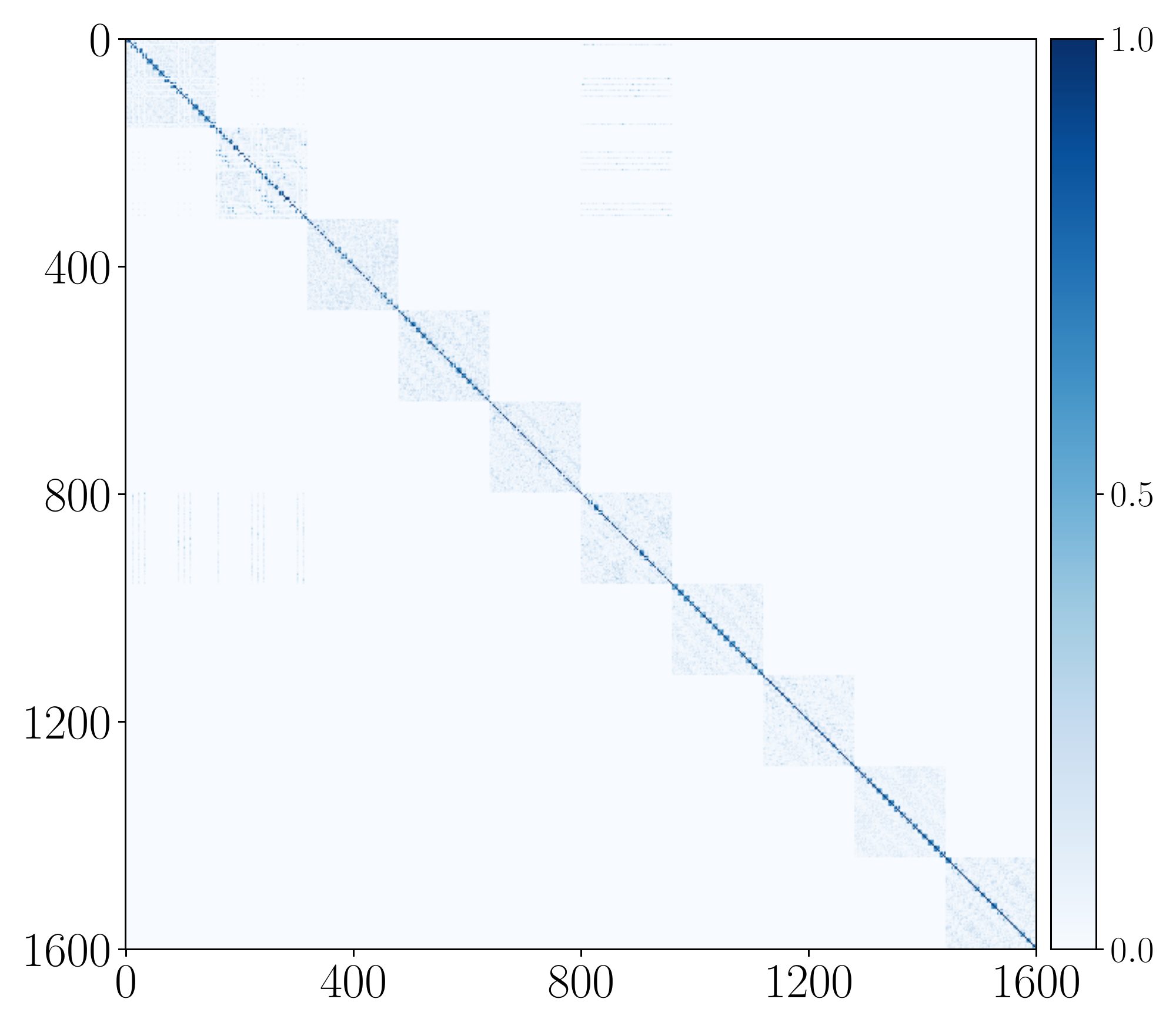}
    }
    \subfigure[\label{fig:1d-invariance-plots-g}Similarity (TI-MNIST)]{
    \includegraphics[width=0.23\textwidth]{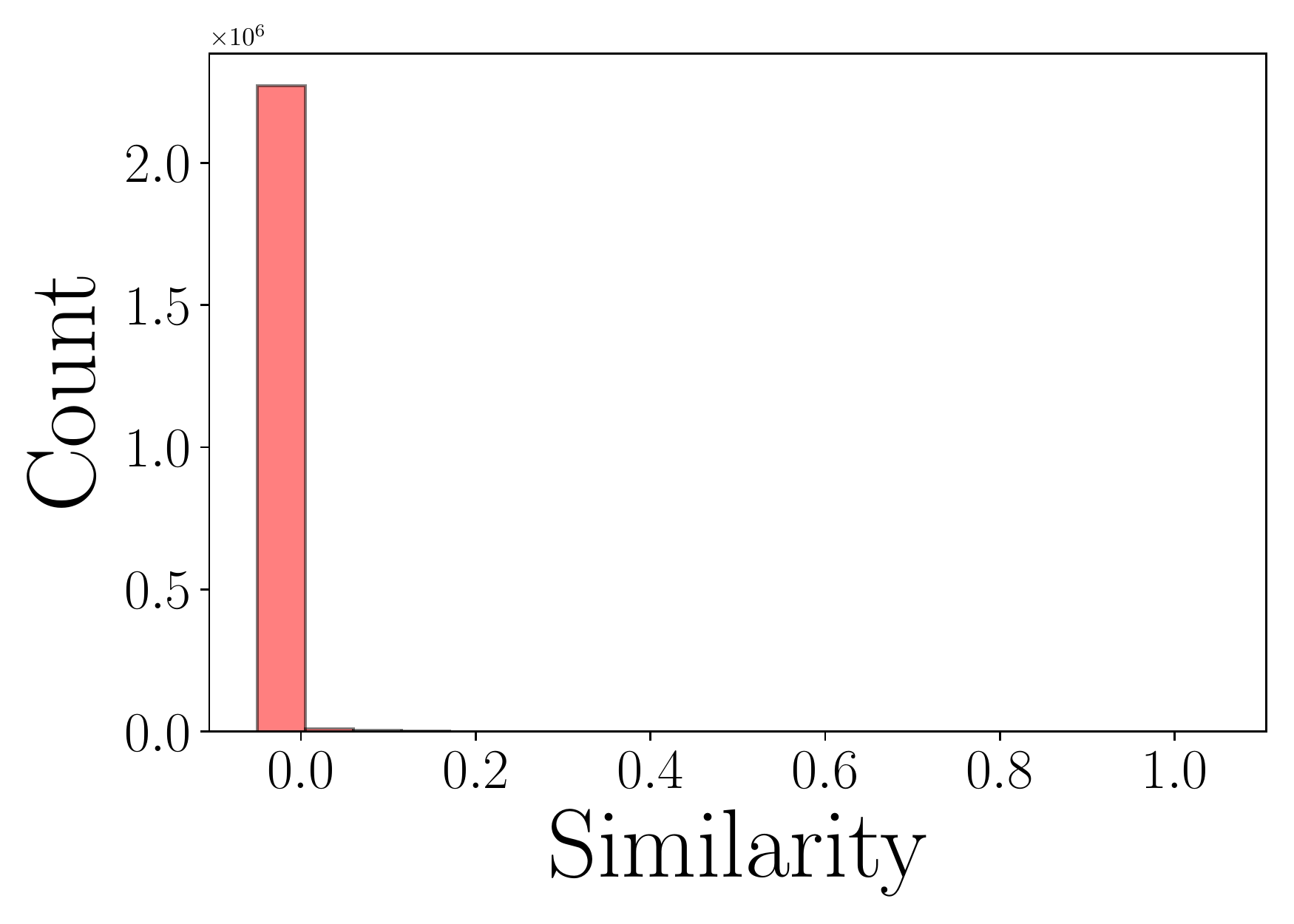}
    }
    \subfigure[\label{fig:1d-invariance-plots-h}Loss (TI-MNIST)]{
    \includegraphics[width=0.23\textwidth]{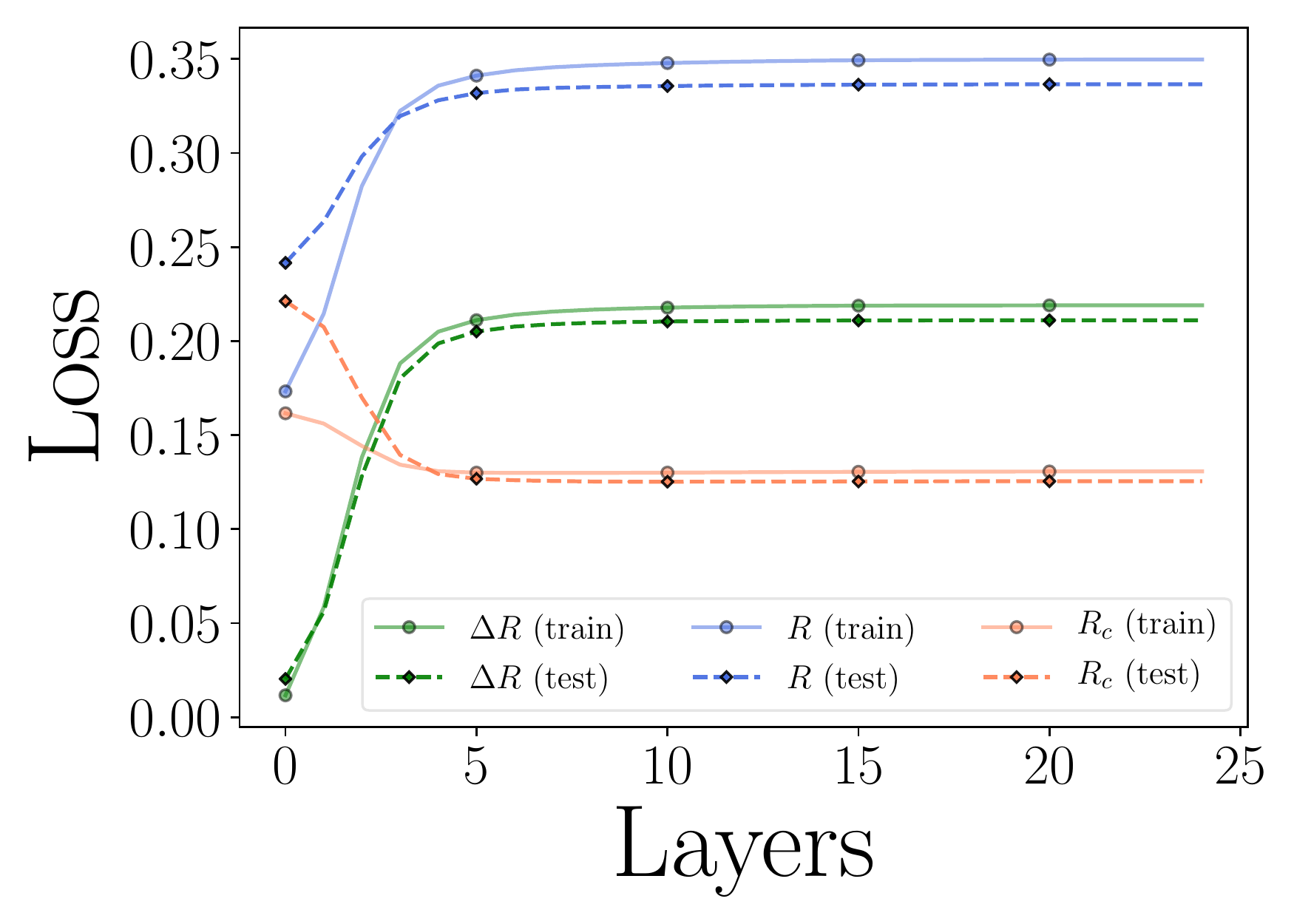}
    }
    \vskip -0.1in
    \caption{\small (a)(b) and (e)(f) are heatmaps of cosine similarity among shifted training data $\X_{\text{shift}}$ and learned features $\bar{\Z}_{\text{shift}}$, for rotation and translation invariance respectively.   
    (c)(g) are histograms of the cosine similarity (in absolute value) between all pairs of features across different classes: for each pair, one sample is from the training dataset (including all shifts) and one sample is from another class in the test dataset (including all possible shifts). There are $4\times 10^{6}$ pairs for the rotation~(c) and $2.56\times 10^{6}$ pairs for the translation~(g).
    }\label{fig:1d-invariance-plots}
  \end{center}
  \vspace{-0.5in}
\end{figure*}

To evaluate the learned representation, each training sample is augmented by 20 of its rotated version, each shifted with stride=10. We compute the cosine similarities among the $m \times 20$ augmented training inputs $\X_{\text{shift}}$ and the results are shown in Figure~\ref{fig:1d-invariance-plots-a}. We compare the cosine similarities among the learned features of all the augmented versions, i.e., $\bar{\Z}_{\text{shift}}$ and summarize the results in  Figure~\ref{fig:1d-invariance-plots-b}. As we see, the so constructed rotation-invariant ReduNet is able to  map the training data (as well as all its rotated versions) from the 10 different classes into 10 nearly orthogonal subspaces. That is, the learnt subspaces are truly invariant to shift transformation in polar angle.  Next, we randomly draw another $100$ test samples followed by the same augmentation procedure. We compute the cosine similarity histogram between  features of all shifted training and those of the new test samples in Figure~\ref{fig:1d-invariance-plots-c}. In Figure~\ref{fig:1d-invariance-plots-d}, we visualize the MCR$^{2}$ loss on the  $\ell$-th layer representation of the ReduNet on the training and test dataset. From  Figure~\ref{fig:1d-invariance-plots-c} and Figure~\ref{fig:1d-invariance-plots-d}, we can find that the constructed ReduNet is indeed able to maximize the MCR$^{2}$ loss as well as generalize to the test data.

\vspace{-0.05in}
\paragraph{2D Translation Invariance on MNIST Digits.} 
In this part, we provide experimental results to use the invariant ReduNet to learn representations for images that are invariant to the 2D cyclic translation. Essentially we view the image as painted on a torus and can be translated arbitrarily, as illustrated in Figure \ref{fig:appendix-mnist-translation-visualize} in the Appendix. Again we use the 10-class MNIST dataset.  We use $m=100$ training samples (10 samples from each class) for constructing the ReduNet and use another $100$ samples ($10$ samples from each class) as the test dataset. We apply 2D circulant convolution to the (1-channel) inputs with 75 random Gaussian kernels of size $9\times 9$ before the first layer of ReduNet.
For the translation-invariant ReduNet, we set $L=25$, step size $\eta=0.5$, precision $\epsilon=0.1$.  Similar to the rotational invariance task,
to evaluate the performance of ReduNet with regard to translation, we augment each training/testing sample by 16 of its translational shifted version (with stride=7). The heatmap of the similarities among the $m\times 16$ augmented training inputs $\X_{\text{shift}}$ and the learned features $\circm(\bar\Z)_{\text{shift}}$ are compared in Figure~\ref{fig:1d-invariance-plots-e} and Figure~\ref{fig:1d-invariance-plots-f}. Similar to the rotation case, Figure~\ref{fig:1d-invariance-plots-g} shows the histogram of similarity between features of all shifted training samples and test samples. Clearly, the ReduNet can map training samples from the 10 classes to 10 nearly orthogonal subspaces and invariant to all possible 2D translations on the training dataset. Also, the MCR$^2$ loss in Figure~\ref{fig:1d-invariance-plots-h} is increasing with the increased layer. This verifies that the proposed ReduNet can indeed maximize the objective and be invariant to transformations as it is designed to.

\begin{table}[t]
\begin{center}
\begin{small}
\begin{sc}
\begin{tabular}{ccc}
\toprule
Initialization & Backpropagation & Test Accuracy\\
\midrule
\cmark & \xmark & 0.898\\
\xmark & \cmark & 0.932\\
\cmark & \cmark & 0.978\\
\bottomrule
\end{tabular}
\end{sc}
\end{small}
\caption{\small Test accuracy of 2D translation-invariant ReduNet, ReduNet-bp (without initialization), and ReduNet-bp (with initialization) on the MNIST dataset.}
\label{table:backprop_acc}
\end{center}
\vskip -0.2in
\end{table}

\paragraph{Back Propagation on ReduNet.} 
Once the $\{\bm E_{\ell}\}_{\ell=1}^{L}$ and $\{\bm{C}_{\ell}^{1}, \dots, \bm{C}_{\ell}^{k} \}_{\ell=1}^{L}$ of the ReduNet are constructed, we here further test whether the architecture is amenable to fine-tuning via back propagation. For the 2D translation-invariant ReduNet, we add a fully-connected layer with weight having dimensions $CHW \times k$ after the last layer of the ReduNet, where $k$ is the number of classes. We denote this modified architecture as \textit{ReduNet-bp}. 
We consider the 10-class classification problem on the MNIST dataset, and compare three networks: (1). ReduNet by the forward construction, (2). ReduNet-bp initialized by the construction, (3). ReduNet-bp with random initialization using the same backbone architecture. To initialize the networks, we use $m=500$ samples ($50$ from each class) and set number of layers/iterations $L=30$, step size $\eta = 0.5$, precision $\epsilon = 0.1$. Before the first layer, we perform 2D circulant-convolution to the inputs with $16$ channels with $7 \times 7$ random Gaussian kernels. For ReduNet-bp, we further train the constructed ReduNet via back propagation by adding an extra fully connected layer and using cross-entropy loss. We update the model parameters by using SGD to minimize the loss on the entire MNIST training data. To evaluate the three networks, we compare the standard test accuracy on $10,000$ MNIST test samples. For ReduNet, we apply the nearest subspace classifier for prediction. For ReduNet-bp with and without initialization, we take the argmax of the final fully connected layer as the prediction. The results are summarized in Table~\ref{table:backprop_acc}. We find that (1). The ReduNet architecture \textit{can} be optimized by SGD and achieve better standard accuracy after back propagation; (2). Using constructed ReduNet for initialization can achieve better performance compared with the same architecture with random initialization.

\section{Conclusions and Discussions}
In this paper, we have laid out a theoretical and computational framework based on data compression which allows us to understand and interpret not only the characteristics of modern deep networks but also reveal their purposes and functions as a white box. From this new perspective, we see that at a high level, the objective of  seeking a linear discriminative representation (via deep learning)  aligns well with the objectives of the classic {\em linear discriminant analysis} \citep{HastieTiFr09}, {\em independent component analysis} \citep{HYVARINEN2000411}, and {\em generalized principal component analysis} \citep{GPCA}. The main difference is that now we are able to conduct such analyses through the lens of a constructive nonlinear mapping and an intrinsic measure. This renders all these analyses so much more general hence practical for real-world data. 

Despite the long history of practicing artificial (deep) neural networks since their inception in 1940-1950's \citep{McCulloch1943ALC,Rosenblatt1958ThePA}, the architectures and operators in deep networks have been mainly proposed or designed empirically and trained via back propagation as a black box \citep{Back-Prop}. Table \ref{tab:comparison} left column summarizes the main characteristics of the conventional practice of deep networks,\footnote{There are exceptions such as ScatteringNets \citep{scattering-net,Wiatowski-2018}, whose operators are pre-designed and fixed, as also  indicated in the table.} whereas the right column highlights comparison of our new compression based framework to the current practice of deep neural networks, on which we will elaborate a little more below.

\begin{table}[t]
\begin{center}
\begin{tabular}{| l || c | c |}
\hline
  & Conventional DNNs & Compression (ReduNets) \\ [0.5ex]
  \hline \hline
Objectives & input/output fitting & rate reduction \\ [0.5ex]
  \hline
Deep architectures & trial \& error & iterative optimization \\  [0.5ex]
\hline
Layer operators & empirical & projected gradient \\  [0.5ex]
\hline
Shift invariance & CNNs + augmentation & invariant ReduNets \\  [0.5ex]
\hline
Initializations & random/pre-designed & forward computed \\ [0.5ex]
\hline
Training/fine-tuning & back prop/fixed & forward/back prop\\ [0.5ex]
\hline
Interpretability & black box & white box \\ [0.5ex]
\hline
Representations & unknown or unclear & incoherent subspaces (LDR) \\ [0.5ex]
\hline
\end{tabular}
\end{center}
\caption{Comparison between conventional deep networks and compression based ReduNets.}\label{tab:comparison}
\end{table}

\paragraph{White box versus black box.}
This new framework offers a constructive approach to derive deep (convolution) networks entirely as a white box from the objective of learning a low-dimensional linear discriminative representation for the given (mixed) data. The goodness of the representation is measured by the intrinsic rate reduction. The resulting network, called the ReduNet, emulates a gradient-based iterative scheme to optimize the rate reduction objective. It shares almost all the main structural characteristics of modern deep networks. Nevertheless, its architectures, linear and nonlinear operators and even their values are all derived from the data and all have precise geometric and statistical interpretation. In particular, we find it is rather intriguing that the linear operator of each layer has a non-parameteric ``data auto-regression'' interpretation. Together with the ``forward construction,'' they give rather basic but universal computing mechanisms that even simple organisms/systems can use to learn good representations from the data that help future classification tasks (e.g. object  detection or recognition).

\paragraph{Forward versus backward optimization and refinement.}
This constructive and white-box approach has demonstrated potential in liberating the practice of deep networks from relying (almost entirely) on random initialization and back propagation of all network parameters. It offers effective mechanisms to construct (hence initialize) deep networks in a forward fashion. The forward-constructed ReduNet already exhibits descent classification performance. Preliminary experiments given in this paper indicate  that the ReduNet architecture is amenable to fine-tuning via back propagation too (say with new training data), and initialization with the forward-constructed network has advantages over random initialization. Furthermore, the constructive nature of ReduNet makes it amenable to other fine-tuning schemes such as forward propagation or incremental learning. Since one no longer has to update all network parameters simultaneously as a black box, one can potentially avoid the so-called ``catastrophic forgetting'' \citep{catastrophic} in the sequential or {\em incremental learning} setting, as the recent work of \cite{Wu-CVPR2021} suggests. 

\paragraph{Modeling invariance and equivariance.}
In the case of seeking classification invariant to certain transformation groups (say translation), the new framework models both  eqvuivariance and invariance in a very natural way: all equivariant instances are mapped into the same subspace and the resulting subspaces are hence invariant to the transformation. As we have shown, in this case, the ReduNet naturally becomes a deep convolution network. Arguably, this work gives a new constructive and explicit justification for the role of multi-channel convolutions in each layer (widely adopted in modern CNNs) as incremental operators to compress or expand all equivariant instances for learning an LDR for the data.
Moreover, our derivation reveals the fundamental computational advantage in constructing and learning such multi-channel convolutions in the spectral domain. Simulations and experiments on synthetic and real data sets clearly verify such forward-constructed ReduNet can be invariant for {\em all} transformed instances. The computation scales gracefully with the number of classes, channels and sample dimension/size. 

\paragraph{Invariance and sparsity.}
The new framework also reveals a {\em fundamental trade-off} between sparsity and invariance: essentially one cannot expect to separate different classes of signals/data if signals in each class can be both arbitrarily shifted and arbitrarily superimposed. To achieve invariance to all translation, one must impose that signals in each class are sparsely generated so that all shifted samples span a proper submanifold (in the high-dimensional space) and features can be mapped to a proper subspace. Although sparse representation for individual signals have been extensively studied and well understood in the literature \citep{Wright-Ma-2021}, very little is yet known about how to characterize the distribution of sparse codes of a class of (equivariant or locally equivariant) signals and its separability from other classes. This fundamental trade-off between sparsity and invariance certainly merits further theoretical study, as it will lead to more precise characterization of the statistical resource (network width) and computational resource (network depth) needed to provide performance guarantees, say for (multi-manifold) classification \citep{buchanan2020deep}.

\paragraph{Further improvements and extensions.}
We believe the proposed rate reduction provides a principled framework for designing new networks with interpretable architectures and operators that can provide performance guarantees (say invariance) when applied to real-world datasets and problems. Nevertheless, the purposes of this paper are to introduce the basic principles and concepts of this new framework. We have chosen arguably the simplest and most basic gradient-based scheme to construct the ReduNet for optimizing the rate reduction. As we have touched upon briefly in the paper, many powerful ideas from optimization can be further applied to improve the efficiency and performance of the network, such as acceleration, precondition/normalization, and regularization. Also there are additional relationships and structures among the channel operators $\bm E$ and $\bm C^j$ that have not been exploited in this work for computational efficiency. 

The reader may have realized that the basic ReduNet is expected to work when the data have relatively benign nonlinear structures when each class is close to be a linear subspace or Gaussian distribution. Real data (say image classes) have much more complicated structures: each class can have highly nonlinear geometry and topology or even be multi-modal itself. Hence in practice, to learn a better LDR, one may have to resort to more sophisticated strategies to control the compression (linearization) and expansion process. Real data also have additional priors and data structures that can be exploited. For example, one can reduce the dimension of the ambient feature space whenever the intrinsic dimension of the features are low (or sparse) enough. Such dimension-reduction operations (e.g. pooling or striding) are widely practiced in modern deep (convolution) networks for both computational efficiency and even accuracy. According to the theory of compressive sensing, even a random projection would be rather efficient and effective in preserving the (discriminative)  low-dimensional structures \citep{Wright-Ma-2021}.

In this work, we have mainly considered learning a good representation $\bm Z$ for the data $\X$ when the class label $\bm \Pi$ is given and fixed. Nevertheless, notice that in its most general form, the maximal rate reduction objective can be optimized against both the representation $\Z$ and the membership $\bm \Pi$. In fact, the original work of \cite{ma2007segmentation} precisely studies the complementary problem of learning $\bm \Pi$ by maximizing $\Delta R(\Z, \bm{\Pi}, \epsilon)$ with $\Z$ fixed, hence equivalent to minimizing only the compression term: $ \min_{\bm \Pi}R_c(\Z, \epsilon \mid  \bm{\Pi})$. Therefore, it is obvious that this framework can be  naturally extended to {\em unsupervised} settings if the membership $\bm \Pi$ is partially known or entirely unknown and it is to be optimized together with the representation $\bm Z$. In a similar vein to the construction of the ReduNet, this may entail us to examine the joint dynamics of the gradient of the representation $\Z$ and the membership $\bm \Pi$:
\begin{equation}
\dot{\bm Z} = \eta \cdot \frac{\partial \Delta R}{\partial \bm Z}, \quad \dot{\bm \Pi} = \gamma \cdot \frac{\partial \Delta R}{\partial \bm \Pi}.
\end{equation}

Last but not the least, in this work, we only considered data that are naturally embedded (as submanifolds) in a vector space (real or complex). There have been many work that study and apply deep networks to data with additional structures or in a non-Euclidean space. For example, in reinforcement learning and optimal control, people often use deep networks to process data with additional dynamical structures, say linearizing the dynamics  \citep{koopman}. In computer graphics or many other fields, people deal with data on a non-Euclidean domain such as a mesh or a graph  \citep{Geometric-DNN}. It remains interesting to see how the principles of {\em data compression} and {\em linear discriminative representation} can be extended to help study or design principled white-box deep networks associated with dynamical or graphical data and problems. 

\acks{Yi would like to thank professor Yann LeCun of New York University for a stimulating discussion in his office back in November 2019 when they contemplated a  fundamental question: {\em what does or should a deep network try to optimize?} At the time, they both believed the low-dimensionality of the data (say sparsity) and discriminativeness of the representation (e.g. contrastive learning) have something to do with the answer. The conversation had inspired Yi to delve into this problem more deeply while self-isolated at home during the pandemic. 

Yi would also like to thank Dr. Harry Shum who has had many hours of conversations with Yi about how to understand and interpret deep networks during the past couple of years. In particular, Harry  suggested how to better visualize the representations learned by the rate reduction, including the results shown in Figure \ref{fig:visual-class-2-8}.

Yi acknowledges support from ONR grant N00014-20-1-2002 and the joint Simons Foundation-NSF DMS grant \#2031899, as well as support from Berkeley FHL Vive Center for Enhanced Reality and Berkeley Center for Augmented Cognition. Chong and Yi acknowledge support from Tsinghua-Berkeley Shenzhen Institute (TBSI) Research Fund. Yaodong, Haozhi, and Yi acknowledge support from Berkeley AI Research (BAIR). John acknowledges support from NSF grants 1838061, 1740833, and 1733857.}

\newpage

\appendix

\section{Properties of the Rate Reduction Function}\label{ap:rate-reduction}

This section is organized as follows. 
We present background and preliminary results for the $\log\det(\cdot)$ function and the coding rate function in Section~\ref{sec:theory-preliminary}. 
Then, Section~\ref{sec:theory-bounds-rate} and \ref{sec:theory-bounds-rate-reduction} provide technical lemmas for bounding the coding rate and coding rate reduction functions, respectively. 
Finally, these lemmas are used to prove our main theoretical results about the properties of the rate reduction function. The main results (given informally as Theorem \ref{thm:MCR2-properties} in the main body) are stated formally in Section~\ref{sec:theory-main} and a proof is given in Section~\ref{sec:theory-proof}. 

\paragraph{Notations} Throughout this section, we use $\bbS_{++}^n$, $\Re_+$ and $\mathbb Z_{++}$ to denote the set of symmetric positive definite matrices of size $n\times n$, nonnegative real numbers and positive integers, respectively.

\subsection{Preliminaries}
\label{sec:theory-preliminary}
\paragraph{Properties of the $\log\det(\cdot)$ function. }

\begin{lemma}\label{thm:logdet-strictly-concave}
The function $\log\det(\cdot): \bbS_{++}^n \to \R$ is strictly concave. That is, 
\begin{equation*}
    \log\det((1-\beta) \Z_1 + \beta \Z_2)) \ge (1-\beta)\log\det(\Z_1) + \beta\log\det(\Z_2)
\end{equation*}
for any $\beta \in (0, 1)$ and $\{\Z_1, \Z_2\} \subseteq \bbS_{++}^n$, with equality holds if and only if $\Z_1 = \Z_2$.
\begin{proof}
Consider an arbitrary line given by $\Z = \Z_0 + t \Delta\Z$ where $\Z_0$ and $\Delta\Z \ne \0$ are symmetric matrices of size $n\times n$. 
Let $f(t) \doteq \log\det(\Z_0 + t \Delta\Z)$ be a function defined on an interval of values of $t$ for which $\Z_0 + t\Delta\Z \in \bbS_{++}^n$.
Following the same argument as in \cite{boyd2004convex}, we may assume $\Z_0 \in \bbS_{++}^n$ and get
\begin{equation*}
    f(t) = \log\det \Z_0 + \sum_{i=1}^n \log(1+ t\lambda_i),
\end{equation*}
where $\{\lambda_i\}_{i=1}^n$ are eigenvalues of $\Z_0^{-\frac{1}{2}}\Delta\Z \Z_0^{-\frac{1}{2}}$. 
The second order derivative of $f(t)$ is given by
\begin{equation*}
    f''(t) = -\sum_{i=1}^n \frac{\lambda_i^2}{(1+t\lambda_i)^2} < 0.
\end{equation*}
Therefore, $f(t)$ is strictly concave along the line $\Z = \Z_0 + t\Delta\Z$. 
By definition, we conclude that $\log\det(\cdot)$ is strictly concave.
\end{proof}
\end{lemma}

\paragraph{Properties of the coding rate function. }
The following properties, also known as the Sylvester's determinant theorem, for the coding rate function are known in the paper of  \cite{ma2007segmentation}.
\begin{lemma}[Commutative property]\label{thm:coding-rate-commute}
For any $\Z \in \R^{n\times m}$ we have 
\begin{equation*}
    R(\Z,\epsilon) \doteq \frac{1}{2}\log \det\left(\I + \frac{n}{m\epsilon^{2}}\Z\Z^{*}\right) =\frac{1}{2}\log\det\left(\I + \frac{n}{m\epsilon^{2}}\Z^{*}\Z\right).
\end{equation*}
\end{lemma}

\begin{lemma}[Invariant property]\label{thm:coding-rate-invariant}
For any $\Z \in \R^{n\times m}$ and any orthogonal matrices $\U\in \R^{n\times n}$ and $\V \in \R^{m\times m}$ we have 
\begin{equation*}
    R(\Z,\epsilon) = R(\U\Z\V^*,\epsilon).
\end{equation*}
\end{lemma}

\subsection{Lower and Upper Bounds for Coding Rate}
\label{sec:theory-bounds-rate}

The following result provides an upper and a lower bound on the coding rate of $\Z$ as a function of the coding rate for its components $\{\Z^j\}_{j=1}^k$. 
The lower bound is tight when all the components $\{\Z^j\}_{j=1}^k$ have the same covariance (assuming that they have zero mean).
The upper bound is tight when the components $\{\Z^j\}_{j=1}^k$ are pair-wise orthogonal.

\begin{lemma}\label{thm:coding-rate-bounds}
For any $\{\Z^j \in \Re^{n\times m_j}\}_{j=1}^k$ and any $\epsilon > 0$, let $\Z = [\Z^1, \cdots, \Z^k] \in \Re^{n\times m}$ with $m=\sum_{j=1}^k m_j$. We have
\begin{equation}\label{eq:coding-rate-bounds}
\begin{split}
\sum_{j=1}^k \frac{m_j}{2}  \log \det\left(\I + \frac{n}{m_j\epsilon^2}\Z^j (\Z^j)^*\right)
&\le  
\frac{m}{2} \log\det\left(\I + \frac{n}{m\epsilon^2}\Z \Z^*\right)  \\
&\le 
\sum_{j=1}^k \frac{m}{2}  \log \det\left(\I + \frac{n}{m\epsilon^2}\Z^j (\Z^j)^*\right),
\end{split}
\end{equation}
where the first equality holds if and only if
$$\frac{\Z^1 (\Z^1)^*}{m_1} = \frac{\Z^2 (\Z^2)^*}{m_2}=\cdots= \frac{\Z^k (\Z^k)^*}{m_k},$$
and the second equality holds if and only if $(\Z^{j_1})^* \Z^{j_2} = \0$ for all $1 \le j_1 < j_2 \le k$.
\end{lemma}
\begin{proof}
By Lemma~\ref{thm:logdet-strictly-concave}, $\log \det(\cdot)$ is strictly concave. Therefore,
\begin{align*}
    \log\det\Big(\sum_{j=1}^k \beta_j \S^j\Big) \ge \sum_{j=1}^k \beta_j \log\det(\S^j), ~\text{for all}~ \{\beta_j > 0\}_{j=1}^k, \sum_{j=1}^k \beta_j = 1 ~\text{and}~\{\S^j \in \bbS_{++}^n\}_{j=1}^k,
\end{align*}
where equality holds if and only if $\S^1 = \S^2 = \cdots = \S^k$. Take $\beta_j = \frac{m_j}{m}$ and $\S^j = \I + \frac{n}{m_j \epsilon^2} \Z^j(\Z^j)^*$, we get
\begin{equation*}
     \frac{m}{2} \log\det\left(\I + \frac{n}{m\epsilon^2}\Z \Z^*\right) 
     \ge
     \sum_{j=1}^k \frac{m_j}{2}  \log \det\left(\I + \frac{n}{m_j\epsilon^2}\Z^j (\Z^j)^*\right),
\end{equation*}
with equality holds if and only if $\frac{\Z^1(\Z^1)^*}{m_1} = \cdots = \frac{\Z^k(\Z^k)^*}{m_k}$. 
This proves the lower bound in \eqref{eq:coding-rate-bounds}. 

We now prove the upper bound. 
By the strict concavity of $\log\det(\cdot)$, we have
\begin{equation*}
    \log\det(\Q) \le \log\det(\S) + \langle \nabla \log\det(\S), \,\Q - \S\rangle, ~\text{for all}~\{\Q, \S\}\subseteq \bbS_{++}^{m},
\end{equation*}
where equality holds if and only if $\Q = \S$. 
Plugging in $\nabla \log\det(\S) = \S^{-1}$ (see e.g., \cite{boyd2004convex}) and $\S^{-1} = (\S^{-1})^{*}$ gives
\begin{equation}\label{eq:prf-logdet-gradient-inequality}
    \log\det(\Q) \le \log\det(\S) + \tr(\S^{-1} \Q) - m.
\end{equation}

We now take
\begin{gather}\label{eq:prf-logdet-gradient-inequality-QS}
    \Q = \I + \frac{n}{m\epsilon^2}\Z^* \Z = \I + \frac{n}{m\epsilon^2}
    \begin{bmatrix}
    (\Z^1)^* \Z^1 & (\Z^1)^* \Z^2  & \cdots & (\Z^1)^* \Z^k \\
    (\Z^2)^* \Z^1 & (\Z^2)^* \Z^2 & \cdots & (\Z^2)^* \Z^2 \\
    \vdots         & \vdots          & \ddots & \vdots \\
    (\Z^k)^* \Z^1 & (\Z^k)^* \Z^2  & \cdots & (\Z^k)^* \Z^k \\
    \end{bmatrix}, ~\text{and}~ 
    \\
    \S = \I + \frac{n}{m\epsilon^2}
    \begin{bmatrix}
    (\Z^1)^* \Z^1 & \0              & \cdots & \0 \\
    \0             & (\Z^2)^*  \Z^2 & \cdots & \0 \\
    \vdots         & \vdots          & \ddots & \vdots \\
    \0             & \0              & \cdots & (\Z^k)^* \Z^k \\
    \end{bmatrix}. \nonumber
\end{gather}
From the property of determinant for block diagonal matrix, we have
\begin{equation}\label{eq:prf-logdet-gradient-inequality-term1}
    \log\det(\S) = \sum_{j=1}^k \log\det \left(\I + \frac{n}{m\epsilon^2} (\Z^j)^* \Z^j\right).
\end{equation}
Also, note that
\begin{align}\label{eq:prf-logdet-gradient-inequality-term2}
&\tr(\S^{-1} \Q) 
\nonumber\\
= \ 
&\tr
\begin{bmatrix}
(\I +\frac{n}{m\epsilon^2}(\Z^1)^* \Z^1)^{-1}(\I +\frac{n}{m\epsilon^2}(\Z^1)^* \Z^1)  & \cdots & (\I +\frac{n}{m\epsilon^2}(\Z^1)^* \Z^1)^{-1}(\I +\frac{n}{m\epsilon^2}(\Z^1)^* \Z^k) \\
\vdots         & \ddots & \vdots \\
(\I +\frac{n}{m\epsilon^2}(\Z^k)^* \Z^k)^{-1}(\I +\frac{n}{m\epsilon^2}(\Z^k)^* \Z^1)            & \cdots & (\I +\frac{n}{m\epsilon^2}(\Z^k)^* \Z^k)^{-1}(\I +\frac{n}{m\epsilon^2}(\Z^k)^* \Z^k) \\
\end{bmatrix}
\nonumber\\
= \
&\tr     \begin{bmatrix}
\I            & \cdots   & \bigcirc  \\
\vdots        & \ddots   & \vdots \\
\bigcirc              & \cdots   & \I \\
\end{bmatrix}
= m,
\end{align}
where ``$\bigcirc$'' denotes nonzero quantities that are irrelevant for the purpose of computing the trace. 
Plugging \eqref{eq:prf-logdet-gradient-inequality-term1} and \eqref{eq:prf-logdet-gradient-inequality-term2} back in \eqref{eq:prf-logdet-gradient-inequality}
gives 
\begin{equation*}
    \frac{m}{2} \log\det\left(\I + \frac{n}{m\epsilon^2}\Z^* \Z\right) \le \sum_{j=1}^k \frac{m}{2}  \log \det\left(\I + \frac{n}{m\epsilon^2}(\Z^j)^* \Z^j\right),
\end{equation*}
where the equality holds if and only if $\Q = \S$, which by the formulation in \eqref{eq:prf-logdet-gradient-inequality-QS}, holds if and only if $(\Z^{j_1})^* \Z^{j_2} = \0$ for all $1 \le j_1 < j_2 \le k$. 
Further using the result in Lemma~\ref{thm:coding-rate-commute} gives 
\begin{equation*}
    \frac{m}{2} \log\det\left(\I + \frac{n}{m\epsilon^2}\Z \Z^*\right) \le \sum_{j=1}^k \frac{m}{2}  \log \det\left(\I + \frac{n}{m\epsilon^2}\Z^j (\Z^j)^*\right),
\end{equation*}
which produces the upper bound in \eqref{eq:coding-rate-bounds}. 
\end{proof}

\subsection{An Upper Bound on Coding Rate Reduction}
\label{sec:theory-bounds-rate-reduction}

We may now provide an upper bound on the coding rate reduction $\Delta R(\Z, \bm{\Pi}, \epsilon)$ (defined in \eqref{eqn:maximal-rate-reduction}) in terms of its individual components $\{\Z^j\}_{j=1}^k$.
\begin{lemma}\label{thm:rate-reduction-bound}
For any $\Z \in \Re^{n\times m}, \bm{\Pi} \in \Omega$ and $\epsilon > 0$, let $\Z^j \in \Re^{n\times m_j}$ be $\Z \bm{\Pi}^j$ with zero columns removed. We have
\begin{equation}\label{eq:rate-reduction-bound}
    \Delta R(\Z, \bm{\Pi}, \epsilon) \le 
    \sum_{j=1}^k \frac{1}{2m}\log\left( \frac{\det^m\left(\I + \frac{n}{m\epsilon^2}\Z^j (\Z^j)^*\right)}{\det^{m_j}\left(\I + \frac{n}{m_j\epsilon^2}\Z^j (\Z^j)^*\right)}\right),
\end{equation}
with equality holds if and only if $(\Z^{j_1})^* \Z^{j_2} = \0$ for all $1 \le j_1 < j_2 \le k$.
\end{lemma}
\begin{proof}
From the definition of $\Delta R(\Z, \bm{\Pi}, \epsilon)$ in Eq.~\eqref{eqn:maximal-rate-reduction}, 
we have
\begin{equation*}
\begin{split}
&\quad \,\, \Delta R(\Z, \bm{\Pi}, \epsilon) \\
&= R(\Z, \epsilon) - R_c(\Z, \epsilon \mid  \bm{\Pi})\\
&= \frac{1}{2}\log \left(\det\left(\I + \frac{n}{m\epsilon^{2}}\Z\Z^{*}\right)\right) - \sum_{j=1}^{k}\left\{\frac{\tr(\bm{\Pi}^j)}{2m}\log \left(\det\left(\I + n\frac{\Z\bm{\Pi}^j\Z^{*}}{\tr(\bm{\Pi}^j)\epsilon^{2}}\right)\right)\right\}\\
&= \frac{1}{2}\log \left(\det\left(\I + \frac{n}{m\epsilon^{2}}\Z\Z^{*}\right)\right) - \sum_{j=1}^{k}\left\{\frac{m_j}{2m}\log \left(\det\left(\I + n\frac{\Z^j(\Z^j)^{*}}{m_j\epsilon^{2}}\right)\right)\right\}\\
&\le \sum_{j=1}^k \frac{1}{2}  \log\left( \det\left(\I + \frac{n}{m\epsilon^2}\Z^j (\Z^j)^*\right)\right) - \sum_{j=1}^{k}\left\{\frac{m_j}{2m}\log \left(\det\left(\I + n\frac{\Z^j (\Z^j)^{*}}{m_j\epsilon^{2}}\right)\right)\right\}\\
&= \sum_{j=1}^k \frac{1}{2m}  \log\left(\det {\!}^m\left(\I + \frac{n}{m\epsilon^2}\Z^j (\Z^j)^*\right)\right) - \sum_{j=1}^{k}\left\{\frac{1}{2m}\log \left(\det{\!}^{m_j}\left(\I + n\frac{\Z^j (\Z^j)^{*}}{m_j\epsilon^{2}}\right)\right)\right\}\\
&= \sum_{j=1}^k \frac{1}{2m}\log\left( \frac{\det^m\left(\I + \frac{n}{m\epsilon^2}\Z^j (\Z^j)^*\right)}{\det^{m_j}\left(\I + \frac{n}{m_j\epsilon^2}\Z^j (\Z^j)^*\right)}\right),
\end{split}
\end{equation*}
where the inequality follows from the upper bound in Lemma~\ref{thm:coding-rate-bounds}, and that the equality holds if and only if $(\Z^{j_1})^{*} \Z^{j_2} = \0$ for all $1 \le j_1 < j_2 \le k$.
\end{proof}

\subsection{Main Results: Properties of Maximal Coding Rate Reduction}
\label{sec:theory-main}

We now present our main theoretical results.  
The following theorem states that for any fixed encoding of the partition $\bm{\Pi}$, the coding rate reduction is maximized by data $\Z$ that is maximally discriminative between different classes and is diverse within each of the classes. 
This result holds provided that the sum of rank for different classes is small relative to the ambient dimension, and that $\epsilon$ is small. 

\begin{theorem}\label{thm:maximal-rate-reduction}
Let $\bm{\Pi} = \{\bm{\Pi}^j \in \Re^{m \times m}\}_{j=1}^{k}$ with $\{\bm{\Pi}^j \ge \mathbf{0}\}_{j=1}^k$ and \, $\bm{\Pi}_1 + \cdots + \bm{\Pi}_k = \I$ be a given set of diagonal matrices whose diagonal entries encode the membership of the $m$ samples in the $k$ classes.
Given any $\epsilon > 0$, $n > 0$ and  $\{n \ge d_j>0\}_{j=1}^k$, consider the optimization problem
\begin{equation}\label{eq:maximal-rate-reduction-thm}
\begin{split}
    \Z_\star \in &\argmax_{\Z\in \Re^{n\times m}} \Delta R(\Z, \bm{\Pi}, \epsilon) \\ 
    & \ \text{s.t.}~\|\Z\bm{\Pi}^j\|_F^2 = \tr({\bm{\Pi}^j}), \ \rank(\Z\bm{\Pi}^j) \le d_j, \ \forall j \in \{1, \ldots, k\}.
\end{split}
\end{equation}
Under the conditions 
\begin{itemize}
\item \emph{(Large ambient dimension)} $n \ge \sum_{j=1}^k d_j$, and
\item \emph{(High coding precision)} $\epsilon ^4 < \min_{j \in \{1, \ldots, k\}}\left\{\frac{\tr({\bm{\Pi}}^j)}{m}\frac{n^2}{d_j^2}\right\}$,
\end{itemize}
the optimal solution $\Z_\star$ satisfies
\begin{itemize}
    \item \emph{(Between-class discriminative)} $(\Z^{j_1}_\star)^* \Z^{j_2}_\star = \0$ for all $1 \le j_1 < j_2 \le k$, i.e., $\Z^{j_1}_\star$ and $\Z^{j_2}_\star$ lie in orthogonal subspaces, and
    \item \emph{(Within-class diverse)} For each $j \in \{1, \ldots, k\}$, the rank of $\Z^j_\star$ is equal to $d_j$ and either all singular values of $\Z^j_\star$ are equal to $\frac{\tr({\bm{\Pi}^j})}{d_j}$, or the $d_j -1$ largest singular values of $\Z^j_\star$ are equal and have value larger than $\frac{\tr({\bm{\Pi}^j})}{d_j}$,
\end{itemize}
where $\Z^j_\star \in \Re^{n\times \tr{(\bm{\Pi}}^j)}$ denotes $\Z_\star \bm{\Pi}^j$ with zero columns removed. 
\end{theorem}

\subsection{Proof of Main Results}
\label{sec:theory-proof}

We start with presenting a lemma that will be used in the proof to Theorem~\ref{thm:maximal-rate-reduction}.
\begin{lemma}\label{thm:generic-simplex-optimization}
Given any twice differentiable $f: \Re_+ \to \Re$, integer $r \in \mathbb{Z}_{++}$ and $c \in \Re_+$, consider the optimization problem
\begin{equation}
\label{eq:generic-simplex-optimization}
\begin{split}
    &\max_{\x} \ \sum_{p=1}^r f(x_p)  \\
    & \ \ \text{s.t.} \ \x=[x_1, \ldots, x_r] \in \Re_+^r, \ x_1 \ge x_2 \ge \cdots \ge x_r, \ \text{and} \ \sum_{p=1}^r x_p = c. 
\end{split}
\end{equation}
Let $\x_\star$ be an arbitrary global solution to \eqref{eq:generic-simplex-optimization}.
If the conditions
\begin{itemize}
    \item $f'(0) < f'(x)$ for all $x > 0$,
    \item There exists $x_T > 0$ such that $f'(x)$ is strictly increasing in $[0, x_T]$ and strictly decreasing in $[x_T, \infty)$,
    \item $f''(\frac{c}{r}) < 0$ (equivalently, $\frac{c}{r} > x_T$),
\end{itemize}
are satisfied, then we have either
\begin{itemize}
    \item $\x_\star = [\frac{c}{r}, \ldots, \frac{c}{r}]$, or
    \item $\x_\star = [x_H, \ldots, x_H, x_L]$ for some $x_H \in (\frac{c}{r}, \frac{c}{r-1})$ and $x_L > 0$.
\end{itemize} 
\end{lemma}
\begin{proof}
The result holds trivially if $r = 1$. Throughout the proof we consider the case where $r > 1$.

We consider the optimization problem with the inequality constraint $x_1 \ge \cdots \ge x_r$ in \eqref{eq:generic-simplex-optimization} removed:
\begin{equation}\label{eq:prf-generic-simplex-optimization}
\max_{\x=[x_1, \ldots, x_r] \in \Re_+^r} \ \sum_{p=1}^r f(x_p)  ~~~~\text{s.t.}~ \sum_{p=1}^r x_p = c.
\end{equation}
We need to show that any global solution $\x_\star = [(x_1)_\star, \ldots, (x_r)_\star]$ to \eqref{eq:prf-generic-simplex-optimization} is either $\x_\star = [\frac{c}{r}, \ldots, \frac{c}{r}]$ or $\x_\star = [x_H, \ldots, x_H, x_L]\cdot \bm{P}$ for some $x_H > \frac{c}{r}$, $x_L > 0$ and permutation matrix $\bm{P} \in \Re^{r \times r}$.
Let 
\begin{equation*}
    \cL(\x, \blambda) = \sum_{p=1}^r f(x_p) - \lambda_0 \cdot \left(\sum_{p=1}^r x_p - c\right) - \sum_{p=1}^r \lambda_p x_p
\end{equation*}
be the Lagragian function for \eqref{eq:prf-generic-simplex-optimization} where $\blambda = [\lambda_0, \lambda_1, \ldots, \lambda_r]$ is the Lagragian multiplier. 
By the first order optimality conditions (i.e., the Karush–Kuhn–Tucker (KKT) conditions, see, e.g., \cite[Theorem 12.1]{nocedal2006numerical}), there exists $\blambda_\star= [(\lambda_0)_\star, (\lambda_1)_\star, \ldots, (\lambda_r)_\star]$ such that
\begin{align}
    \sum_{p=1}^r (x_q)_\star &= c,\label{eq:kkt1}\\
    (x_q)_\star &\ge 0, ~\forall q\in \{1, \ldots, r\},\label{eq:kkt2}\\
    (\lambda_q)_\star &\ge 0, ~\forall q\in \{1, \ldots, r\},\label{eq:kkt3}\\
    (\lambda_q)_\star \cdot (x_q)_\star &= 0, ~\forall q\in \{1, \ldots, r\}, ~~\text{and}~\label{eq:kkt4}\\
    [f'((x_1)_\star), \ldots, f'((x_r)_\star)] &= [(\lambda_0)_\star, \ldots, (\lambda_0)_\star] + [(\lambda_1)_\star, \ldots, (\lambda_r)_\star].\label{eq:kkt5}
\end{align}
By using the KKT conditions, we first show that all entries of $\x_\star$ are strictly positive. 
To prove by contradiction, suppose that $\x_\star$ has $r_0$ nonzero entries and $r - r_0$ zero entries for some $1 \le r_0 < r$. 
Note that $r_0 \ge 1$ since an all zero vector $\x_\star$ does not satisfy the equality constraint \eqref{eq:kkt1}.

Without loss of generality, we may assume that $(x_p)_\star > 0$ for $p \le r_0$ and $(x_p)_\star = 0$ otherwise. 
By \eqref{eq:kkt4}, we have
\begin{equation*}
    (\lambda_1)_\star = \cdots = (\lambda_{r_0})_\star = 0.
\end{equation*}
Plugging it into \eqref{eq:kkt5}, we get
\begin{equation*}
    f'((x_1)_\star) = \cdots = f'((x_{r_0})_\star) = (\lambda_0)_\star.
\end{equation*}
From \eqref{eq:kkt5} and noting that $x_{r_0+1}=0$ we get
\begin{equation*}
    f'(0) = f'(x_{r_0+1}) = (\lambda_0)_\star + (\lambda_{r_0 + 1})_\star.
\end{equation*}
Finally, from \eqref{eq:kkt3}, we have
\begin{equation*}
    (\lambda_{r_0+1})_\star \ge 0.
\end{equation*}
Combining the last three equations above gives $f'(0) - f'((x_1)_\star) \ge 0$, contradicting the assumption that $f'(0) < f'(x)$ for all $x > 0$. 
This shows that $r_0 = r$, i.e., all entries of $\x_\star$ are strictly positive.
Using this fact and \eqref{eq:kkt4} gives 
\begin{equation*}
    (\lambda_p)_\star = 0 ~~\text{for all}~p \in \{1, \ldots, r\}.
\end{equation*}
Combining this with \eqref{eq:kkt5} gives
\begin{equation}\label{eq:prf-equal-first-order}
    f'((x_1)_\star) = \cdots = f'((x_{r})_\star) = (\lambda_0)_\star.
\end{equation}
It follows from the fact that $f'(x)$ is strictly unimodal that
\begin{equation}\label{eq:prf-two-values}
    \exists \ x_H \ge x_L > 0 ~~\text{s.t.}~~\{(x_p)_\star\}_{p=1}^r \subseteq \{x_L, x_H\}.
\end{equation}
That is, the set $\{(x_p)_\star\}_{p=1}^r$ may contain no more than two values. 
To see why this is true, suppose that there exists three distinct values for $\{(x_p)_\star\}_{p=1}^r$. 
Without loss of generality we may assume that $0 < (x_1)_\star < (x_2)_\star < (x_3)_\star$. 
If $(x_2)_\star \le x_T$ (recall $x_T := \arg\max_{x\ge 0} f'(x)$), then by using the fact that $f'(x)$ is strictly increasing in $[0, x_T]$, we must have $f'((x_1)_\star) < f'((x_2)_\star)$ which contradicts \eqref{eq:prf-equal-first-order}. 
A similar contradiction is arrived by considering $f'((x_2)_\star)$ and $f'((x_3)_\star)$ for the case where $(x_2)_\star > x_T$. 

There are two possible cases as a consequence of \eqref{eq:prf-two-values}. 
First, if $x_L = x_H$, then we have $(x_1)_\star = \cdots = (x_r)_\star$. 
By further using \eqref{eq:kkt1} we get 
\begin{equation*}
     (x_1)_\star = \cdots = (x_r)_\star = \frac{c}{r}.
\end{equation*}

It remains to consider the case where $x_L < x_H$. 
First, by the unimodality of $f'(x)$, we must have $x_L < x_T < x_H$, therefore 
\begin{equation}\label{eq:prf-second-order-sign}
    f''(x_L) > 0 ~\text{and}~f''(x_H) < 0.
\end{equation} 
Let $\ell := |\{p: x_p = x_L\}|$ be the number of entries of $\x_\star$ that are equal to $x_L$ and $h:= r - \ell$. 
We show that it is necessary to have $\ell = 1$ and $h = r-1$. 
To prove by contradiction, assume that $\ell > 1$ and $h < r-1$. 
Without loss of generality we may assume $\{(x_p)_\star = x_H\}_{p=1}^{h}$ and $\{(x_p)_\star = x_L\}_{p=h+1}^{r}$. 
By \eqref{eq:prf-second-order-sign}, we have
\begin{equation*}
    f''((x_p)_\star) > 0 ~\text{for all}~p > h.
\end{equation*}
In particular, by using $h < r-1$ we have
\begin{equation}\label{eq:prf-last-two-positive}
    f''((x_{r-1})_\star) > 0 ~\text{and}~f''((x_{r})_\star) > 0.
\end{equation}
On the other hand, by using the second order necessary conditions for constraint optimization (see, e.g., \cite[Theorem 12.5]{nocedal2006numerical}), the following result holds
\begin{equation}\label{eq:prf-second-order}
\begin{split}
    \v_\star \nabla_{\x\x}\mathcal{L}(\x_\star, \blambda_\star) \v &\le 0, ~~\text{for all}~ \left\{\v: \left\langle \nabla_{\x}\left(\sum_{p=1}^r (x_p)_\star - c\right), \v \right\rangle = 0\right\}\\
    \iff\quad \sum_{p=1}^r f''((x_p)_\star) \cdot v_p^2 &\le 0, ~~\text{for all}~ \left\{\v=[v_1, \ldots, v_r]: \sum_{p=1}^r v_p = 0 \right\}.
\end{split}
\end{equation}
Take $\v$ to be such that $v_1 = \cdots = v_{r-2} = 0$ and $v_{r-1} = - v_r \ne 0$. Plugging it into \eqref{eq:prf-second-order} gives
\begin{equation*}
    f''((x_{r-1})_\star) + f''((x_{r})_\star) \le 0, 
\end{equation*}
which contradicts \eqref{eq:prf-last-two-positive}. 
Therefore, we may conclude that $\ell = 1$. 
That is, $\x_\star$ is given by
\begin{equation*}
\x_\star = [x_H, \ldots, x_H, x_L], \ \text{where} \ x_H > x_L > 0.
\end{equation*}
By using the condition in \eqref{eq:kkt1}, we may further show that
\begin{align*}
&(r-1) x_H + x_L = c \implies x_H = \frac{c}{r-1} - \frac{c}{x_L} < \frac{x_L}{r-1}, \\
&(r-1) x_H + x_L = c \implies (r-1) x_H + x_H > c \implies x_H > \frac{c}{r},
\end{align*}
which completes our proof.
\end{proof}

\begin{proof}[Proof of Theorem~\ref{thm:maximal-rate-reduction}]
Without loss of generality, let $\Z_\star = [\Z^1_\star, \ldots, \Z^k_\star]$ be the optimal solution of problem~\eqref{eq:maximal-rate-reduction-thm}. 

To show that $\Z^j_\star, j \in \left\{ 1, \dots, k\right\}$ are pairwise orthogonal, suppose for the purpose of arriving at a contradiction that $(\Z^{j_1}_\star)^* \Z^{j_2}_\star \ne \0$ for some $1 \le j_1 < j_2 \le k$. 
By using Lemma~\ref{thm:rate-reduction-bound}, the strict inequality in \eqref{eq:rate-reduction-bound} holds for the optimal solution $\Z_\star$. That is, 
\begin{equation}\label{eq:prf-optimal-strict-inequality}
    \Delta R(\Z_\star, \bm{\Pi}, \epsilon) < 
    \sum_{j=1}^k \frac{1}{2m}\log\left( \frac{\det^m \left(\I + \frac{n}{m\epsilon^2}\Z^j_\star (\Z^j_\star)^*\right)}{\det^{m_j}\left(\I + \frac{n}{m_j\epsilon^2}\Z^j_\star (\Z^j_\star)^*\right)}\right).
\end{equation}
On the other hand, since $\sum_{j=1}^k d_j \le n$, there exists $\{\U^j_\diamond \in \R^{n\times d_j}\}_{j=1}^k$ such that the columns of the matrix $[\U^1_\diamond, \ldots, \U^k_\diamond]$ are orthonormal. 
Denote $\Z^j_\star = \U^j_\star \bfSigma^j_\star (\V^j_\star)^*$ the compact SVD of $\Z^j_\star$, and let
\begin{equation*}
    \Z_\diamond = [\Z^1_\diamond, \ldots, \Z^k_\diamond], ~~\text{where}~ \Z^j_\diamond = \U^j_\diamond \bfSigma^j_\star (\V^j_\star)^*.
\end{equation*}
It follows that 
\begin{multline}
    (\Z^{j_1}_\diamond)^* \Z^{j_2}_\diamond = \V^{j_1}_\star \bfSigma^{j_1}_\star (\U^{j_1}_\diamond)^* \U^{j_2}_\diamond \bfSigma^{j_2}_\star (\V^{j_2}_\star)^*= \V^{j_1}_\star \bfSigma^{j_1}_\star  \0  \bfSigma^{j_2}_\star (\V^{j_2}_\star)^* = \0 \\~~\text{for all}~ 1\le j_1 < j_2 \le k. 
\end{multline}
That is, the matrices $\Z^1_\diamond, \ldots, \Z^k_\diamond$ are pairwise orthogonal. 
Applying Lemma~\ref{thm:rate-reduction-bound} for $\Z_\diamond$ gives
\begin{equation}\label{eq:prf-constructed-optimal}
\begin{split}
    \Delta R(\Z_\diamond, \bm{\Pi}, \epsilon) &= 
    \sum_{j=1}^k \frac{1}{2m}\log\left( \frac{\det^m\left(\I + \frac{n}{m\epsilon^2}\Z^j_\diamond (\Z^j_\diamond)^*\right)}{\det^{m_j}\left(\I + \frac{n}{m_j\epsilon^2}\Z^j_\diamond (\Z^j_\diamond)^*\right)}\right)\\
    &= \sum_{j=1}^k \frac{1}{2m}\log\left( \frac{\det^m\left(\I + \frac{n}{m\epsilon^2}\Z^j_\star (\Z^j_\star)^*\right)}{\det^{m_j}\left(\I + \frac{n}{m_j\epsilon^2}\Z^j_\star (\Z^j_\star)^*\right)}\right),
\end{split}
\end{equation}
where the second equality follows from Lemma~\ref{thm:coding-rate-invariant}.
Comparing \eqref{eq:prf-optimal-strict-inequality} and \eqref{eq:prf-constructed-optimal} gives $\Delta R(\Z_\diamond, \bm{\Pi}, \epsilon) > \Delta R(\Z_\star, \bm{\Pi}, \epsilon)$, which contradicts the optimality of $\Z_\star$.
Therefore, we must have 
\begin{equation*}
    (\Z^{j_1}_\star)^* \Z^{j_2}_\star = \0 ~\text{for all}~1 \le j_1 < j_2 \le k.
\end{equation*} 
Moreover, from Lemma~\ref{thm:coding-rate-invariant} we have
\begin{equation}\label{eq:prf-coding-rate-decomposition}
    \Delta R(\Z_\star, \bm{\Pi}, \epsilon) = 
    \sum_{j=1}^k \frac{1}{2m}\log\left( \frac{\det^m \left(\I + \frac{n}{m\epsilon^2}\Z^j_\star (\Z^j_\star)^*\right)}{\det^{m_j}\left(\I + \frac{n}{m_j\epsilon^2}\Z^j_\star (\Z^j_\star)^*\right)}\right).
\end{equation}
We now prove the result concerning the singular values of $\Z^j_\star$. 
To start with, we claim that the following result holds:
\begin{equation}\label{eq:prf-rich-key}
     \Z^j_\star \in \arg\max_{\Z^j} \  \log\left(\frac{\det^m\left(\I + \frac{n}{m\epsilon^2}\Z^j (\Z^j)^*\right)}{\det^{m_j}\left(\I + \frac{n}{m_j\epsilon^2}\Z^j (\Z^j)^*
     \right)}\right) ~~\text{s.t.}~\|\Z^j\|_F^2 = m_j,\, \rank(\Z^j) \le d_j.
\end{equation}
To see why \eqref{eq:prf-rich-key} holds, suppose that there exists $\widetilde{\Z}^j$ such that $\|\widetilde{\Z}^j\|_F^2 = m_j$, $\rank(\widetilde{\Z}^j) \le d_j$ and
\begin{equation}\label{eq:prf-tildeZ-inequality}
    \log\left(\frac{\det^m\left(\I + \frac{n}{m\epsilon^2}\widetilde{\Z}^j (\widetilde{\Z}^j)^*\right)}{\det^{m_j}\left(\I + \frac{n}{m_j\epsilon^2}\widetilde{\Z}^j (\widetilde{\Z}^j)^*\right)}\right) > \log\left(\frac{\det^m\left(\I + \frac{n}{m\epsilon^2}\Z^j_\star (\Z^j_\star)^*\right)}{\det^{m_j}\left(\I + \frac{n}{m_j\epsilon^2}\Z^j_\star (\Z^j_\star)^*\right)}\right).
\end{equation}
Denote $\widetilde{\Z}^j = \widetilde{\U}^j \widetilde{\bfSigma}^j(\widetilde{\V}^j)^*$ the compact SVD of $\widetilde{\Z}^j$ and let
\begin{equation*}
    \Z_\diamond = [\Z^1_\star, \ldots, \Z^{j-1}_\star, \Z^j_\diamond, \Z^{j+1}_\star, \ldots, \Z^k_\star], ~~\text{where}~\Z^j_\diamond := \U^j_\star \widetilde{\bfSigma}^j(\widetilde{\V}^j)^*.
\end{equation*}
Note that $\|\Z^j_\diamond\|_F^2 = m_j$, $\rank(\Z^j_\diamond) \le d_j$ and $(\Z^j_\diamond)^* \Z^{j'}_\star = \0$ for all $j' \ne j$. 
It follows that $\Z_\diamond$ is a feasible solution to \eqref{eq:maximal-rate-reduction-thm} and that the components of $\Z_\diamond$ are pairwise orthogonal.  
By using Lemma~\ref{thm:rate-reduction-bound}, Lemma~\ref{thm:coding-rate-invariant} and \eqref{eq:prf-tildeZ-inequality} we have
\begin{equation*}
\begin{split}
    &\Delta R(\Z_\diamond, \bm{\Pi}, \epsilon) \\
    =\ & \frac{1}{2m}\log\left( \frac{\det^m\left(\I + \frac{n}{m\epsilon^2}\Z^j_\diamond (\Z^j_\diamond)^*\right)}{\det^{m_j}\left(\I + \frac{n}{m_j\epsilon^2}\Z^j_\diamond (\Z^j_\diamond)^*\right)}\right) + 
    \sum_{j' \ne j} \frac{1}{2m}\log\left( \frac{\det^m\left(\I + \frac{n}{m\epsilon^2}\Z^{j'}_\star (\Z^{j'}_\star)^*\right)}{\det^{m_{j'}}\left(\I + \frac{n}{m_{j'}\epsilon^2}\Z^{j'}_\star (\Z^{j'}_\star)^*\right)}\right)\\
    =\ & \frac{1}{2m}\log\left( \frac{\det^m\left(\I + \frac{n}{m\epsilon^2}\widetilde{\Z}^j (\widetilde{\Z}^j)^*\right)}{\det^{m_j}\left(\I + \frac{n}{m_j\epsilon^2}\widetilde{\Z}^j (\widetilde{\Z}^j)^*\right)}\right) + 
    \sum_{j' \ne j} \frac{1}{2m}\log\left( \frac{\det^m\left(\I + \frac{n}{m\epsilon^2}\Z^{j'}_\star (\Z^{j'}_\star)^*\right)}{\det^{m_{j'}}\left(\I + \frac{n}{m_{j'}\epsilon^2}\Z^{j'}_\star (\Z^{j'}_\star)^*\right)}\right)\\
    >\ & \frac{1}{2m}\log\left(\frac{\det^m\left(\I + \frac{n}{m\epsilon^2}\Z^j_\star (\Z^j_\star)^*\right)}{\det^{m_j}\left(\I + \frac{n}{m_j\epsilon^2}\Z^j_\star (\Z^j_\star)^*\right)}\right) + 
    \sum_{j' \ne j} \frac{1}{2m}\log\left( \frac{\det^m\left(\I + \frac{n}{m\epsilon^2}\Z^{j'}_\star (\Z^{j'}_\star)^*\right)}{\det^{m_{j'}}\left(\I + \frac{n}{m_{j'}\epsilon^2}\Z^{j'}_\star (\Z^{j'}_\star)^*\right)}\right)\\
    =\ & \sum_{j=1}^k \frac{1}{2m}\log \left(\frac{\det^m \left(\I + \frac{n}{m\epsilon^2}\Z^j_\star (\Z^j_\star)^*\right)}{\det^{m_j}\left(\I + \frac{n}{m_j\epsilon^2}\Z^j_\star (\Z^j_\star)^*\right)}\right).
\end{split}
\end{equation*}
Combining it with \eqref{eq:prf-coding-rate-decomposition} shows $\Delta R(\Z_\diamond, \bm{\Pi}, \epsilon) > \Delta R(\Z_\star, \bm{\Pi}, \epsilon)$, contradicting the optimality of $\Z_\star$.
Therefore, the result in \eqref{eq:prf-rich-key} holds.

Observe that the optimization problem in \eqref{eq:prf-rich-key} depends on $\Z^j$ only through its singular values. 
That is, by letting $\bfsigma_j:=[\sigma_{1,j}, \ldots, \sigma_{\min(m_j, n),j}]$ be the singular values of $\Z^j$, we have
\begin{equation*}
     \log\left(\frac{\det^m\left(\I + \frac{n}{m\epsilon^2}\Z^j (\Z^j)^*\right)}{\det^{m_j}\left(\I + \frac{n}{m_j\epsilon^2}\Z^j (\Z^j)^*
     \right)} \right)
     = \sum_{p=1}^{\min\{m_j, n\}} \log\left( \frac{(1+\frac{n}{m \epsilon^2} \sigma_{p,j}^2)^m}{(1+\frac{n}{m_j \epsilon^2} \sigma_{p,j}^2)^{m_j}}\right),
\end{equation*}
also, we have
\begin{equation*}
  \|\Z^j\|_F^2 = \sum_{p=1}^{\min\{m_j, n\}} \sigma_{p,j}^2  ~~\text{and}~~ \rank(\Z^j) = \|\bfsigma_j\|_0. 
\end{equation*}
Using these relations, \eqref{eq:prf-rich-key} is equivalent to
\begin{equation}\label{eq:prf-sigma-optimization-all}
\begin{split}
    &\max_{\bfsigma_j \in \Re_+^{\min\{m_j, n\}}} \sum_{p=1}^{\min\{m_j, n\}} \log\left( \frac{(1+\frac{n}{m \epsilon^2} \sigma_{p,j}^2)^m}{(1+\frac{n}{m_j \epsilon^2} \sigma_{p,j}^2)^{m_j}}\right) \\
    &\ ~~\text{s.t.}~\sum_{p=1}^{\min\{m_j, n\}} \sigma_{p,j}^2 = m_j, ~\text{and}~ \ \rank(\Z^j) = \|\bfsigma_j\|_0 
\end{split}
\end{equation}
Let $(\bfsigma_j)_\star = [(\sigma_{1,j})_\star, \ldots, (\sigma_{\min\{m_j, n\},j})_\star]$ be an optimal solution to \eqref{eq:prf-sigma-optimization-all}. 
Without loss of generality we assume that the entries of $(\bfsigma_j)_\star$ are sorted in descending order. 
It follows that 
\begin{equation*}
(\sigma_{p, j})_\star = 0 \ ~\text{for all}~ \ p > d_j,
\end{equation*} 
and
\begin{equation}\label{eq:prf-sigma-optimization}
    [(\sigma_{1,j})_\star, \ldots, (\sigma_{d_j, j})_\star] = \argmax_{\substack{[\sigma_{1,j}, \ldots, \sigma_{d_j, j}] \in \Re^{d_j}_{+}\\ \sigma_{1,j} \ge \cdots \ge \sigma_{d_j,j}}} \  \sum_{p=1}^{d_j} \log\left( \frac{(1+\frac{n}{m \epsilon^2} \sigma_{p,j}^2)^m}{(1+\frac{n}{m_j \epsilon^2} \sigma_{p,j}^2)^{m_j}}\right)~~~~\text{s.t.}~\sum_{p=1}^{d_j} \sigma_{p,j}^2 = m_j.
\end{equation}

Then we define 
\begin{equation*}
    f(x; n, \epsilon, m_j, m) = \log\left( \frac{(1+\frac{n}{m \epsilon^2} x)^m}{(1+\frac{n}{m_j \epsilon^2} x)^{m_j}}\right),
\end{equation*}
and rewrite \eqref{eq:prf-sigma-optimization} as
\begin{equation}\label{eq:prf-sigma-optimization_f}
    \max_{\substack{[x_1, \ldots, x_{d_j}]  \in \Re_+^{d_j}\\x_1 \ge \cdots \ge x_{d_j}}} \ \sum_{p=1}^{d_j}  f(x_p; n, \epsilon, m_j, m) \ ~~\text{s.t.}~ \sum_{p=1}^{d_j} x_p = m_j.
\end{equation}
We compute the first and second derivative for $f$ with respect to $x$, which are given by
\begin{align*}
    f'(x; n, \epsilon, m_j, m) &= \frac{n^2 x (m-m_j)}{(nx+m\epsilon^2)(nx+m_j\epsilon^2)}, \\
    f''(x; n, \epsilon, m_j, m) &= \frac{n^2(m-m_j)(m m_j\epsilon^4 - n^2 x^2)}{(nx+m\epsilon^2)^2(nx+m_j\epsilon^2)^2}.
\end{align*}
Note that
\begin{itemize}
    \item $0 = f'(0) < f'(x)$ for all $x > 0$,
    \item $f'(x)$ is strictly increasing in $[0, x_T]$ and strictly decreasing in $[x_T, \infty)$, where $x_T =\epsilon^2\sqrt{\frac{m}{n}\frac{m_j}{n}}$, and
    \item by using the condition $\epsilon ^4 < \frac{m_j}{m}\frac{n^2}{d_j^2}$, we have $f''(\frac{m_j}{d_j}) < 0$.
\end{itemize}
Therefore, we may apply Lemma~\ref{thm:generic-simplex-optimization} and conclude that the unique optimal solution to \eqref{eq:prf-sigma-optimization_f} is either
\begin{itemize}
    \item $\x_\star = [\frac{m_j}{d_j}, \ldots, \frac{m_j}{d_j}]$, or
    \item $\x_\star = [x_H, \ldots, x_H, x_L]$ for some $x_H \in (\frac{m_j}{d_j}, \frac{m_j}{d_j -1})$ and $x_L > 0$.
\end{itemize} 
Equivalently, we have either
\begin{itemize}
    \item $[(\sigma_{1, j})_\star, \ldots, (\sigma_{d_j, j})_\star] = \left[\sqrt{\frac{m_j}{d_j}}, \ldots, \sqrt{\frac{m_j}{d_j}}\right]$, or
    \item $[(\sigma_{1, j})_\star, \ldots, (\sigma_{d_j, j})_\star] = [\sigma_H, \ldots, \sigma_H, \sigma_L]$ for some $\sigma_H \in \left(\sqrt{\frac{m_j}{d_j}}, \sqrt{\frac{m_j}{d_j -1}}\right)$ and $\sigma_L > 0$,
\end{itemize} 
as claimed.
\end{proof}

\newpage
\section{ReduNet for 1D Circular Shift Invariance}\label{app:1D}
It has been long known that to implement a convolutional neural network, one can achieve higher computational efficiency by implementing the network in the spectral domain via the fast Fourier transform \citep{mathieu2013fast,lavin2015fast,Vasilache2015FastCN}. However, our purpose here is different: We want to show that the linear operators $\bm E$ and $\bm C^j$ (or  $\bar{\bm E}$ and $\bar{\bm C}^j$) derived from the gradient flow of MCR$^2$ are naturally convolutions when we enforce shift-invariance rigorously. Their convolution structure is derived from the rate reduction objective, rather than imposed upon the network. Furthermore, the computation involved in constructing these linear operators has a naturally efficient implementation in the spectral domain via fast Fourier transform. Arguably this work is the first to show multi-channel convolutions, together with other convolution-preserving nonlinear operations in the ReduNet, are both necessary and sufficient to ensure shift invariance.

To be somewhat self-contained and self-consistent, in this section, we first introduce our notation and review some of the key properties of circulant matrices which will be used to characterize the properties of the linear operators $\bar{\bm E}$ and $\bar{\bm C}^j$ and to compute them efficiently. The reader may refer to \cite{Kra2012OnCM} for a more rigorous exposition on circulant matrices.

\subsection{Properties of Circulant Matrix and Circular Convolution}\label{ap:circulant}

Given a vector $\z = [z(0), z(1), \ldots, z{(n-1)}]^* \in \Re^n$, we may arrange all its circular shifted versions in a circulant matrix form as
\begin{equation}\label{eq:def-circulant}
\circm(\z) \quad\doteq\quad \left[ \begin{array}{ccccc} z(0) & z(n-1) & \dots & z(2) & z(1) \\ z(1) & z(0) & z(n-1) & \cdots & z(2) \\ \vdots & z(1) & z(0) &\ddots & \vdots \\ z(n-2) &  \vdots & \ddots & \ddots & z(n-1) \\ z(n-1) & z(n-2) & \dots & z(1) & z(0)   \end{array} \right] \quad \in \Re^{n \times n}.
\end{equation}
\begin{fact}[Convolution as matrix multiplication via circulant matrix] The multiplication of a circulant matrix $\circm(\z)$ with a vector $\x \in \Re^n$ gives a circular (or cyclic) convolution, i.e., 
    \begin{equation}
        \circm(\z) \cdot \x = \z \circledast \x,
    \end{equation} 
    where
    \begin{equation}\label{eq:def-convolution}
    (\bm z \circledast \bm x)_{i} = \sum_{j=0}^{n-1} x(j) z(i+ n-j \, \mathrm{mod} \,n).
    \end{equation}
\end{fact}

\begin{fact}[Properties of circulant matrices]
\label{fact:circ-properties}
Circulant matrices have the following properties:
\begin{itemize}
    \item Transpose of a circulant matrix, say $\circm(\z)^*$, is circulant;
    \item Multiplication of two circulant matrices is circulant, for example $\circm(\z)\circm(z)^*$;
    \item For a non-singular circulant matrix,  its inverse is also circulant (hence representing a circular convolution). 
\end{itemize}
\end{fact}

These properties of circulant matrices are extensively used in this work as for characterizing the convolution structures of the operators $\E$ and $\C^j$. Given a set of vectors $[\z^{1}, \dots, \z^{m}] \in \Re^{n\times m}$, let $\circm(\z^i) \in \Re^{n\times n}$ be the circulant matrix for $\z^i$. 
Then we have the following (Proposition \ref{prop:circular-conv-1} in the main body and here restated for convenience):
\begin{proposition}[Convolution structures of $\bm E$ and $\bm C^j$]
Given a set of vectors $\bm Z =[\z^1, \ldots, \z^m]$, the matrix: 
$$\E = \alpha\big(\bm I + \alpha\sum_{i=1}^m \circm(\z^i)\circm(\z^i)^*\big)^{-1}
$$ 
is a circulant matrix and represents a circular convolution: $$\E \z = \bm e \circledast \z,$$ where $\bm e$ is the first column vector of $\E$. Similarly, the matrices $\bm C^j$ associated with any subsets of $\bm Z$ are also circular convolutions. 
\label{prop:circular-conv}
\end{proposition}

\subsection{Circulant Matrix and Circulant Convolution for Multi-channel Signals}\label{ap:multichannel-circulant}
In the remainder of this section, we view $\z$ as a 1D signal such as an audio signal. Since we will deal with the more general case of multi-channel signals, we will use the traditional notation $T$ to denote the temporal length of the signal and $C$ for the number of channels. Conceptually, the ``dimension'' $n$ of such a multi-channel signal, if viewed as a vector, should be $n = CT$.\footnote{Notice that in the main paper, for simplicity, we have used $n$ to indicate both the 1D ``temporal'' or 2D ``spatial'' dimension of a signal, just to be consistent with the vector case, which corresponds to $T$ here. All notation should be clear within the context.} As we will also reveal additional interesting structures of the operators $\bm E$ and $\bm C^j$ in the spectral domain, we use $t$ as the index for time, $p$ for the index of frequency, and $c$ for the index of channel.

Given a multi-channel 1D signal  $\bar\z \in \Re^{C \times T}$, we denote
\begin{equation}\label{eq:index-multichannel}
    \bar\z = 
        \begin{bmatrix}
        \bar\z[1]^*\\
        \vdots\\
        \bar\z[C]^*\\
    \end{bmatrix}
    = [\bar\z(0), \bar\z(1), \ldots, \bar\z(T-1)] = \{\bar\z[c](t)\}_{c=1, t=0}^{c=C, t=T-1}.
\end{equation}
To compute the coding rate reduction for a collection of such multi-channel 1D signals, we may flatten the matrix representation into a vector representation by stacking the multiple channels of $\bar\z$ as a column vector.  
In particular, we let
\begin{equation}
    \vec(\bar\z) = \left[ \bar\z[1](0), \bar\z[1](1), \ldots, \bar\z[1](T-1), \bar\z[2](0), \ldots \right] \quad \in \Re^{(C\times T)}.
\end{equation}
Furthermore, to obtain shift invariance for the coding rate reduction, we may generate a collection of shifted copies of $\bar\z$ (along the temporal dimension). 
Stacking the vector representations for such shifted copies as column vectors, we obtain
\begin{equation}
    \circm(\bar\z) \doteq 
    \begin{bmatrix}
    \circm(\bar\z[1])\\
    \vdots\\
    \circm(\bar\z[C])
    \end{bmatrix} \quad
    \in \Re^{(C\times T) \times T}.
\end{equation}
In above, we overload the notation ``$\circm(\cdot)$'' defined in \eqref{eq:def-circulant}. 

We now consider a collection of $m$ multi-channel 1D signals  $\{\bar\z^i \in \Re^{C \times T}\}_{i=1}^m$.
Compactly representing the data by $\bar\Z \in \Re^{C \times T \times m}$ in which the $i$-th slice on the last dimension is $\bar\z^i$, we denote
\begin{equation}\label{eq:index-multichannel-collection}
    \bar\Z[c] = [\bar\z^1[c], \ldots, \bar\z^m[c]] \in \Re^{T \times m}, \qquad \bar\Z(t) = [\bar\z^1(t), \ldots, \bar\z^m(t)] \in \Re^{C \times m}.
\end{equation}
In addition, we denote
\begin{equation}
\begin{split}
    \vec(\bar\Z) &= [\vec(\bar\z^1), \ldots, \vec(\bar\z^m)] \in \Re^{(C\times T) \times m}, \\
    \circm(\bar\Z) &= [\circm(\bar\z^1), \ldots, \circm(\bar\z^m)] \in \Re^{(C\times T) \times (T \times m)}.
\end{split}
\end{equation}
Then, we define the \emph{shift invariant coding rate reduction} for $\bar\Z \in \Re^{C \times T \times m}$ as
\begin{multline}
\label{eq:1D-MCR2}
    \Delta R_\circm(\bar\Z, \bm{\Pi}) \doteq \frac{1}{T}\Delta R(\circm(\bar\Z), \bar{\bm{\Pi}}) \\= 
    \frac{1}{2T}\log\det \Bigg(\I + \alpha \cdot \circm(\bar\Z) \cdot \circm(\bar\Z)^{*} \Bigg) 
    - \sum_{j=1}^{k}\frac{\gamma_j}{2T}\log\det\Bigg(\I + \alpha_j \cdot \circm(\bar\Z) \cdot \bar{\bm{\Pi}}^{j} \cdot \circm(\bar\Z)^{*} \Bigg),
\end{multline}
where $\alpha = \frac{CT}{mT\epsilon^{2}} = \frac{C}{m\epsilon^{2}}$, $\alpha_j = \frac{CT}{\textsf{tr}\left(\bm{\Pi}^{j}\right)T\epsilon^{2}} = \frac{C}{\textsf{tr}\left(\bm{\Pi}^{j}\right)\epsilon^{2}}$, $\gamma_j = \frac{\textsf{tr}\left(\bm{\Pi}^{j}\right)}{m}$, and $\bar{\bm \Pi}^j$ is augmented membership matrix in an obvious way.
Note that we introduce the normalization factor $T$ in \eqref{eq:1D-MCR2} because the circulant matrix $\circm(\bar\Z)$ contains $T$ (shifted) copies of each signal.

By applying \eqref{eqn:expand-directions} and \eqref{eqn:compress-directions}, we obtain the derivative of $\Delta R_\circm(\bar\Z, \bm{\Pi})$ as
\begin{equation}\label{eqn:expand-directions-multichannel} 
\begin{split}
    \frac{1}{2T}\frac{\partial \log \det \Big(\I + \alpha \circm(\bar\Z) \circm(\bar\Z)^{*} \Big)}{\partial  \vec(\bar\Z)} 
    &= \frac{1}{2T}\frac{\partial \log \det \Big(\I + \alpha \circm(\bar\Z) \circm(\bar\Z)^{*} \Big)}{\partial \circm(\bar\Z)} \frac{\partial \circm(\bar\Z)}{\partial  \vec(\bar\Z)}\\
    &= \underbrace{\alpha\Big(\I + \alpha\circm(\bar\Z) \circm(\bar\Z)^{*}\Big)^{-1}}_{\bar\E{} \; \in \Re^{(C\times T)\times (C \times T)}}\vec(\bar\Z), 
\end{split}
\end{equation}    
\begin{equation}\label{eqn:compress-directions-multichannel}
    \frac{\gamma_j }{2T}\frac{\partial \log \det \Big(\I + \alpha_j \circm(\bar\Z) \bm \Pi^j \circm(\bar\Z)^{*} \Big)}{\partial  \vec(\bar\Z)} = \gamma_j  \underbrace{ \alpha_j  \Big(\I +  \alpha_j \circm(\bar\Z) \bm \Pi^j \circm(\bar\Z)^{*}\Big)^{-1}}_{\bar\C^j \; \in \Re^{(C\times T)\times (C \times T)}} \vec(\bar\Z) \bm \Pi^j.
\end{equation}

In the following, we show that $\bar\E \cdot \vec(\bar\z)$ represents a multi-channel circular convolution. 
Note that
\begin{equation}\label{eq:calc-E-bar}
    \bar{\bm E} =
    \alpha 
    \left[\begin{smallmatrix}
    \bm I + \alpha \sum_{i=1}^m\circm(\z^i[1])\circm(\z^i[1])^* & \cdots & \sum_{i=1}^m\circm(\z^i[1])\circm(\z^i[C])^* \\
    \vdots & \ddots & \vdots \\
    \sum_{i=1}^m\circm(\z^i[C])\circm(\z^i[1])^* & \cdots & \bm I + \sum_{i=1}^m\alpha \circm(\z^i[C])\circm(\z^i[C])^* \\
    \end{smallmatrix}\right]^{-1}.
\end{equation}
By using Fact~\ref{fact:circ-properties}, the matrix in the inverse above is a \emph{block circulant matrix}, i.e., a block matrix where each block is a circulant matrix. 
A useful fact about the inverse of such a matrix is the following. 
\begin{fact}[Inverse of block circulant matrices]
The inverse of a block circulant matrix is a block circulant matrix (with respect to the same block partition).
\end{fact}

The main result of this subsection is the following (Proposition \ref{prop:multichannel-circular-conv-1} in the main body and here restated for convenience):
\begin{proposition}[Convolution structures of $\bar{\bm E}$ and $\bar{\bm C}^j$] Given a collection of multi-channel 1D signals  $\{\bar\z^i \in \Re^{C \times T}\}_{i=1}^m$, the matrix $\bar\E$
is a block circulant matrix, i.e.,
\begin{equation}\label{eq:E-bar}
    \bar{\bm E} \doteq 
    \begin{bmatrix}
        \bar{\bm E}_{1, 1} & \cdots & \bar{\bm E}_{1, C}\\
        \vdots & \ddots & \vdots \\
        \bar{\bm E}_{C, 1} & \cdots & \bar{\bm E}_{C, C}\\
    \end{bmatrix},
\end{equation}
where each $\bar{\bm E}_{c, c'}\in \Re^{T \times T}$ is a circulant matrix. Moreover, $\bar{\bm E}$ represents a multi-channel circular convolution, i.e., for any multi-channel signal $\bar\z \in \Re^{C \times T}$ we have 
$$\bar\E \cdot \vec(\bar\z) = \vec( \bar{\bm e} \circledast \bar\z).$$ 
In above, $\bar{\bm e} \in \Re^{C \times C \times T}$ is a multi-channel convolutional kernel with $\bar{\bm e}[c, c'] \in \Re^{T}$ being the first column vector of $\bar{\bm E}_{c, c'}$, and $\bar{\bm e} \circledast \bar\z \in \Re^{C \times T}$ is the multi-channel circular convolution (with ``$\circledast$'' overloading the notation from Eq.  \eqref{eq:def-convolution}) defined as
\begin{equation}
    (\bar{\bm e} \circledast \bar\z)[c] = \sum_{c'=1}^C \bar{\bm e}[c, c'] \circledast \bar{\z}[c'], \quad \forall c = 1, \ldots, C.
\end{equation}
Similarly, the matrices $\bar{\bm C}^j$ associated with any subsets of $\bar{\bm Z}$ are also multi-channel circular convolutions. 
\label{prop:multichannel-circular-conv}
\end{proposition}

Note that the calculation of $\bar\E$ in \eqref{eq:calc-E-bar} requires inverting a matrix of size $(C\times T) \times (C \times T)$. 
In the following, we show that this computation can be accelerated by working in the frequency domain.

\subsection{Fast Computation in Spectral Domain}
\label{ap:1D-shift}

\paragraph{Circulant matrix and Discrete Fourier Transform.}
A remarkable property of circulant matrices is that {\em they all share the same set of eigenvectors that form a unitary matrix}. 
We define the matrix:
\begin{equation}\label{eq:dft-matrix}
    \F_T \doteq \frac{1}{\sqrt{T}} 
    \begin{bmatrix}
        \omega_T^0 & \omega_T^0 & \cdots & \omega_T^0 &  \omega_T^0\\
        \omega_T^0 & \omega_T^1 & \cdots & \omega_T^{T-2} &  \omega_T^{T-1}\\
        \vdots     & \vdots     & \ddots     & \vdots & \vdots \\
        \omega_T^0 & \omega_T^{ T-2}& \cdots & \omega_T^{(T-2)^2} &  \omega_T^{(T-2)(T-1)} \\
        \omega_T^0 & \omega_T^{T-1}& \cdots & \omega_T^{(T-2)(T-1)} &  \omega_T^{(T-1)^2}
    \end{bmatrix} \quad \in \mathbb{C}^{T\times T},
\end{equation} 
where $\omega_T \doteq \exp(- \frac{2\pi\sqrt{-1}}{T})$ is the roots of unit (as $\omega_T^T = 1$). The matrix $\F_T$ is a unitary matrix: $\F_T \F_T^* = \bm I$ and is the well known {\em Vandermonde matrix}. Multiplying a vector with $\F_T$ is known as the {\em discrete Fourier transform} (DFT). Be aware that the conventional DFT matrix differs from our definition of $\bm F_T$ here by a scale: it does not have the $\frac{1}{\sqrt{T}}$ in front. Here for simplicity, we scale it so that $\bm F_T$ is a unitary matrix and its inverse is simply its conjugate transpose $\bm F_T^*$, columns of which represent the eigenvectors of a circulant matrix \citep{abidi2016optimization}. 

\begin{fact}[DFT as matrix-vector multiplication]\label{fact:dft}
The DFT of a vector $\z \in \Re^T$ can be computed as
\begin{equation}
    \dft(\z) \doteq \F_T \cdot \z \quad \in \Co^T,
\end{equation}
where
\begin{equation}
    \dft(\z)(p) = \frac{1}{\sqrt{T}} \sum_{t=0}^{T-1} z(t) \cdot \omega_T ^{p \cdot t}, \quad \forall p = 0, 1, \ldots, T-1.
\end{equation}
The Inverse Discrete Fourier Transform (IDFT) of a signal $\bv \in \Co^{T}$ can be computed as
        \begin{equation}
            \idft(\bv) \doteq \F_T^* \cdot \bv \quad \in \Co^T
        \end{equation}
        where
\begin{equation}
    \idft(\bv)(t) = \frac{1}{\sqrt{T}} \sum_{p=0}^{T-1} v(p) \cdot \omega_T ^{-p \cdot t}, \quad \forall t = 0, 1, \ldots, T-1.
\end{equation}
\end{fact}

Regarding the relationship between a circulant matrix (convolution) and discrete Fourier transform, we have: 
\begin{fact}
An $n\times n$ matrix $\bm M \in \mathbb{C}^{n\times n}$ is a circulant matrix if and only if it is diagonalizable by the unitary matrix $\bm F_n$:
\begin{equation}
   \F_n \bm M \F_n^* = \D \quad \mbox{or} \quad \bm M = \F_n^* \D \F_n,
\end{equation}
where $\D$ is a diagonal matrix of eigenvalues. 
\label{fact:circulant}
\end{fact}

\begin{fact}[DFT are eigenvalues of the circulant matrix] Given a vector $\z \in \Co^T$, we have
    \begin{equation}\label{eq:dft-circulant}
        \F_T \cdot \circm(\z) \cdot \F_T^* = \diag(\dft(\z))  \quad \mbox{or} \quad \circm(\z) = \F_T^* \cdot \diag(\dft(\z)) \cdot \F_T.
    \end{equation}
That is, the eigenvalues of the circulant matrix associated with a vector are given by its DFT.
\label{fact:dft-circulant}
\end{fact}

\begin{fact}[Parseval's theorem]
Given any $\z \in \Co^T$, we have $\|\z\|_2 = \|\dft(\z)\|_2$. More precisely, 
\begin{equation}
    \sum_{t = 0}^{T-1} |\z[t]| ^2 =  \sum_{p = 0}^{T-1} |\dft(\z)[p]| ^2.
\end{equation}
\label{fact:parseval}
\end{fact}
This property allows us to easily ``normalize'' features after each layer onto the sphere $\mathbb S^{n-1}$ directly in the spectral domain (see Eq. \eqref{eqn:layer-approximate} and \eqref{eqn:layer-approximate-spectral}).

\paragraph{Circulant matrix and Discrete Fourier Transform for multi-channel signals.}

We now consider multi-channel 1D signals $\bar\z \in \Re^{C \times T}$. 
Let $\dft(\bar\z) \in \Co^{C \times T}$ be a matrix where the $c$-th row is the DFT of the corresponding signal $\z[c]$, i.e., 
\begin{equation}
    \dft(\bar\z) \doteq 
    \begin{bmatrix}
    \dft(\z[1]) ^*\\
    \vdots\\
    \dft(\z[C]) ^*
    \end{bmatrix} \quad
    \in \Co^{C \times T}.
\end{equation}
Similar to the notation in \eqref{eq:index-multichannel}, we denote
\begin{equation}
\begin{aligned}
    \dft(\bar\z) 
    &= 
    \left[\begin{matrix}
        \dft(\bar\z)[1]^*\\
        \vdots\\
        \dft(\bar\z)[C]^*\\
    \end{matrix}\right]\\
    &= [\dft(\bar\z)(0), \dft(\bar\z)(1), \ldots, \dft(\bar\z)(T-1)] 
    \\
    &= \{\dft(\bar\z)[c](t)\}_{c=1, t=0}^{c=C, t=T-1}.
\end{aligned}
\end{equation}
As such, we have $\dft(\z[c]) = \dft(\bar\z)[c]$. 

By using Fact~\ref{fact:dft-circulant}, $\circm(\bar\z)$ and $\dft(\bar\z)$ are related as follows:
\begin{equation}\label{eq:dft-circulant-multichannel}
    \circm(\bar\z) = 
    \begin{bmatrix}
    \F_T^*\cdot  \diag(\dft(\z[1])) \cdot \F_T\\
    \vdots\\
    \F_T^*\cdot  \diag(\dft(\z[C])) \cdot \F_T
    \end{bmatrix}
    =
    \begin{bmatrix}
        \F_T^* &  \cdots & \0 \\
        \vdots& \ddots & \vdots \\
        \0   &  \cdots & \F_T^* \\
    \end{bmatrix}
    \cdot
    \begin{bmatrix}
     \diag(\dft(\z[1]))\\
    \vdots\\
     \diag(\dft(\z[C]))
    \end{bmatrix}
     \cdot \F_T.
\end{equation}
We now explain how this relationship can be leveraged to produce a fast computation of $\bar\E$ defined in \eqref{eqn:expand-directions-multichannel}. 
First, there exists a permutation matrix $\P$ such that
\begin{equation}\label{eq:permute-dft}
\begin{bmatrix}
    \diag(\dft(\z[1]))  \\
    \diag(\dft(\z[2]))  \\
    \vdots \\
    \diag(\dft(\z[C])) \\
\end{bmatrix}
= \P \cdot
\begin{bmatrix}
    \dft(\bar\z)(0) & \0         & \cdots & 0\\
    \0         & \dft(\bar\z)(1) & \cdots & 0\\
    \vdots     & \vdots     & \ddots & \vdots \\
    \0         & \0         & \cdots & \dft(\bar\z)(T-1)
\end{bmatrix}.
\end{equation}
Combining \eqref{eq:dft-circulant-multichannel} and \eqref{eq:permute-dft}, we have
\begin{equation}\label{eq:multi-channel-circ-covariance}
\circm(\bar\z) \cdot \circm(\bar\z)^* = 
\begin{bmatrix}
    \F_T^* &  \cdots & \0 \\
    \vdots& \ddots & \vdots \\
    \0   &  \cdots & \F_T^* \\
\end{bmatrix}
\cdot \P \cdot \D(\bar\z) \cdot \P^* \cdot
\begin{bmatrix}
    \F_T &  \cdots & \0 \\
    \vdots& \ddots & \vdots \\
    \0   &  \cdots & \F_T \\
\end{bmatrix},
\end{equation}
where
\begin{equation}
\D(\bar\z) \doteq 
\begin{bmatrix}
    \dft(\bar\z)(0) \cdot \dft(\bar\z)(0)^*  & \cdots & \0\\
     \vdots                              & \ddots & \vdots \\
    \0         & \cdots                 & \dft(\bar\z)(T-1) \cdot \dft(\bar\z)(T-1)^*
\end{bmatrix}.
\end{equation}

Now, consider a collection of $m$ multi-channel 1D signals $\bar\Z \in \Re^{C \times T \times m}$. 
Similar to the notation in \eqref{eq:index-multichannel-collection}, we denote
\begin{equation}
\begin{split}
    \dft(\bar\Z)[c] &= [\dft(\bar\z^1)[c], \ldots, \dft(\bar\z^m)[c]] \in \Re^{T \times m}, \\
    \dft(\bar\Z)(p) &= [\dft(\bar\z^1)(p), \ldots, \dft(\bar\z^m)(p)] \in \Re^{C \times m}.
\end{split}
\end{equation}

By using \eqref{eq:multi-channel-circ-covariance}, we have
\begin{multline}\label{eq:bar-E-frequency}
    \bar\E =  
    \begin{bmatrix}
        \F_T^* &  \cdots & \0 \\
        \vdots& \ddots & \vdots \\
        \0   &  \cdots & \F_T^* \\
    \end{bmatrix}
    \cdot \P \cdot
    \alpha \cdot \left[\bm I + \alpha \cdot \sum_{i=1}^m \D(\bar\z^i)\right]^{-1} 
    \cdot \P^* \cdot
    \begin{bmatrix}
        \F_T &  \cdots & \0 \\
        \vdots& \ddots & \vdots \\
        \0   &  \cdots & \F_T \\
    \end{bmatrix}.
\end{multline}
Note that $\alpha \cdot \left[\bm I + \alpha \cdot \sum_{i=1}^m \D(\bar\z^i)\right]^{-1}$ is equal to
\begin{multline}\label{eq:E-bar-fast-inverse}
\alpha
\left[\begin{smallmatrix}
    \bm I + \alpha \dft(\bar\Z)(0) \cdot \dft(\bar\Z^i)(0)^*  & \cdots & \0\\
     \vdots                              & \ddots & \vdots \\
    \0         & \cdots                 & \bm I + \alpha \dft(\bar\Z)(T-1) \cdot \dft(\bar\Z)(T-1)^*
\end{smallmatrix}\right] ^{-1} \\
=
\left[\begin{smallmatrix}
    \alpha \left(\bm I + \alpha \dft(\bar\Z)(0) \cdot \dft(\bar\Z)(0)^*\right)^{-1}  & \cdots & \0\\
     \vdots                              & \ddots & \vdots \\
    \0         & \cdots                 & \alpha \left(\bm I + \alpha \dft(\bar\Z)(T-1) \cdot \dft(\bar\Z)(T-1)^*\right)^{-1}
\end{smallmatrix}\right].   
\end{multline}
Therefore, the calculation of $\bar\E$ only requires inverting $T$ matrices of size $C\times C$. 
This motivates us to construct the ReduNet in the spectral domain for the purpose of accelerating the computation, as we explain next.

\paragraph{Shift-invariant ReduNet in the Spectral  Domain.}

Motivated by the result in \eqref{eq:E-bar-fast-inverse}, we introduce the notations $\bar\cE(p) \in \Re^{C \times C \times T}$ and $\bar\cC^j(p) \in \Re^{C \times C \times T}$ given by
\begin{eqnarray}
    \bar\cE(p) &\doteq& \alpha \cdot \left[\I + \alpha \cdot \dft(\bar\Z)(p) \cdot \dft(\bar\Z)(p)^* \right]^{-1} \quad \in \Co^{C\times C}, \\
    \bar\cC^j(p) &\doteq& \alpha_j \cdot\left[\I + \alpha_j \cdot \dft(\bar\Z)(p) \cdot \bm{\Pi}_j \cdot \dft(\bar\Z)(p)^*\right]^{-1} \quad \in \Co^{C\times C}.
\end{eqnarray}
In above, $\bar\cE(p)$ (resp., $\bar\cC^j(p)$) is the $p$-th slice of $\bar\cE$ (resp., $\bar\cC^j$) on the last dimension. 
Then, the gradient of $\Delta R_\circm(\bar\Z, \bm{\Pi})$ with respect to $\bar\Z$ can be calculated by the following result (Theorem \ref{thm:1D-convolution} in the main body and here restated for convenience).

\begin{theorem} [Computing multi-channel convolutions $\bar{\bm E}$ and $\bar{\bm C}^j$]
Let $\bar\U \in \Co^{C \times T \times m}$ and $\bar\W^{j} \in \Co^{C \times T \times m}, j=1,\ldots, k$ be given by 
\begin{eqnarray}
    \bar\U(p) &\doteq& \bar\cE(p) \cdot \dft(\bar\Z)(p), \\
    \bar\W^{j}(p) &\doteq& \bar\cC^j(p) \cdot \dft(\bar\Z)(p), \quad j=1,\ldots, k,
\end{eqnarray}
for each  $p \in \{0, \ldots, T-1\}$. Then, we have
\begin{eqnarray}
    \frac{1}{2T}\frac{\partial \log \det (\I + \alpha \cdot \circm(\bar\Z) \circm(\bar\Z)^{*} )}{\partial \bar\Z} &=& \idft(\bar\U), \\
    \frac{\gamma_j}{2T}\frac{\partial  \log\det (\I + \alpha_j \cdot \circm(\bar\Z) \bm \bar{\bm \Pi}^j \circm(\bar\Z)^{*})}{\partial \bar\Z}&=&
    \gamma_j \cdot \idft(\bar\W^{j} \bm \Pi^j).
\end{eqnarray}
\end{theorem}
\begin{proof}[Also proof to Theorem \ref{thm:1D-convolution} in the main body]

From \eqref{eqn:expand-directions}, \eqref{eq:dft-circulant-multichannel} and  \eqref{eq:bar-E-frequency}, we have 
\begin{gather}
    \frac{1}{2}\frac{\partial \log\det \Big(\I + \alpha \circm(\bar\Z)  \circm(\bar\Z)^{*} \Big)}{\partial\circm(\bar\z^i)}
    = 
    \bar\E  \circm(\bar\z^i)
    =  
    \bar\E  
    \left[\begin{smallmatrix}
        \F_T^* &  \cdots & \0 \\
        \vdots& \ddots & \vdots \\
        \0   &  \cdots & \F_T^* \\
    \end{smallmatrix}\right]
    \left[\begin{smallmatrix}
     \diag(\dft(\z^i[1]))\\
    \vdots\\
     \diag(\dft(\z^i[C]))
    \end{smallmatrix}\right]
    \F_T\\
    =
    \left[\begin{smallmatrix}
        \F_T^* &  \cdots & \0 \\
        \vdots& \ddots & \vdots \\
        \0   &  \cdots & \F_T^* \\
    \end{smallmatrix}\right]
    \cdot \P \cdot
    \alpha \cdot \left[\bm I + \alpha \cdot \sum_{i} \D(\bar\z^i)\right]^{-1} \cdot 
    \left[\begin{smallmatrix}
        \dft(\bar\z^i)(0)   & \cdots & \0\\
        \vdots          & \ddots & \vdots \\
        \0              & \cdots & \dft(\bar\z^i)(T-1)
    \end{smallmatrix}\right]
    \cdot \F_T \\
    = 
    \left[\begin{smallmatrix}
        \F_T^* &  \cdots & \0 \\
        \vdots& \ddots & \vdots \\
        \0   &  \cdots & \F_T^* \\
    \end{smallmatrix}\right]
    \cdot \P  \cdot 
    \left[\begin{smallmatrix}
        \bar\cE(0) \cdot \dft(\bar\z^i)(0)   & \cdots & \0\\
        \vdots          & \ddots & \vdots \\
        \0              & \cdots & \bar\cE(T-1) \cdot \dft(\bar\z^i)(T-1)
    \end{smallmatrix}\right]
    \cdot \F_T\\
    =
    \left[\begin{smallmatrix}
        \F_T^* &  \cdots & \0 \\
        \vdots& \ddots & \vdots \\
        \0   &  \cdots & \F_T^* \\
    \end{smallmatrix}\right]
    \cdot \P  \cdot   
    \left[\begin{smallmatrix}
        \bar\u^i(0)   & \cdots & \0\\
        \vdots          & \ddots & \vdots \\
        \0              & \cdots & \bar\u^i(T-1)
    \end{smallmatrix}\right]
    \cdot \F_T
    =
    \left[\begin{smallmatrix}
        \F_T^* &  \cdots & \0 \\
        \vdots& \ddots & \vdots \\
        \0   &  \cdots & \F_T^* \\
    \end{smallmatrix}\right]
    \cdot     
    \left[\begin{smallmatrix}
        \diag(\bar\u^i[1])  \\
        \vdots \\
        \diag(\bar\u^i[C]) \\
    \end{smallmatrix}\right]
    \cdot \F_T\\
    = \circm(\idft(\bar\u^i)).
\end{gather}
Therefore, we have
\begin{equation}
\begin{aligned}
    &\frac{1}{2}\frac{\partial \log\det \Big(\I + \alpha \cdot \circm(\bar\Z) \cdot \circm(\bar\Z)^{*} \Big)}{\partial \bar\z^i} \\
    =\,& \frac{1}{2}\frac{\partial \log\det \Big(\I + \alpha \cdot \circm(\bar\Z) \cdot \circm(\bar\Z)^{*} \Big)}{\partial\circm(\bar\z^i)} \cdot \frac{\partial\circm(\bar\z^i)}{\partial \bar\z^i} \\
    =\,& T \cdot \idft(\bar\u^i).
\end{aligned}
\end{equation}
By collecting the results for all $i$, we have
\begin{gather}
    \frac{\partial \frac{1}{2T}\log\det \Big(\I + \alpha \cdot \circm(\bar\Z) \cdot \circm(\bar\Z)^{*} \Big)}{\partial \bar\Z}
    = \idft(\bar\U).
\end{gather}
In a similar fashion, we get
\begin{equation}
    \frac{\partial \frac{\gamma_j}{2T}\log\det\Big(\I + \alpha_j \cdot \circm(\bar\Z) \cdot \bar{\bm{\Pi}}^{j} \cdot \circm(\bar\Z)^{*} \Big)}{\partial \bar\Z} 
    = \gamma_j \cdot \idft(\bar\W^{j} \cdot \bm{\Pi}^j).
\end{equation}

\end{proof}

By the above theorem, the gradient ascent update in \eqref{eqn:gradient-descent} (when applied to $\Delta R_\circm(\bar\Z, \bm{\Pi})$) can be equivalently expressed as an update in frequency domain on $\bar\V_\ell \doteq \dft(\bar\Z_\ell)$ as
\begin{equation}
    \bar\V_{\ell+1}(p) \; \propto \; \bar\V_{\ell}(p) + \eta \; \Big(\bar\cE_\ell(p) \cdot \bar\V_\ell(p) - \sum_{j=1}^k \gamma_j \bar\cC_\ell^j(p) \cdot \bar\V_\ell(p) \bm \Pi^j \Big), \quad p = 0, \ldots, T-1.
\end{equation}
Similarly, the gradient-guided feature map increment in \eqref{eqn:layer-approximate} can be equivalently expressed as an update in frequency domain on $\bar\bv_\ell \doteq \dft(\bar\z_\ell)$ as
\begin{equation}
\bar\bv_{\ell+1}(p) \propto \bar\bv_\ell(p) +  \eta \cdot  \bar\cE_{\ell}(p) \bar\bv_{\ell}(p) - \eta\cdot  \bm \sigma\Big([\bar\cC_{\ell}^{1}(p) \bar\bv_{\ell}(p), \dots, \bar\cC_{\ell}^{k}(p) \bar\bv_{\ell}(p)]\Big), \quad p = 0, \ldots, T-1,
\label{eqn:layer-approximate-spectral}
\end{equation}
subject to the constraint that $\|\bar\bv_{\ell+1}\|_F = \|\bar\z_{\ell+1}\|_F = 1$ (the first equality follows from Fact~\ref{fact:parseval}). 

We summarize the training, or construction to be more precise, of ReduNet in the spectral domain in Algorithm~\ref{alg:training-1D}.

\begin{algorithm}[t]
	\caption{\textbf{Training Algorithm} (1D Signal, Shift Invariance, Spectral Domain)}
	\label{alg:training-1D}
	\begin{algorithmic}[1]
		\REQUIRE $\bar\Z \in \Re^{C \times T \times m}$, $\bm{\Pi}$, $\epsilon > 0 $, $\lambda$, and a learning rate $\eta$.
		\STATE Set $\alpha = \frac{C}{m \epsilon^2}$, $\{\alpha_j = \frac{C}{\textsf{tr}\left(\bm{\Pi}^{j}\right)\epsilon^{2}}\}_{j=1}^k$, 
		$\{\gamma_j = \frac{\textsf{tr}\left(\bm{\Pi}^{j}\right)}{m}\}_{j=1}^k$.
		\STATE Set $\bar\V_1 = \{\bar\bv_1^{i}(p) \in \Co^C\}_{p=0,i=1}^{T-1, m}\doteq \dft(\bar\Z) \in \Co^{C \times T \times m}$.
		\FOR{$\ell = 1, 2, \dots, L$} 
		    \STATE {\texttt{\# Step 1:Compute network parameters} $\cE_\ell$ \texttt{and} $\{\cC_\ell^j\}_{j=1}^k$.}
    		\FOR{$p = 0, 1, \dots, T-1$} 
        		\STATE 
        		$\bar\cE_\ell(p) \doteq \alpha \cdot \left[\I + \alpha \cdot \bar\V_{\ell}(p) \cdot \bar\V_{\ell}(p)^* \right]^{-1}$, \\
        		$\bar\cC_{\ell}^j(p) \doteq \alpha_j \cdot\left[\I + \alpha_j \cdot \bar\V_{\ell}(p) \cdot \bm{\Pi}^j \cdot \bar\V_{\ell}(p)^*\right]^{-1}$;
    		\ENDFOR
		    \STATE {\texttt{\# Step 2:Update feature} $\bar\V$.}
    		\FOR{$i=1, \ldots, m$}
    		    \STATE \texttt{\# Compute  (approximately) projection at each frequency $p$.}
        		\FOR{$p = 0, 1, \dots, T-1$}     
    		        \STATE Compute $\{\bar\p_\ell^{ij} (p) \doteq \bar\cC^j_{\ell}(p) \cdot \bar\bv_\ell^{i}(p) \in \Co^{C \times 1}\}_{j=1}^k$;
        		\ENDFOR    	
    		    \STATE \texttt{\# Compute overall (approximately) projection by aggregating over frequency $p$.}
        		\STATE Let $\{\bar\P_\ell^{ij} = [\bar\p_\ell^{ij} (0), \ldots, \bar\p_\ell^{ij} (T-1)] \in \Co^{C \times T}\}_{j=1}^k$;
    		    \STATE \texttt{\# Compute soft   assignment from  (approximately) projection.}
            	\STATE Compute $\Big\{\widehat{\bm \pi}_\ell^{ij} = \frac{\exp(-\lambda \|\bar\P_\ell^{ij}\|_F)}{\sum_{j=1}^k \exp(-\lambda \|\bar\P_\ell^{ij}\|_F)}\Big\}_{j=1}^k$;
    		    \STATE \texttt{\# Compute update at each frequency $p$.}     	
        		\FOR{$p = 0, 1, \dots, T-1$}         		
            		\STATE $\bar\bv_{\ell+1}^{i}(p) = \mathcal{P}_{\mathbb{S}^{n-1}} \left( \bar\bv_{\ell}^{i}(p) + \eta \left(\bar\cE_\ell(p) \bar\bv_\ell^{i}(p) - \sum_{j=1}^k \gamma_j \cdot \widehat{\bm \pi}_\ell^{ij} \cdot \bar\cC_\ell^j(p) \cdot \bar\bv_\ell^{i}(p)\right)\right)$;
        		\ENDFOR
        	\ENDFOR
    		\STATE Set $\bar\Z_{\ell+1} = \idft(\bar\V_{\ell+1})$ as the feature at the $\ell$-th layer;	
		\ENDFOR
		\ENSURE features $\bar\Z_{L+1}$, the learned filters $\{\bar\E_\ell(p)\}_{\ell, p}$ and  $\{\bar\cC^j_\ell(p)\}_{j, \ell, p}$.
	\end{algorithmic}
\end{algorithm}

\section{ReduNet for 2D Circular Translation Invariance}\label{ap:2D-translation}
To a large degree, both conceptually and technically, the 2D case is very similar to the 1D case that we have studied carefully in the previous Appendix \ref{app:1D}. For the sake of consistency and completeness, we here give a brief account.

\subsection{Doubly Block Circulant Matrix}

In this section, we consider $\z$ as a 2D signal such as an image, and use $H$ and $W$ to denote its ``height'' and ``width'', respectively. 
It will be convenient to work with both a matrix representation
\begin{equation}
\z = 
\begin{bmatrix}
z(0, 0) & z(0, 1) & \cdots & z(0, W-1)\\
z(1, 0) & z(1, 1) & \cdots & z(1, W-1)\\
\vdots       & \vdots       & \ddots & \vdots        \\
z(H-1, 0) & z(H-1, 1) & \cdots & z(H-1, W-1)\\
\end{bmatrix} \quad \in \Re^{H \times W},
\end{equation}
as well as a vector representation
\begin{equation}
\begin{aligned}
\vec(\z)\doteq
\Big[z(0, 0), &\ldots, z(0, W-1), z(1, 0), \ldots,\\
&z(1, W-1), 
\ldots, z(H-1, 0), \ldots, z(H-1, W-1)\Big]^* \in \Re^{(H \times W)}.
\end{aligned}  
\end{equation}
We represent the circular translated version of $\z$ as $\trans_{p, q}(\z) \in \Re^{H \times W}$ by an amount of $p$ and $q$ on the vertical and horizontal directions, respectively. That is, we let
\begin{equation}
    \trans_{p, q}(\z) (h, w) \doteq \z(h - p \mod H, w - q \mod W),
\end{equation}
where $\forall (h, w) \in \{0, \ldots, H-1\} \times \{0, \ldots, W-1\}$. It is obvious that $\trans_{0, 0}(\z) = \z$. 
Moreover, there is a total number of $H \times W$ distinct translations given by $\{\trans_{p, q}(\z), (p, q) \in \{0, \ldots, H-1\} \times \{0, \ldots, W-1\}\}$.  We may arrange the vector representations of them into a matrix and obtain 
\begin{multline}
    \circm(\z) \doteq \Big[\vec(\trans_{0, 0}(\z)), \ldots, \vec(\trans_{0, W-1}(\z)), \vec(\trans_{1, 0}(\z)), \ldots, \vec(\trans_{1, W-1}(\z)),\\
    \ldots, \\
    \vec(\trans_{H-1, 0}(\z)), \ldots, \vec(\trans_{H-1, W-1}(\z))\Big] \in \Re^{(H\times W)\times (H\times W)}.
\end{multline}
The matrix $\circm(\z)$ is known as the \emph{doubly block circulant matrix} associated with $\z$ (see, e.g., \cite{abidi2016optimization,sedghi2018singular}). 

We now consider a multi-channel 2D signal represented as a tensor $\bar\z \in \Re^{C \times H \times W}$, where $C$ is the number of channels. 
The $c$-th channel of $\bar\z$ is represented as $\bar\z[c] \in \Re^{H \times W}$, and the $(h, w)$-th pixel is represented as $\bar\z(h, w) \in \Re^C$. 
To compute the coding rate reduction for a collection of such multi-channel 2D signals, we may flatten the tenor representation into a vector representation by concatenating the vector representation of each channel, i.e., we let
\begin{equation}
    \vec(\bar\z) = [\vec(\bar\z[1]) ^*, \ldots, \vec(\bar\z[C])^*]^* \quad \in \Re^{(C \times H \times W)}
\end{equation}
Furthermore, to obtain shift invariance for coding rate reduction, we may generate a collection of translated versions of $\bar\z$ (along two spatial dimensions). Stacking the vector representation for such translated copies as column vectors, we obtain
\begin{equation}
    \circm(\bar\z) \doteq 
    \begin{bmatrix}
    \circm(\bar\z[1])\\
    \vdots\\
    \circm(\bar\z[C])
    \end{bmatrix} \quad 
    \in \Re^{(C \times H \times W) \times (H \times W)}.
\end{equation}

We can now define a \emph{translation invariant coding rate reduction} for multi-channel 2D signals. 
Consider a collection of $m$ multi-channel 2D signals $\{\bar\z^i \in \Re^{C \times H \times W}\}_{i=1}^m$. 
Compactly representing the data by $\bar\Z \in \Re^{C \times H \times W \times m}$ where the $i$-th slice on the last dimension is $\bar\z^i$, we denote
\begin{equation}
    \circm(\bar\Z) = [\circm(\bar\z^1), \ldots, \circm(\bar\z^m)] \quad \in \Re^{(C\times H \times W) \times (H \times W \times m)}.
\end{equation}
Then, we define
\begin{multline}
\label{eq:2D-MCR2}
    \Delta R_\circm(\bar\Z, \bm{\Pi}) \doteq \frac{1}{HW}\Delta R(\circm(\bar\Z), \bar{\bm{\Pi}}) = 
    \frac{1}{2HW}\log\det \Bigg(\I + \alpha \cdot \circm(\bar\Z) \cdot \circm(\bar\Z)^{*} \Bigg) \\
    - \sum_{j=1}^{k}\frac{\gamma_j}{2HW}\log\det\Bigg(\I + \alpha_j \cdot \circm(\bar\Z) \cdot \bar{\bm{\Pi}}^{j} \cdot \circm(\bar\Z)^{*} \Bigg),
\end{multline}
where $\alpha = \frac{CHW}{mHW\epsilon^{2}} = \frac{C}{m\epsilon^{2}}$, $\alpha_j = \frac{CHW}{\textsf{tr}\left(\bm{\Pi}^{j}\right)HW\epsilon^{2}} = \frac{C}{\textsf{tr}\left(\bm{\Pi}^{j}\right)\epsilon^{2}}$, $\gamma_j = \frac{\textsf{tr}\left(\bm{\Pi}^{j}\right)}{m}$, and $\bar{\bm \Pi}^j$ is augmented membership matrix in an obvious way.

By following an analogous argument as in the 1D case, one can show that ReduNet for multi-channel 2D signals naturally gives rise to the multi-channel 2D circulant convolution operations. 
We omit the details, and focus on the construction of ReduNet in the frequency domain.

\subsection{Fast Computation in Spectral Domain}

\paragraph{Doubly block circulant matrix and 2D-DFT. }

Similar to the case of circulant matrices for 1D signals, all doubly block circulant matrices share the same set of eigenvectors, and these eigenvectors form a unitary matrix given by 
\begin{equation}
    \F \doteq \F_H \otimes \F_W  \quad \in \Co^{(H\times W) \times (H \times W)},
\end{equation}
where $\otimes$ denotes the Kronecker product and $\F_H, \F_W$ are defined as in \eqref{eq:dft-matrix}.

Analogous to Fact \ref{fact:dft}, $\F$ defines 2D-DFT as follows.

\begin{fact}[2D-DFT as matrix-vector multiplication]\label{fact:2D-dft}
The 2D-DFT of a signal $\z \in \Re^{H \times W}$ can be computed as
    \begin{equation}
        \vec(\dft(\z)) \doteq \F \cdot \vec(\z) \quad \in \Co^{(H \times W)},
    \end{equation}
where $\forall (p, q) \in \{0, \ldots, H-1\} \times \{0, \ldots, W-1\}$,
\begin{equation}\label{eq:2d-dft}
    \dft(\z)(p, q) = \frac{1}{\sqrt{H \cdot W}} \sum_{h=0}^{H-1} \sum_{w=0}^{W-1} \z(h, w) \cdot \omega_H ^{p \cdot h} \omega_W ^{q\cdot w}.
\end{equation}
The 2D-IDFT of a signal $\bv \in \Co^{H \times W}$ can be computed as
\begin{equation} 
    \vec(\idft(\bv)) \doteq \F_T^* \cdot \vec(\bv) \quad \in \Co^{(H \times W)},
\end{equation}
where $\forall (h, w) \in \{0, \ldots, H-1\} \times \{0, \ldots, W-1\}$,
\begin{equation}
\begin{aligned}
    \idft(\bv)(h, w) =& \frac{1}{\sqrt{H \cdot W}} \sum_{p=0}^{H-1}\sum_{q=0}^{W-1} v(p, q) \cdot \omega_H ^{-p \cdot h}\omega_W ^{-q \cdot w}.
\end{aligned}
\end{equation}
\end{fact}

Analogous to Fact \ref{fact:2d-dft-circulant}, $\F$ relates $\dft(\z)$ and $\circm(\z)$ as follows.
\begin{fact}[2D-DFT are eigenvalues of the doubly block circulant matrix] Given a signal $\z \in \Co^{H \times W}$, we have
\begin{equation}\label{eq:dft-circulant-v1}
    \F \cdot \circm(\z) \cdot \F^* = \diag(\vec(\dft(\z)))  \quad \mbox{or} \quad \circm(\z) = \F^* \cdot \diag(\vec(\dft(\z))) \cdot \F.
\end{equation}
\label{fact:2d-dft-circulant}
\end{fact}

\paragraph{Doubly block circulant matrix and 2D-DFT for multi-channel signals.}

We now consider multi-channel 2D signals $\bar\z \in \Re^{C \times H \times W}$. 
Let $\dft(\bar\z) \in \Co^{C \times H \times W}$ be a matrix where the $c$-th slice on the first dimension is the DFT of the corresponding signal $\z[c]$.
That is, $\dft(\bar\z)[c] = \dft(\z[c]) \in \Co^{H \times W}$. 
We use $\dft(\bar\z)(p, q) \in \Co^{C}$ to denote slicing of $\bar\z$ on the frequency dimensions. 

By using Fact~\ref{fact:2d-dft-circulant}, $\circm(\bar\z)$ and $\dft(\bar\z)$ are related as follows:
\begin{equation}
\begin{aligned}
\circm(\bar\z) 
&=
\begin{bmatrix}
\F^* \cdot \diag(\vec(\dft(\z[1])))  \cdot \F\\
\vdots \\
\F^* \cdot \diag(\vec(\dft(\z[C]))) \cdot \F\\
\end{bmatrix} 
\\
&=
\begin{bmatrix}
    \F^*  & \cdots & \0 \\
    \0    & \cdots & \0 \\
    \vdots  & \ddots & \vdots \\
    \0   & \cdots & \F^* \\
\end{bmatrix}
\cdot
\begin{bmatrix}
    \diag(\vec(\dft(\z[1])))  \\
    \diag(\vec(\dft(\z[2])))  \\
    \vdots \\
    \diag(\vec(\dft(\z[C])))  \\
\end{bmatrix}
\cdot \F.
\end{aligned}
\end{equation}
Similar to the 1D case, this relation can be leveraged to produce a fast implementation of ReduNet in the spectral domain. 

\paragraph{Translation-invariant ReduNet in the Spectral Domain. }

Given a collection of multi-channel 2D signals $\bar\Z \in \Re^{C \times H \times W \times m}$, we denote
\begin{equation}
    \dft(\bar\Z)(p, q) \doteq [\dft(\bar\z^1)(p, q), \ldots, \dft(\bar\z^m)(p, q)] \quad \in \Re^{C \times m}.
\end{equation}
We introduce the notations $\bar\cE(p, q) \in \Re^{C \times C \times H \times W}$ and $\bar\cC^j(p, q) \in \Re^{C \times C \times H \times W}$ given by
\begin{eqnarray}
    \bar\cE(p, q) &\doteq& \alpha \cdot \left[\I + \alpha \cdot \dft(\bar\Z)(p, q) \cdot \dft(\bar\Z)(p, q)^* \right]^{-1} \quad \in \Co^{C\times C}, \\
    \bar\cC^j(p, q) &\doteq& \alpha_j \cdot\left[\I + \alpha_j \cdot \dft(\bar\Z)(p, q) \cdot \bm{\Pi}_j \cdot \dft(\bar\Z)(p, q)^*\right]^{-1} \quad \in \Co^{C\times C}.
\end{eqnarray}
In above, $\bar\cE(p, q)$ (resp., $\bar\cC^j(p, q)$) is the $(p, q)$-th slice of $\bar\cE$ (resp., $\bar\cC^j$) on the last two dimensions. 
Then, the gradient of $\Delta R_\circm(\bar\Z, \bm{\Pi})$ with respect to $\bar\Z$ can be calculated by the following result.

\begin{theorem}[Computing multi-channel 2D convolutions $\bar{\bm E}$ and $\bar{\bm C}^j$]\label{thm:2D-convolution}
Suppose $\bar\U \in \Co^{C \times H \times W \times m}$ and $\bar\W^{j} \in \Co^{C \times H \times W \times m}, j=1,\ldots, k$ are given by 
\begin{eqnarray}
    \bar\U(p, q) &\doteq& \bar\cE(p, q) \cdot \dft(\bar\Z)(p, q), \\
    \bar\W^{j}(p, q) &\doteq& \bar\cC^j(p, q) \cdot \dft(\bar\Z)(p, q), \quad j=1,\ldots, k,
\end{eqnarray}
for each  $(p, q) \in \{0, \ldots, H-1\}\times \{0, \ldots, W-1\}$. 
Then, we have
\begin{eqnarray}
    \frac{1}{2HW}\frac{\partial \log \det (\I + \alpha \cdot \circm(\bar\Z) \circm(\bar\Z)^{*} )}{\partial \bar\Z} &=& \idft(\bar\U), \\
    \frac{1}{2HW}\frac{\partial \left( \gamma_j  \log \det (\I + \alpha_j \cdot \circm(\bar\Z) \bm \bar{\bm \Pi}^j \circm(\bar\Z)^{*} )  \right)}{\partial \bar\Z}&=&
    \gamma_j \cdot \idft(\bar\W^{j} \bm \Pi^j).
\end{eqnarray}
\end{theorem}
This result shows that the calculation of the derivatives for the 2D case is analogous to that of the 1D case. 
Therefore, the construction of the ReduNet for 2D translation invariance can be performed using Algorithm~\ref{alg:training-1D} with straightforward extensions. 

\newpage

\section{Additional Simulations and Experiments for MCR$^2$}\label{ap:additional-exp}

\subsection{Simulations - Verifying Diversity Promoting Properties of MCR$^2$}
As proved in Theorem~\ref{thm:maximal-rate-reduction}, the proposed MCR$^2$ objective promotes within-class diversity. In this section, we use simulated data to verify the diversity promoting property of MCR$^2$. As shown in Table~\ref{table:simulations}, we calculate our proposed MCR$^2$ objective on simulated data. We observe that orthogonal subspaces with \textit{higher} dimension achieve higher MCR$^2$ value, which is consistent with our theoretical analysis in Theorem~\ref{thm:maximal-rate-reduction}.

\begin{table}[h]
\begin{center}
\begin{small}
\begin{sc}
\begin{tabular}{l | c c c c c}
\toprule
& ${R}$ &  ${R}_{c}$ &  $\Delta{R}$ &   Orthogonal? & Output Dimension\\
\midrule
Random Gaussian & 552.70 & 193.29 & 360.41 & {\cmark} & 512 \\
Subspace  ($d_j = 50$) & 545.63 & 108.46 & \textbf{437.17} & {\cmark} & 512 \\
Subspace  ($d_j = 40$)  & 487.07 & 92.71 & 394.36 & {\cmark} & 512 \\
Subspace  ($d_j = 30$) & 413.08 & 74.84 & 338.24 & {\cmark} & 512 \\
Subspace  ($d_j = 20$) & 318.52 &	54.48 &	264.04 & {\cmark} & 512 \\
Subspace  ($d_j = 10$) & 195.46 &	30.97 &	164.49 & {\cmark} & 512 \\
Subspace  ($d_j = 1$) & 31.18 &	4.27 &	26.91 & {\cmark} & 512 \\
\midrule
Random Gaussian & 292.71 &	154.13 &	138.57 & {\cmark} & 256 \\
Subspace  ($d_j = 25$) & 288.65 &	56.34 &	\textbf{232.31} & {\cmark} & 256 \\
Subspace  ($d_j = 20$)  & 253.51 &	47.58 &	205.92 & {\cmark} & 256 \\
Subspace  ($d_j = 15$) & 211.97 &	38.04 &	173.93 & {\cmark} & 256 \\
Subspace  ($d_j = 10$) & 161.87 &	27.52 &	134.35 & {\cmark} & 256 \\
Subspace  ($d_j = 5$) & 98.35 &	15.55 &	82.79 & {\cmark} & 256 \\
Subspace  ($d_j = 1$) & 27.73 &	3.92 &	23.80 & {\cmark} & 256 \\
\midrule
Random Gaussian & 150.05 &	110.85 &	39.19 & {\cmark} & 128 \\
Subspace  ($d_j = 12$) & 144.36 &	27.72 &	\textbf{116.63} & {\cmark} & 128 \\
Subspace  ($d_j = 10$) & 129.12 &	24.06 &	105.05 & {\cmark} & 128 \\
Subspace  ($d_j = 8$) & 112.01 &	20.18 &	91.83 & {\cmark} & 128 \\
Subspace  ($d_j = 6$) & 92.55 &	16.04 &	76.51 & {\cmark} & 128 \\
Subspace  ($d_j = 4$) & 69.57 &	11.51 &	58.06 & {\cmark} & 128 \\
Subspace  ($d_j = 2$) & 41.68 &	6.45 &	35.23 & {\cmark} & 128 \\
Subspace  ($d_j = 1$) & 24.28 &	3.57 &	20.70 & {\cmark} & 128 \\
\midrule
Subspace  ($d_j = 50$) & 145.60 &	75.31 &	70.29 & {\xmark} & 128 \\
Subspace  ($d_j = 40$) & 142.69 &	65.68 &	77.01 & {\xmark} & 128 \\
Subspace  ($d_j = 30$) & 135.42 &	54.27 &	81.15 & {\xmark} & 128 \\
Subspace  ($d_j = 20$) & 120.98 &	40.71 &	80.27 & {\xmark} & 128 \\
Subspace  ($d_j = 15$) & 111.10 &	32.89 &	78.21 & {\xmark} & 128 \\
Subspace  ($d_j = 12$) & 101.94 &	27.73 &	74.21 & {\xmark} & 128 \\
\bottomrule
\end{tabular}
\end{sc}
\end{small}
\caption{\small \textbf{MCR$^2$ objective on simulated data.} We evaluate the proposed MCR$^2$ objective defined in \eqref{eqn:maximal-rate-reduction}, including ${R}$,  ${R}_{c}$, and  $\Delta{R}$, on simulated data. The output dimension $d$ is set as 512, 256, and 128. We set the batch size as $m=1000$ and random assign the label of each sample from $0$ to $9$, i.e., 10 classes.  We generate two types of data: 1) (\textsc{Random Gaussian}) For comparison with data without structures, for each class we generate random vectors sampled from Gaussian distribution (the dimension is set as the output dimension $d$) and normalize each vector to be on the unit sphere. 2) (\textsc{Subspace}) For each class, we generate vectors sampled from its corresponding subspace with  dimension $d_j$ and normalize each vector to be on the unit sphere. We consider the subspaces from different classes are orthogonal/nonorthogonal to each other.}
\label{table:simulations}
\end{center}
\vskip -0.3in
\end{table}

\subsection{Implementation Details}\label{sec:appendix-mcr2-exp}
\paragraph{Training Setting.} We mainly use ResNet-18~\citep{he2016deep} in our experiments, where we use 4 residual blocks with layer widths $\{64, 128, 256, 512\}$.  The implementation of network architectures used in this paper are mainly based on this github repo.\footnote{\url{https://github.com/kuangliu/pytorch-cifar}} For data augmentation in the supervised setting, we apply the \texttt{RandomCrop} and \texttt{RandomHorizontalFlip}. For the supervised setting, we train the models for 500 epochs and use stage-wise learning rate decay every 200 epochs (decay by a factor of 10). For the supervised setting, we train the models for 100 epochs and use stage-wise learning rate decay at 20-th epoch and 40-th epoch (decay by a factor of 10).

\paragraph{Evaluation Details.} For the supervised setting, we set the number of principal components for nearest subspace classifier $r_j = 30$. We also study the effect of $r_j$ in Section~\ref{sec:appendix-subsec-sup}. For the CIFAR100 dataset, we consider 20 superclasses and set the cluster number as 20, which is the same setting as in \cite{chang2017deep, wu2018unsupervised}.

\paragraph{Datasets.} We apply the default datasets in PyTorch, including CIFAR10, CIFAR100, and STL10.

\paragraph{Augmentations $\mathcal{T}$ used for the self-supervised setting.} We apply the same data augmentation for CIFAR10 dataset and CIFAR100 dataset and the pseudo-code is as follows.

\begin{tcolorbox}
\begin{footnotesize}
\begin{verbatim}
import torchvision.transforms as transforms
TRANSFORM = transforms.Compose([
    transforms.RandomResizedCrop(32),
    transforms.RandomHorizontalFlip(),
    transforms.RandomApply([transforms.ColorJitter(0.4, 0.4, 0.4, 0.1)], p=0.8),
    transforms.RandomGrayscale(p=0.2),
    transforms.ToTensor()])
\end{verbatim}
\end{footnotesize}
\end{tcolorbox}

\noindent
The augmentations we use for STL10 dataset and the pseudo-code is as follows.

\begin{tcolorbox}
\begin{footnotesize}
\begin{verbatim}
import torchvision.transforms as transforms
TRANSFORM = transforms.Compose([
    transforms.RandomResizedCrop(96),
    transforms.RandomHorizontalFlip(),
    transforms.RandomApply([transforms.ColorJitter(0.8, 0.8, 0.8, 0.2)], p=0.8),
    transforms.RandomGrayscale(p=0.2),
    GaussianBlur(kernel_size=9),
    transforms.ToTensor()])
\end{verbatim}
\end{footnotesize}
\end{tcolorbox}

\paragraph{Cross-entropy training details.} For CE models presented in Table~\ref{table:label-noise}, Figure \ref{fig:pca-ce-1}-\ref{fig:pca-ce-3}, and Figure~\ref{fig:heatmap-plot}, we use the same network architecture, ResNet-18~\citep{he2016deep}, for cross-entropy training on CIFAR10, and set the output dimension as 10 for the last layer.  We apply SGD, and set learning rate $\texttt{lr=0.1}$, momentum $\texttt{momentum=0.9}$, and weight decay $\texttt{wd=5e{-4}}$. We set the total number of training epoch as 400, and use stage-wise learning rate decay every 150 epochs (decay by a factor of 10).

\subsection{Additional Experimental Results}

\subsubsection{PCA Results of MCR$^2$ Training versus Cross-Entropy Training}\label{sec:subsec-pca}

\begin{figure*}[h]
\subcapcentertrue
  \begin{center}
    \subfigure[\label{fig:pca-mcr-1}PCA: MCR$^2$ training learned features for overall data (first 30 components).]{\includegraphics[width=0.31\textwidth]{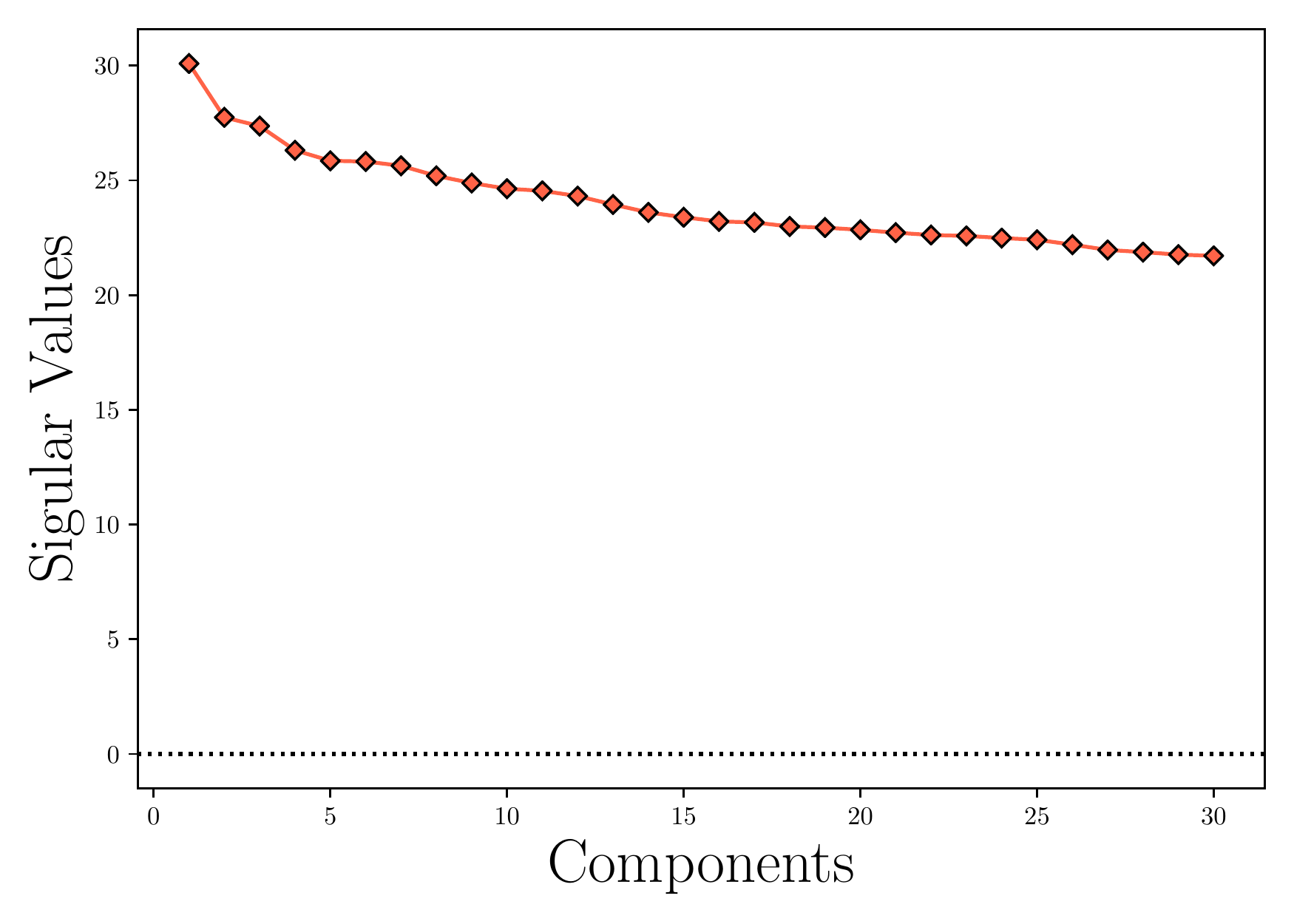}}
    \subfigure[\label{fig:pca-mcr-2}PCA: MCR$^2$ training learned features for overall data.]{\includegraphics[width=0.31\textwidth]{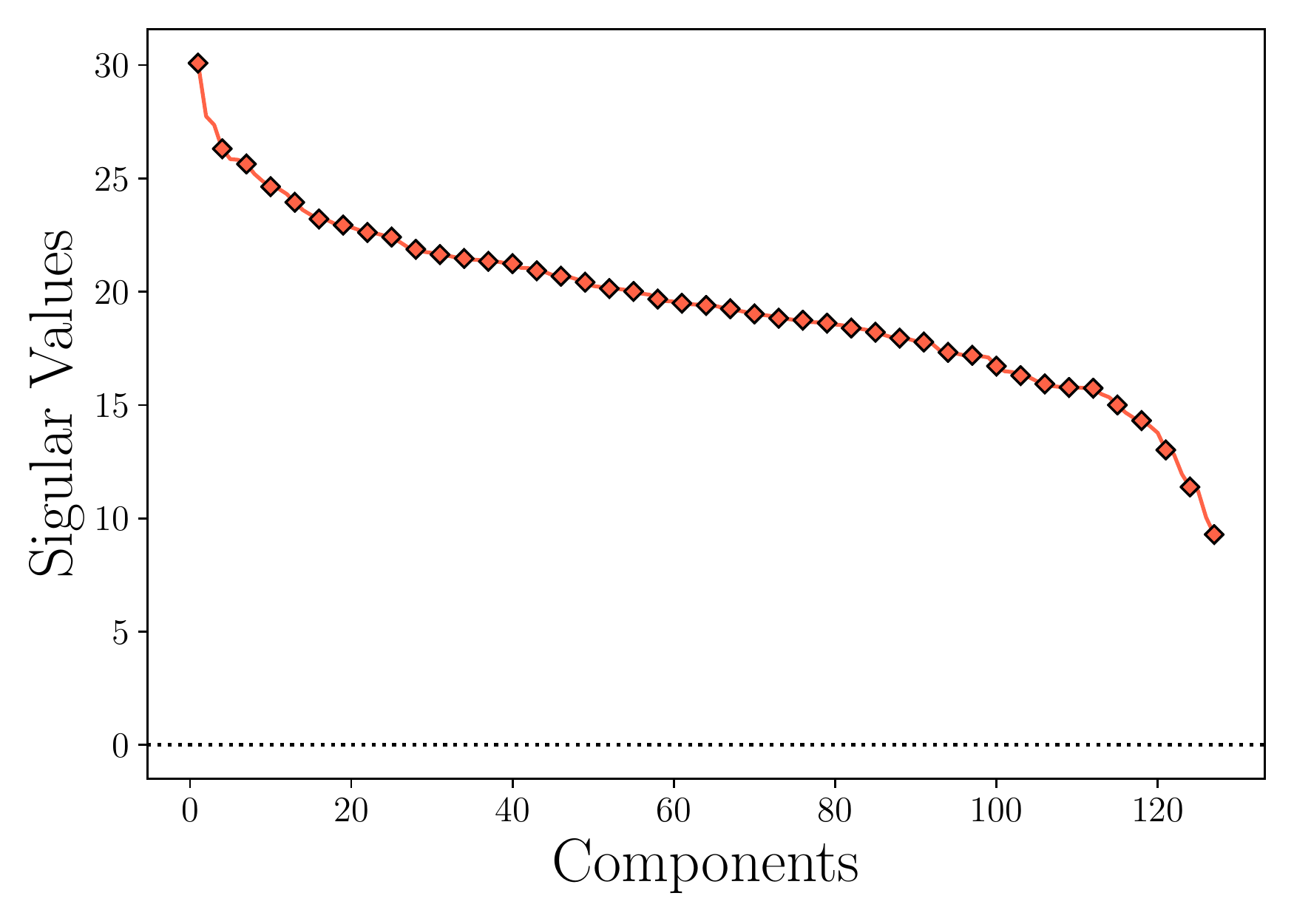}}
    \subfigure[\label{fig:pca-mcr-3}PCA: MCR$^2$ training learned features for every class.]{\includegraphics[width=0.31\textwidth]{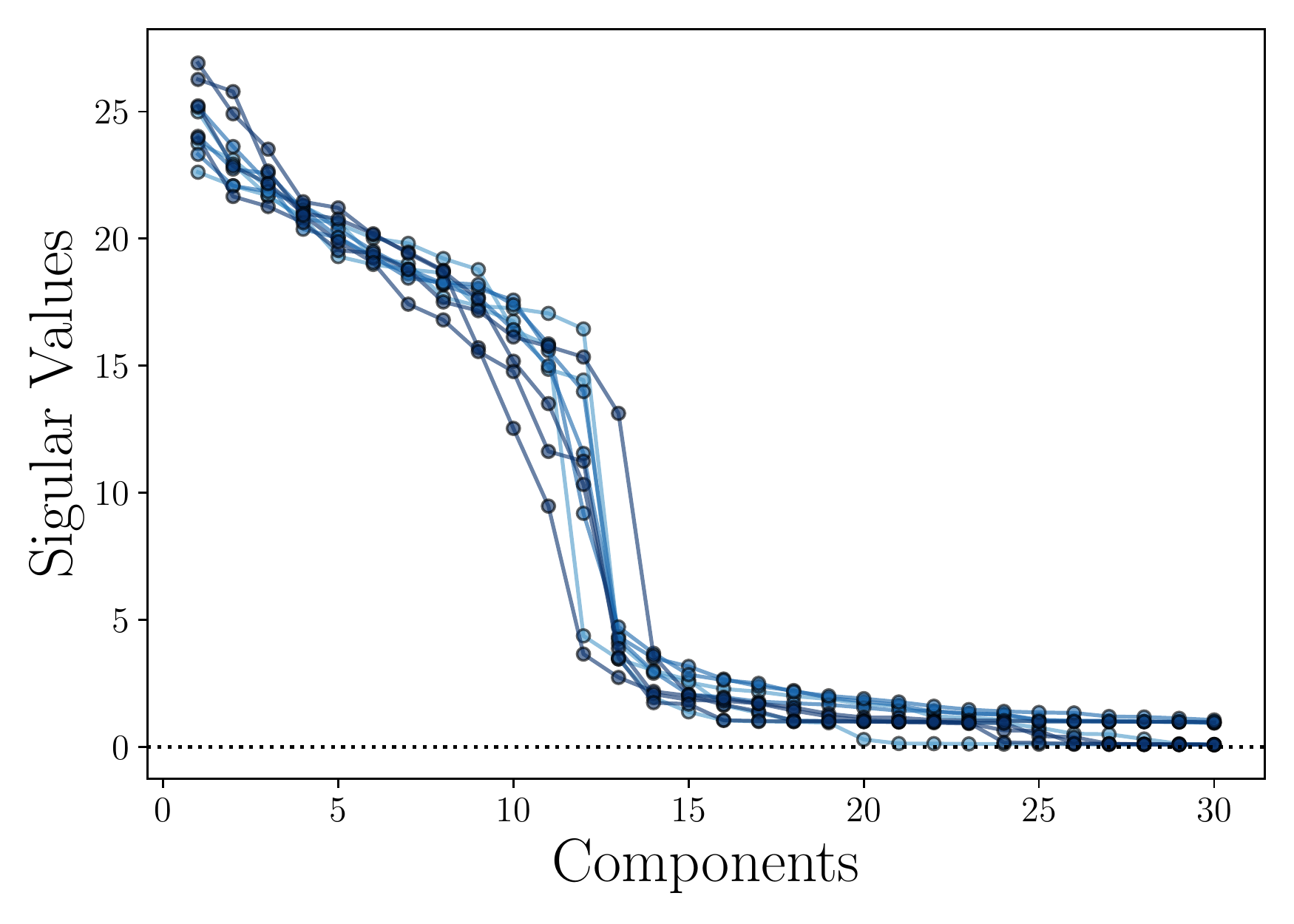}}
    \subfigure[\label{fig:pca-ce-1}PCA: cross-entropy training learned features for overall data (first 30 components).]{\includegraphics[width=0.31\textwidth]{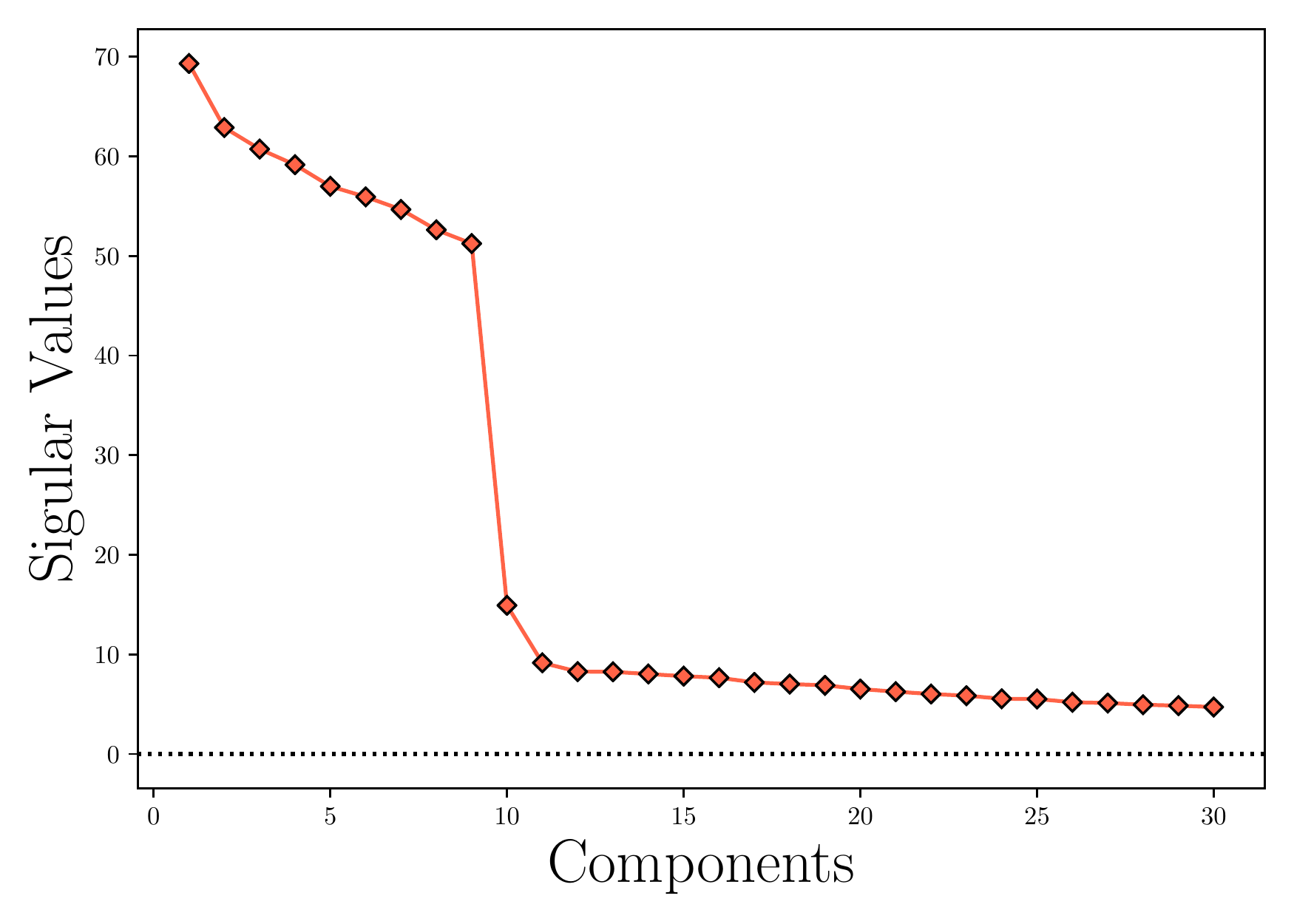}}
    \subfigure[\label{fig:pca-ce-2}PCA: cross-entropy training learned features for overall data.]{\includegraphics[width=0.31\textwidth]{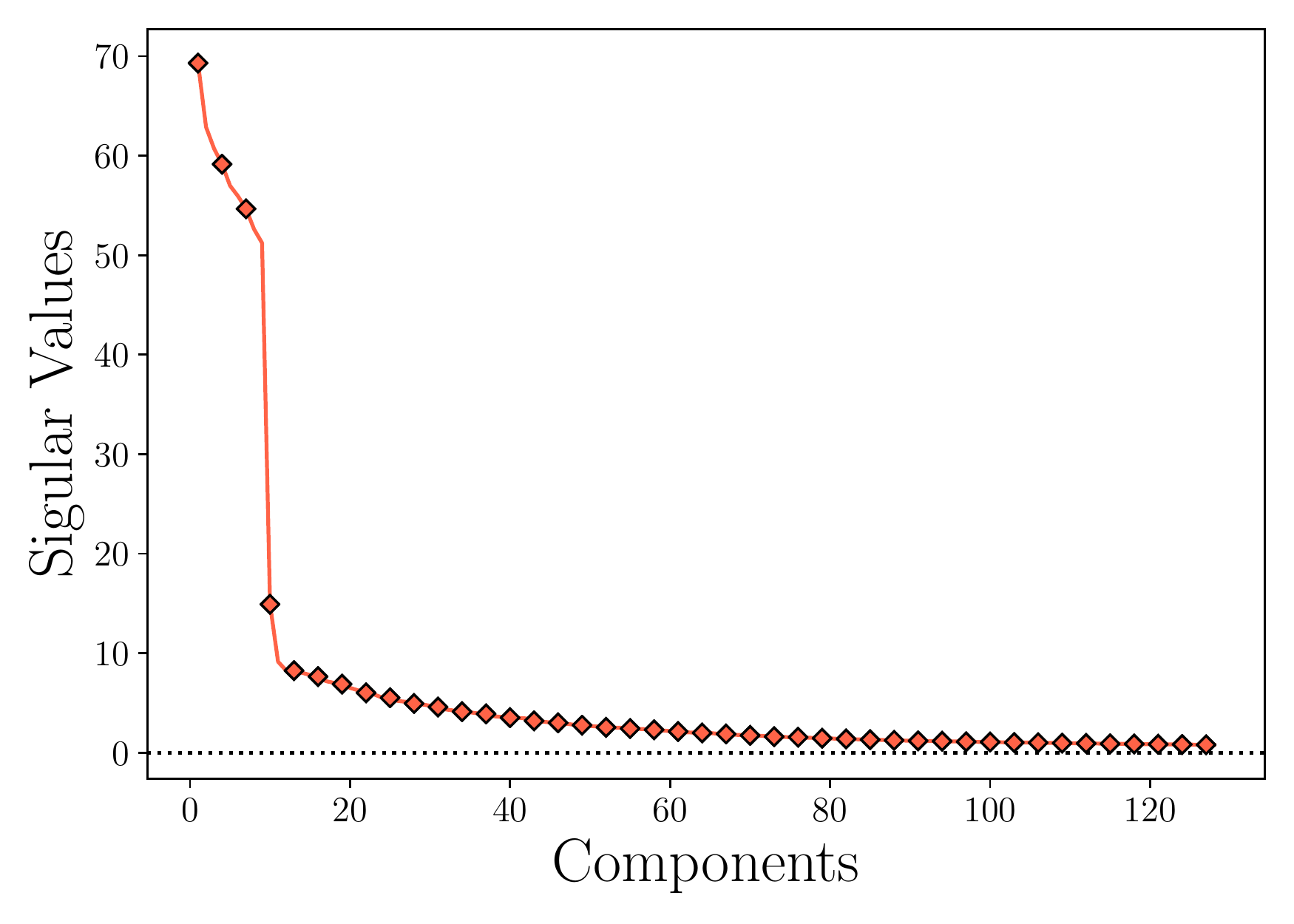}}
    \subfigure[\label{fig:pca-ce-3}PCA: cross-entropy training learned features for every class.]{\includegraphics[width=0.31\textwidth]{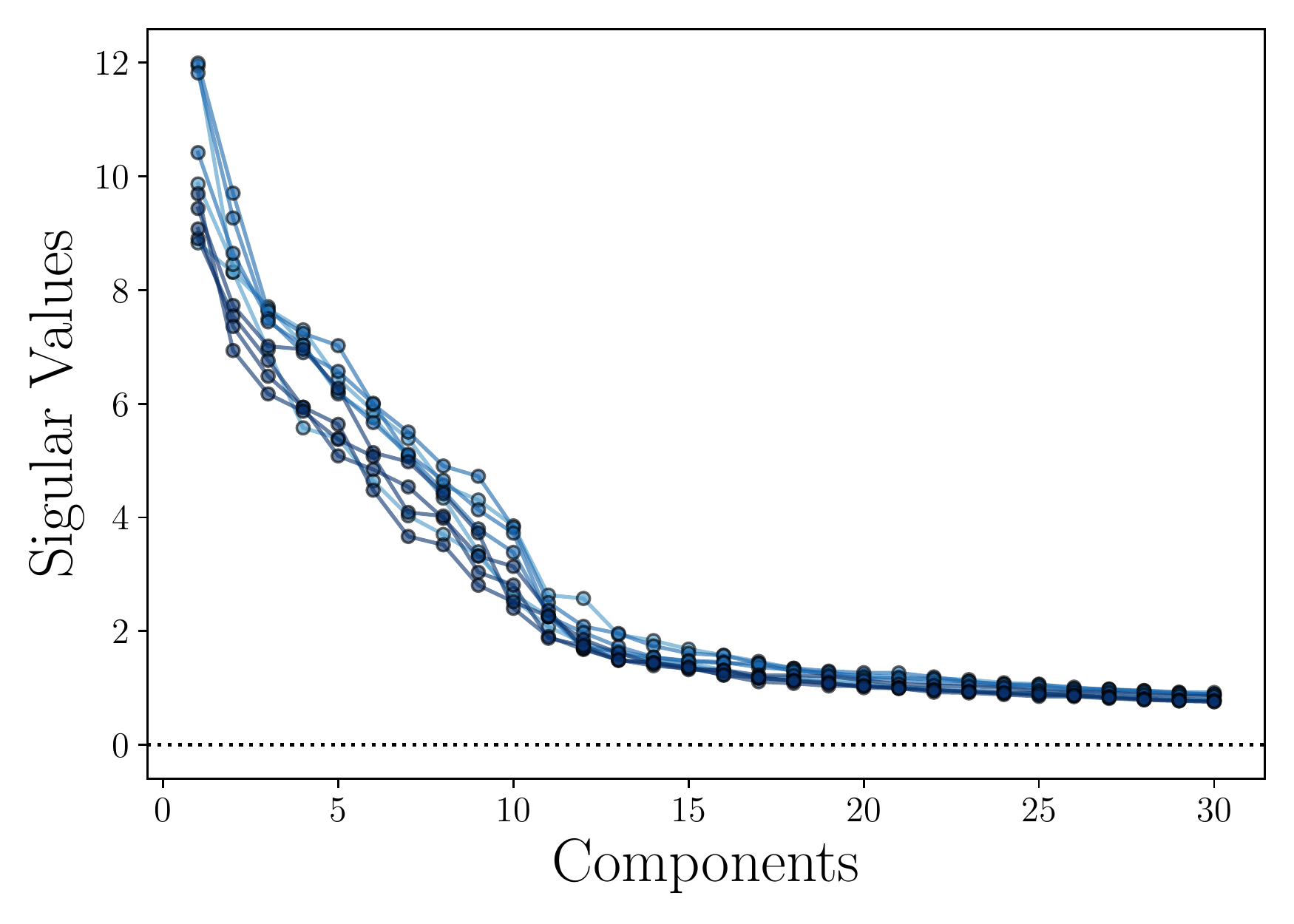}}
    \vskip -0.05in
    \caption{\small Principal component analysis (PCA) of learned representations for the  MCR$^2$ trained model (\textbf{first row}) and the cross-entropy trained model (\textbf{second row}).}
    \label{fig:pca-plot}
  \end{center}
\end{figure*}

\begin{figure*}[h]
  \begin{center}
    \subfigure{\includegraphics[width=0.47\textwidth]{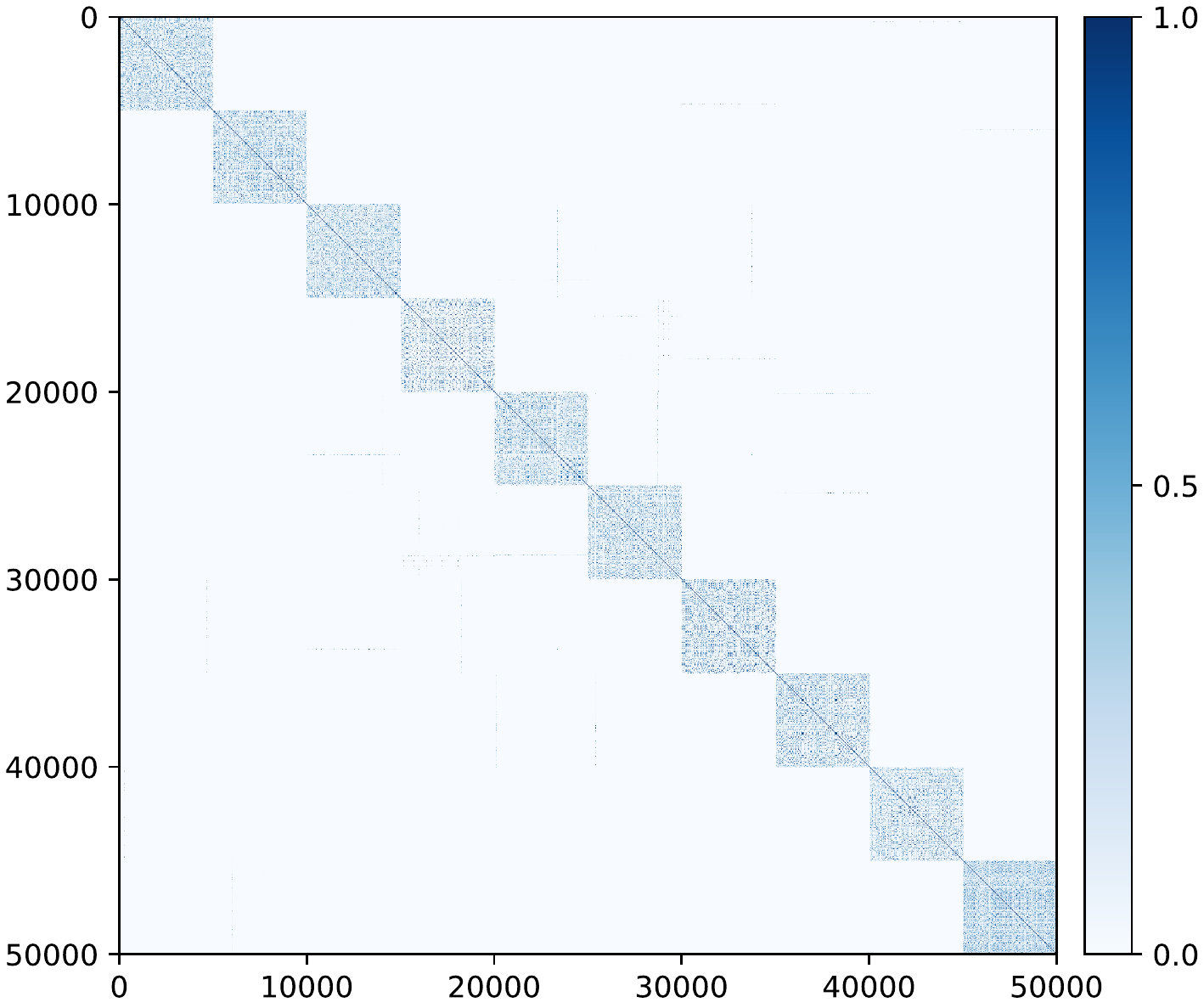}}
    \hspace{0.25cm}
    \subfigure{\includegraphics[width=0.47\textwidth]{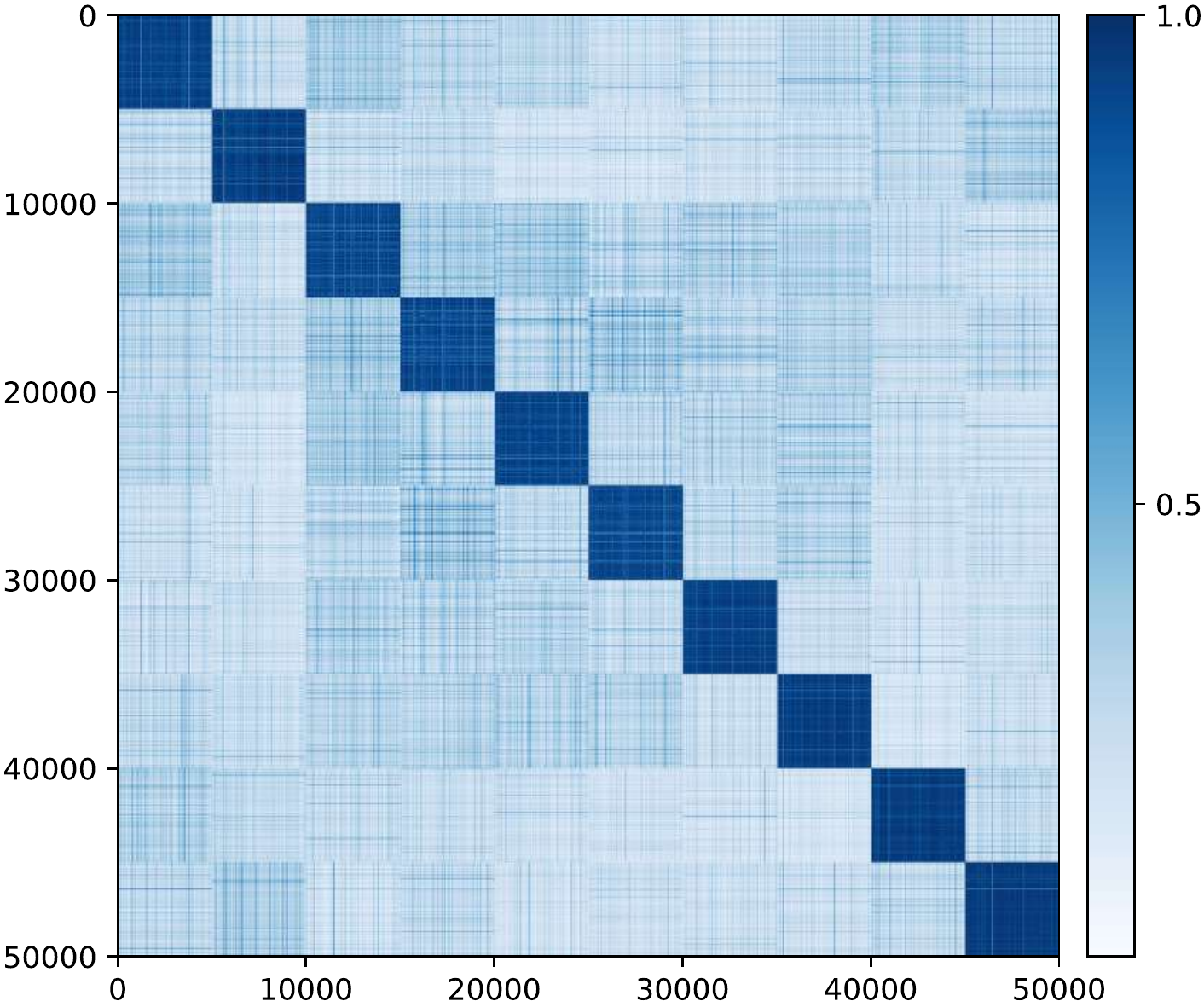}}
    \vskip -0.1in
    \caption{\small Cosine similarity between learned features by using the MCR$^2$ objective  (\textbf{left}) and CE loss (\textbf{right}).}
    \label{fig:heatmap-plot}
  \end{center}
  \vskip -0.1in
\end{figure*}

\begin{figure*}[!h]
\subcapcentertrue
\begin{center}
    \subfigure[\label{fig:visual-overall-mcr}10  representative images from each class based on top-10 principal components of the SVD of learned representations by MCR$^2$.]{\includegraphics[width=0.46\textwidth]{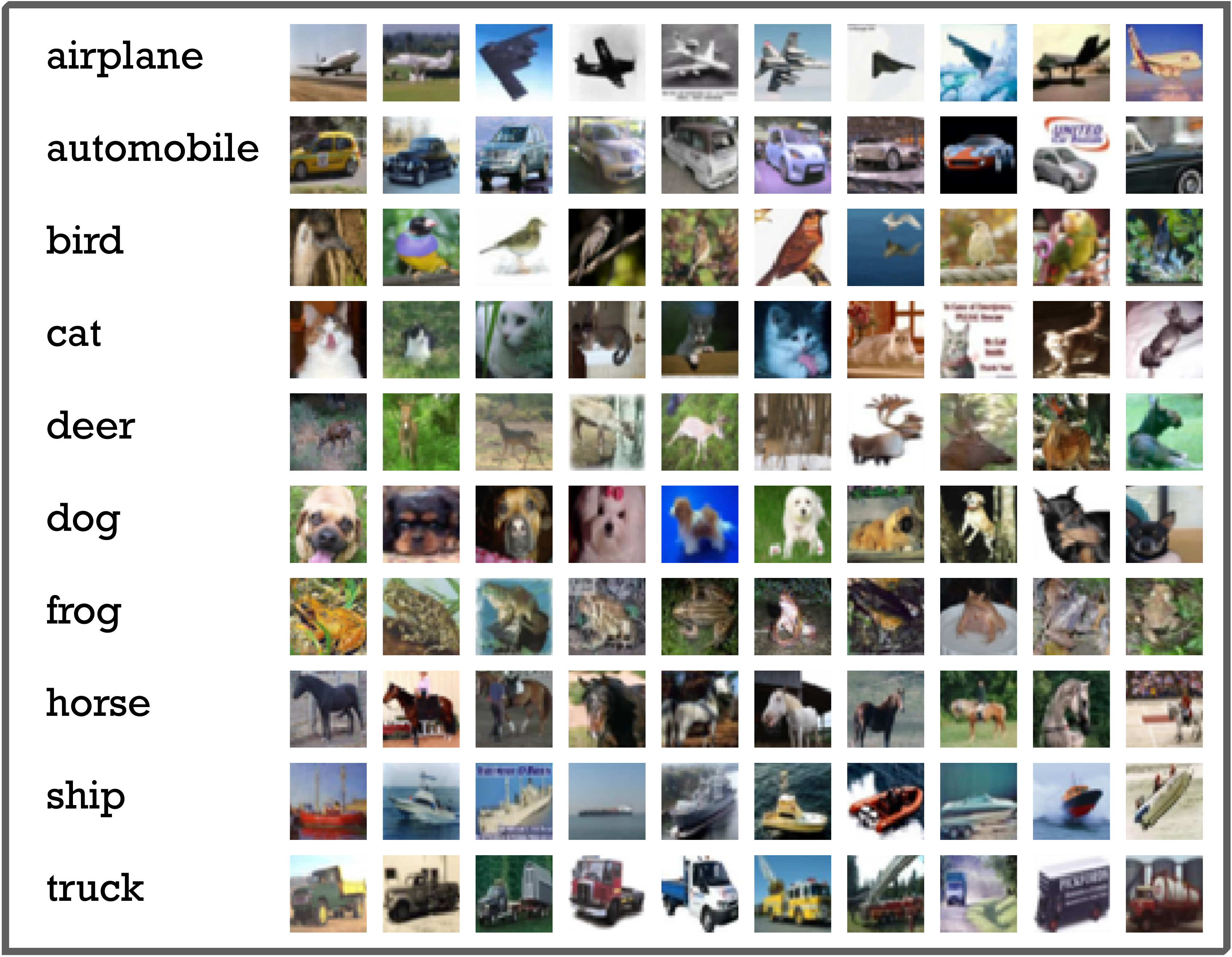}}
    \subfigure[\label{fig:visual-overall-cifar-website} Randomly selected 10 images from each class.]{\includegraphics[width=0.46\textwidth]{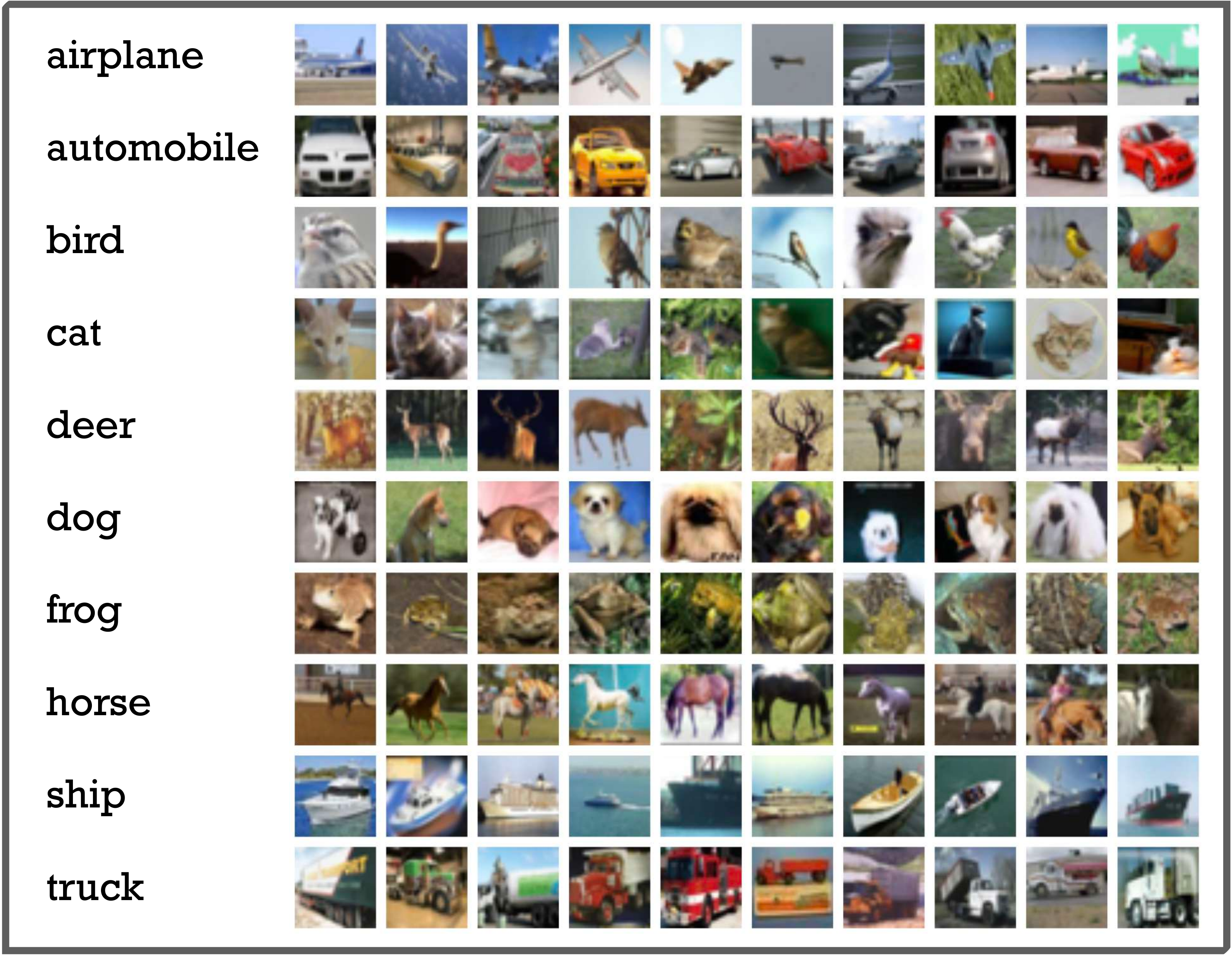}}
    \caption{\small Visualization of top-10 ``principal'' images for each class in the CIFAR10 dataset. \textbf{(a)} For each class-$j$, we first compute the top-10 singular vectors of the SVD of the learned features $\Z^j$. Then for the $l$-th singular vector of class $j$, $\u_{j}^{l}$, and for the feature of the $i$-th image of class $j$, $\z_{j}^{i}$, we calculate the absolute value of inner product, $| \langle  \z_{j}^{i}, \u_{j}^{l} \rangle|$, then we select the largest one for each singular vector within class $j$. Each row corresponds to one class, and each image corresponds to one singular vector, ordered by the value of the associated singular value. \textbf{(b)} For each class, 10 images are randomly selected in the dataset. These images are the ones displayed in the CIFAR dataset website~\citep{krizhevsky2009learning}.}
\label{fig:visual-overall-data}
\end{center}
\vskip -0.1in
\end{figure*}
For comparison, similar to Figure~\ref{fig:train-test-loss-pca-3}, we calculate the principle components of representations learned by MCR$^2$ training and cross-entropy training. For cross-entropy training, we take the output of the second last layer as the learned representation. The results are summarized in Figure~\ref{fig:pca-plot}. We also compare the cosine similarity between learned representations for both MCR$^2$ training and cross-entropy training, and the results are presented in  Figure~\ref{fig:heatmap-plot}. 

As shown in Figure~\ref{fig:pca-plot}, we observe that representations learned by MCR$^2$ are much more diverse, the dimension of learned features (each class) is  around a dozen, and the dimension of the overall features is nearly 120, and the output dimension is 128. In contrast, the dimension of the overall features learned using entropy is slightly greater than 10, which is much smaller than that learned by MCR$^2$. From Figure~\ref{fig:heatmap-plot}, for MCR$^2$ training,  we find that the  features of different class are almost orthogonal.

\paragraph{Visualize representative images selected from CIFAR10 dataset by using MCR$^2$.} As mentioned in Section \ref{sec:principled-objective-via-compression}, obtaining the properties of desired representation in the proposed MCR$^2$ principle is equivalent to performing {\em nonlinear generalized principle components} on the given dataset. As shown in Figure~\ref{fig:pca-mcr-1}-\ref{fig:pca-mcr-3}, MCR$^2$ can indeed learn such diverse and discriminative  representations. In order to better interpret the representations learned by MCR$^2$, we select images according to their ``principal'' components (singular vectors using SVD) of the learned features.  In Figure~\ref{fig:visual-class-2-8}, we visualize images selected from class-`Bird' and class-`Ship'. For each class, we first compute top-10 singular vectors of the SVD of the learned features and then for each of the top singular vectors, we display in each row the top-10 images whose
corresponding features are closest to the singular vector. As shown in Figure~ \ref{fig:visual-class-2-8}, we observe that images in the same row share many common characteristics such as shapes, textures, patterns, and styles, whereas images in different rows are significantly different from each other---suggesting our method captures all the different ``modes'' of the data even within the same class. Notice that top rows are associated with components with larger singular values, hence they are images that show up more frequently in the dataset.

In Figure~\ref{fig:visual-overall-mcr}, we visualize the 10 ``principal'' images selected from CIFAR10 for each of the 10 classes. That is, for each class, we display the 10 images whose corresponding features are most coherent with the top-10 singular vectors. We observe that the selected images are much more diverse and representative than those selected randomly from the dataset (displayed on the CIFAR official website), indicating such principal images can be used as a good ``summary'' of the dataset.

\subsubsection{Experimental Results of MCR$^2$ in the Supervised Learning Setting.}\label{sec:appendix-subsec-sup}

\paragraph{Training details for mainline experiment.} For the model presented in Figure~\ref{fig:low-dim} (\textbf{Right}) and Figure~\ref{fig:train-test-loss-pca},  we use ResNet-18 to parameterize $f(\cdot, \bm \theta)$, and we set the output dimension $d=128$, precision $\epsilon^2=0.5$, mini-batch size $m=1,000$. We use SGD in Pytorch~\citep{paszke2019pytorch} as the optimizer, and set the  learning rate \texttt{lr=0.01}, weight decay \texttt{wd=5e-4}, and \texttt{momentum=0.9}.

\paragraph{Experiments for studying the effect of hyperparameters and architectures.} We present the experimental results of MCR$^2$ training in the supervised setting by using various training hyperparameters and different network architectures. The results are summarized in Table~\ref{table:ablation-supervise}. Besides the ResNet architecture, we also consider VGG architecture~\citep{simonyan2014very} and ResNext achitecture~\citep{xie2017aggregated}. From Table~\ref{table:ablation-supervise}, we find that larger batch size $m$ can lead to  better performance. Also, models with higher output dimension $d$ require larger training batch size $m$.

\begin{table}[h]
\vskip -0.05in
\centering
\begin{small}
\begin{sc}
\begin{tabular}{ lcccccl}
\toprule
Arch & Dim $n$ &  Precision $\epsilon^2$ & BatchSize $m$ & {\texttt{lr}} & ACC & Comment \\
\midrule
\multirow{1}{*}{ResNet-18} & 128  & 0.5 & 1,000 & 0.01 &  0.922 & Mainline, Fig~\ref{fig:train-test-loss-pca} \\
\midrule
ResNext-29  & 128  & 0.5 & 1,000 & 0.01  & 0.925 & \multirow{2}{6em}{Different Architecture}  \\
VGG-11      & 128  & 0.5 & 1,000 & 0.01 & 0.907 &        \\
\midrule
ResNet-18   & 512  & 0.5 & 1,000 & 0.01 & 0.886 & \multirow{3}{6em}{Effect of Output Dimension}  \\
ResNet-18   & 256  & 0.5 & 1,000 & 0.01 & 0.921 &        \\
ResNet-18   & 64  & 0.5 & 1,000 & 0.01 & 0.922 &        \\
\midrule
ResNet-18   & 128  & 1.0 & 1,000 & 0.01 & 0.930 & \multirow{3}{6em}{Effect of precision}  \\
ResNet-18   & 128  & 0.4 & 1,000 & 0.01 & 0.919 &        \\
ResNet-18   & 128  & 0.2 & 1,000 & 0.01 & 0.900 &        \\
\midrule
ResNet-18   & 128  & 0.5 & 500 & 0.01 & 0.823 & \multirow{5}{6em}{Effect of Batch Size}  \\
ResNet-18   & 128  & 0.5 & 2,000 & 0.01 & 0.930 &        \\
ResNet-18   & 128  & 0.5 & 4,000 & 0.01 & 0.925 &        \\
ResNet-18   & 512  & 0.5 & 2,000 & 0.01 & 0.924 &        \\
ResNet-18   & 512  & 0.5 & 4,000 & 0.01 & 0.921 &        \\
\midrule
ResNet-18   & 128  & 0.5 & 1,000 & 0.05 & 0.860 & \multirow{3}{6em}{Effect of \texttt{lr}}  \\
ResNet-18   & 128  & 0.5 & 1,000 & 0.005 & 0.923 &        \\
ResNet-18   & 128  & 0.5 & 1,000 & 0.001 & 0.922 &        \\
\bottomrule
\end{tabular}
\end{sc}
\end{small}
\caption{\small Experiments of MCR$^2$ in the supervised setting on the CIFAR10 dataset.}
\label{table:ablation-supervise}
\end{table}

\paragraph{Effect of  $r_j$ on classification.} Unless otherwise stated, we set the number of components $r_j=30$ for nearest subspace classification. We study the effect of $r_j$ when used for classification, and the results are summarized in Table~\ref{table:effect-rj}. We observe that the nearest subspace classification works for a wide range of $r_j$.

\begin{table}[h]
\vskip -0.05in
\begin{center}
\begin{small}
\begin{sc}
\begin{tabular}{l|ccccc}
\toprule
Number of components & $r_j=10$ & $r_j=20$ & $r_j=30$ & $r_j=40$ & $r_j=50$  \\
\midrule
Mainline (Label Noise Ratio=0.0) & 0.926  & 0.925 & 0.922 & 0.923 & 0.921  \\
\midrule
Label Noise Ratio=0.1 & 0.917  & 0.917 & 0.911 & 0.918 & 0.917  \\
Label Noise Ratio=0.2 & 0.906  & 0.906 & 0.897 & 0.906 & 0.905  \\
Label Noise Ratio=0.3 & 0.882  & 0.879 & 0.881 & 0.881 & 0.881  \\
Label Noise Ratio=0.4 & 0.864  & 0.866 & 0.866 & 0.867 & 0.864  \\
Label Noise Ratio=0.5 & 0.839  & 0.841 & 0.843 & 0.841 & 0.837  \\
\bottomrule
\end{tabular}
\end{sc}
\end{small}
\caption{\small Effect of number of components $r_j$ for nearest subspace classification in the supervised setting.}
\label{table:effect-rj}
\end{center}
\vskip -0.05in
\end{table}

\paragraph{Effect of  $\epsilon^2$ on learning from corrupted labels.} To further study the proposed MCR$^2$ on learning from corrupted labels, we use different precision parameters, $\epsilon^2 = 0.75, 1.0$, in addition to the one shown in Table~\ref{table:label-noise}. Except for the precision parameter $\epsilon^2$, all the other parameters are the same as the mainline experiment (the first row in Table~\ref{table:ablation-supervise}). The first row ($\epsilon^2=0.5$) in Table~\ref{table:label-noise-appendix-precision} is identical to the \textsc{MCR$^2$ training} in Table~\ref{table:clustering}. Notice that with slightly different choices in $
\epsilon^2$, one might even see slightly improved performance over the ones reported in the main body.

\begin{table}[h]
\vspace{-2mm}
\begin{center}
\begin{small}
\begin{sc}
\begin{tabular}{l | c c c c c }
\toprule
Precision  & Ratio=0.1 &  Ratio=0.2 &  Ratio=0.3 &  Ratio=0.4 &  Ratio=0.5 \\
\midrule
$\epsilon^2=0.5$  & {0.911} & {0.897} & {0.881} & {0.866} &  {0.843}\\
$\epsilon^2=0.75$ & \textbf{0.923} & {0.908} & \textbf{0.899} & \textbf{0.876} &  {0.836}\\
$\epsilon^2=1.0$  & {0.919} & \textbf{0.911} & {0.896} & {0.870} &  \textbf{0.845}\\
\bottomrule
\end{tabular}
\end{sc}
\end{small}
\caption{\small Effect of Precision $\epsilon^2$ on classification results with features learned with labels corrupted at different levels by using MCR$^2$ training.}
\label{table:label-noise-appendix-precision}
\end{center}
\end{table}

\subsection{Comparison with Related Work on Label Noise}\label{sec:appendix-label-noise-related-work}
We compare the proposed MCR$^2$ with OLE~\citep{lezama2018ole}, Large Margin Deep Networks~\citep{elsayed2018large}, and ITLM~\citep{shen2019learning} in label noise robustness experiments on CIFAR10 dataset. In Table~\ref{table:label-noise-related-work}, we compare MCR$^2$ with OLE~\citep{lezama2018ole} and Large Margin Deep Networks~\citep{elsayed2018large} on the corrupted label task using the same network, MCR$^2$ achieves significant better performance. We compare MCR$^2$ with ITLM~\citep{shen2019learning} using the same network. MCR2 achieves better performance without any noise ratio dependent hyperparameters as required by \cite{shen2019learning}.

\begin{table}[h]
\begin{center}
\begin{small}
\begin{sc}
\begin{tabular}{l | c c c c c }
\toprule
ResNet18 & Ratio=0.1 &  Ratio=0.2 &  Ratio=0.3 & Ratio=0.4 &  Ratio=0.5 \\
\midrule
OLE &  0.910 &  0.860 &  0.806 &  0.717  &  0.610 \\
LargeMargin &  0.901 &  0.874 &  0.837 &  0.785  &  0.724 \\
MCR$^2$  & \textbf{0.911} & \textbf{0.897} & \textbf{0.881} & \textbf{0.866} &  \textbf{0.843}\\
\midrule
WRN16 & Ratio=0.1 &  Ratio=0.3 &  Ratio=0.5 & Ratio=0.7\\
\midrule
ITLM &  0.903 &  0.882 &  0.825 &  0.647  \\
MCR$^2$  & \textbf{0.915} & \textbf{0.888} & \textbf{0.842} & \textbf{0.670} \\
\bottomrule
\end{tabular}
\end{sc}
\end{small}
\caption{\small Comparison  with related work (OLE~\citep{lezama2018ole}, LargeMargin~\citep{elsayed2018large}, ITLM~\citep{shen2019learning}) on learning  from noisy labels.}
\label{table:label-noise-related-work}
\end{center}
\end{table}

\subsection{Learning from Gaussian noise corrupted data.}

We investigate the performance of MCR$^2$ training with corrupted data by adding varying levels of Gaussian noise. For each corruption level, we add $\mathcal{N}(0, \sigma^{2})$ to the input images with different standard deviations $\sigma \in \{0.04, 0.06, 0.08, 0.09, 0.1\}$ as in \citet{hendrycks2018benchmarking}. We train the same architecture ResNet-18 as the previous experiments for 500 epochs, set mini-batch size to $m=1000$ and optimize using SGD with learning rate \texttt{lr=0.01}, momentum \texttt{momentum=0.9} and weight decay \texttt{wd=5e-4}. We also decrease the learning rate to 0.001 at epoch 200 and to 0.0001 at epoch 400. In our objective, we set precision $\epsilon^2 = 0.5$. To compare the performance of MCR$^2$ versus cross-entropy (CE), we train the same architecture using the cross-entropy loss for 200 epochs and optimize using SGD with learning rate \texttt{lr=0.1}, momentum \texttt{momentum=0.9} and weight decay \texttt{wd=5e-4}. We also use Cosine Annealing learning rate scheduler during training. We show the respective testing accuracy in Table~\ref{table:gaussian-noise}. Although the classification result of using MCR$^2$ slightly lags behind that of using CE, when the noise level is small, their performances are comparable with each other. Similar sensitivity to the noise indicates that the reason might be because of the choice of the same network architecture. We reserve the study on how to improve robustness to input noise for future work. 

\begin{table}[h]
\begin{center}
\vskip 0.05in
\begin{small}
\begin{sc}
\begin{tabular}{l | c c c c c }
\toprule
 Noise level & $\sigma=0.04$ &  $\sigma=0.06$ &  $\sigma=0.08$ &  $\sigma=0.09$ &  $\sigma=0.1$ \\
\midrule
CE Training & 0.912 & 0.897 & 0.876 & 0.867 & 0.857\\
MCR$^2$ Training & 0.909 & 0.882 & 0.869 & 0.855 & 0.829 \\
\bottomrule
\end{tabular}
\end{sc}
\end{small}
\caption{\small  Classification results of features learned with inputs corrupted by  Gaussian noise at different levels.}
\label{table:gaussian-noise}
\end{center}
\vspace{-4mm}
\end{table}

\subsection{Experimental Results of MCR$^2$ in the Self-supervised Learning Setting}\label{sec:appendix-selfsup}
\subsubsection{Self-supervised Learning of Invariant Features} 
\textbf{Learning invariant features via rate reduction.} Motivated by self-supervised learning algorithms~\citep{lecun2004learning,kavukcuoglu2009learning,oord2018representation,he2019momentum,wu2018unsupervised}, we use the MCR$^2$ principle to learn representations that are {\em invariant} to certain class of transformations/augmentations, say $\mathcal T$ with a distribution $P_{\mathcal T}$. Given a mini-batch of data  $\{ \bm{x}^j\}_{j=1}^{k}$ with mini-batch size equals to $k$, we augment each sample $\x^j$ with $n$  transformations/augmentations $\{\tau^{i}(\cdot)\}_{i=1}^n$ randomly drawn from $P_{\mathcal{T}}$.
We simply label all the augmented samples $\X^j = [\tau_{1}(\bm{x}^j), \ldots, \tau_{n}(\bm{x}^j)]$ of $\bm x^j$ as the $j$-th class, and $\Z^j$ the corresponding learned features. Using this self-labeled data, we train our feature mapping $f(\cdot, \bm{\theta})$ the same way as the supervised setting above. For every mini-batch, the total number of samples for training is $m = k n$.

\noindent \textbf{Evaluation via clustering.} To learn invariant features, our formulation itself does {\em not} require the original samples $\x^j$ come from a fixed number of classes. For evaluation, we may train on a few classes and observe how the learned features facilitate classification or clustering of the data. A common method to evaluate learned features is to train an additional linear classifier~\citep{oord2018representation,he2019momentum}, with ground truth labels. But for our purpose, because we explicitly verify whether the so-learned invariant features have good subspace structures when the samples come from $k$ classes, we use an off-the-shelf subspace clustering algorithm EnSC~\citep{you2016oracle}, which is computationally efficient and is provably correct for data with well-structured subspaces. 
We also use K-Means on the original data $\bm X$ as our baseline for comparison.  We use normalized mutual information (NMI), clustering accuracy (ACC), and adjusted rand index (ARI) for our evaluation metrics, see Appendix~\ref{sec:appendix-subsec-clustering} for their detailed definitions.

\noindent \textbf{Controlling dynamics of expansion and compression.} By directly optimizing the rate reduction $\Delta R = R - R_c$, we achieve $0.570$ clustering accuracy on CIFAR10 dataset, which is the second best result compared with previous methods. Empirically, we observe that, without class labels, the overall \textit{coding rate} $R$ expands quickly and the MCR$^2$ loss saturates (at a local maximum), see Fig~\ref{fig:train-test-loss-selfsup-mcr}.  Our experience suggests that learning a good representation from unlabeled data might be too ambitious when directly optimizing the original $\Delta R$. Nonetheless, from the  geometric meaning of $R$ and $R_c$, one can design a different learning strategy by controlling the dynamics of expansion and compression differently during training. 
For instance, we may re-scale the rate by replacing $R(\Z, \epsilon)$ with 
\begin{equation}
    \widetilde R(\Z,\epsilon) \doteq \frac{1}{2 \gamma_1}\log\det(\I + \frac{\gamma_2 d}{ m\epsilon^{2}}\Z\Z^{*}).
\end{equation}
With $\gamma_1 = \gamma_2 = k$, the learning dynamics change from Figure~\ref{fig:train-test-loss-selfsup-mcr} to Figure~\ref{fig:train-test-loss-selfsup-mcr-ctrl}: All features are first compressed then gradually expand. We denote the controlled MCR$^2$ training by MCR$^2$-{\scriptsize CTRL}.

\begin{figure*}[t]
\subcapcentertrue
\begin{center}
    \subfigure[\label{fig:train-test-loss-selfsup-mcr}MCR$^2$]{\includegraphics[width=0.45\textwidth]{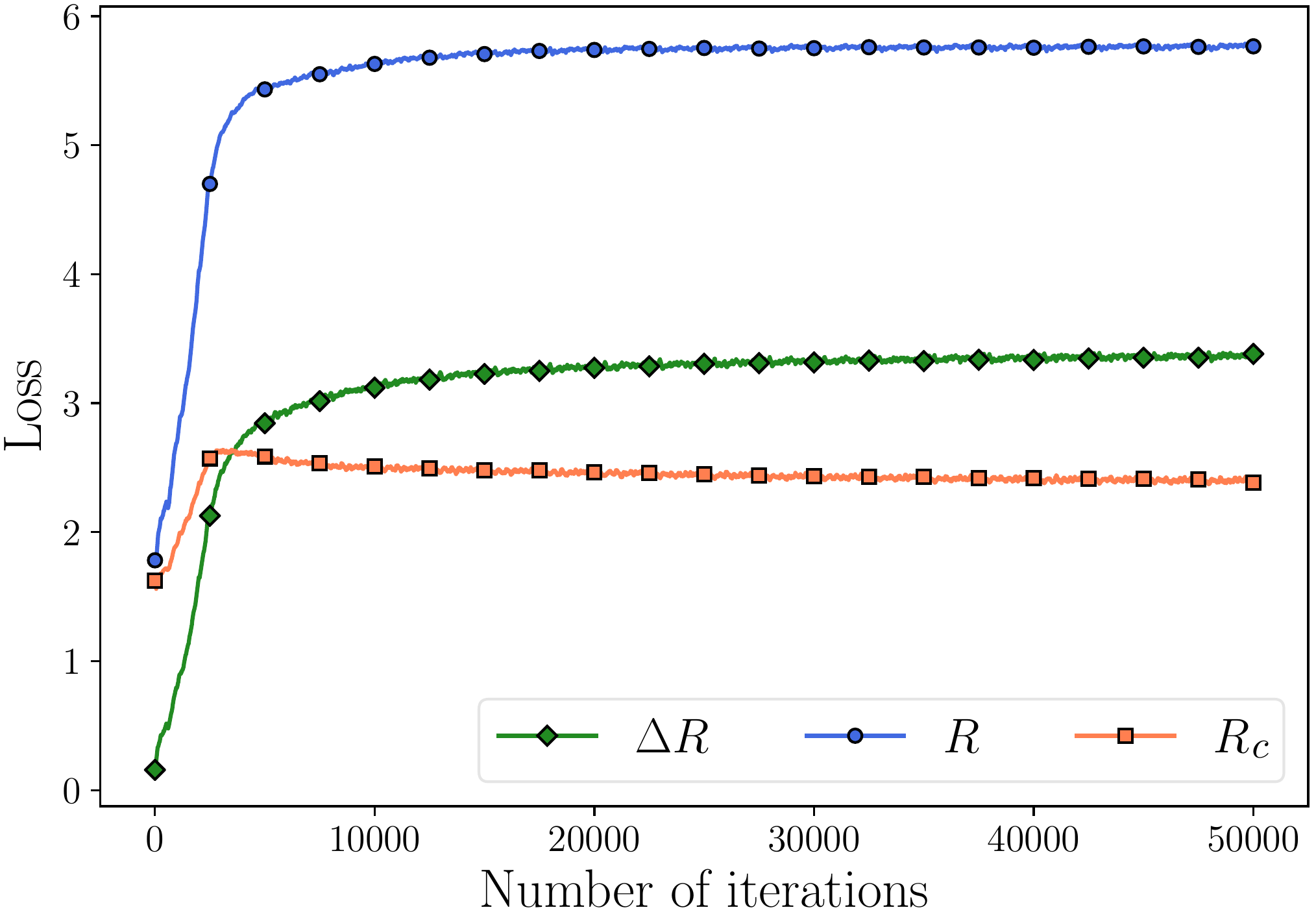}} 
    \hspace{5mm}
    \subfigure[\label{fig:train-test-loss-selfsup-mcr-ctrl}MCR$^2$-{\scriptsize CTRL}.]{\includegraphics[width=0.45\textwidth]{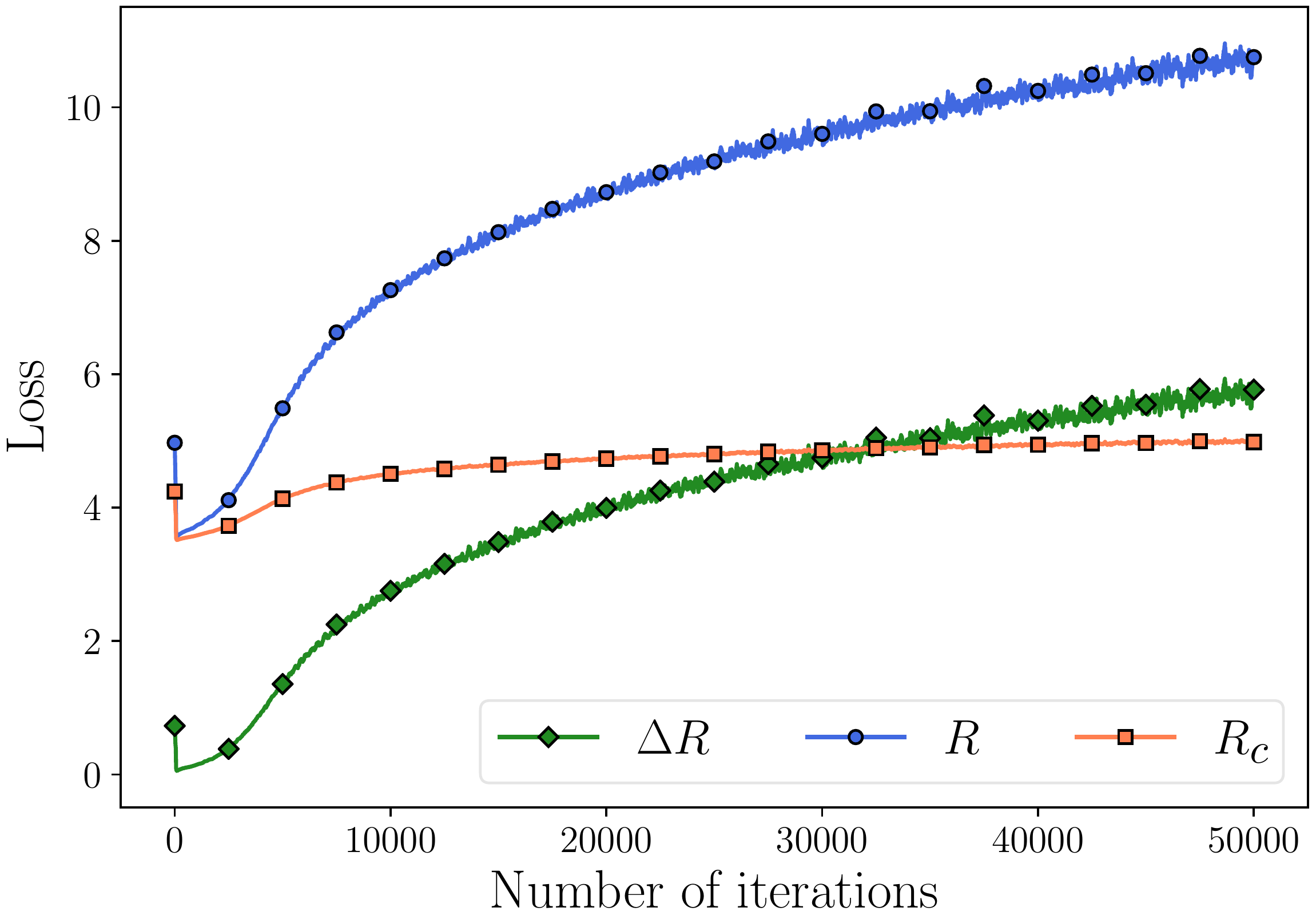}}
    \vskip -0.05in
    \caption{\small Evolution of the rates of (\textbf{left}) MCR$^2$   and (\textbf{right}) MCR$^2$-{\scriptsize CTRL} in the training process in the self-supervised setting on CIFAR10 dataset.}
\label{fig:training-dynamic-controlling-compare}
\end{center}
\vskip -0.1in
\end{figure*}

\noindent \textbf{Experiments on real data.} Similar to the supervised learning setting, we train {\em exactly the same} ResNet-18 network on the CIFAR10, CIFAR100, and STL10~\citep{coates2011analysis} datasets.
We set the mini-batch size as $k = 20$, number of augmentations for each sample as $n=50$ and the precision parameter as $\epsilon^2 = 0.5$. Table \ref{table:clustering} shows the results of the proposed MCR$^2$-{\scriptsize CTRL}  in comparison with methods JULE~\citep{yang2016joint}, RTM~\citep{nina2019decoder}, DEC~\citep{xie2016unsupervised}, DAC~\citep{chang2017deep}, and DCCM~\citep{wu2019deep} that have achieved the best results on these datasets. {Surprisingly, without utilizing any inter-class or inter-sample information and heuristics on the data, the invariant features learned by our method with augmentations alone achieves a better performance over other highly engineered clustering methods.} More comparisons and ablation studies can be found in Appendix~\ref{sec:appendix-subsec-clustering}. 

Nevertheless, compared to the representations learned in the supervised setting where the optimal partition $\bm{\Pi}$ in \eqref{eqn:maximal-rate-reduction} is initialized by correct class information, the representations here learned with self-supervised  classes are far from being optimal. 
It remains wide open how to design better optimization strategies and dynamics to learn from unlabelled or partially-labelled data better representations (and the associated partitions) close to the global maxima of the MCR$^2$ objective \eqref{eqn:maximal-rate-reduction}.

\begin{table}[t]
\begin{center}
\vskip 0.05in
\begin{small}
\begin{sc}
\begin{tabular}{l l | c c c c c c c c c }
\toprule
Dataset & Metric & K-Means & JULE &  RTM    & DEC  & DAC  &  DCCM & MCR$^2$-{\scriptsize Ctrl} \\
\midrule
\multirow{3}{*}{CIFAR10}    
& NMI & 0.087 & 0.192 & 0.197 & 0.257 & 0.395  & 0.496 & \textbf{0.630}       \\
& ACC & 0.229 & 0.272 & 0.309 & 0.301 & 0.521  & 0.623 & \textbf{0.684}       \\
& ARI & 0.049 & 0.138 & 0.115 & 0.161 & 0.305  & 0.408 & \textbf{0.508}       \\
\midrule
\multirow{3}{*}{CIFAR100}    
& NMI & 0.084 & 0.103 & - & 0.136 & 0.185  & 0.285 & \textbf{0.387}       \\
& ACC & 0.130 & 0.137 & - & 0.185 & 0.237  & 0.327 & \textbf{0.375}       \\
& ARI & 0.028 & 0.033 & - & 0.050 & 0.087  & 0.173 & \textbf{0.178}       \\
\midrule
\multirow{3}{*}{STL10}    
& NMI & 0.124 & 0.182 & - & 0.276 &  0.365 & 0.376 &  \textbf{0.446}       \\
& ACC & 0.192 & 0.182 & - & 0.359 & 0.470 & 0.482  &  \textbf{0.491}       \\
& ARI & 0.061 & 0.164 & - & 0.186 & 0.256 & 0.262 &  \textbf{0.290}       \\
\bottomrule
\end{tabular}
\end{sc}
\end{small}
\caption{\small Clustering results on CIFAR10, CIFAR100, and STL10 datasets.}
\label{table:clustering}
\end{center}
\vskip -0.2in
\end{table}

\paragraph{Training details of MCR$^2$-{\scriptsize CTRL}.} For three datasets (CIFAR10, CIFAR100, and STL10), we use ResNet-18 as in the supervised setting, and we set the output dimension $d=128$, precision $\epsilon^2=0.5$, mini-batch size $k=20$, number of augmentations $n=50$, $\gamma_1 = \gamma_2 =20$. We observe that MCR$^2$-{\scriptsize CTRL} can achieve better clustering performance by using smaller $\gamma_2$, i.e., $\gamma_2=15$, on CIFAR10 and CIFAR100 datasets. We use SGD as the optimizer, and set the  learning rate \texttt{lr=0.1}, weight decay \texttt{wd=5e-4}, and \texttt{momentum=0.9}.

\paragraph{Training dynamic comparison between MCR$^2$ and MCR$^2$-{\scriptsize CTRL}.}
In the self-supervised setting, we compare the training process for MCR$^2$ and MCR$^2$-\text{\scriptsize CTRL} in terms of $R, \widetilde{R}, R_c$, and $\Delta R$. For MCR$^2$ training shown in Figure~\ref{fig:train-test-loss-selfsup-mcr}, the features first expand (for both $R$ and $R_c$) then compress (for $R_c$). For MCR$^2$-\text{\scriptsize CTRL}, both $\widetilde{R}$ and $R_c$ first compress then $\widetilde{R}$ expands quickly and $R_c$ remains small, as we have seen in Figure~\ref{fig:train-test-loss-selfsup-mcr-ctrl}.

\paragraph{Clustering results comparison.} 
We compare the clustering performance between MCR$^2$ and MCR$^2$-{\scriptsize CTRL} in terms of NMI, ACC, and ARI. The clustering results are summarized in Table~\ref{table:mcr-mcrctrl-compare-appendix}. We find  that  MCR$^2$-{\scriptsize CTRL}  can achieve better performance for clustering.

\begin{table}[th]
\begin{center}
\begin{small}
\begin{sc}
\begin{tabular}{l | c c c  }
\toprule
 & NMI & ACC & ARI  \\
\midrule
MCR$^2$ & 0.544 & 0.570 & 0.399\\
MCR$^2$-{\scriptsize Ctrl} & 0.630 & 0.684 & 0.508 \\
\bottomrule
\end{tabular}
\end{sc}
\end{small}
\caption{\small Clustering comparison between MCR$^2$ and MCR$^2$-{\scriptsize CTRL} on CIFAR10 dataset.}
\label{table:mcr-mcrctrl-compare-appendix}
\end{center}
\end{table}

\subsubsection{Clustering Metrics and More  Results}\label{sec:appendix-subsec-clustering}
We first introduce the definitions of  normalized mutual information (NMI)~\citep{strehl2002cluster}, clustering accuracy (ACC), and adjusted rand index (ARI)~\citep{hubert1985comparing}.

\paragraph{Normalized mutual information (NMI).} Suppose $Y$ is the ground truth partition and $C$ is the prediction partition. The NMI metric is defined as 
\begin{equation*}
    \text{NMI}(Y, C) = \frac{\sum_{i=1}^{k}\sum_{j=1}^{s}|Y_{i} \cap C_{j}|\log\left(\frac{m |Y_{i} \cap C_{j}| }{|Y_{i}| |C_{j}|}\right)}{\sqrt{\left(\sum_{i=1}^{k}|Y_i|\log\left(\frac{|Y_i|}{m}\right)\right) \left(\sum_{j=1}^{s}|C_j|\log\left(\frac{|C_j|}{m}\right)\right)}},
\end{equation*}
where $Y_i$ is the $i$-th cluster in $Y$ and $C_j$ is the $j$-th cluster in $C$, and $m$ is the total number of samples.

\paragraph{Clustering accuracy (ACC).} Given $m$ samples, $\{(\x^i, \y^i)\}_{i=1}^m$.  For the $i$-th sample $\x^i$, let $\y^i$ be its ground truth label, and let $\bm{c}^i$ be its cluster label. The ACC metric is defined as 
\begin{equation*}
    \text{ACC}(\Y, \bm{C})= \max_{\sigma\in S}\frac{\sum_{i=1}^{m}\mathbf{1}\{\y^i = \sigma(\bm{c}^i)\}}{m},
\end{equation*}
where $S$ is the set includes all the one-to-one mappings from cluster to label, and $\Y = [\y^1, \dots, \y^m]$, $\bm{C} = [\bm{c}^1, \dots, \bm{c}^{m}]$.

\paragraph{Adjusted rand index (ARI).} Suppose there are $m$ samples, and let $Y$ and $C$ be two clustering of these samples, where $Y = \{Y_1, \dots, Y_r\}$ and $C = \{C_1, \dots, C_{s}\}$. Let $m_{ij}$ denote the number of the intersection between $Y_i$ and $C_{j}$, i.e., $m_{ij} = |Y_i \cap C_j|$. The ARI metric is defined as 
\begin{equation*}
    \text{ARI} = \frac{\sum_{ij}\binom{m_{ij}}{2} - \left(\sum_{i}\binom{a_{i}}{2} \sum_{j}\binom{b_{j}}{2} \right)\big/ \binom{m}{2} }{\frac{1}{2}\left(\sum_{i}\binom{a_{i}}{2} +\sum_{j}\binom{b_{j}}{2} \right) - \left(\sum_{i}\binom{a_{i}}{2} \sum_{j}\binom{b_{j}}{2} \right)\big/ \binom{m}{2}},
\end{equation*}
where $a_{i} = \sum_{j}m_{ij}$ and $b_{j} = \sum_{i}m_{ij}$.

\paragraph{Comparison with \cite{ji2019invariant,hu2017learning}.} We compare MCR$^2$ with IIC~\citep{ji2019invariant} and IMSAT~\citep{hu2017learning} in Table~\ref{table:clustering-appendix}. We find that MCR$^2$ outperforms IIC~\citep{ji2019invariant} and IMSAT~\citep{hu2017learning} on both CIFAR10 and CIFAR100 by a large margin. For STL10, \cite{hu2017learning} applied pretrained ImageNet models and \cite{ji2019invariant} used more data for training.

\begin{table}[ht]
\begin{center}
\begin{small}
\begin{sc}
\begin{tabular}{l l | c c c }
\toprule
Dataset & Metric &  IIC  &  IMSAT & MCR$^2$-{\scriptsize Ctrl} \\
\midrule
\multirow{3}{*}{CIFAR10}    
& NMI & -  & - & \textbf{0.630}       \\
& ACC & 0.617  & 0.456  & \textbf{0.684}       \\
& ARI & -  & - & \textbf{0.508}       \\
\midrule
\multirow{3}{*}{CIFAR100}    
& NMI & -  & - & \textbf{0.387}       \\
& ACC & 0.257  & 0.275 & \textbf{0.375}       \\
& ARI & -  & - & \textbf{0.178}       \\
\bottomrule
\end{tabular}
\end{sc}
\end{small}
\caption{\small Compare with \cite{ji2019invariant,hu2017learning} on clustering.}
\label{table:clustering-appendix}
\end{center}
\end{table}

\paragraph{More experiments on the effect of hyperparameters of MCR$^2$-{\scriptsize CTRL}.}We provide more experimental results of MCR$^2$-{\scriptsize CTRL} training in the self-supervised setting by varying training hyperparameters on the STL10 dataset. The results are summarized in Table~\ref{table:ablation-self-supervise}. Notice that the choice of hyperparameters only has small effect on the performance with the MCR$^2$-{\scriptsize CTRL} objective. We may hypothesize that, in order to further improve the performance, one has to seek other, potentially better, control of optimization dynamics or strategies. We leave those for future investigation. 

\begin{table}[htp]
\centering
\begin{small}
\begin{sc}
\begin{tabular}{ lccccc}
\toprule
Arch &   Precision $\epsilon^2$  & Learning Rate \texttt{lr} & NMI & ACC & ARI \\
\midrule
ResNet-18 &  0.5 & 0.1 & 0.446 & 0.491 & 0.290  \\
\midrule
ResNet-18  &  0.75 &  0.1 &   0.450 & 0.484 & 0.288   \\
ResNet-18  &  0.25 &  0.1 &   0.447 & 0.489 & 0.293   \\
ResNet-18  &  0.5 &  0.2  &   0.477 & 0.473 & 0.295   \\
ResNet-18  &  0.5 &  0.05  &   0.444 & 0.496 & 0.293  \\
ResNet-18  &  0.25 &  0.05 &   0.454 & 0.489 & 0.294  \\
\bottomrule
\end{tabular}
\end{sc}
\end{small}
\caption{\small Experiments of MCR$^2$-{\scriptsize CTRL} in the self-supervised setting on STL10 dataset.}
\label{table:ablation-self-supervise}
\end{table}

\clearpage
\section{Implementation Details and Additional Experiments for ReduNets}\label{sec:appendix-redunet-exp}
In this section, we provide additional experimental results related to ReduNet in Section~\ref{sec:experiments}. Obviously, in this work we have chosen a simplest design of the ReduNet and do not particularly optimize any of the hyper parameters, such as the number of initial channels, kernel sizes, normalization, and learning rate etc., for the best performance or scalability. The choices are mostly for convenience and just minimally adequate to verify the concept. We leave all such practical matters for us and others to investigate in the future.

We first present the visualization of  rotated and translated images of the MNIST dataset in Section~\ref{sec:appendix-visualization}. We provide additional experimental results on rotational invariance in Section~\ref{sec:appendix-rotation} and translational invariance in Section~\ref{sec:appendix-translation}. In Section~\ref{sec:appendix-gaussian}, we provide more results on learning mixture of Gaussians.

\subsection{Visualization of Rotation and Translation on MNIST}\label{sec:appendix-visualization}
In this subsection, we present the visualization of rotation and translation images on the MNIST dataset. The rotation examples are shown in Figure~\ref{fig:appendix-mnist-rotation-visualize} and the translation examples are shown in Figure~\ref{fig:appendix-mnist-translation-visualize}.
\begin{figure*}[ht!]
  \begin{center}
    \includegraphics[width=0.24\textwidth]{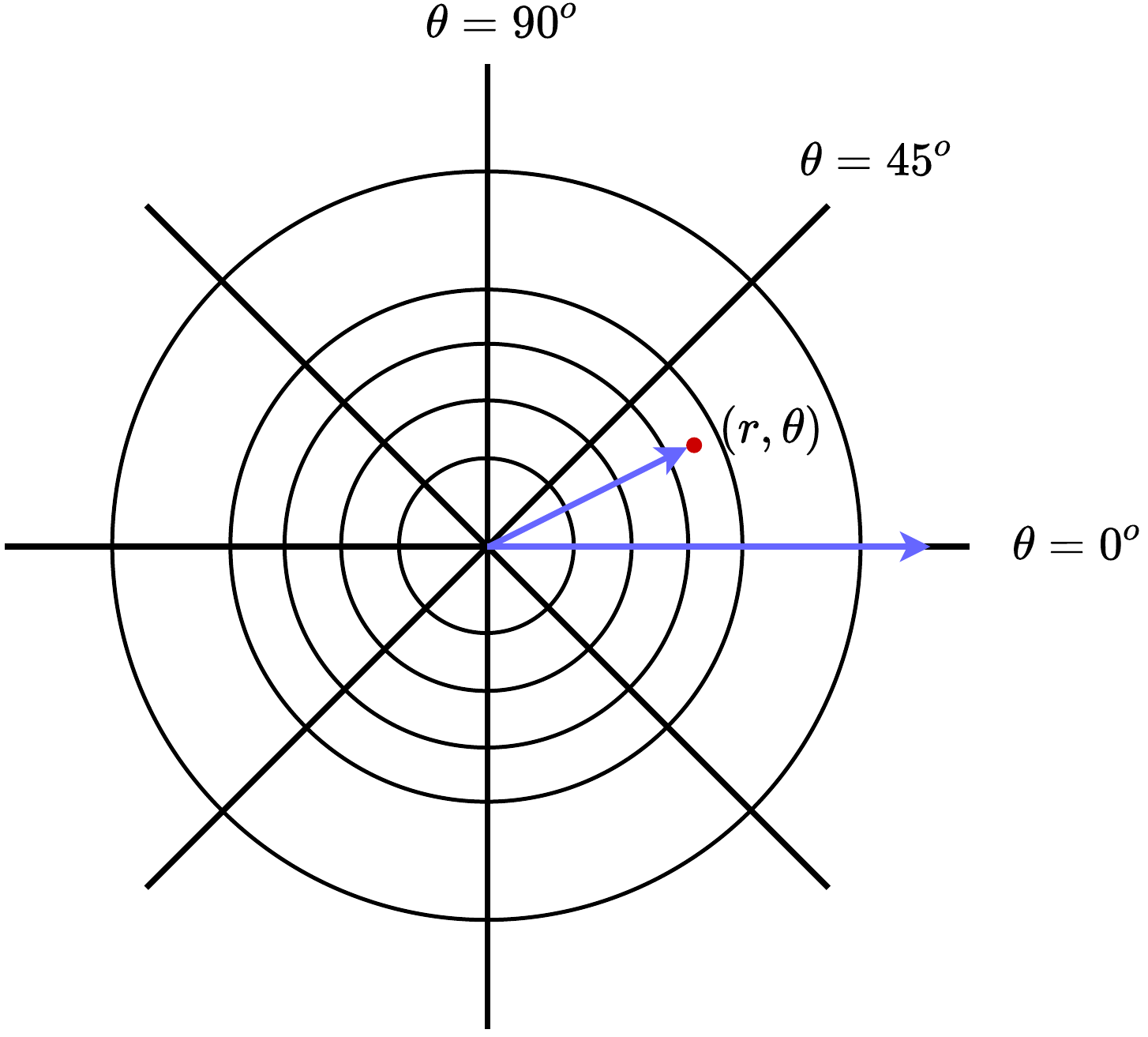} \hspace{4mm}
    \includegraphics[width=0.24\textwidth]{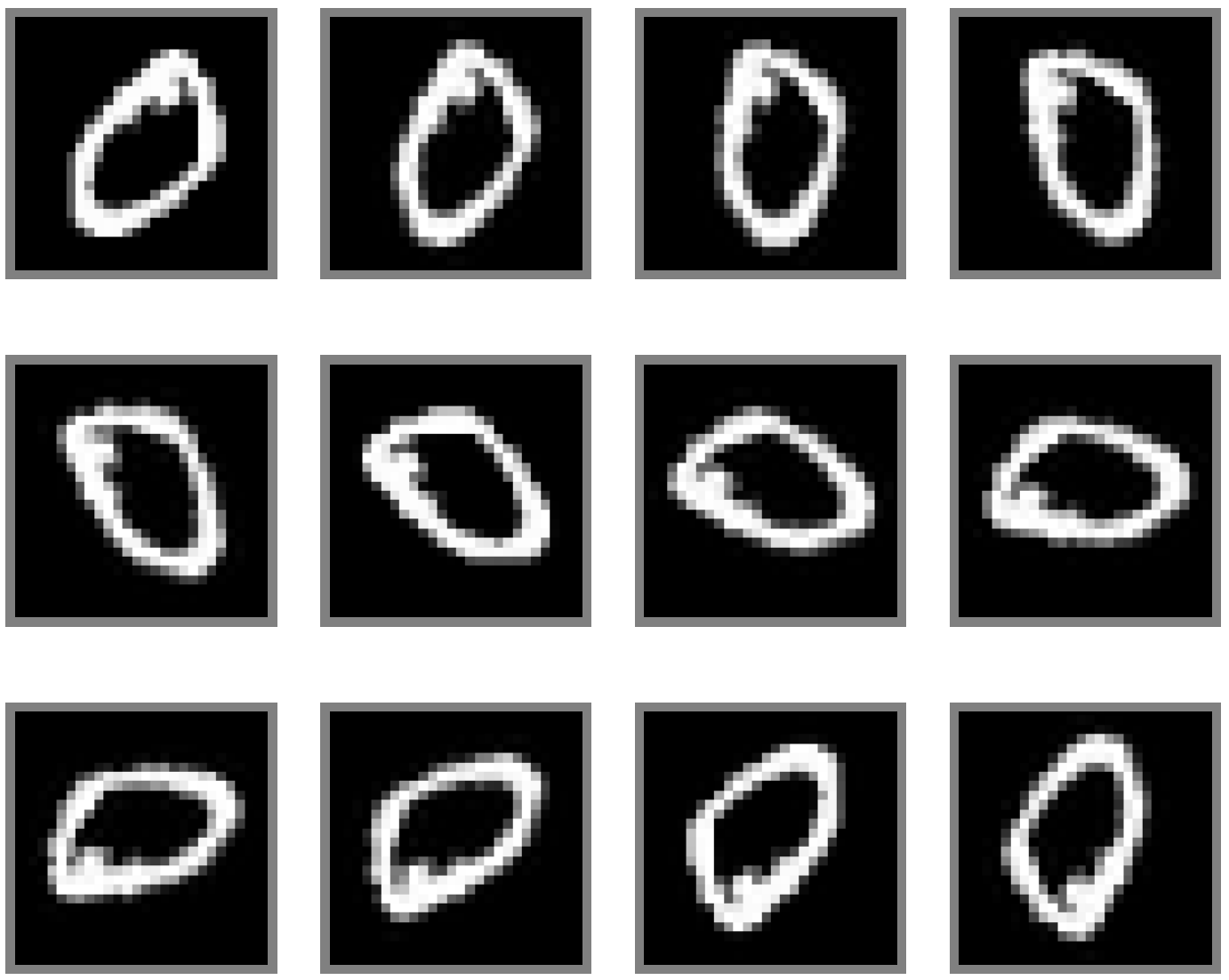}
    \includegraphics[width=0.24\textwidth]{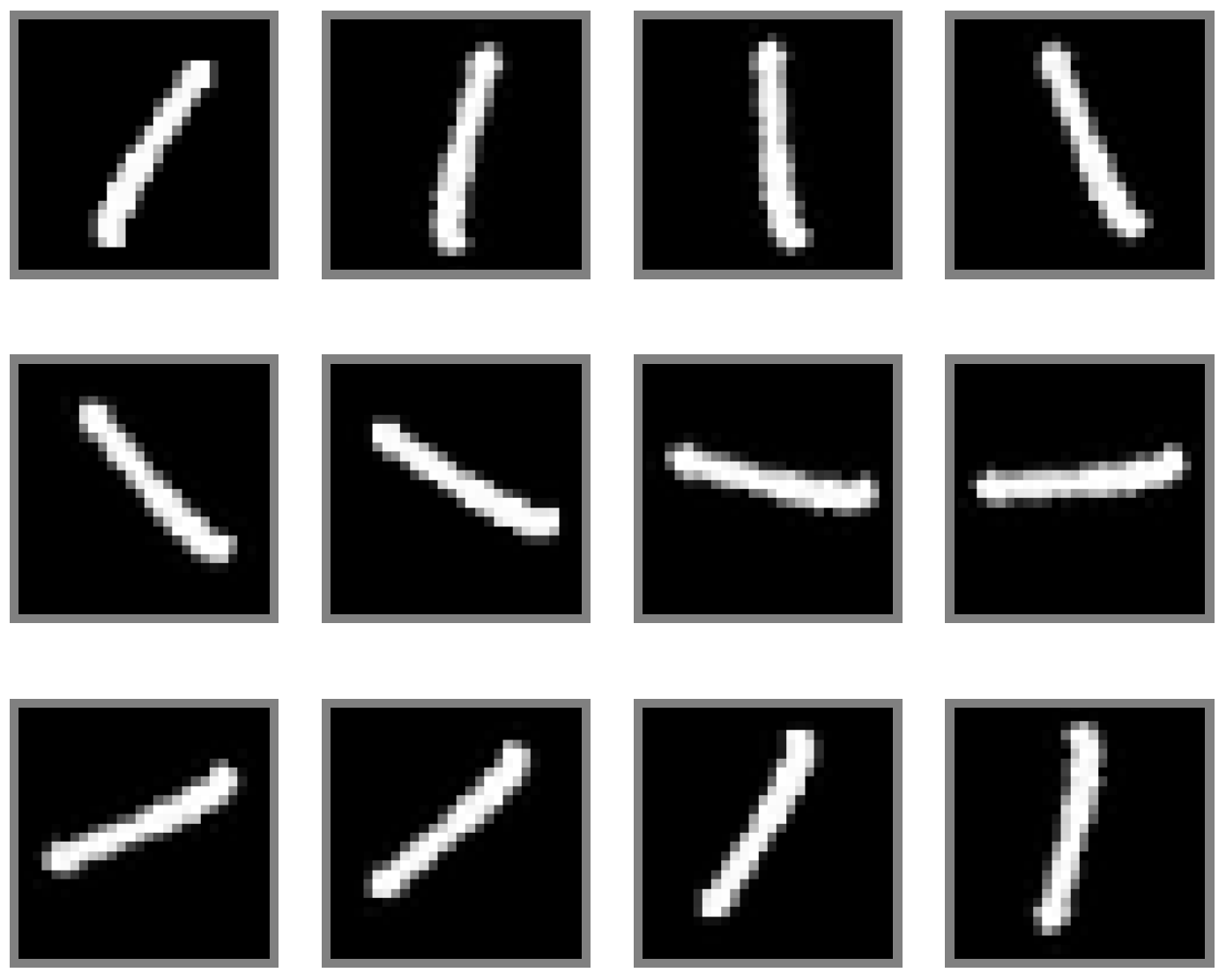}
    \vskip -0.01in
    \caption{Examples of rotated images of MNIST digits, each by 18$^{\circ}$. (\textbf{Left}) Diagram for polar coordinate representation;  (\textbf{Right}) Rotated images of digit `0' and `1'.}
    \label{fig:appendix-mnist-rotation-visualize}
  \end{center}
  \vskip -0.2in
\end{figure*}

\begin{figure*}[ht!]
  \begin{center}
  \includegraphics[width=0.24\textwidth]{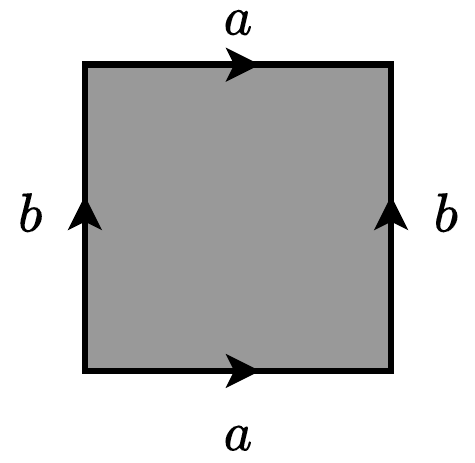} \hspace{4mm}
    \includegraphics[width=0.27\textwidth]{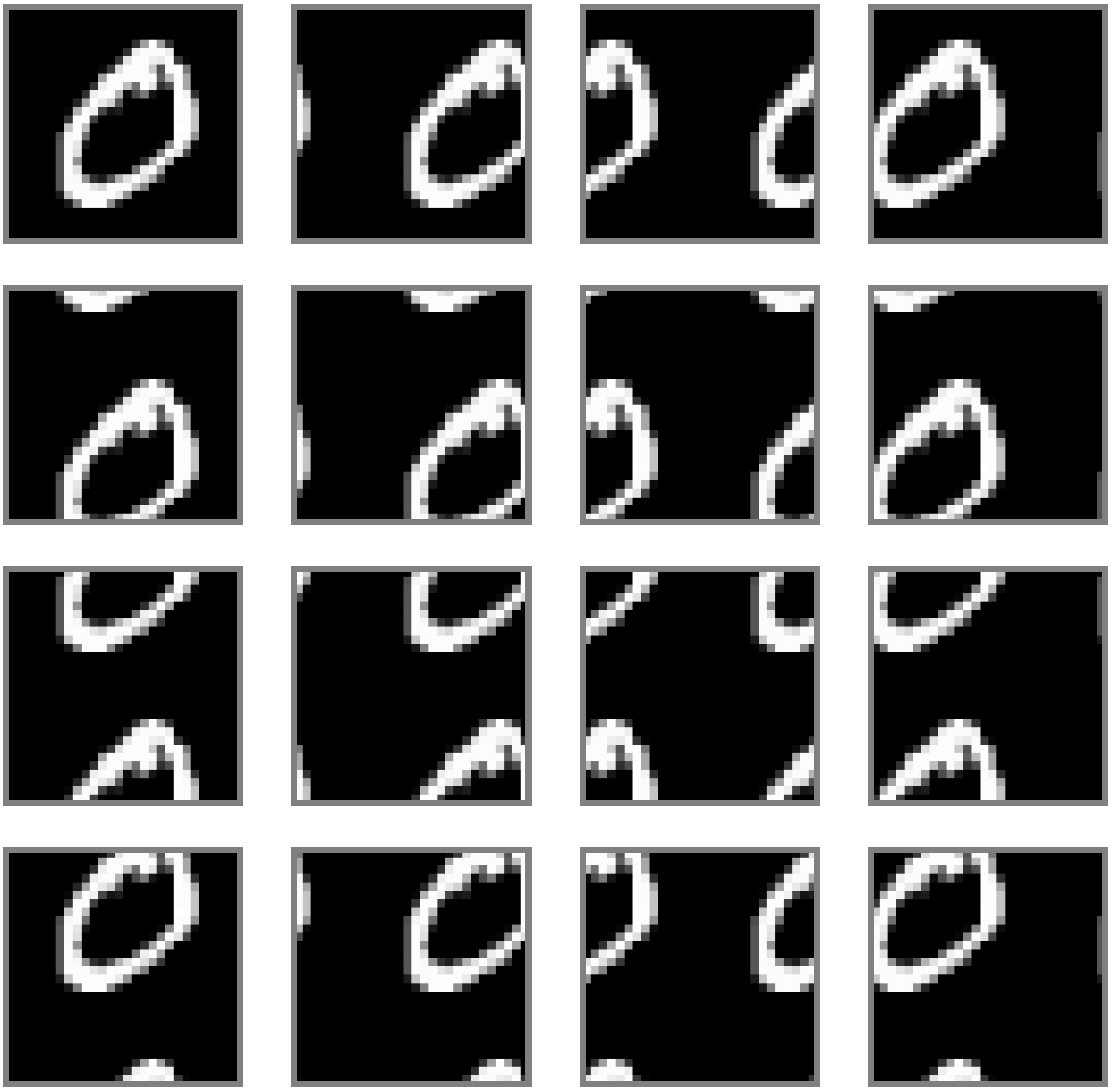}
    \includegraphics[width=0.27\textwidth]{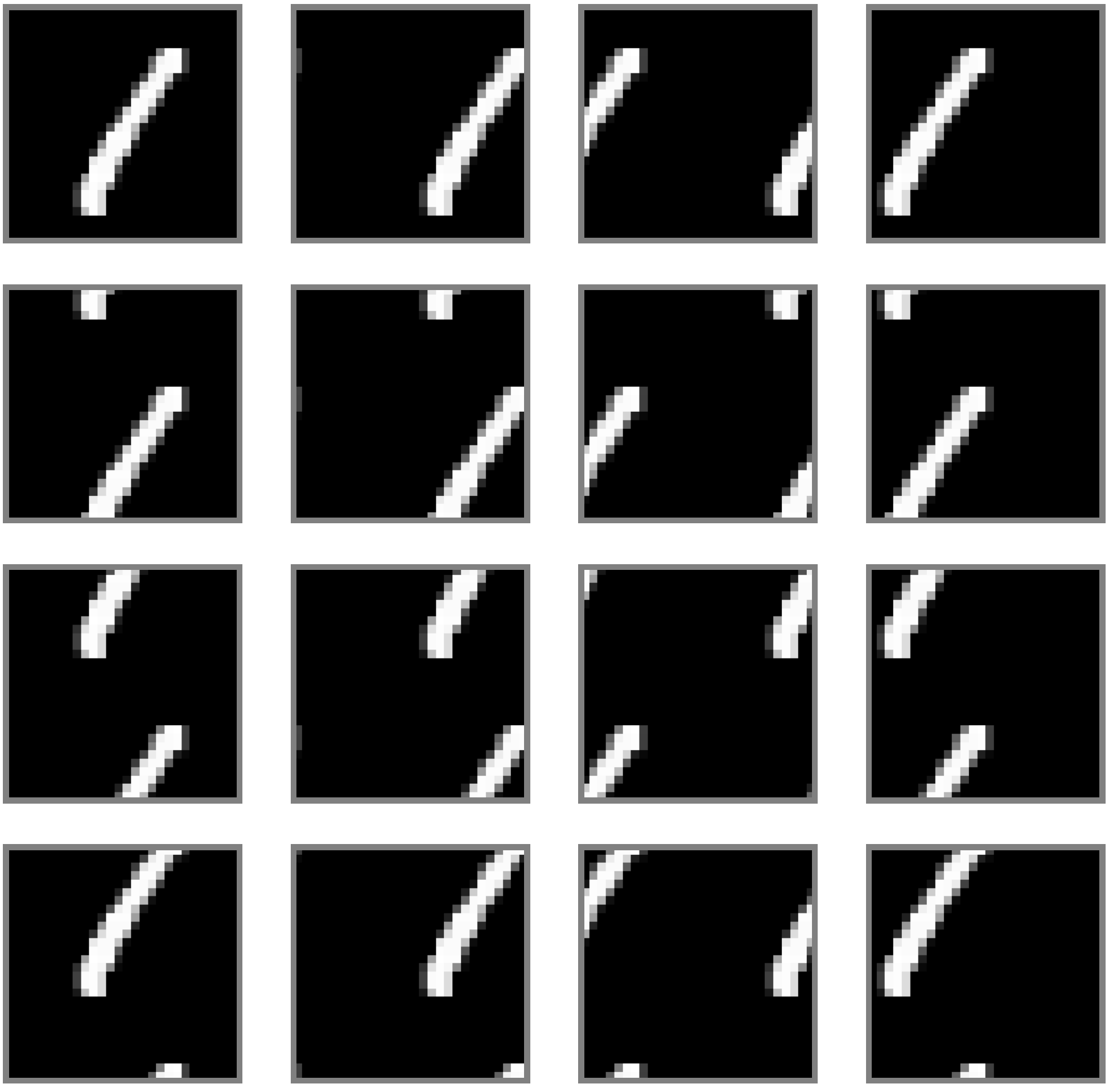}
    \vskip -0.1in
    \caption{(\textbf{Left}) A torus on which 2D cyclic translation is defined;  (\textbf{Right}) Cyclic translated images of MNIST digits `0' and `1' (with \textsf{stride=7}).}
    \label{fig:appendix-mnist-translation-visualize}
  \end{center}
\end{figure*}

\newpage
\subsection{Additional Experiments on Learning Rotational Invariance on MNIST}\label{sec:appendix-rotation}

\paragraph{Effect of channel size.} We study the effect of the number of channels on the rotation invariant task  on the MNIST dataset. We apply the same parameters as in Figure~\ref{fig:1d-invariance-plots-b} except that we vary the channel size from 5 to 20, where the number of channels used in Figure~\ref{fig:1d-invariance-plots-b} is 20. We summarize the results in Table~\ref{table:appendix-rotation-channel}, Figure~\ref{fig:mnist1d_different_channels_heatmap} and Figure~\ref{fig:mnist1d_different_channels_loss}. From Table~\ref{table:appendix-rotation-channel}, we find that the invariance training accuracy increases when we increase the channel size. As we can see from Figure~\ref{fig:mnist1d_different_channels_heatmap}, the shifted learned features become more orthogonal across different classes when the channel size is large. Also, we observe that with more channel size, the training loss converges faster (in Figure~\ref{fig:mnist1d_different_channels_loss}).

\begin{table}[ht]
\begin{center}
\begin{small}
\begin{tabular}{l|cccc}
\toprule
Channel & 5 & 10 & 15 & 20\\
\midrule
Training Acc &              1.000 & 1.000 &  1.000 &  1.000 \\
Test Acc &                  0.560 & 0.610 &  0.610 &  0.610 \\
Invariance Training Acc &   0.838 & 0.972 &  0.990 &  1.000 \\
Invariance Test Acc &       0.559 & 0.609 &  0.610 &  0.610 \\
\bottomrule
\end{tabular}\vspace{-2mm}
\caption{Training accuracy of 1D  rotation-invariant ReduNet with different number of channels on the MNIST dataset. 
}\label{table:appendix-rotation-channel}
\end{small}
\end{center}
\end{table}

\begin{figure}[ht]
    \centering
    \subfigure[5 Channels]{
    \includegraphics[width=0.20\textwidth]{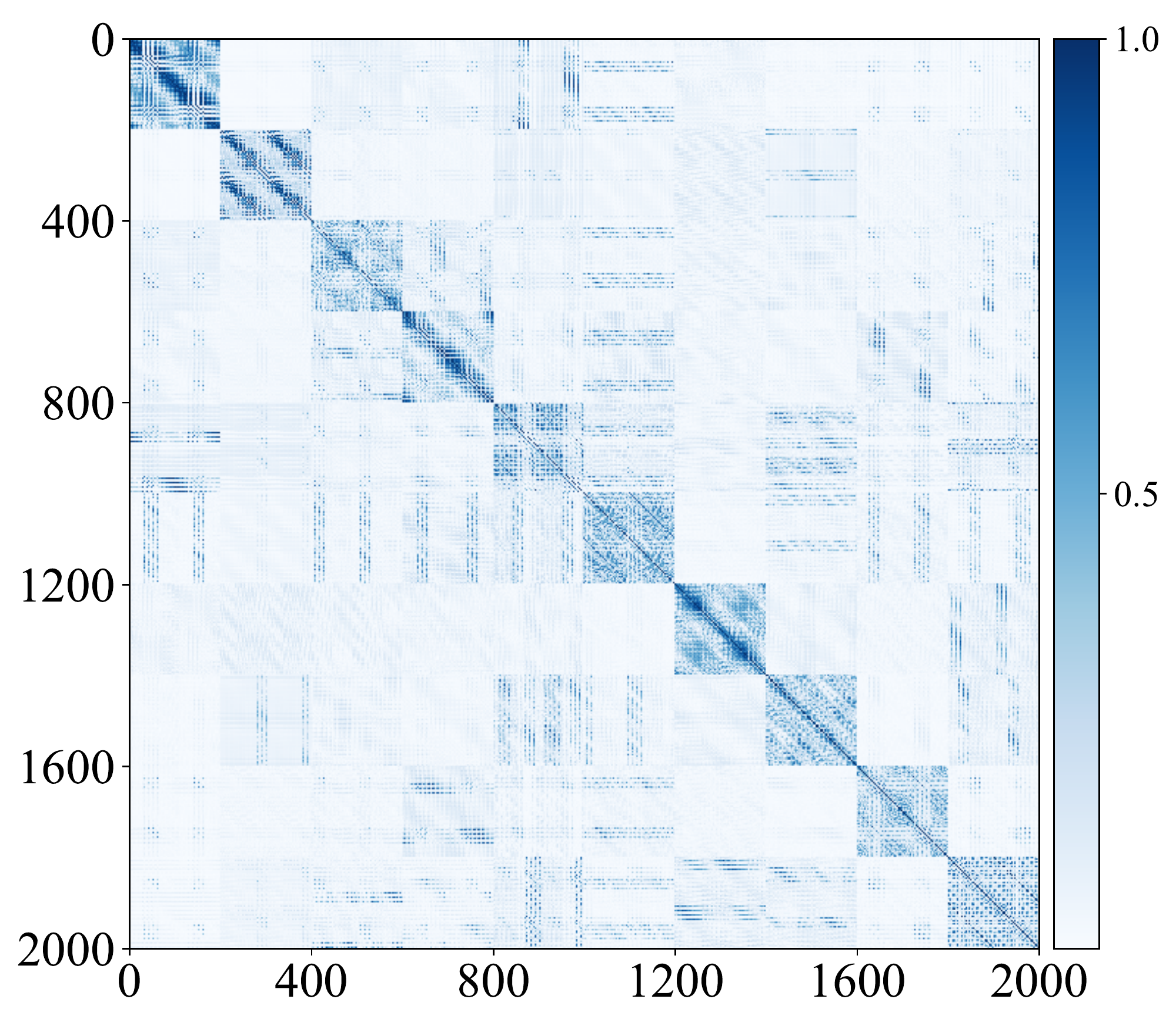}
    }
    \subfigure[10 Channels]{
    \includegraphics[width=0.20\textwidth]{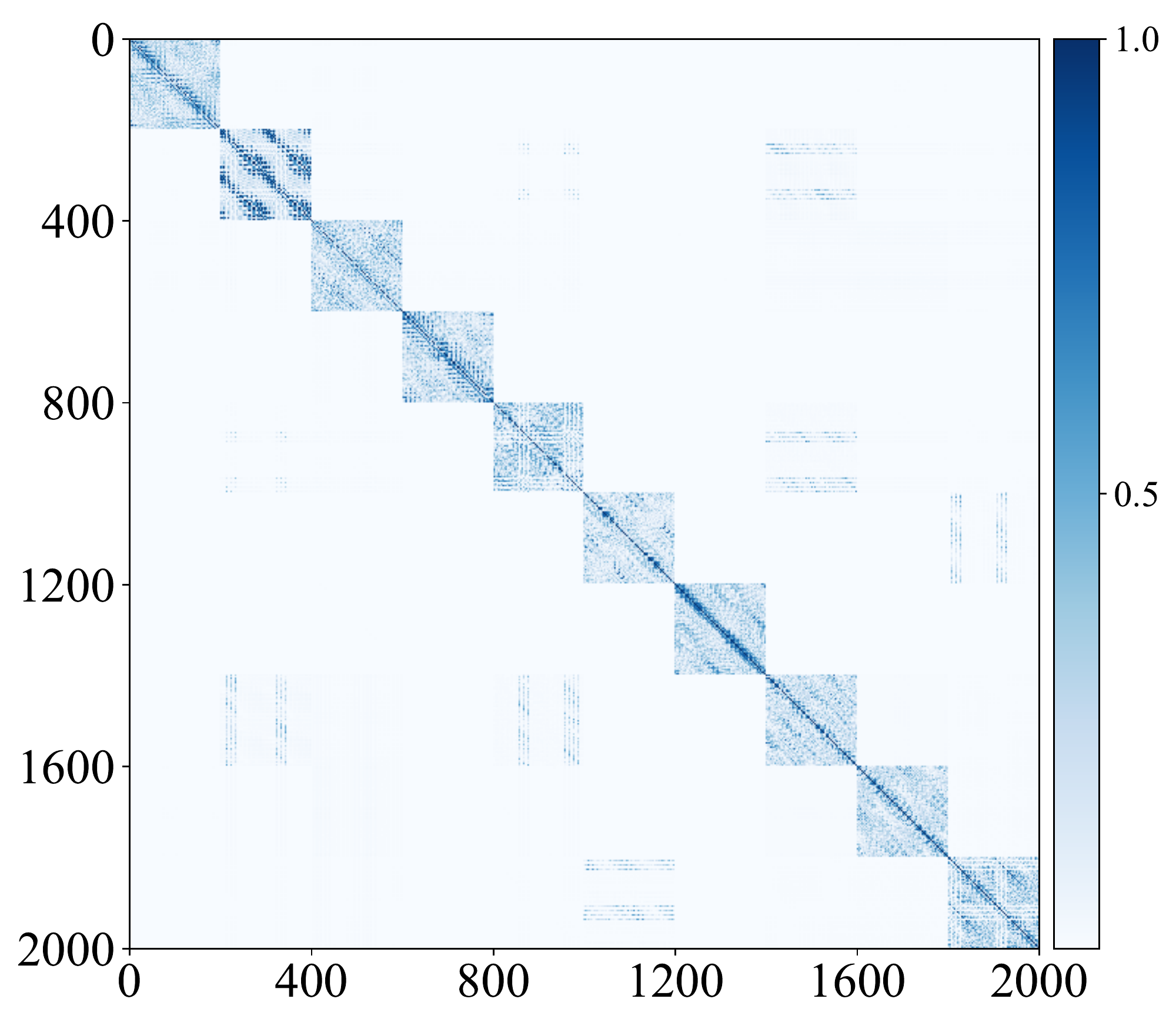}
    }
    \subfigure[15 Channels]{
    \includegraphics[width=0.20\textwidth]{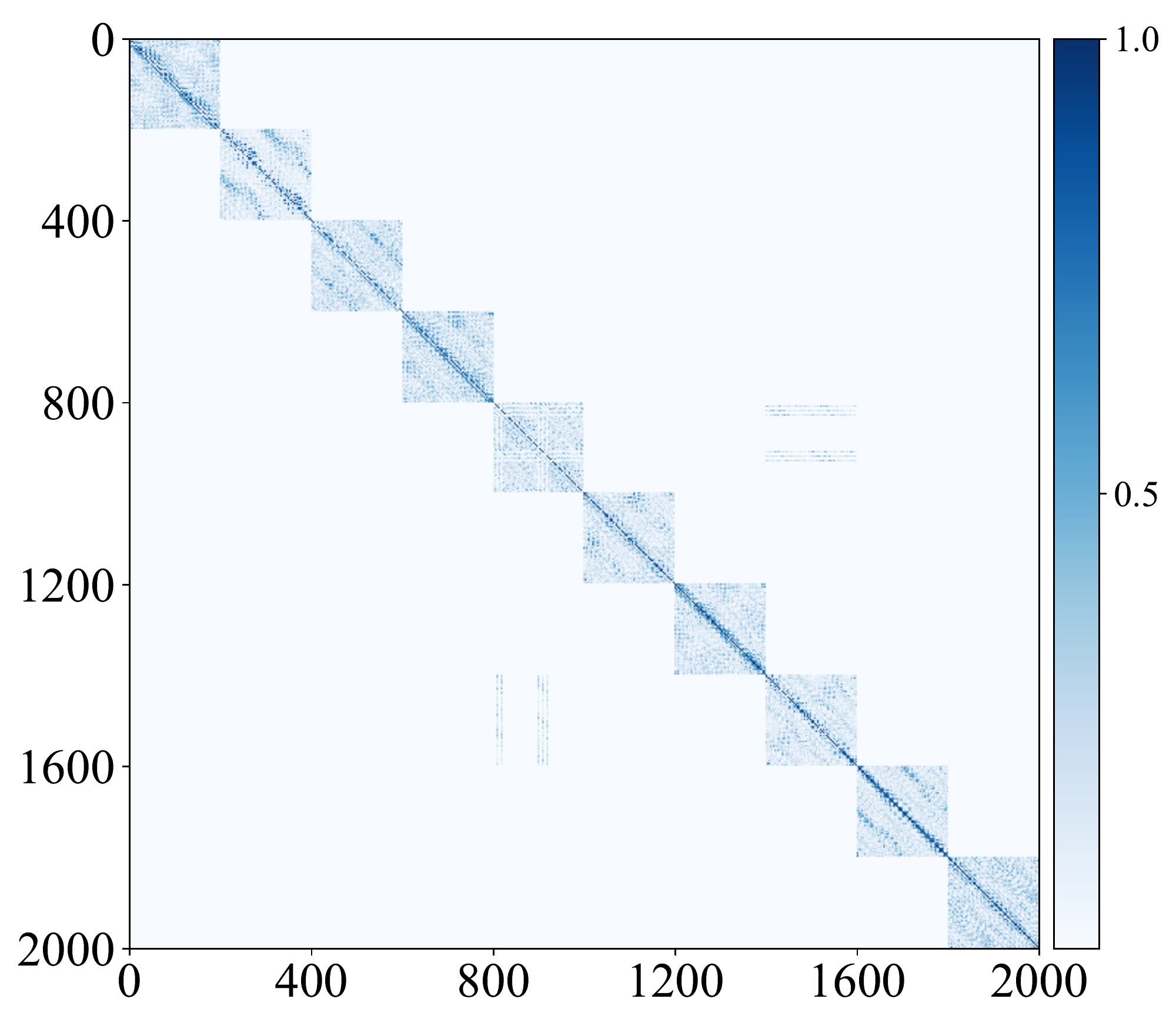}
    }
    \subfigure[20 Channels]{
    \includegraphics[width=0.20\textwidth]{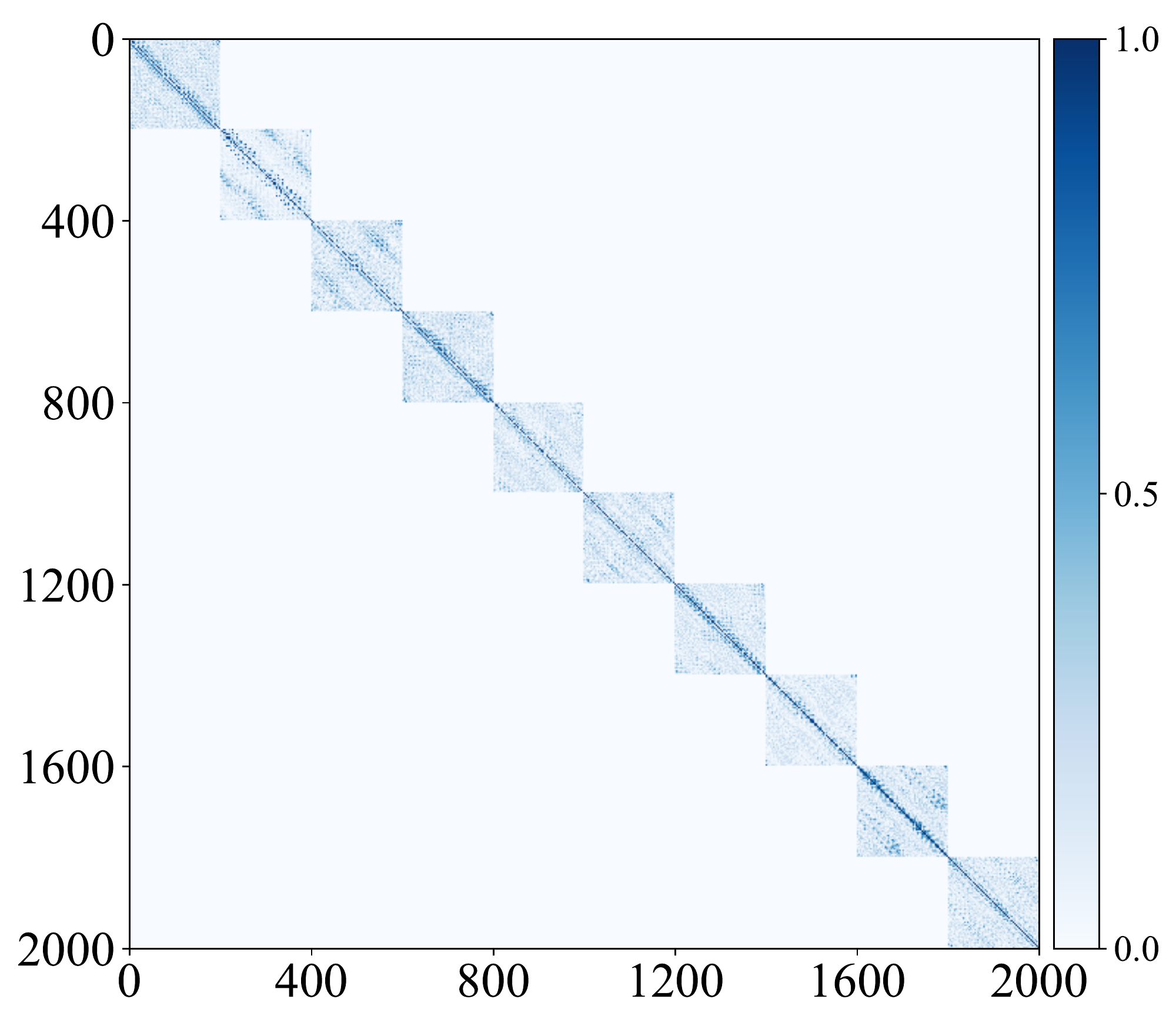}
    }
    \vspace{-0.1in}
    \caption{Heatmaps of cosine similarity among shifted learned features (with different channel sizes) $\bar{\Z}_{\text{shift}}$ for rotation invariance on the MNIST dataset (RI-MNIST). }
    \vspace{-0.01in}
    \label{fig:mnist1d_different_channels_heatmap}
\end{figure}
\begin{figure}[ht!]
    \centering
    \subfigure[5 Channels]{
    \includegraphics[width=0.22\textwidth]{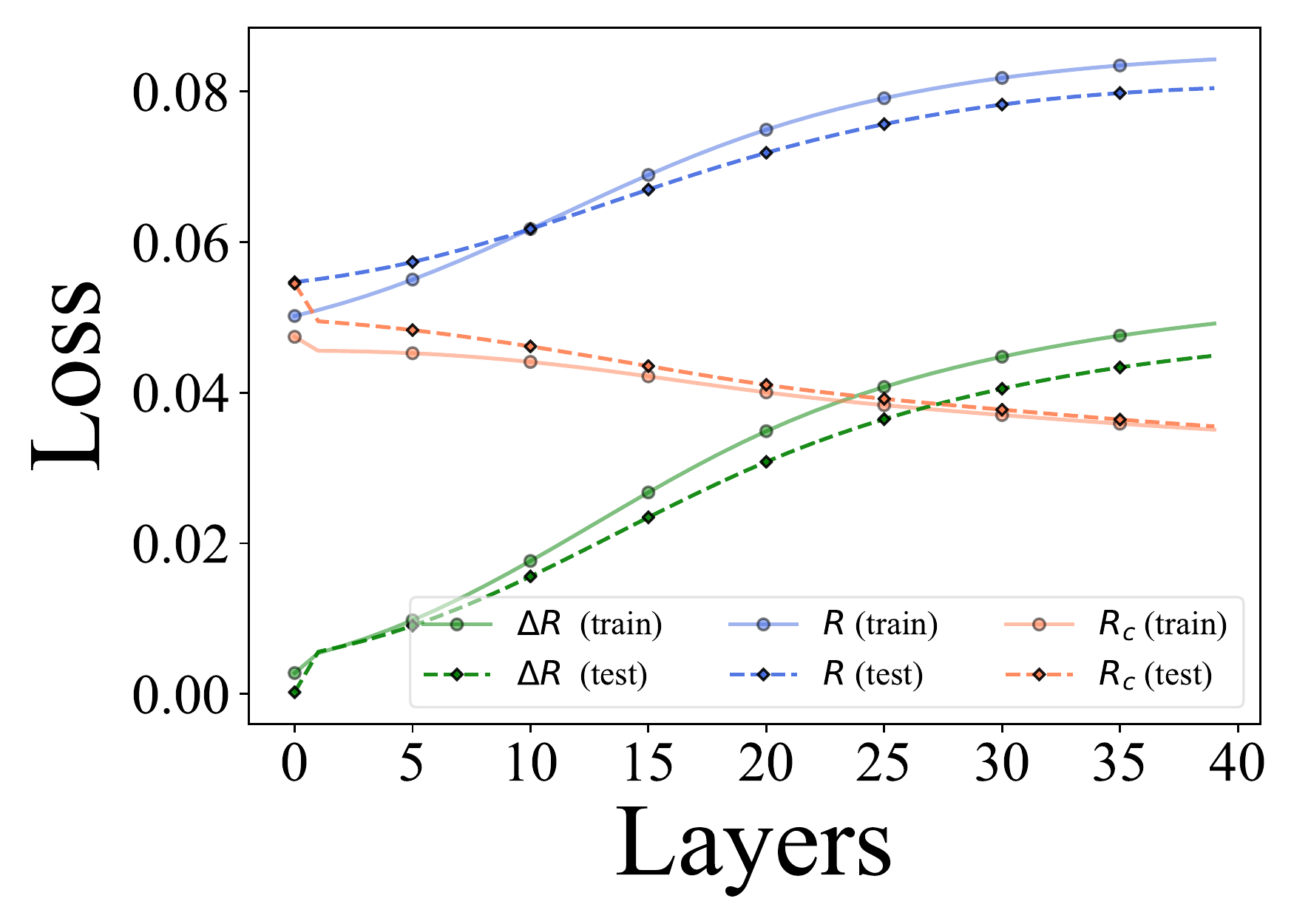}
    }
    \subfigure[10 Channels]{
    \includegraphics[width=0.22\textwidth]{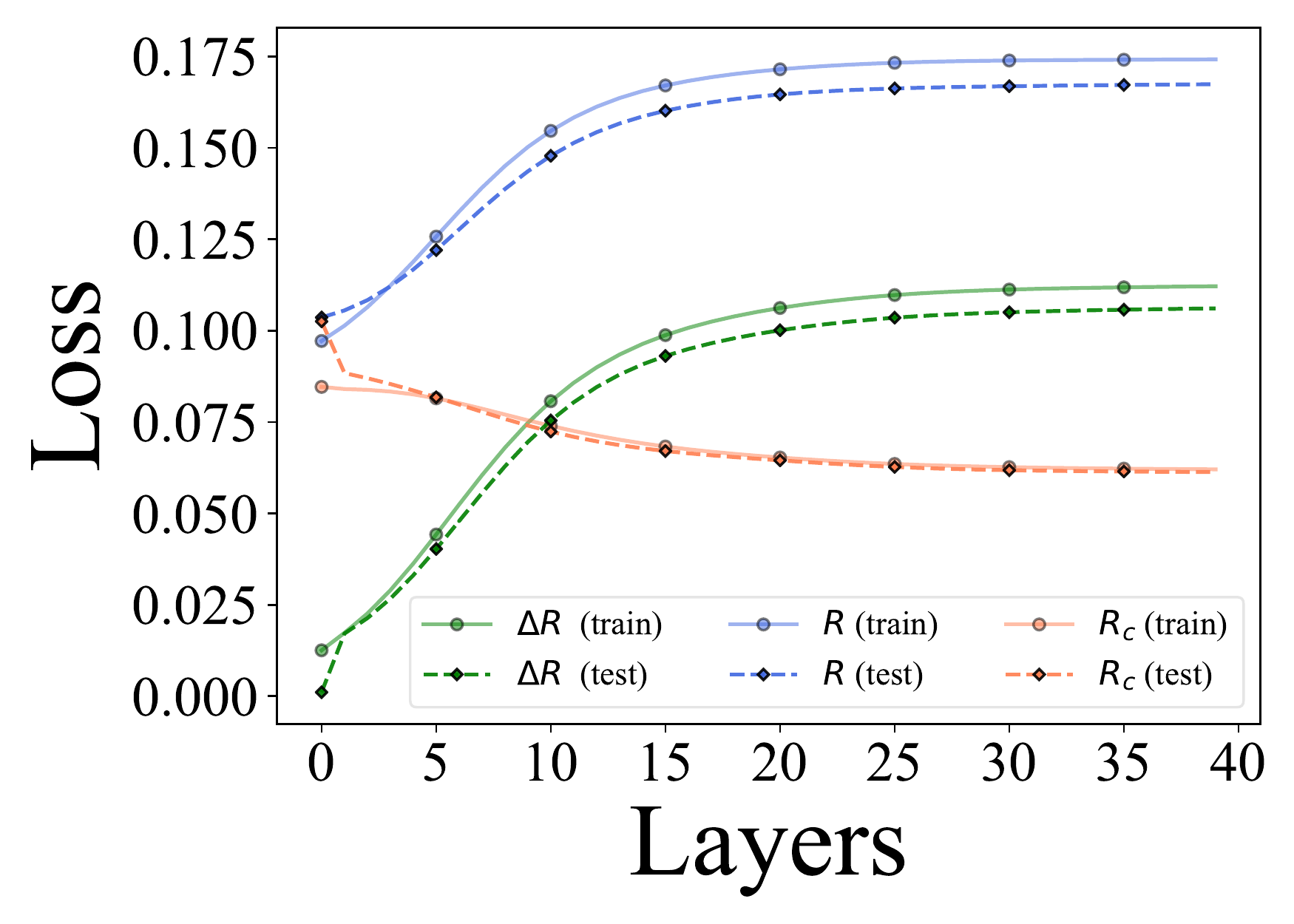}
    }
    \subfigure[15 Channels]{
    \includegraphics[width=0.22\textwidth]{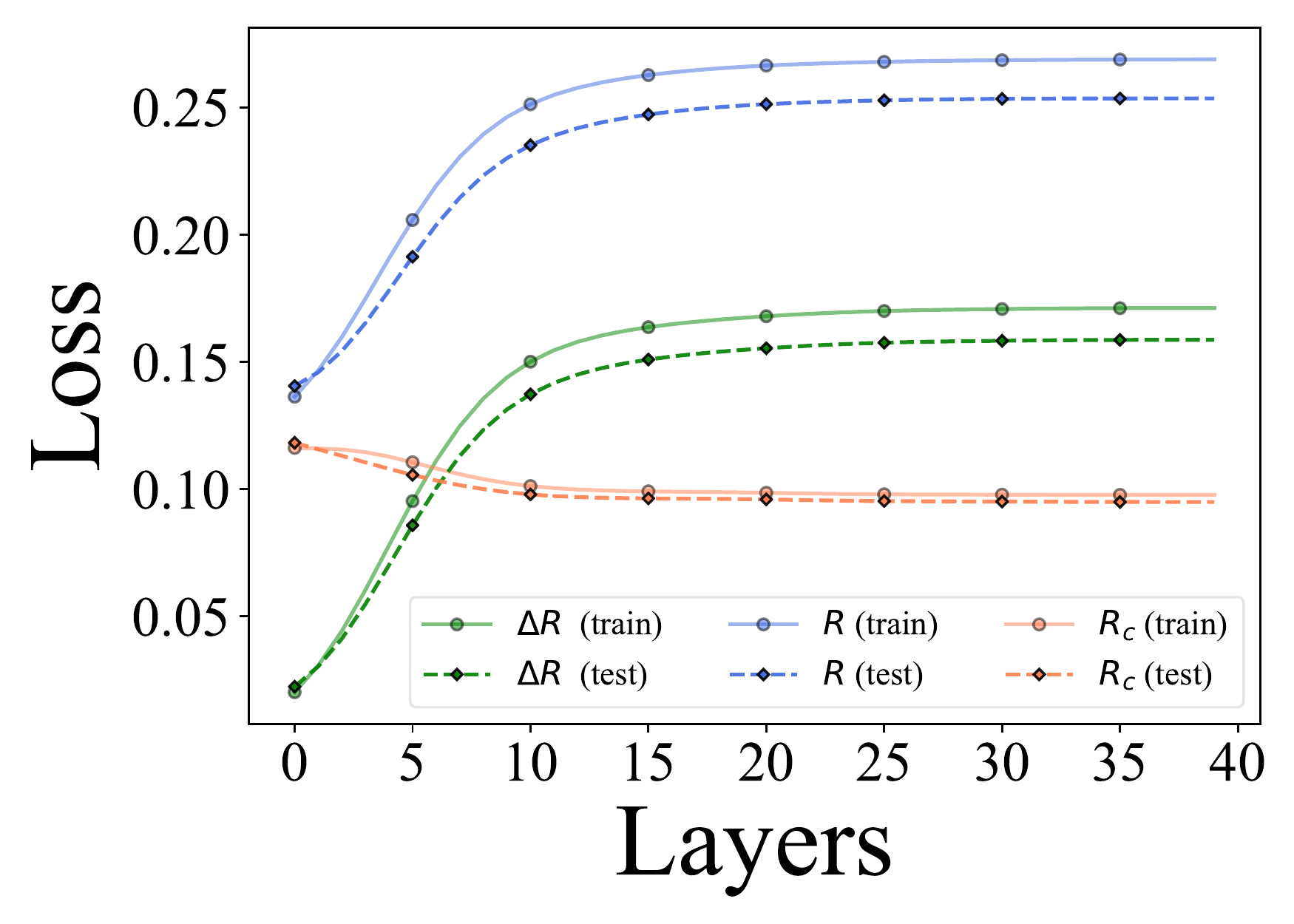}
    }
    \subfigure[20 Channels]{
    \includegraphics[width=0.22\textwidth]{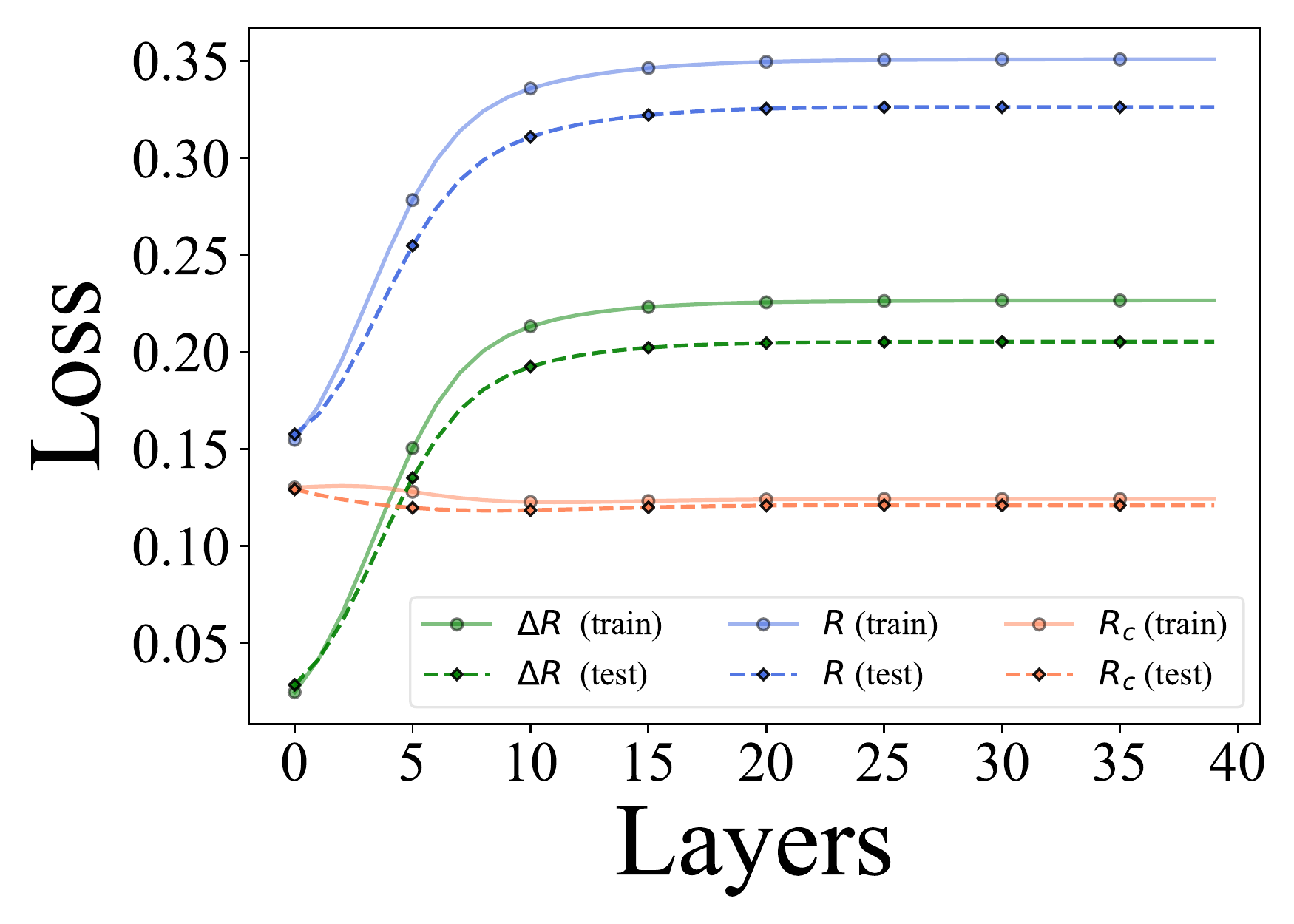}
    }
    \vspace{-0.1in}
    \caption{Training and test losses of rotational invariant ReduNet for rotation invariance on the MNIST dataset (RI-MNIST). }
    \vspace{-0.1in}
    \label{fig:mnist1d_different_channels_loss}
\end{figure}

\subsection{Additional Experiments on Learning 2D Translation Invariance on MNIST}\label{sec:appendix-translation}

\paragraph{Effect of channel size.} We study the effect of the number of channels on the translation invariant task on the MNIST dataset. We use the same parameters as in Figure~\ref{fig:1d-invariance-plots-f} except that we vary the channel size from 5 to 75, where the number of channels used in Figure~\ref{fig:1d-invariance-plots-f} is 75. Similar to the observations in Section~\ref{sec:appendix-rotation}, in Table~\ref{table:appendix-translation-channel}, we observe that both the invariance training accuracy and invariance test accuracy increase when we increase the channel size. From Figure~\ref{fig:mnist2d_different_channels_heatmap} and Figure~\ref{fig:mnist2d_different_channels_loss},  we find that when we increase the number of channels, the shifted learned features become more orthogonal across different classes and the training loss converges faster.

\begin{table}[ht]
\begin{center}
\begin{small}
\begin{tabular}{l|ccccc}
\toprule
Channel & 5 & 15 & 35 & 55 & 75\\
\midrule
Training Acc            & 1.000 & 1.000 & 1.000 & 1.000 & 1.000\\
Test Acc                & 0.680 & 0.610 & 0.670 & 0.770 & 0.840\\
Invariance Training Acc & 0.879 & 0.901 & 0.933 & 0.933 & 0.976\\
Invariance Test Acc     & 0.619 & 0.599 & 0.648 & 0.767 & 0.838\\
\bottomrule
\end{tabular}\vspace{-2mm}
\caption{\small Training/test accuracy of 2D  translation-invariant ReduNet with different number of channels on the MNIST dataset. 
}\label{table:appendix-translation-channel}
\end{small}
\end{center}
\end{table}

\begin{figure}[ht]
    \centering
    \subfigure[5 Channels]{
    \includegraphics[width=0.17\textwidth]{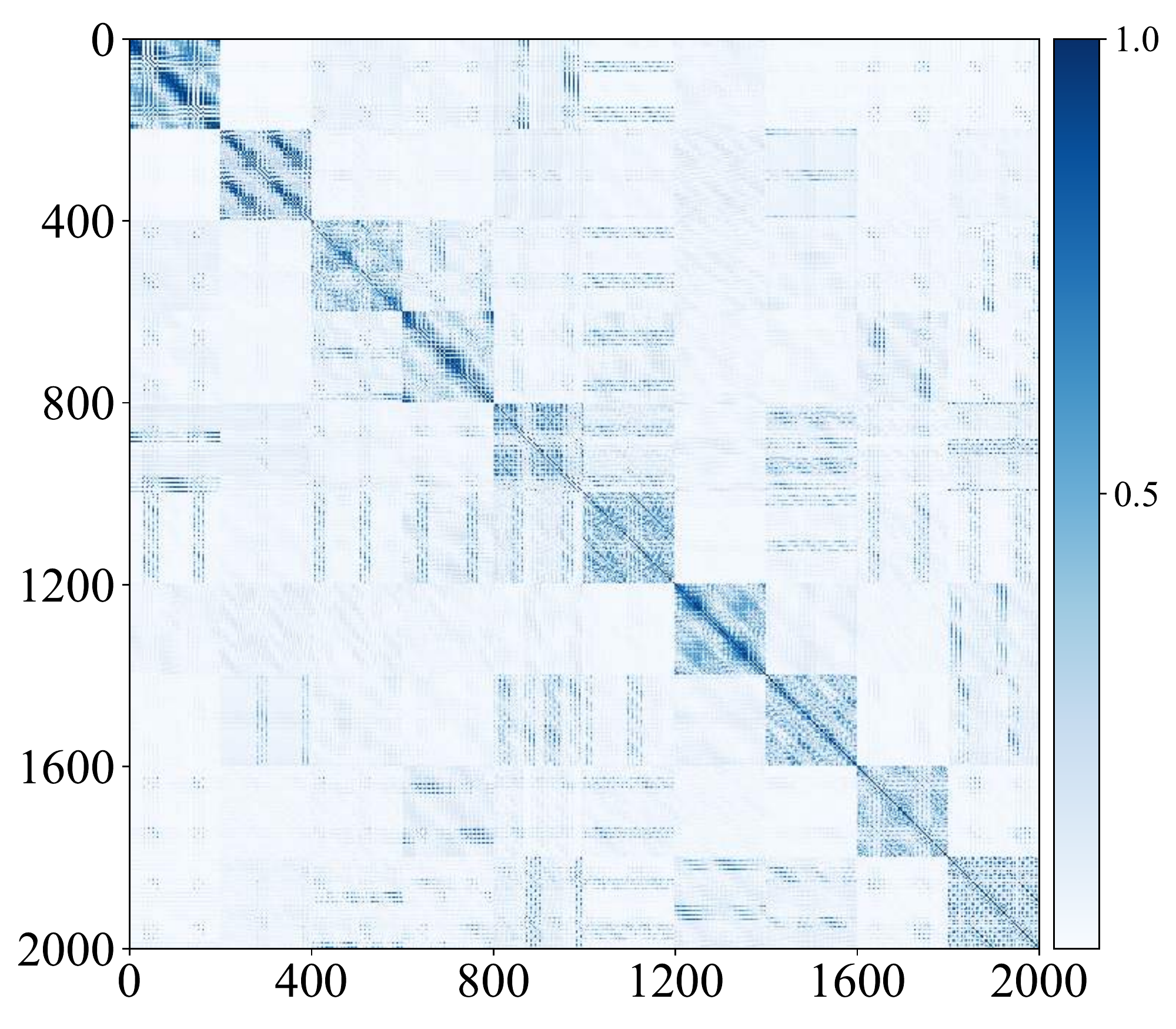}
    }
    \subfigure[15 Channels]{
    \includegraphics[width=0.17\textwidth]{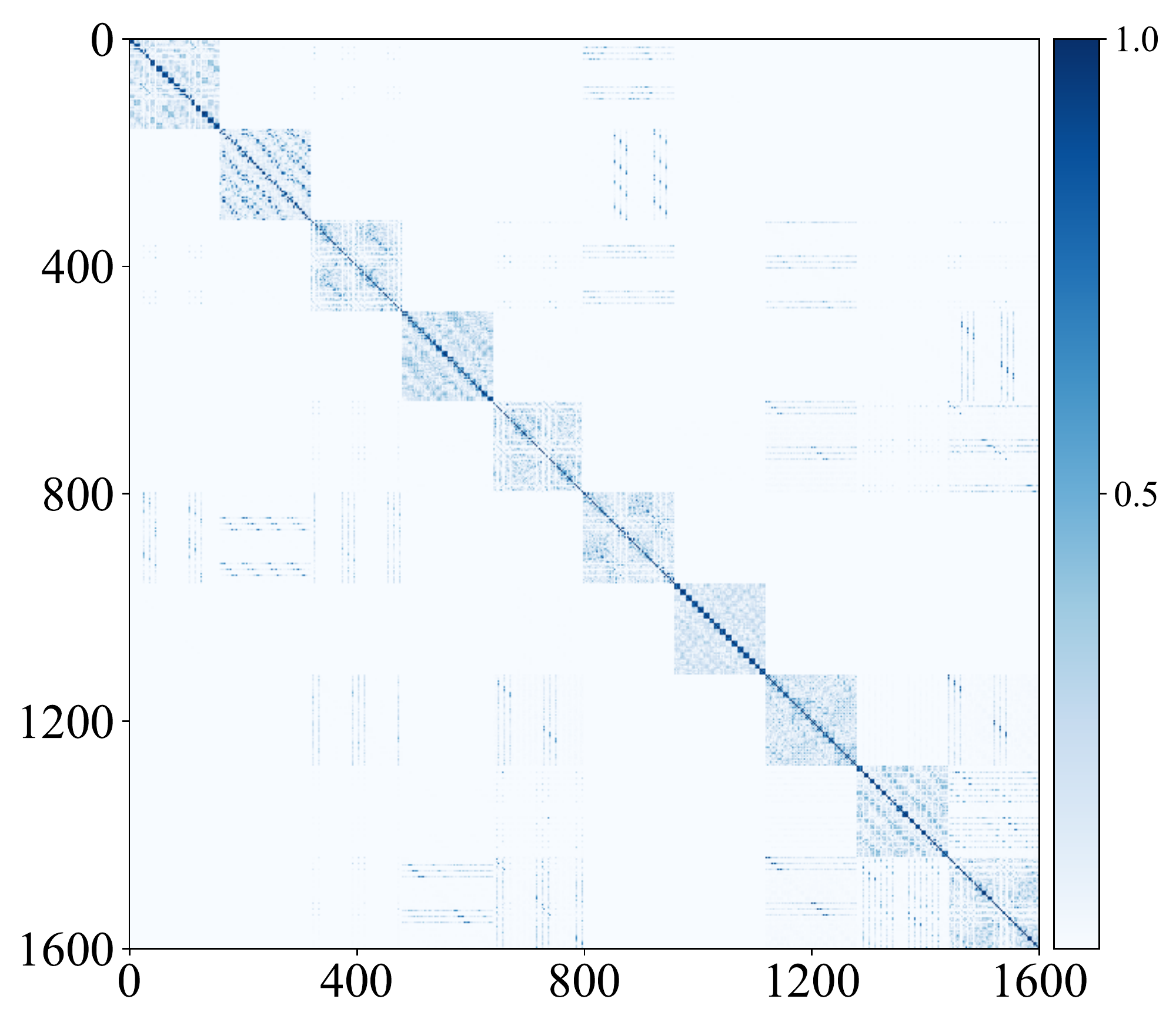}
    }
    \subfigure[35 Channels]{
    \includegraphics[width=0.17\textwidth]{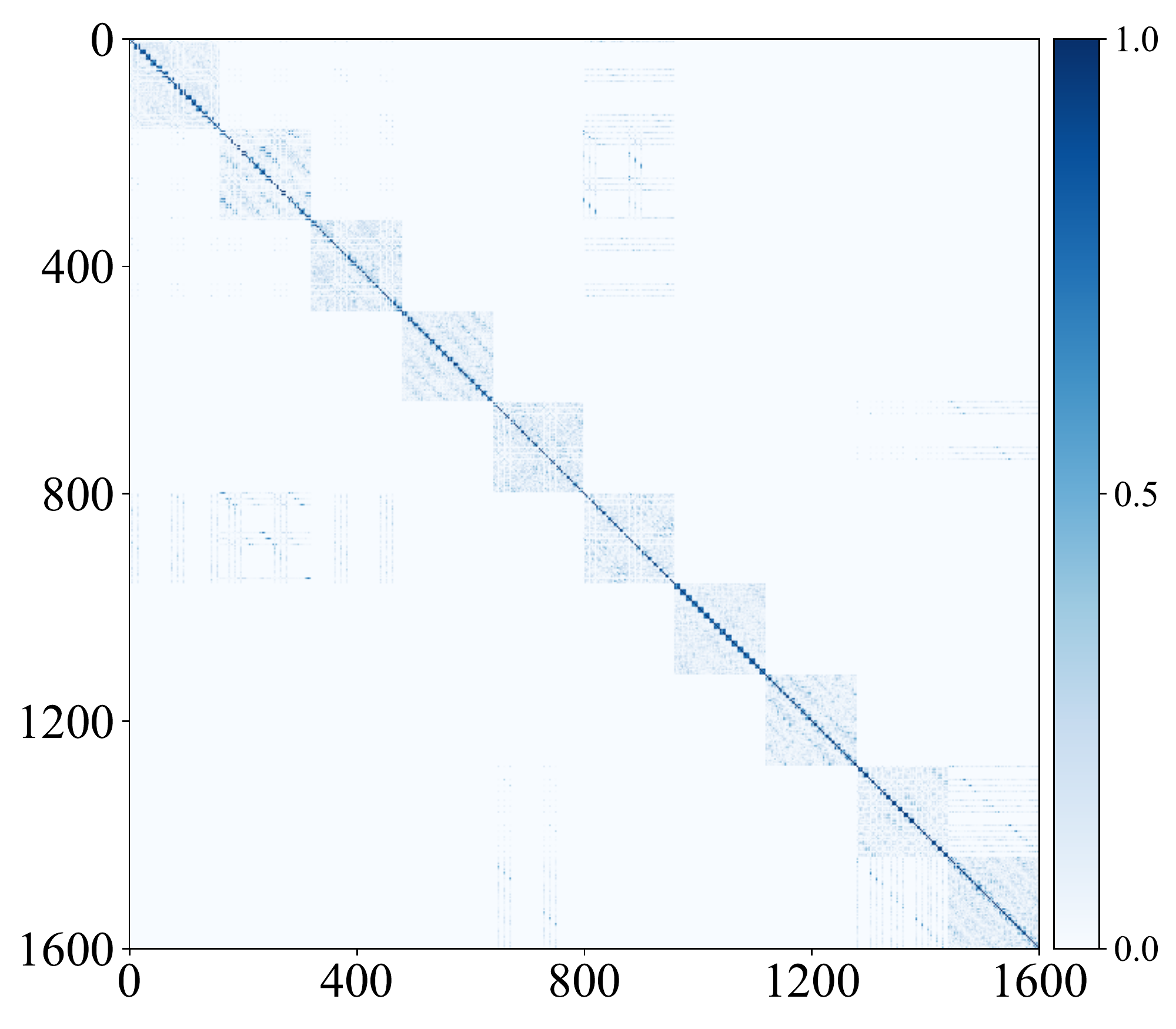}
    }
    \subfigure[55 Channels]{
    \includegraphics[width=0.17\textwidth]{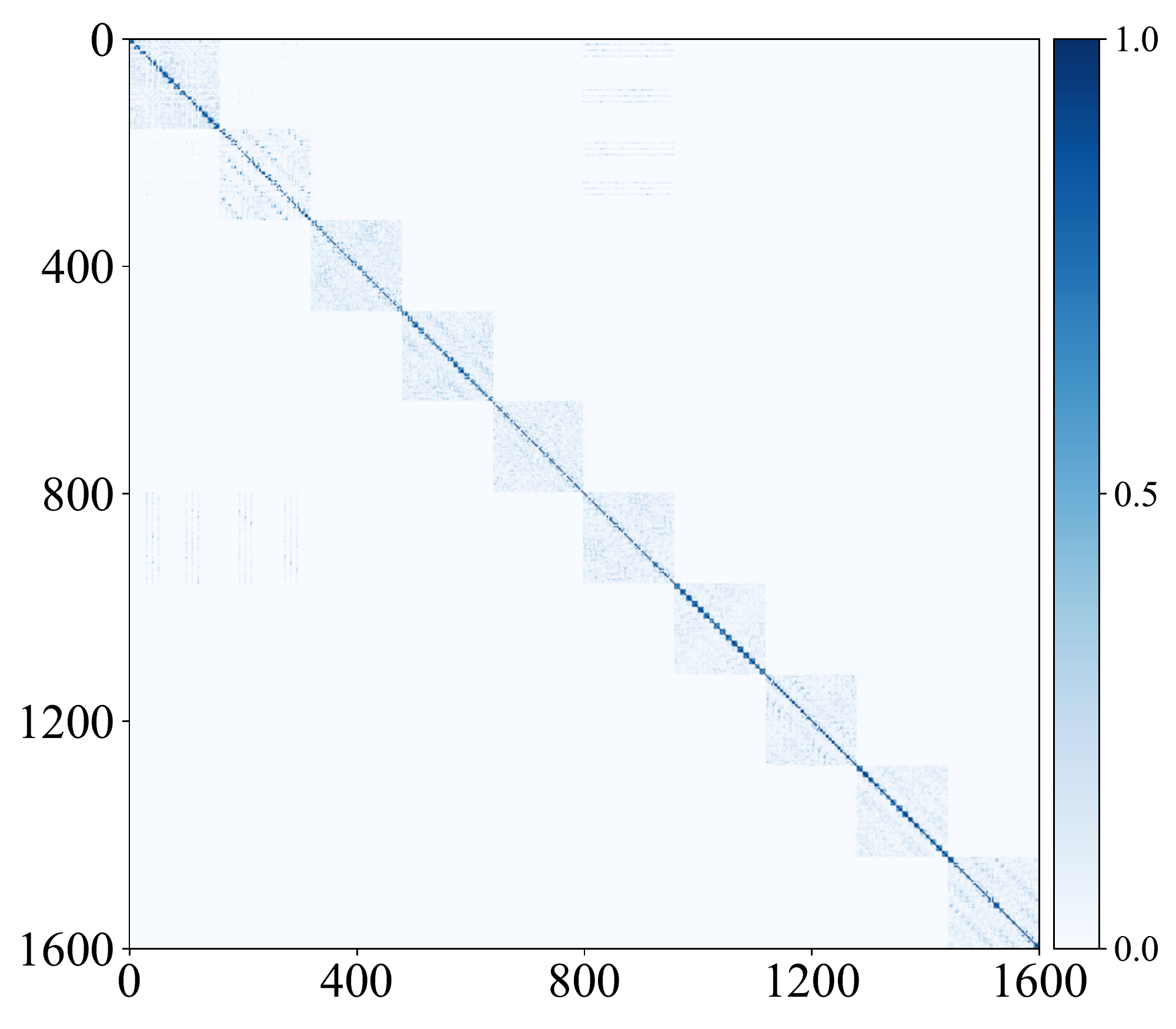}
    }
    \subfigure[75 Channels]{
    \includegraphics[width=0.17\textwidth]{experiments/redunet/mnist2d/class_all/heatmap-Z_translate_train_all.pdf}
    }
    \vspace{-0.1in}
    \caption{Heatmaps of cosine similarity among shifted learned features (with different channel sizes) $\bar{\Z}_{\text{shift}}$ for translation invariance on the MNIST dataset (TI-MNIST).}
    \label{fig:mnist2d_different_channels_heatmap}
\end{figure}
\begin{figure}[ht!]
    \centering
    \subfigure[5 Channels]{
    \includegraphics[width=0.17\textwidth]{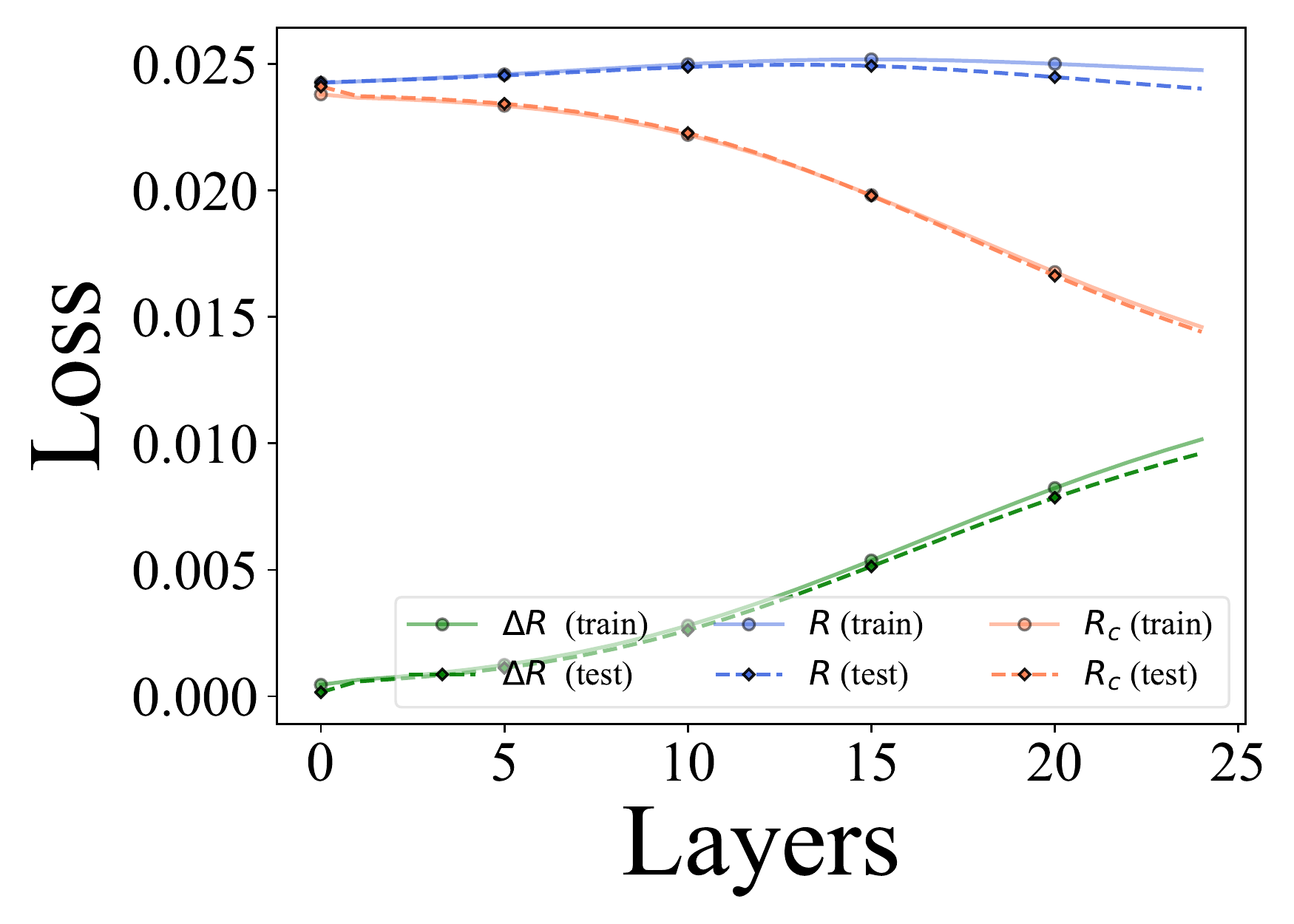}
    }
    \subfigure[15 Channels]{
    \includegraphics[width=0.17\textwidth]{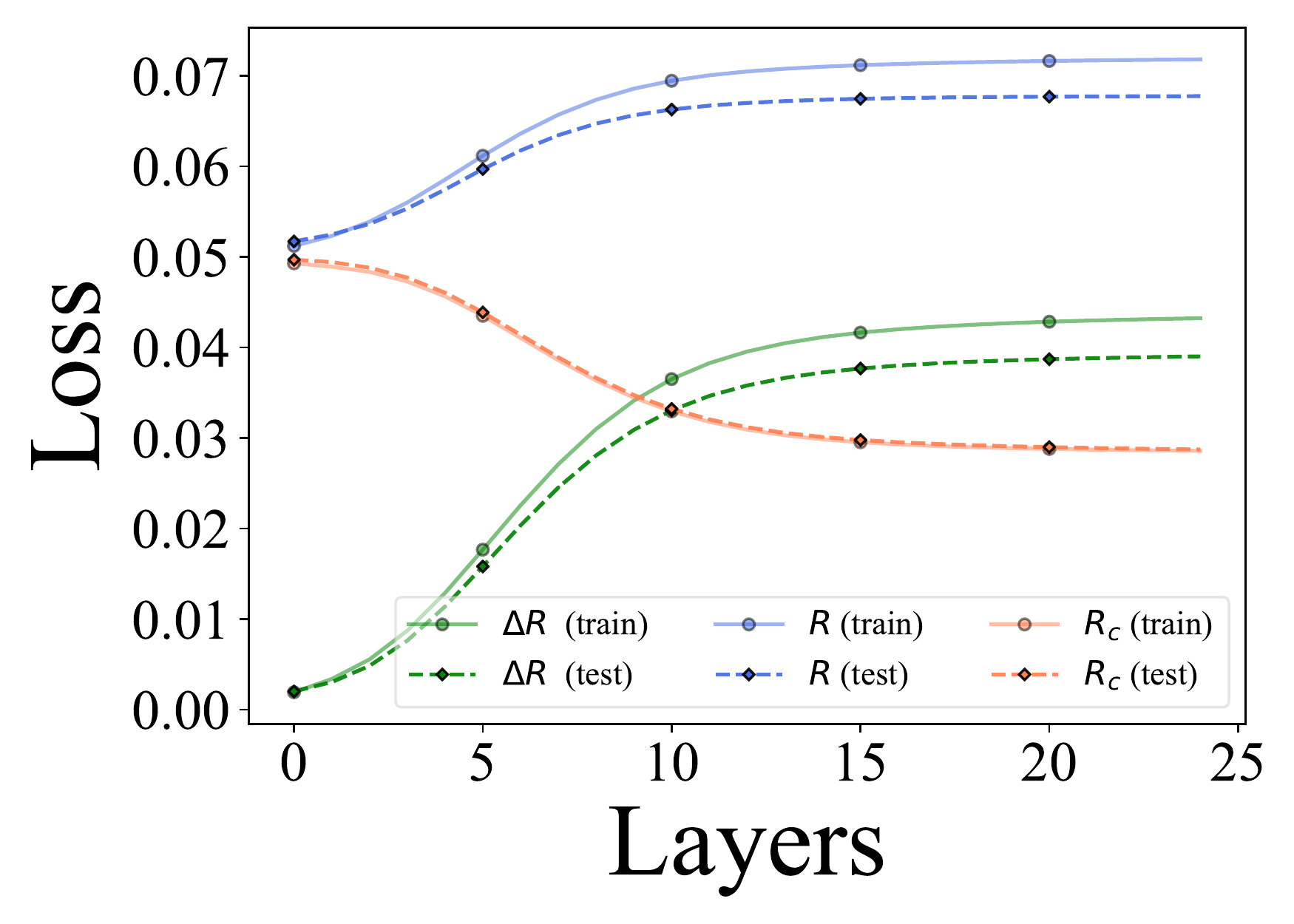}
    }
    \subfigure[35 Channels]{
    \includegraphics[width=0.17\textwidth]{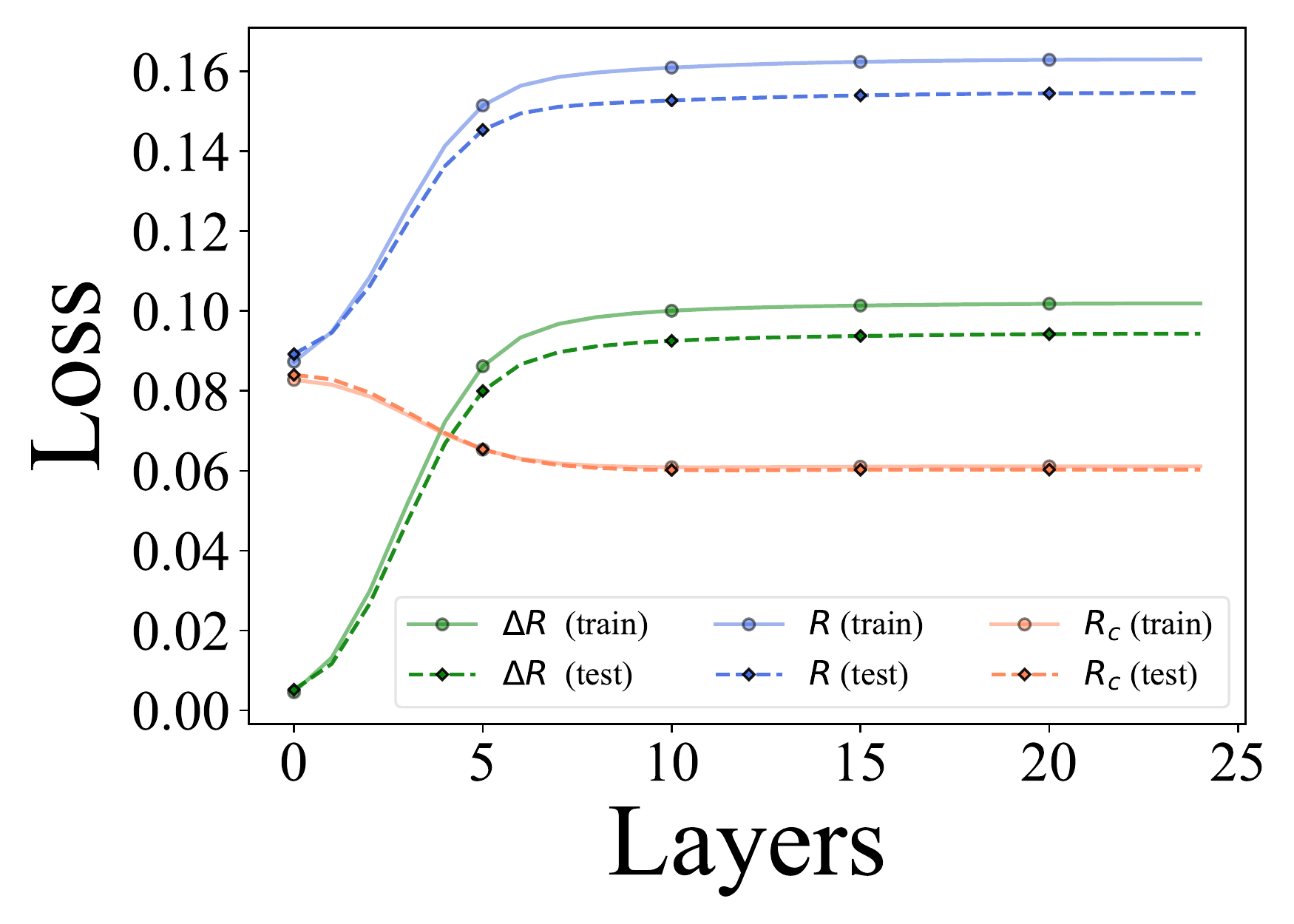}
    }
    \subfigure[55 Channels]{
    \includegraphics[width=0.17\textwidth]{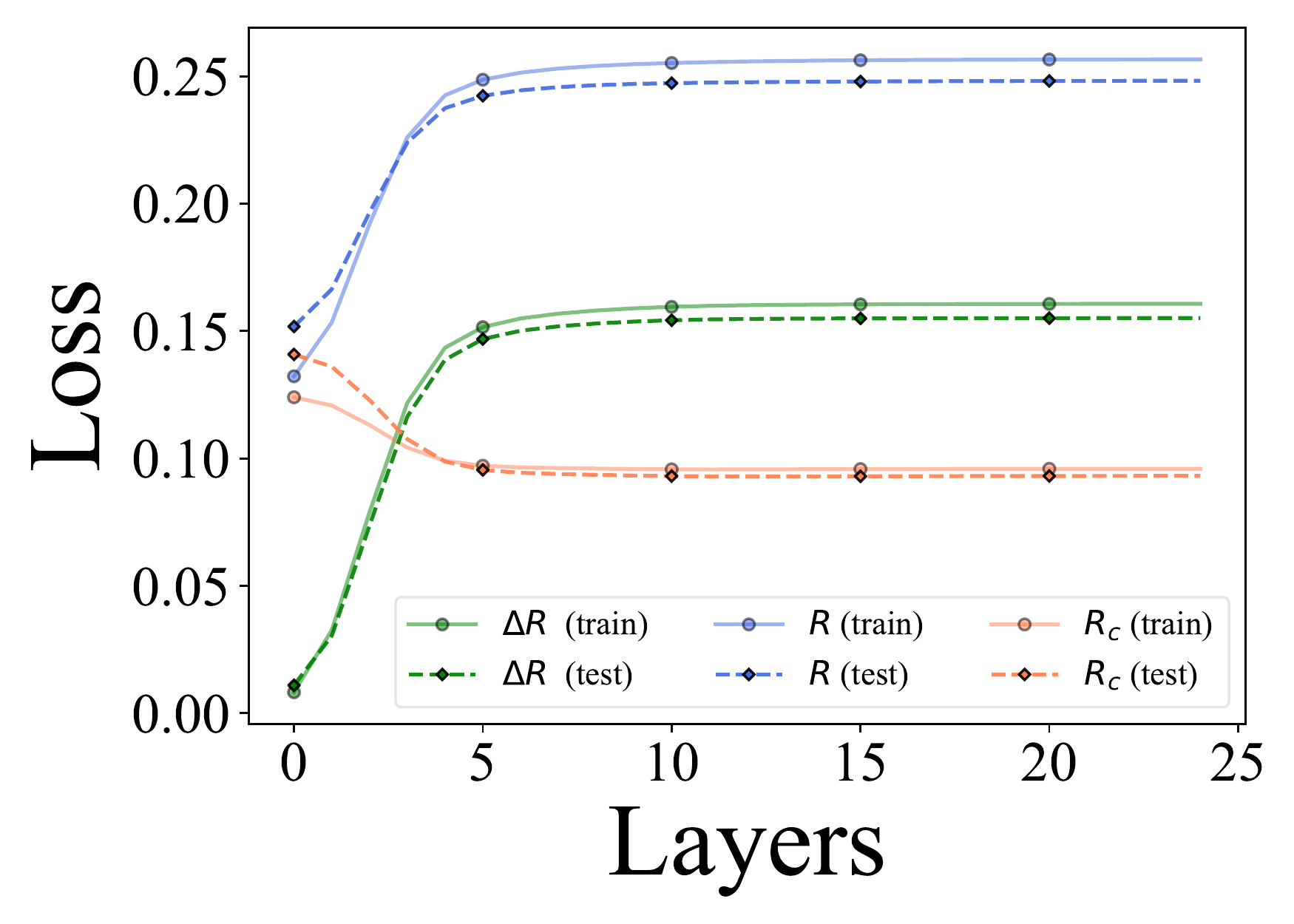}
    }
    \subfigure[75 Channels]{
    \includegraphics[width=0.17\textwidth]{experiments/redunet/mnist2d/class_all/loss-traintest.pdf}
    }
    \vspace{-0.1in}
    \caption{Training and test losses of translation invariant ReduNet for translation invariance on the MNIST dataset (TI-MNIST).}
    \label{fig:mnist2d_different_channels_loss}
\end{figure}

\subsection{Additional Experiments on Learning Mixture of Gaussians in $\mathbb{S}^1$ and $\mathbb{S}^2$}\label{sec:appendix-gaussian}

\paragraph{Additional experiments on  $\mathbb{S}^1$ and $\mathbb{S}^2$.} We also provide additional experiments on learning mixture of Gaussians in $\mathbb{S}^1$ and $\mathbb{S}^2$ in Figure~\ref{fig:appendix-guassian-exp1}. We can observe similar behavior of the proposed ReduNet: the network can map data points from different classes to orthogonal subspaces.

\begin{figure*}[ht!]
  \begin{center}
    \subfigure[$\X (2D)$ (\textbf{left: }scatter plot; \textbf{right: }cosine similarity visualization)]{
        \includegraphics[width=0.17\textwidth]{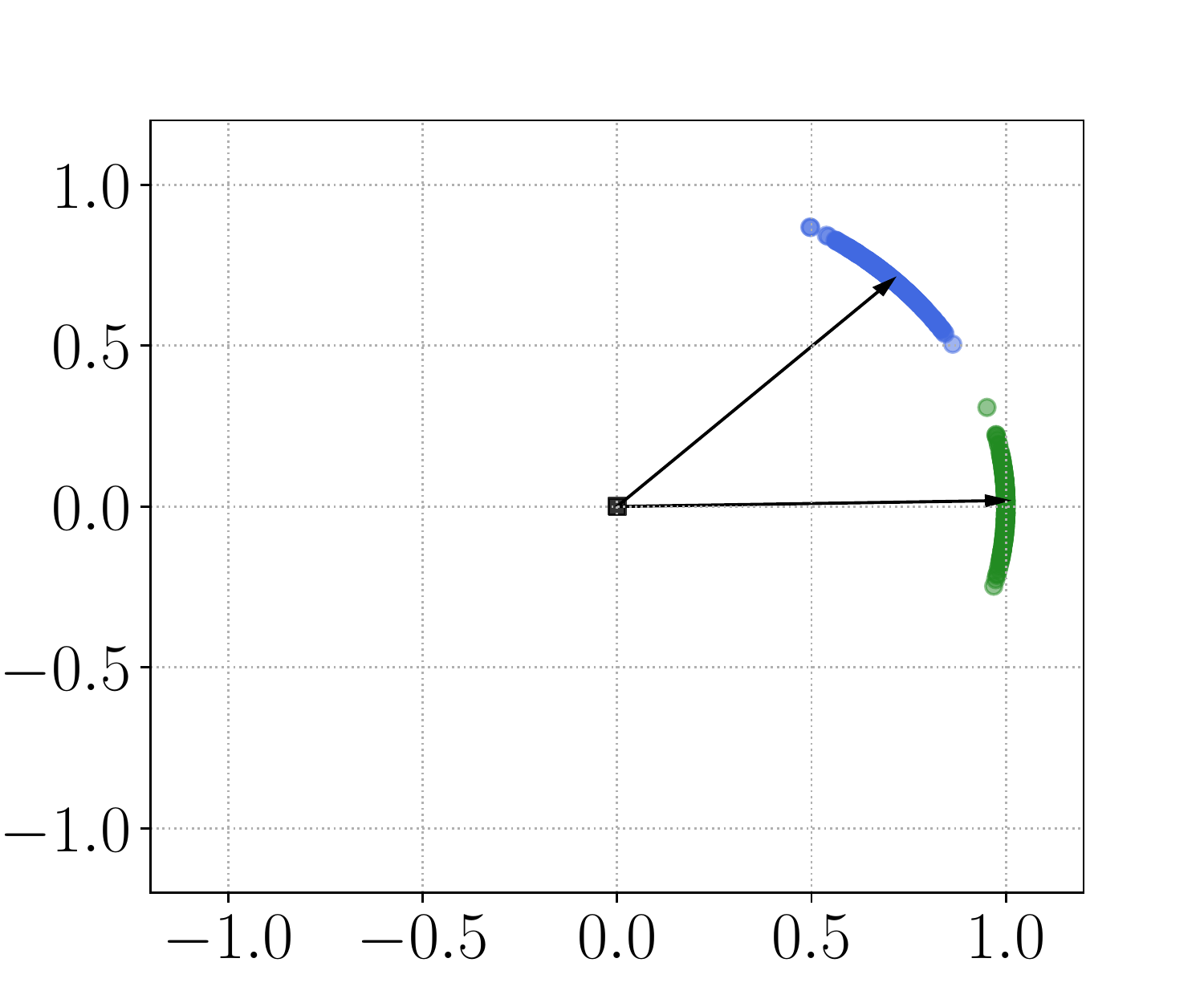}
        \includegraphics[width=0.17\textwidth]{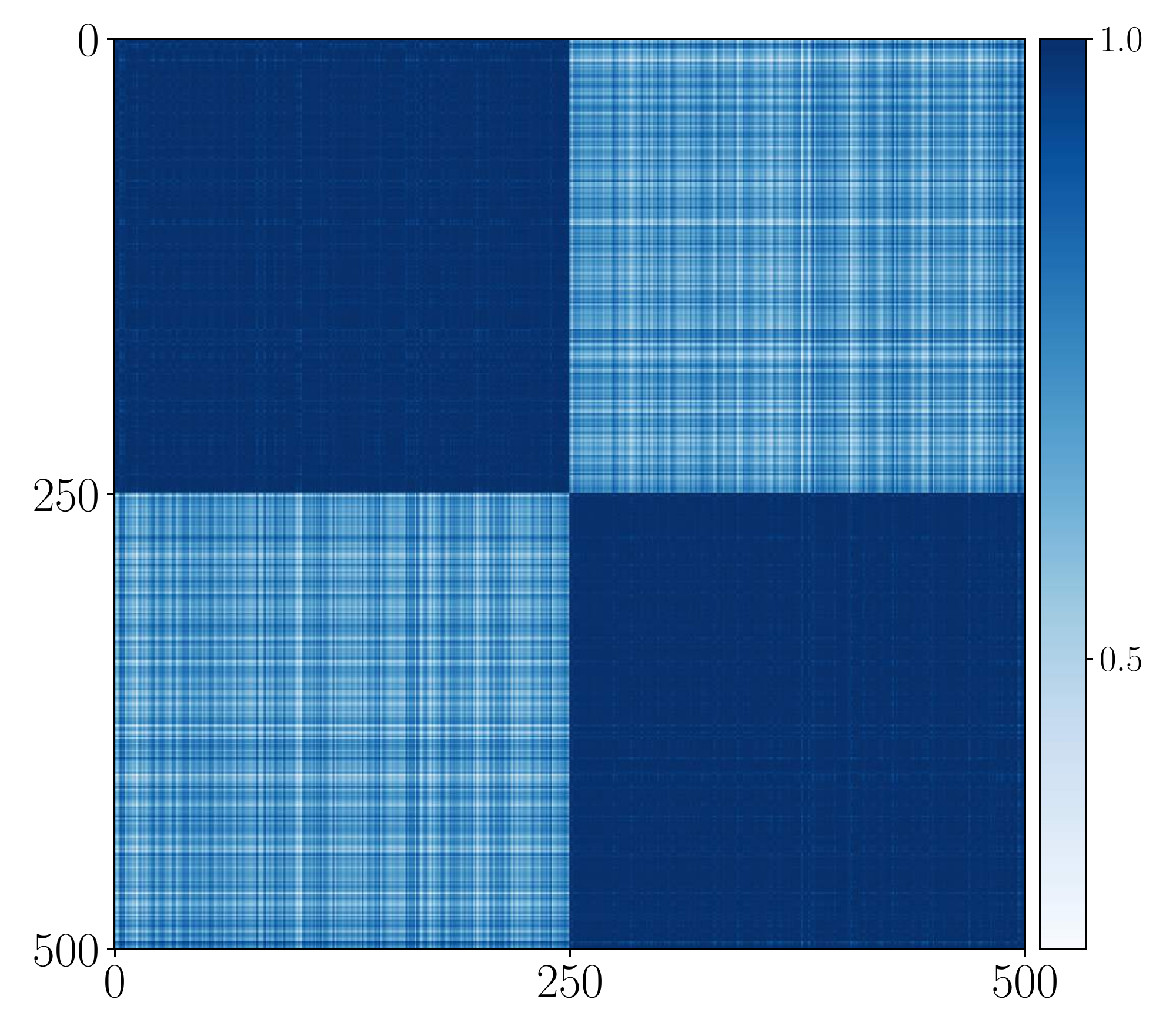}
    }
    \quad
    \subfigure[$\Z (2D)$ (\textbf{left: }scatter plot; \textbf{right: }cosine similarity visualization)]{
        \includegraphics[width=0.17\textwidth]{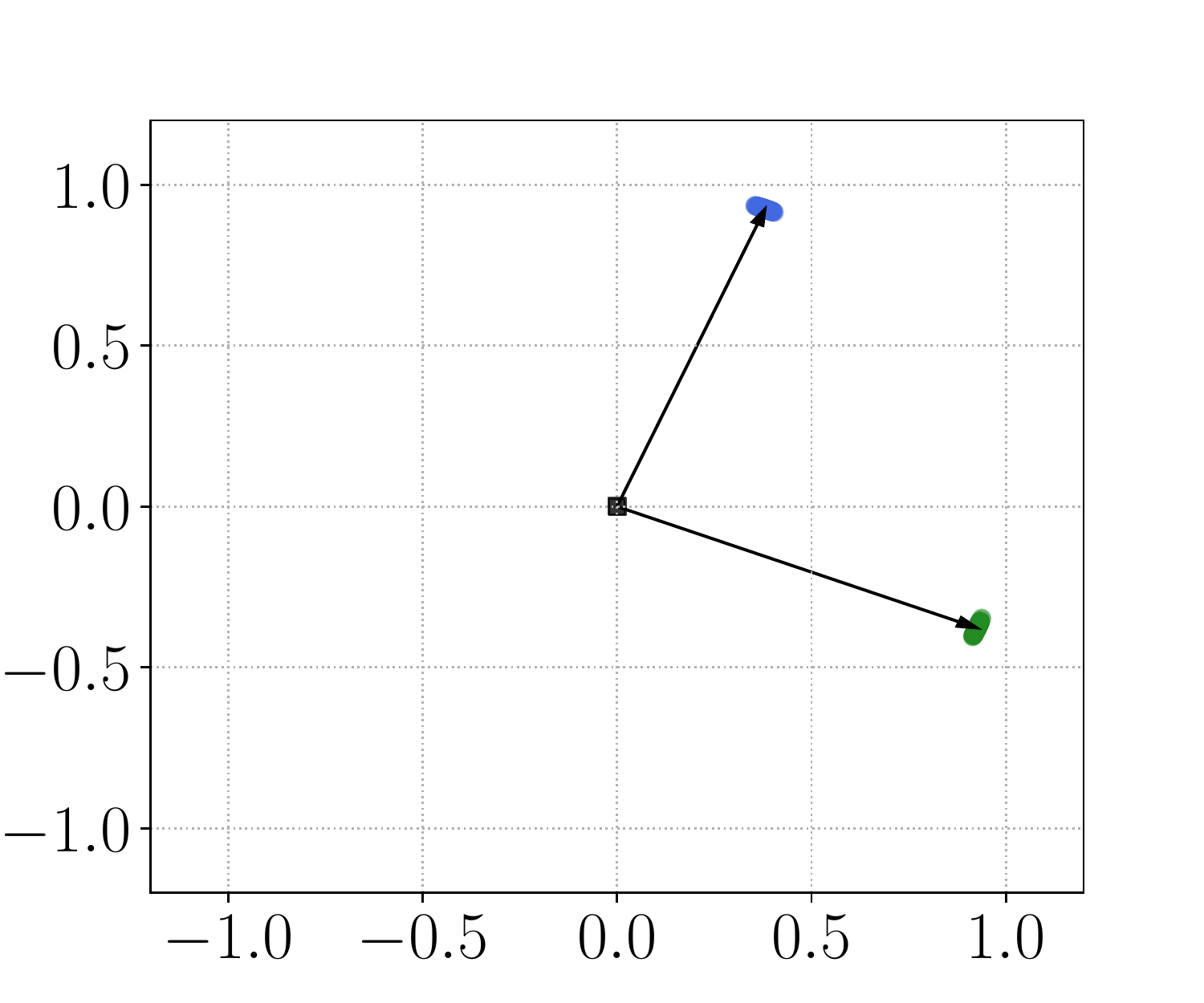}
        \includegraphics[width=0.17\textwidth]{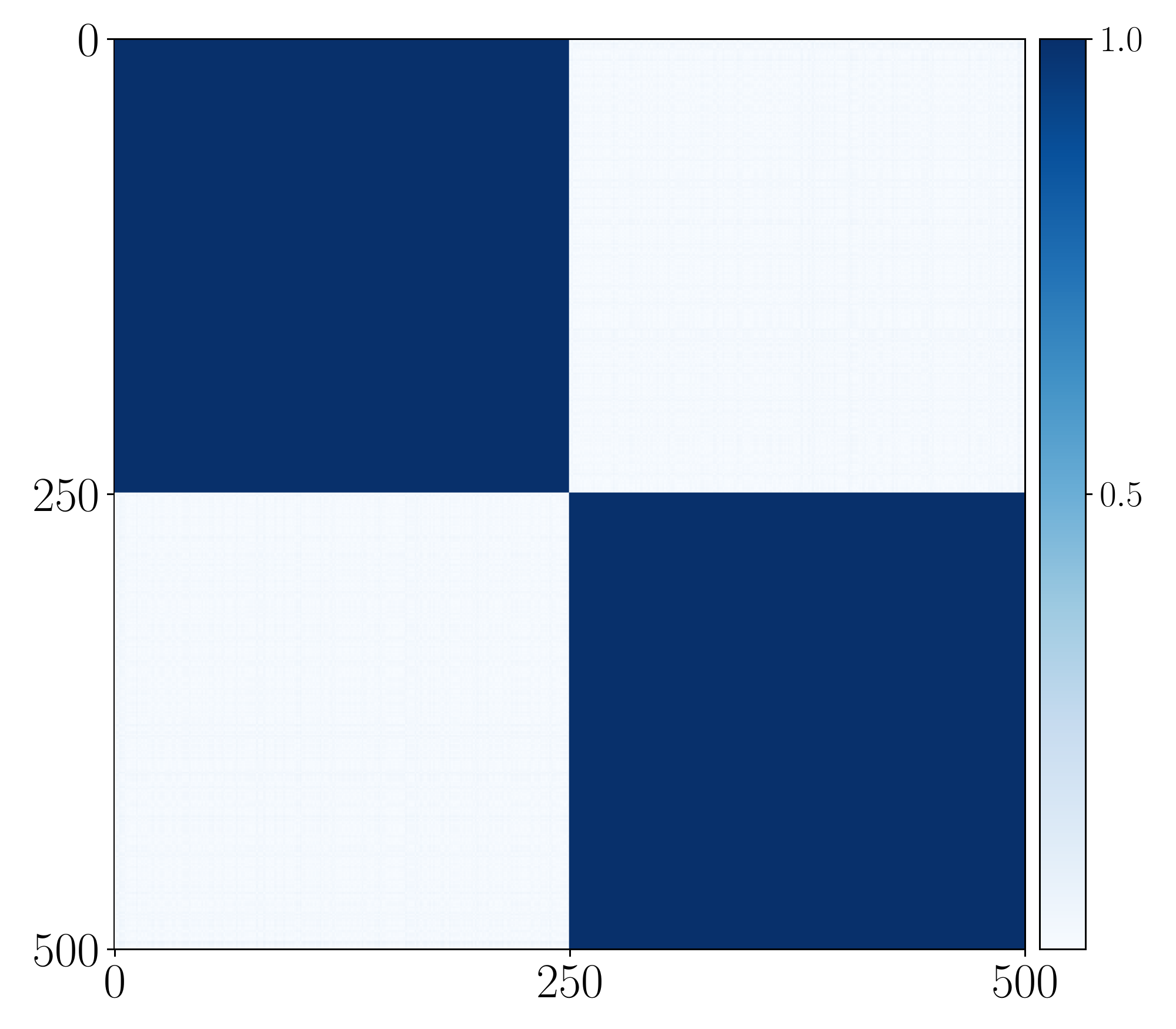}
    }
    \quad
    \subfigure[Loss]{
        \includegraphics[width=0.17\textwidth]{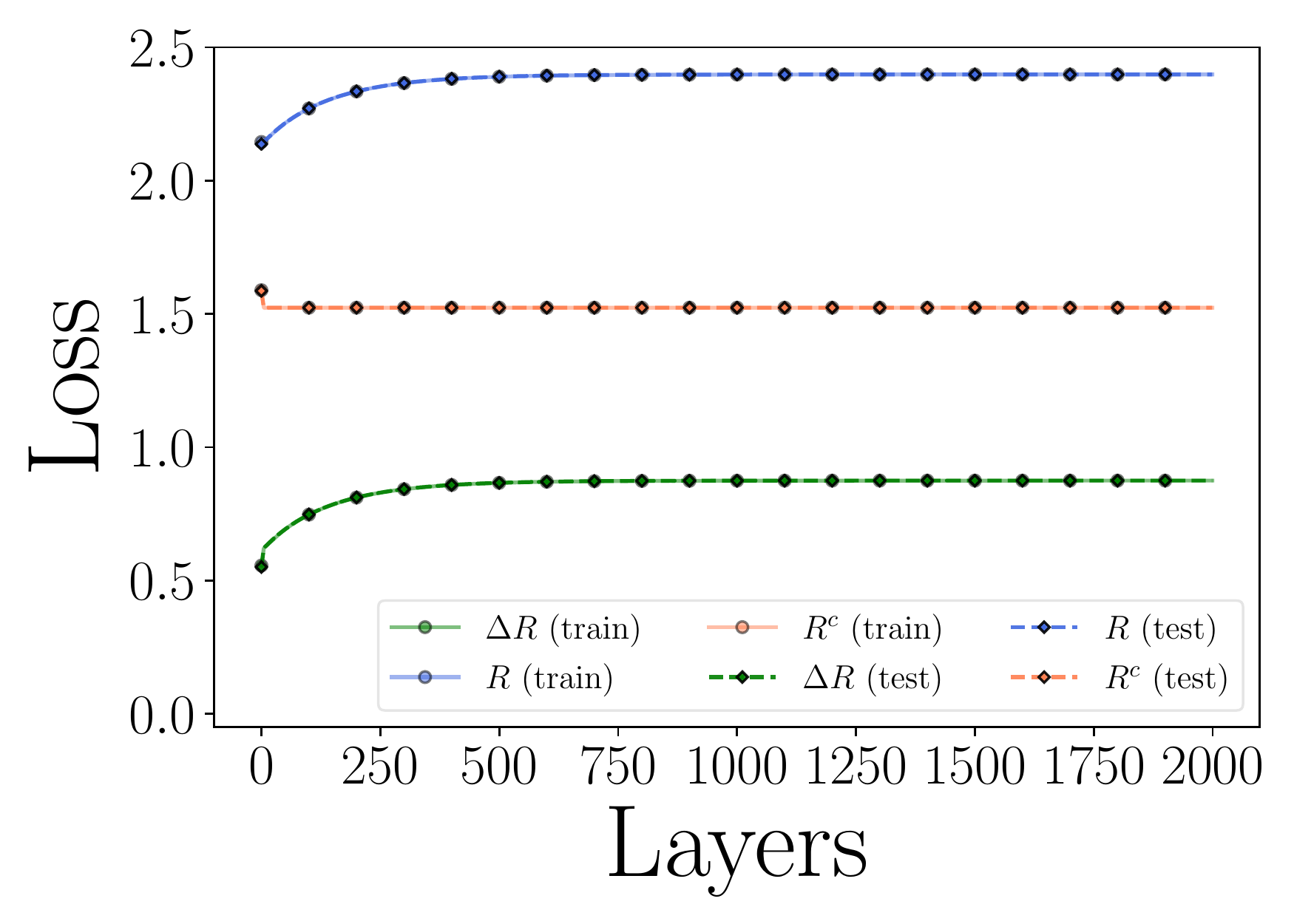}
    }
    \subfigure[$\X (3D)$ (\textbf{left: }scatter plot; \textbf{right: }cosine similarity visualization)]{
        \includegraphics[width=0.17\textwidth]{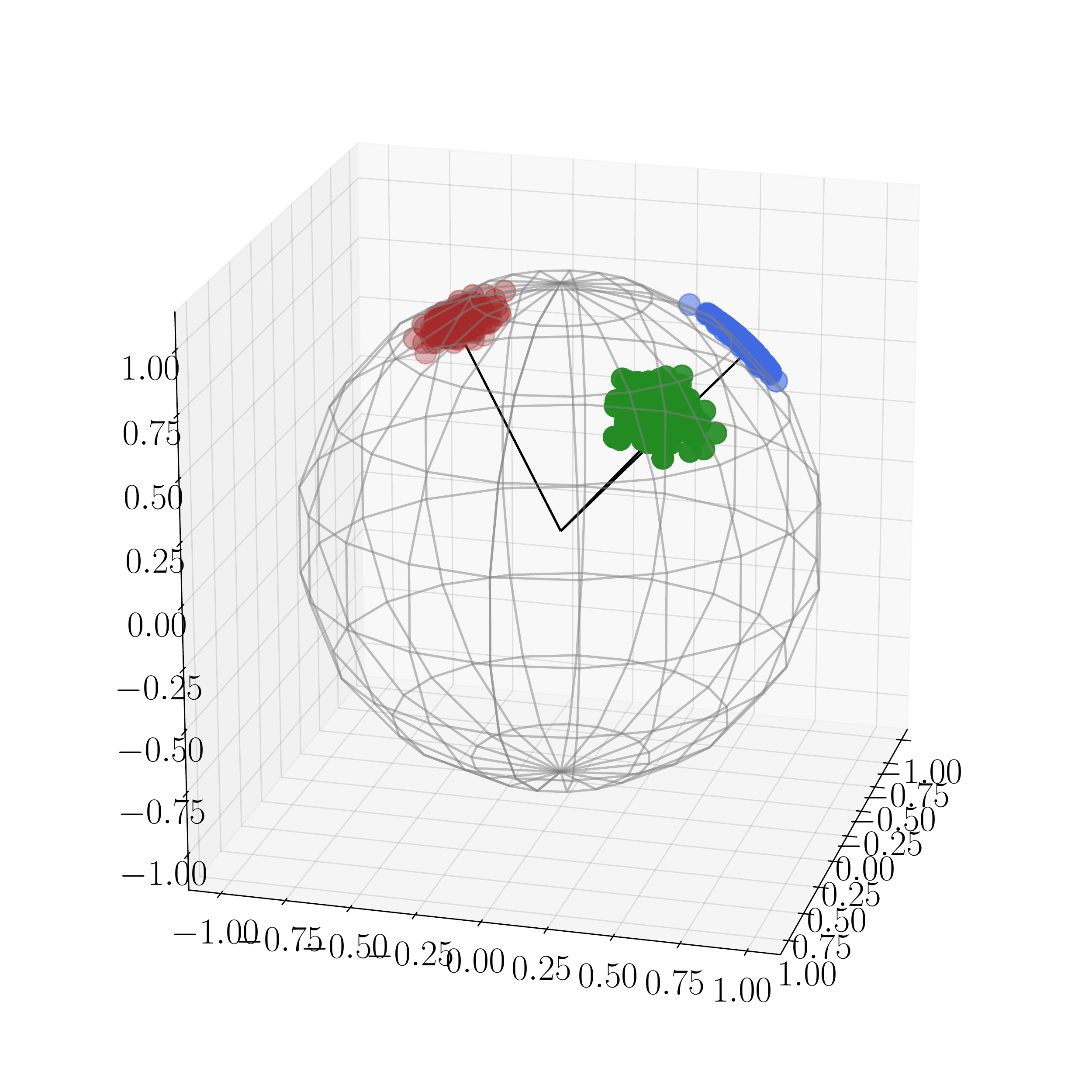}
        \includegraphics[width=0.17\textwidth]{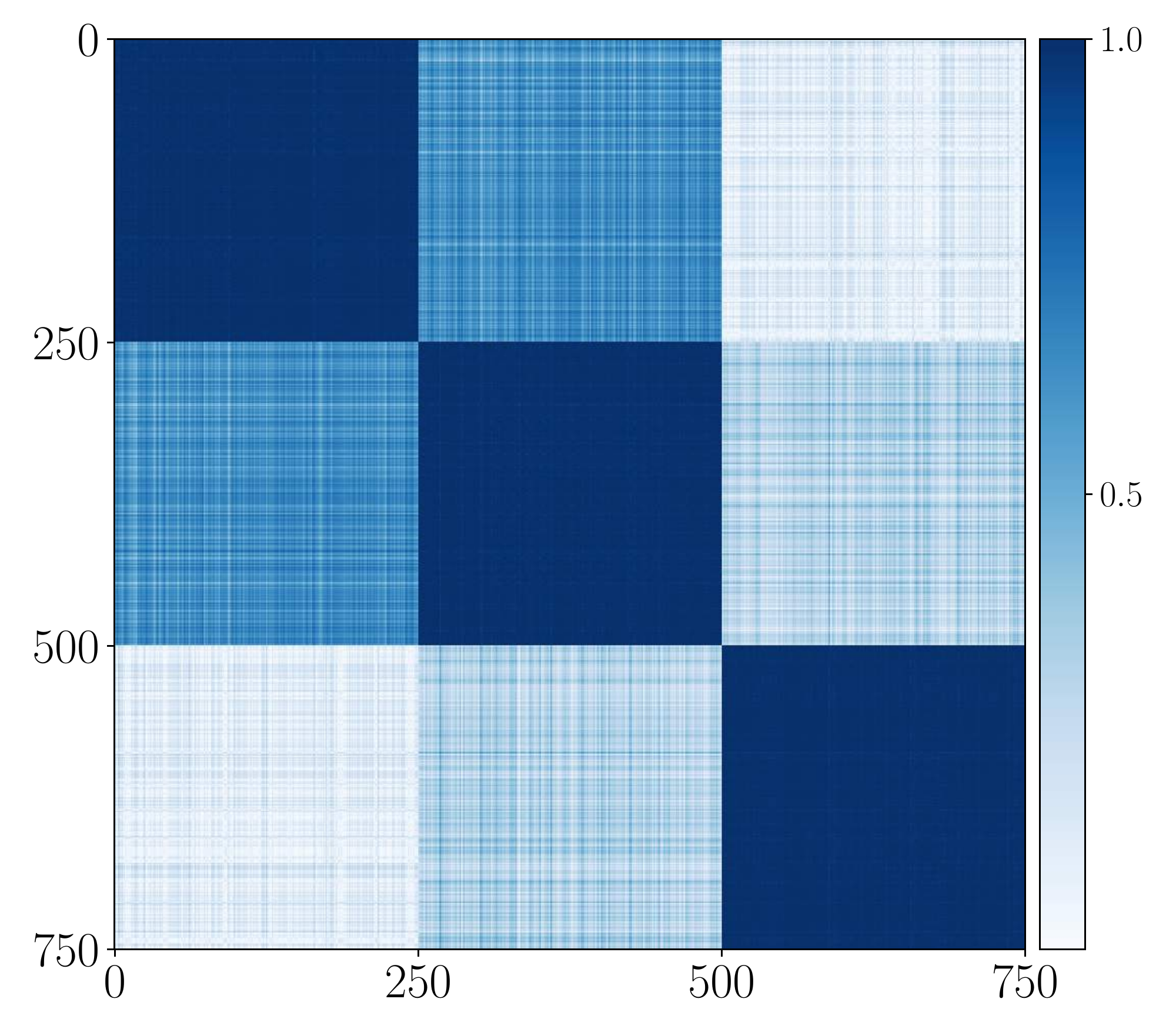}
    }
    \quad
    \subfigure[$\Z (3D)$ (\textbf{left: }scatter plot; \textbf{right: }cosine similarity visualization)]{
        \includegraphics[width=0.17\textwidth]{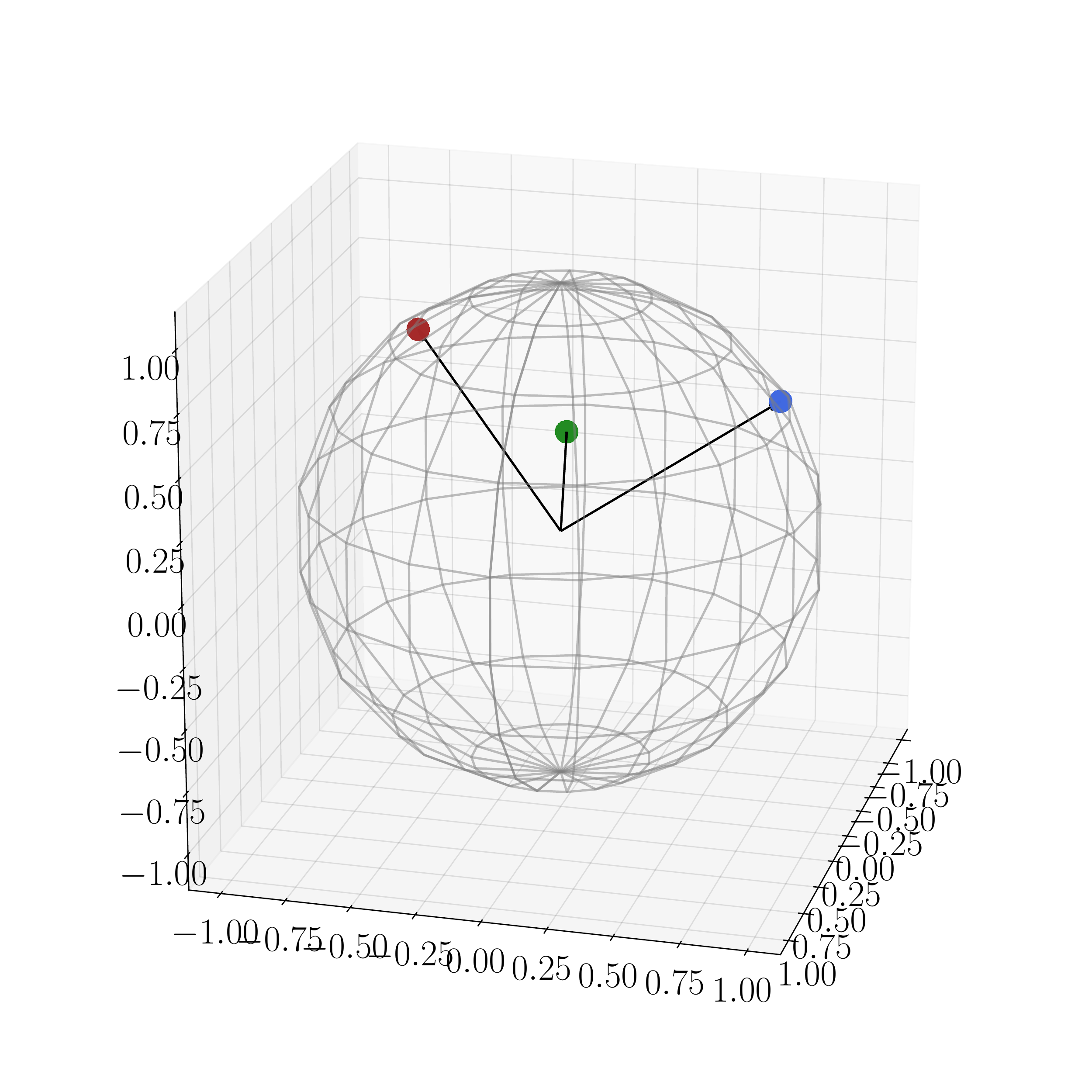}
        \includegraphics[width=0.17\textwidth]{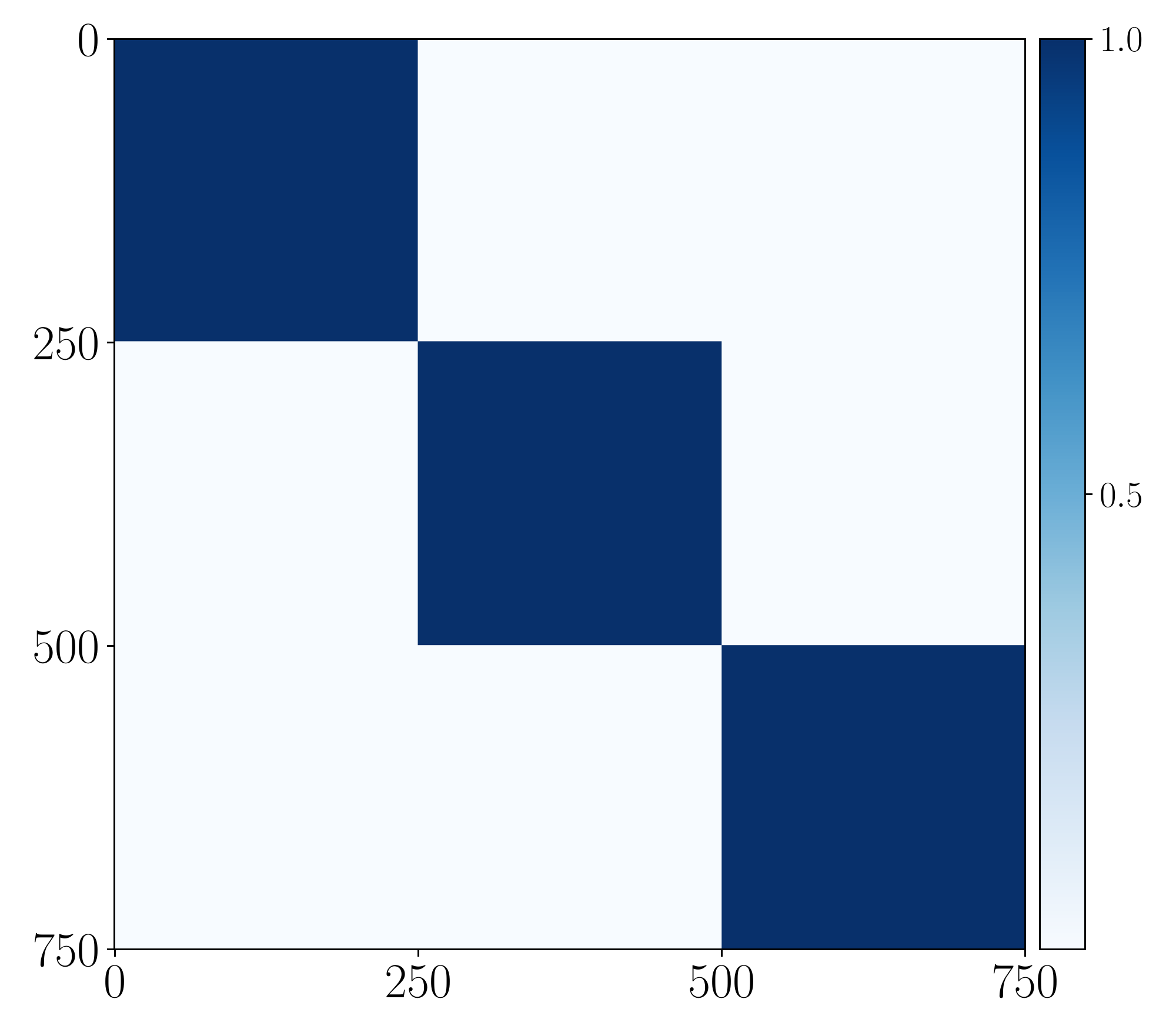}
    }
    \quad
    \subfigure[Loss]{
        \includegraphics[width=0.17\textwidth]{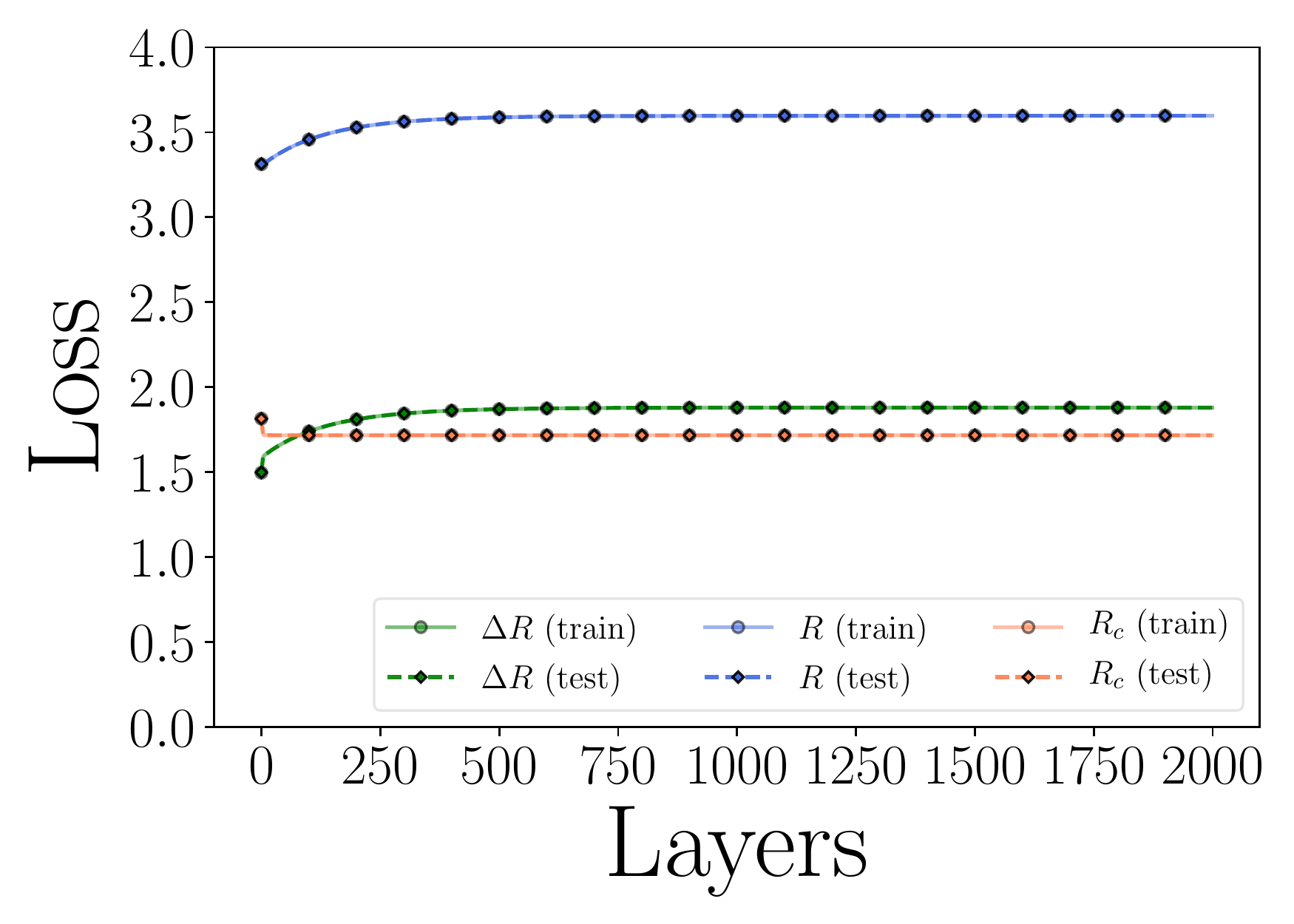}
    }
    \vskip -0.1in
    \caption{Learning mixture of Gaussians in $\mathbb{S}^1$ and $\mathbb{S}^2$. (\textbf{Top}) For $\mathbb{S}^1$, we set $\sigma_1 = \sigma_2 = 0.1$; (\textbf{Bottom}) For $\mathbb{S}^2$, we set $\sigma_1 = \sigma_2 = \sigma_3 = 0.1$.
    }\label{fig:appendix-guassian-exp1}
  \end{center}
  \vspace{-0.4in}
\end{figure*}

\paragraph{Additional experiments on  $\mathbb{S}^1$ with more than 2 classes.}  We try to apply ReduNet to learn mixture of Gaussian distributions on $\mathbb{S}^1$ with the number of class is larger than 2. Notice that these are the cases to which Theorem \ref{thm:MCR2-properties} no longer applies. These experiments suggest that the MCR$^2$ still promotes between-class discriminativeness with so constructed ReduNet. In particular, the case on the left of Figure \ref{fig:appendix-scatter-extra} indicates that the ReduNet has ``merged'' two linearly correlated clusters into one on the same line. This is consistent with the objective of rate reduction to group data as linear subspaces. 

\begin{figure*}[h]
\vspace{-0.2in}
  \begin{center}
    \subfigure[3 classes. (\textbf{Left}) $\X$; (\textbf{Right}) $\Z$]{
        \includegraphics[width=0.23\textwidth]{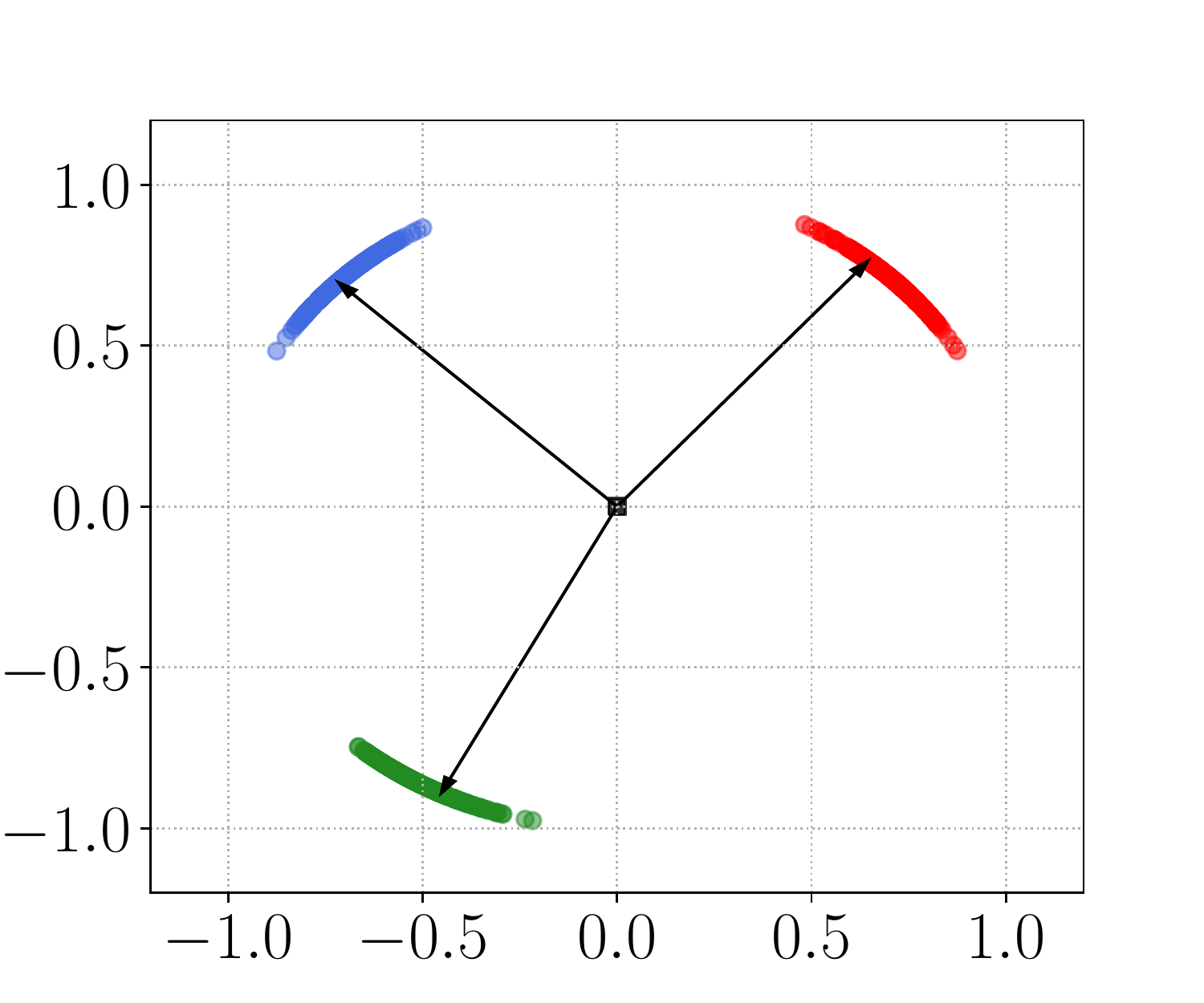}
        \includegraphics[width=0.23\textwidth]{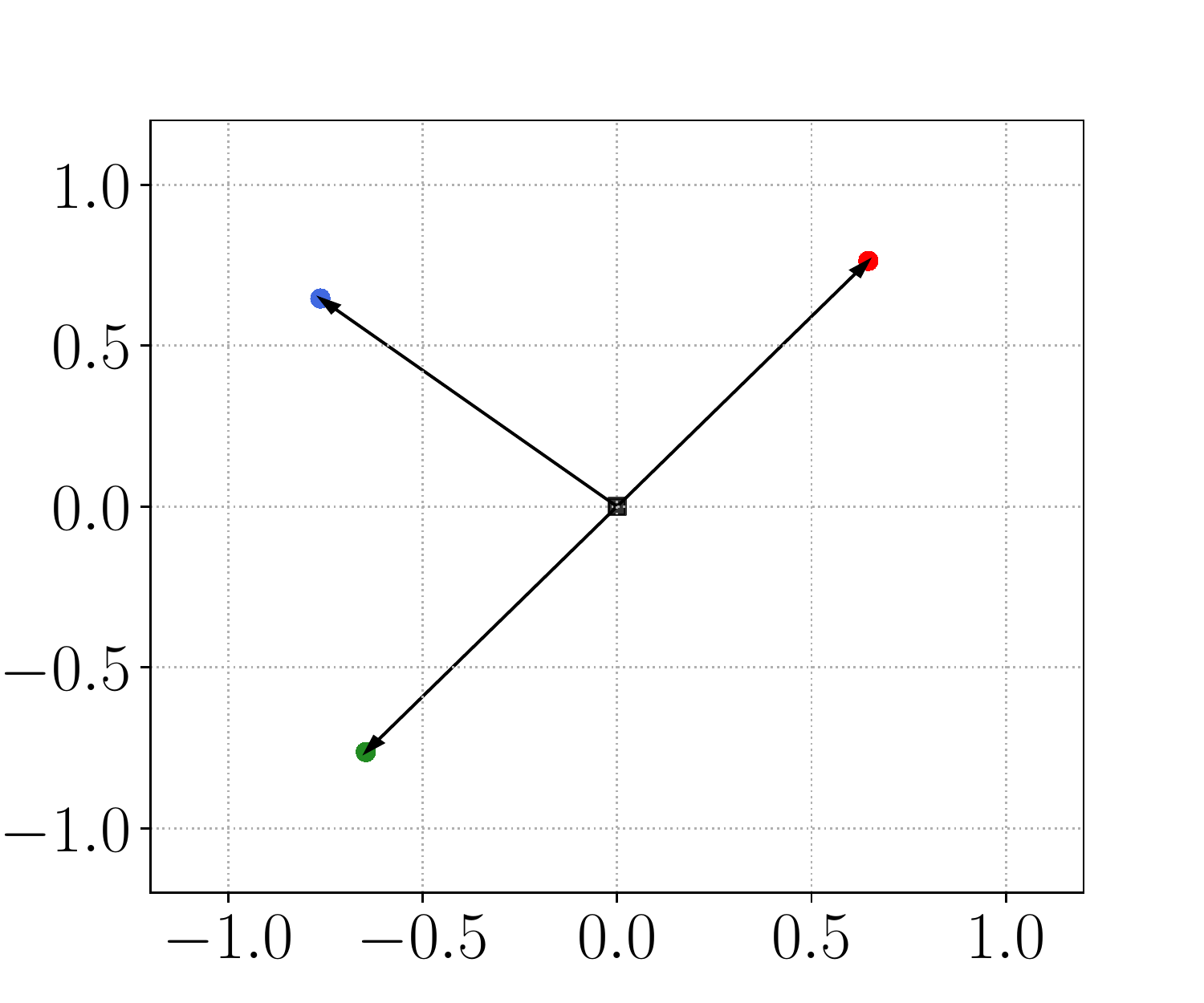}
    }
    \subfigure[6 classes. (\textbf{Left}) $\X$; (\textbf{Right}) $\Z$]{
        \includegraphics[width=0.23\textwidth]{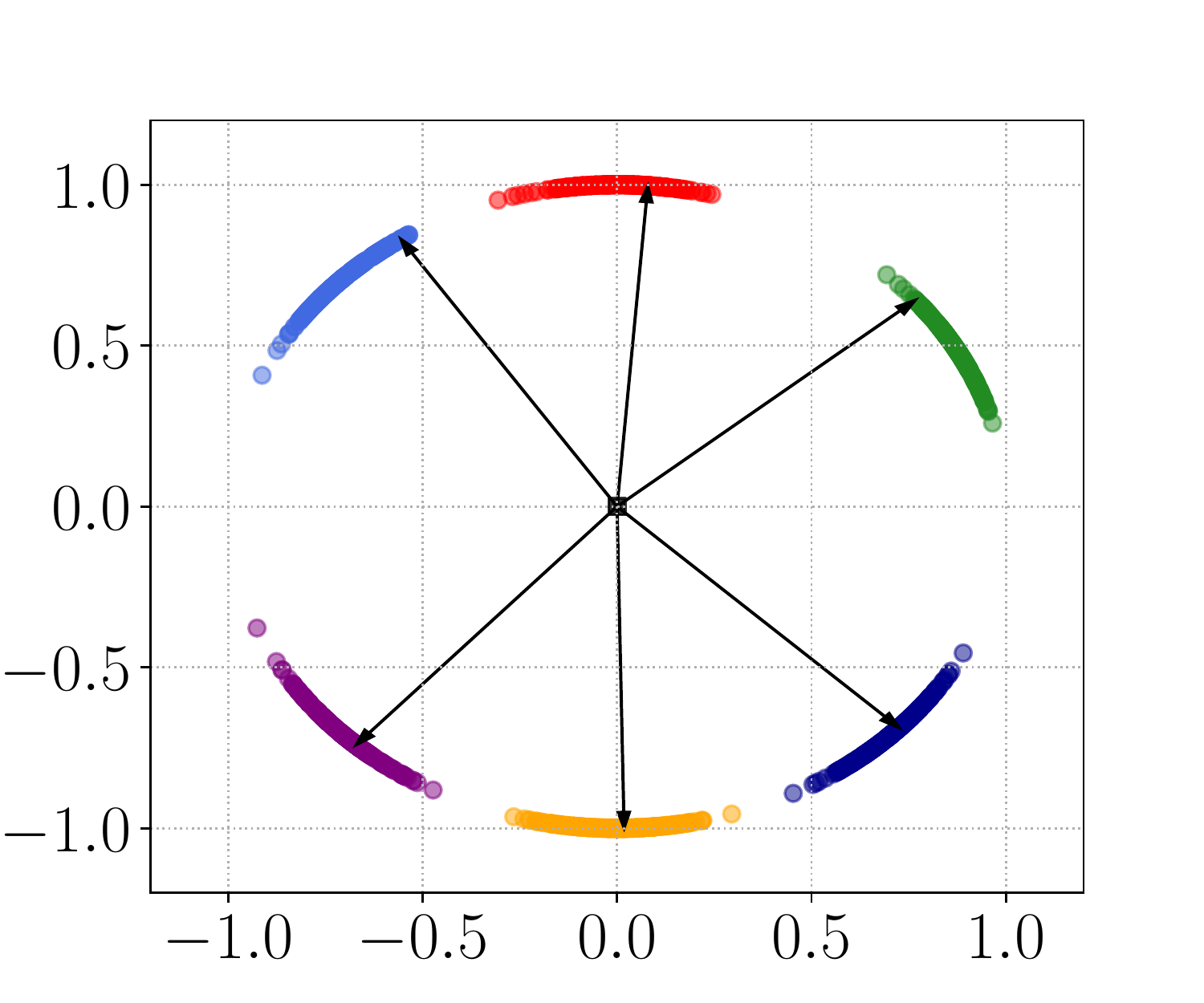}
        \includegraphics[width=0.23\textwidth]{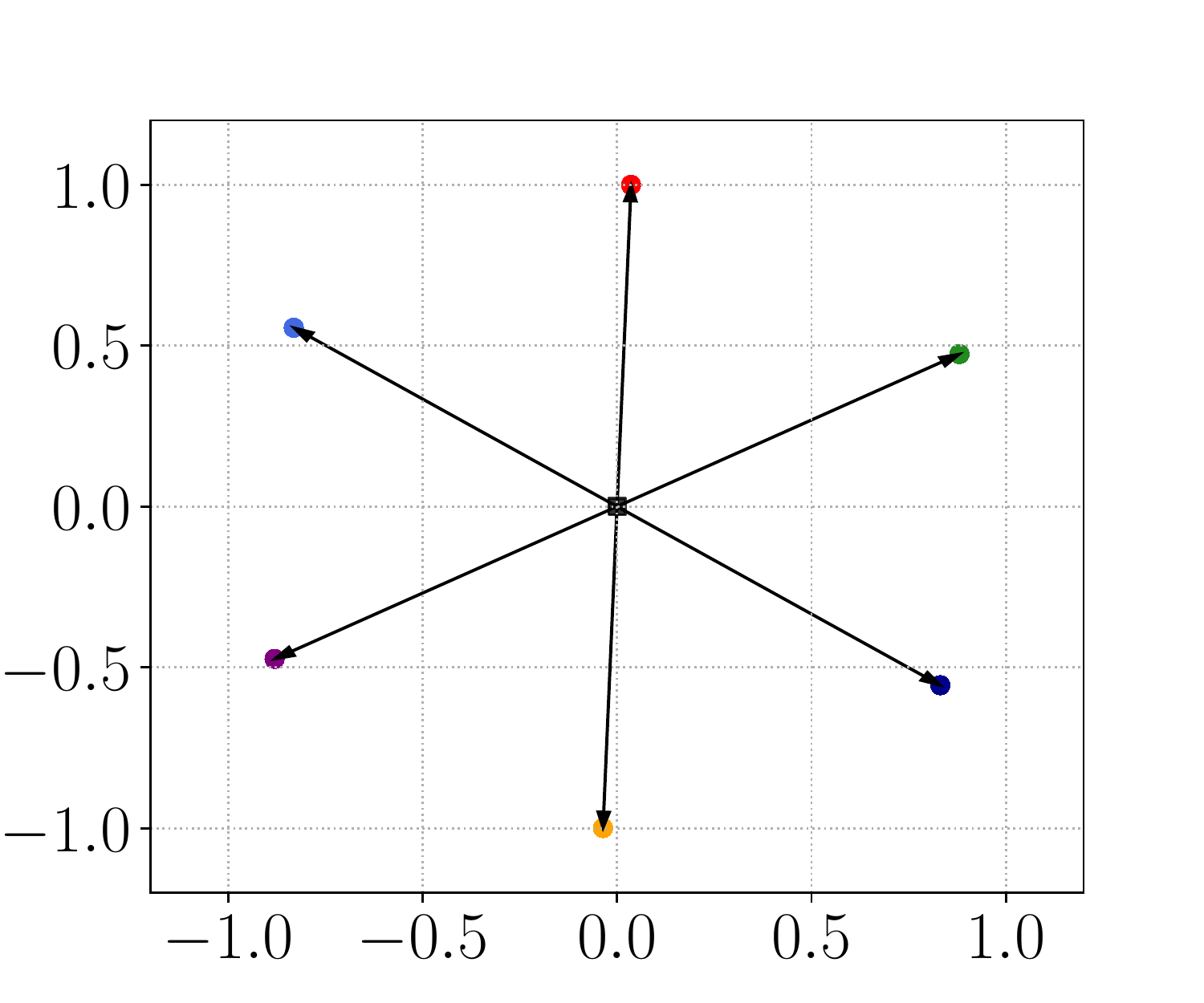}
    }
        \vskip -0.1in
    \caption{Learning mixture of Gaussian distributions with more than 2 classes. For both cases, we use step size $\eta=0.5$ and precision $\epsilon=0.1$. For (a), we set iteration $L=2,500$; for (b), we set iteration $L=4,000$.}
    \label{fig:appendix-scatter-extra}
  \end{center}
\end{figure*}


\subsection{Additional Experimental Results of ReduNet and Scattering Transform} \label{app:ReduNet-Scattering}

\paragraph{Comparing ResNet and ReduNet on maximizing  $\Delta R$ using CIFAR10.}

We compare the objective value reached by training a ResNet-18 versus constructing a ReduNet that maximizes the MCR$^2$ objective in an iterative fashion. For training ResNet-18, we use the same setting as Section~\ref{sec:experiment-objective}, with mini-batch size $m=1000$ and precision $\epsilon^2 = 0.5$. For training ReduNet, we first apply scattering transform to the training samples, then flatten the scatter-transformed features and project them to a 128-dimensional feature vector using a random linear projection. Then we construct a 4000-layer ReduNet with precision $\epsilon^2 = 0.1$, step size $\eta = 0.5$ and $\lambda = 500$. We vary the scales $\{2^3, 2^4\}$ and angles $\{4, 8, 16\}$ in scattering transforms, but set the phase to 4 for all sets of experiments. In both cases, we use all 50000 training samples. Our results are listed in Table~\ref{table:appendix-obj-val}. Empirically, features learned by training a ResNet-18 with MCR$^2$ achieves a higher $\Delta R$ of 48.65 than those learned by constructing a ReduNet, achieving a $\Delta R$ of 46.14.

\begin{table}[ht]
\begin{center}
\begin{small}

\begin{tabular}{cc|ccc}
\toprule
Scales         & Angles        & $\Delta R$                  & $R$                         & $R_c$                       \\
\midrule
$2^3$         & 4             & 46.1418                     & 66.0627                     & 19.9208                     \\
$2^3$         & 8             & 46.3207                     & 66.1358                     & 19.8151                     \\
$2^3$         & 16            & 46.4162                     & 66.1474                     & 19.7312                     \\
$2^4$         & 4             & 40.5377                     & 58.8128                     & 18.2751                     \\
$2^4$         & 8             & 41.3374                     & 60.1895                     & 18.8521                     \\
$2^4$        & 16            & 42.5311                     & 61.5164                     & 18.9854                     \\
\midrule
\multicolumn{2}{l}{ResNet-18} & \multicolumn{1}{r}{48.6497} & \multicolumn{1}{r}{69.2653} & \multicolumn{1}{r}{20.6155}\\
\bottomrule
\end{tabular}\vspace{-2mm}
\caption{Objective values under scattering transform with varying scales and angles. 
}\label{table:appendix-obj-val}
\end{small}
\end{center}
\end{table}

\paragraph{Comparing ReduNet with random filters and scattering transform.
} Here we provide comparisons of using ReduNet on MNIST digits that are lifted by scattering transform versus by random filters. We select $m = 5000$ training samples (500 from each class), and set precision $\epsilon^2 = 0.1$ and step size $\eta = 0.1$. We vary the number of layers for each comparison, as the numbers of layers needed for the rate reduction objective to converge also varies. For each comparison, we construct two networks: 1). We apply scattering transform to each training sample, then flatten the feature vector and construct a ReduNet, and 2). We apply a number random filters, such that the resulting dimensions of the feature are the same. We ablate across different scales $\{2^3, 2^4\}$ and different angles $\{4, 8, 16\}$, but keep the number of phases to 4 the same. To compute the testing accuracy of each architecture, we use all 10000 testing samples in the MNIST dataset. We list details about the feature dimensions and test accuracies in Table~\ref{table:appendix-scatter-compare}. From these empirical results, we have shown that using scattering transform outperforms using random filters in all scenarios. This implies that scattering transform is better at making the original data samples more separable than the random filters. With a more suitable lifting operator, ReduNet is able to linearize low-dimensional structure into incoherent subspaces. This also showcases that scattering transform and ReduNet complement each other when learning diverse and discrminative features.

\begin{table}[ht]
\begin{center}
\begin{small}
\begin{tabular}{lcc|cc}
\toprule
Lifting & $n$ & $L$ & Acc \\
\midrule
Scattering ($2^3$ scales, $4$ angles)   & 441             & 150          & \textbf{0.9734}   \\
Random filters    & 441             & 150          & 0.9225   \\
\midrule
Scattering ($2^3$ scales, $8$ angles)   & 873             & 100          & \textbf{0.9781}   \\
Random filters    & 873             & 100          & 0.9153   \\
\midrule
Scattering ($2^3$ scales, $16$ angles)   & 1737            & 50           & \textbf{0.9795}   \\
Random filters    & 1737            & 50           & 0.9061   \\
\midrule
Scattering ($2^4$ scales, $4$ angles)   & 65              & 400          & \textbf{0.8952}   \\
Random filters    & 65              & 400          & 0.8794   \\
\midrule
Scattering ($2^4$ scales, $8$ angles)   & 129             & 300          & \textbf{0.9392}   \\
Random filters    & 129             & 300          & 0.9172   \\
\midrule
Scattering ($2^4$ scales, $16$ angles)   & 257             & 200          & \textbf{0.9569}   \\
Random filters    & 257             & 200          & 0.8794   \\
\bottomrule
\end{tabular}\vspace{-2mm}
\caption{Lifting by Scattering transform versus by random filters. The dimension of the feature vector $n$ and the number of layers used in the ReduNet $L$ are also stated.
}\label{table:appendix-scatter-compare}
\end{small}
\end{center}
\end{table}

\subsection{Equivariance of learned features using Translation-Invariant ReduNet}

We investigate how features learned using Translational-Invariant ReduNet possess equivariant properties. More specifically, we construct a ReduNet using 500 training samples of MNIST digits (50 from each class) with number of layers $L = 30$, precision $\epsilon^2 = 0.1$, step size $\eta = 0.5$, and $\lambda = 500$. We also apply 2D circulant convolution to the (1-channel) inputs with 16 random Gaussian kernels of size $7 \times 7$.

To evaluate its equivariant properties, we augment each training and test sample by shifting 7 pixels in each canonical direction, resulting in 9 augmented images for each original image. 
As shown in Figure~\ref{appendix:equivariance}, we compute the distance between the representation of a (shifted) test sample and representations of training samples (including all translation augmentations) using cosine similarity. By computing the top-9 largest inner products, we observe that each augmented test sample is closest to a training sample with the similar translation.

\begin{figure*}[t]
  \begin{center}
    \subfigure[\label{fig:equi-a}Class 3]{\includegraphics[width=0.49\textwidth]{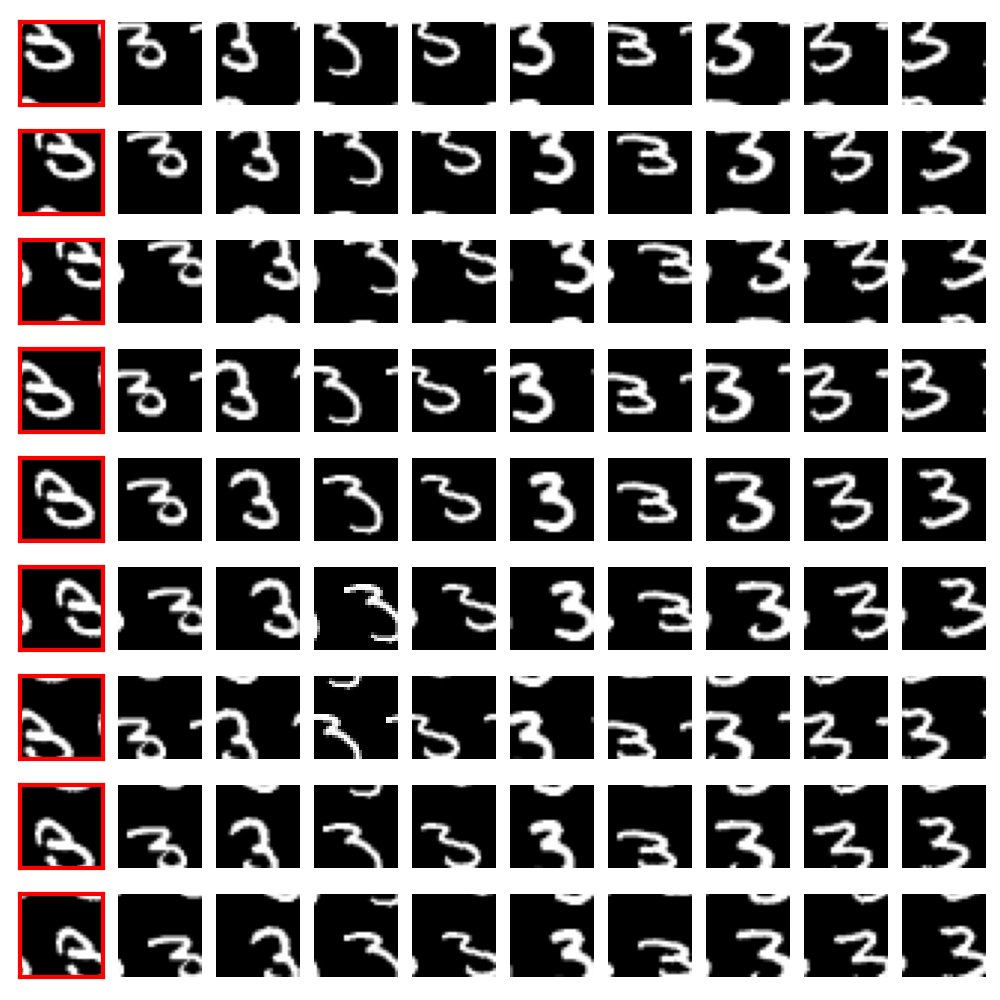}}
    \subfigure[\label{fig:equi-b}Class 7]{\includegraphics[width=0.49\textwidth]{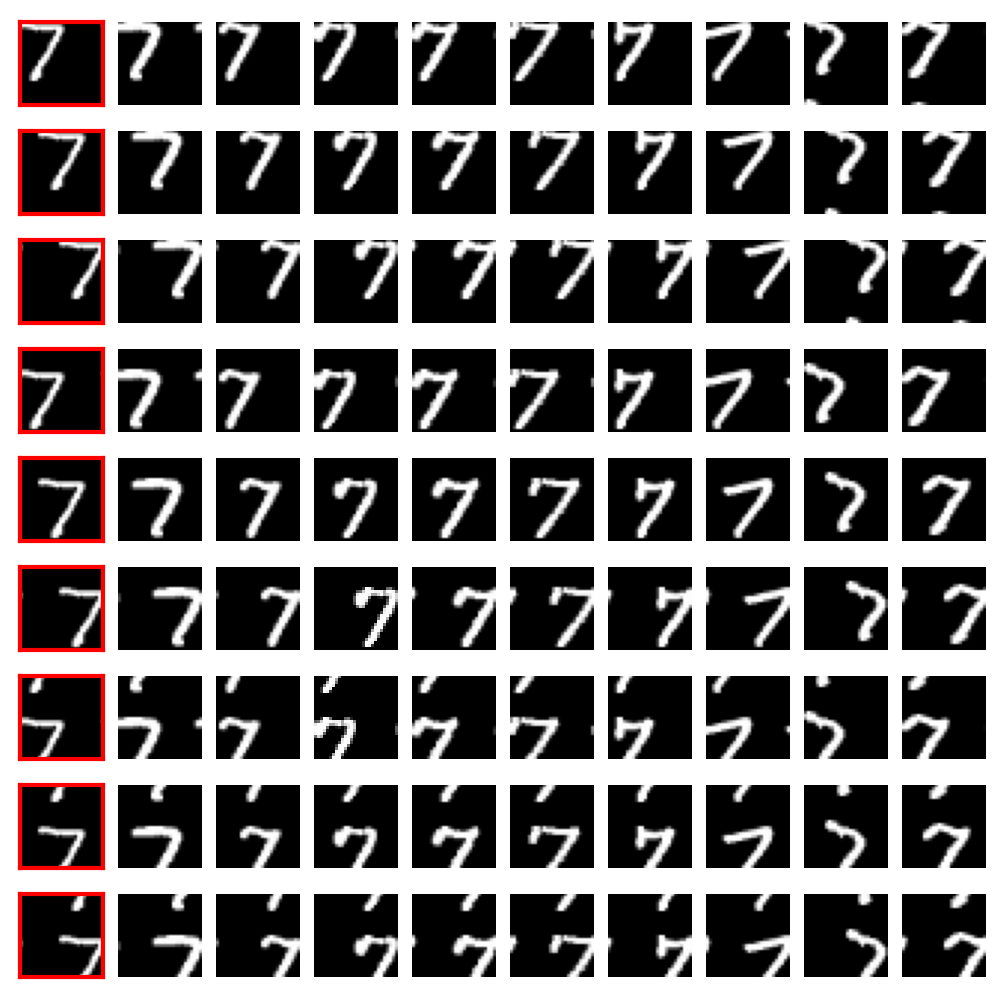}}
    \caption{Verification of Equivariance in ReduNet. For each class of samples, and for all training samples (and their augmentations) $\z^i$ and a test sample (and its augmentations) $\hat{\z}$, we compute their inner product $|\langle \hat{\z}, \z^i \rangle|,\, \forall i$. We select the top-9 largest inner products  (sorted from left to right), and visualize their respective image. The test sample and its augmented version are highlighted in red.  }\label{appendix:equivariance}
  \end{center}
\vskip -0.3in
\end{figure*} 

\newpage
\addcontentsline{toc}{section}{References}
\bibliography{references/references1,references/references2,references/references3}

\end{document}